\documentclass[times, review, 10pt]{elsarticle}

\usepackage{amsmath,amssymb,amsthm}
\DeclareMathOperator{\lcm}{lcm}
\newtheorem{theorem}{Theorem}[section]
\newtheorem{lemma}{Lemma}[section]
\newtheorem{definition}{Definition}[section]
\newtheorem{example}{Example}[section]
\newtheorem{remark}{Remark}[section]
\usepackage{placeins} 
\usepackage{geometry}
\usepackage{amsfonts}
\usepackage{mathrsfs} 
\usepackage{euscript}
\usepackage{booktabs}
\usepackage{graphicx}
\usepackage{makecell}
\usepackage{multirow}
\usepackage{comment}
\usepackage{diagbox}
\usepackage{hyperref}
\usepackage{appendix}
\usepackage{caption}

\usepackage{tikz}      
\usepackage[ruled, vlined, linesnumbered]{algorithm2e}

\journal{Pattern Recognition}
\begin{document}

\begin{frontmatter}



\title{Semi-Tensor Product-Based Multi-Term Randomized T-SVD and Its Visual Applications
}


\author[1]{Xingchen Xiao}
\author[1]{Feng Zhang\corref{cor1}}
\author[1]{Wenjin Qin}
\author[1]{Jianjun Wang}
\cortext[cor1]{Corresponding author. Email addresses: zfmath@swu.edu.cn (F. Zhang).}
\affiliation[1]{organization={School of Mathematics and Statistics, Southwest University},
            city={Chongqing},
            postcode={400715},
            country={China}}

\begin{abstract}
Tensor singular value decomposition (T-SVD), which is built upon the tensor-tensor product (t-product), has emerged as a powerful tool for processing high-dimensional visual data such as color images and videos. 
However, the standard t-product imposes strict dimensional compatibility constraints. Although extensions based on the semi-tensor product (STP) relax this restriction, their single-term formulations still suffer from limited approximation accuracy. Moreover, these deterministic methods incur high computational costs when processing large-scale tensor data. 
To address these issues, this paper introduces a novel semi-tensor product for third-order tensors under the t-product framework induced by arbitrary invertible linear transforms. The resulting tensor semi-tensor product breaks the rigid dimension matching requirement of the standard t-product, while retaining the closed-form property of T-SVD. 
Based on this construction, we develop a multi-term semi-tensor product singular value decomposition (MSTP-SVD), which integrates multiple orthogonal decomposition terms to significantly improve low-rank approximation accuracy compared with single-term schemes. 
To reduce the computational cost of multi-term modeling, we incorporate randomized projection and power iteration techniques into the MSTP-SVD framework, yielding an accelerated multi-term randomized semi-tensor product SVD (MRSTP-SVD) algorithm that achieves a balance between reconstruction accuracy and computational efficiency. 
Experiments on image and video compression and completion tasks demonstrate the effectiveness of the proposed method.

\end{abstract}





\begin{keyword}
Tensor singular value decomposition \sep Semi-tensor product \sep Multi-term Kronecker product decomposition \sep Randomized algorithm \sep  Tensor compression and completion
\end{keyword}

\end{frontmatter}



\section{Introduction}
\label{intro}
The rapid advancement of data acquisition technologies has led to an explosion of high-dimensional visual data, including color images, color videos, hyperspectral remote sensing images, and medical imaging data.
Tensors, as natural higher-order extensions of matrices, have become the standard representation for such multidimensional data and found widespread applications in computer vision \cite{SONG2025111600,YANG2022108311,zhang2020low,hou2021robust}, 
machine learning \cite{xie2017kronecker,xue_enhanced_2020}, signal processing \cite{sidiropoulos2017tensor,cichocki_tensor_2015}, and data mining \cite{papalexakis2016tensors,sun2009multivis}. 
Compared with flattening multidimensional signals into matrices or vectors, tensor representations preserve inherent cross-dimensional structural correlations, which are crucial for extracting meaningful features and achieving superior performance in various analytical tasks \cite{ZHU2026114295,luo2023low,wang_guaranteed_2023,11370221,KONG2023109545}. 
Nevertheless, the development of efficient and accurate tensor decomposition algorithms for large-scale visual data remains a fundamental challenge.

\par Among numerous tensor decomposition techniques, CANDECOMP/PARAFAC (CP) \cite{hitchcock1928multiple}, Tucker decomposition \cite{tucker1966some}, Tensor Train (TT) \cite{oseledets2011tensor}, and Tensor Ring (TR) \cite{zhao2016tensor} have been extensively studied and successfully applied. However, these decompositions either suffer from NP-hard rank determination (CP), lack optimal truncation properties (Tucker), 
or involve unbalanced matricization schemes that may not capture global information effectively \cite{hillar_most_2013,bengua2017efficient}. In this context, the tensor singular value decomposition (T-SVD) pioneered by Kilmer et al. \cite{kilmer2011factorization,kernfeld2015tensor} has emerged as a theoretically attractive alternative. Built upon the tensor-tensor product (t-product), T-SVD provides a closed-form factorization (see Theorem \ref{th:t-svd}). 
Moreover, truncated T-SVD achieves optimal approximation in the Frobenius norm sense for any unitary-invariant tensor norm \cite{kilmer2021tensor}, analogous to the classical Eckart–Young–Mirsky theorem for matrices. These theoretical advantages, combined with natural parallelizability across frontal slices, 
have made T-SVD particularly successful in image compression \cite{ahmadi2024randomizedtsvd}, video completion \cite{LIU2025111610}, face recognition \cite{8908805}, and background modeling \cite{10420484}.

\par Despite its elegance, the standard t-product framework imposes strict dimensional compatibility requirements. Specifically, for two third-order tensors $\mathcal{X} \in \mathbb{R}^{n_1 \times n_2 \times n_3}$ and $\mathcal{Y} \in \mathbb{R}^{m_1 \times m_2 \times n_3}$, their t-product $\mathcal{X} * \mathcal{Y}$ is well-defined only when $n_2 = m_1$. 
This constraint severely limits the applicability of T-SVD to real-world tensors whose modes do not satisfy such rigid matching conditions. To address this limitation, researchers have explored alternative algebraic structures. 
Notably, the matrix semi-tensor product (STP) proposed by Cheng \cite{cheng2012introduction} supports dimensionally mismatched operations by introducing block Kronecker products, offering greater flexibility than conventional matrix multiplication.
 Recent work by Chen et al. \cite{chen2023tensor} has extended this idea to the tensor setting, defining a tensor semi-tensor product that relaxes dimension constraints while preserving certain algebraic properties.

\par However, a critical observation is that existing tensor STP-based decompositions adopt a single-term formulation, which represents the target tensor using only one set of orthogonal factors. This single-component structure inherently restricts the representation power, particularly when the underlying tensor exhibits complex multi-mode correlations or does not possess rapidly decaying singular values. 
In such scenarios, single-term approximations inevitably suffer from limited accuracy, failing to capture fine-grained structural information effectively \cite{batselier2017constructive}. This observation motivates the exploration of multi-term decomposition strategies, 
wherein multiple orthogonal components are combined to achieve more accurate and flexible tensor representations—a paradigm that has proven successful in classical matrix decompositions and other tensor frameworks, yet remains largely underexplored in the STP-based context.

\par Furthermore, deterministic tensor decomposition algorithms, whether based on exact SVD computations or iterative optimization procedures, incur substantial computational costs. As shown in \cite{zhang2018randomized}, computing a $k$-term truncated T-SVD requires \(\mathcal{O} \left( n_1 n_2 n_3 log(n_3) + m_1 m_2 k \right)\) operations, 
which becomes prohibitively expensive for high-resolution videos (e.g., $1920 \times 1080 \times 1000$) or large-scale hyperspectral datasets. This challenge is further exacerbated in multi-term decomposition scenarios, where multiple factorizations must be performed, 
multiplying the computational burden accordingly. Although randomized algorithms have been successfully developed for Tucker decomposition \cite{che2025efficient}, CP approximation \cite{malik2020fast}, and even T-SVD \cite{WANG2026112066,liu2021hyperspectral} to reduce complexity, their application to STP-based frameworks remains largely unexplored.

\par Motivated by these observations, we aim to develop a comprehensive framework that simultaneously addresses the aforementioned three challenges: 
(i) dimensional rigidity of traditional t-product operations, 
(ii) approximation accuracy limitations of single-term decomposition schemes, 
and (iii) computational inefficiency of deterministic methods for large-scale data. These three challenges are not isolated; addressing one often exacerbates another. For instance, relaxing dimensional constraints through semi-tensor products may introduce additional algebraic complexity, 
while improving accuracy via multi-term modeling inevitably increases computational cost. Moreover, we observe that existing tensor STP constructions rely exclusively on specific transform bases (e.g. discrete Fourier transform (DFT)), which may not be optimal for all application scenarios. 
A framework based on arbitrary invertible linear transforms could provide greater flexibility and potentially superior approximation properties for diverse data characteristics. Therefore, a unified approach that carefully balances these competing objectives is highly desirable.

\par To this end, the main contributions of this paper are summarized as follows:
\begin{itemize}
        \item We propose a novel semi-tensor product for third-order tensors under a generalized t-product framework induced by arbitrary invertible linear transforms. Unlike prior work that relies exclusively on DFT, our construction accommodates any unitary transform, providing a more flexible algebraic foundation while retaining the closed-form property of T-SVD and essential algebraic properties. 
        \item We establish a multi-term semi-tensor product singular value decomposition (MSTP-SVD) model that integrates multiple orthogonal decomposition terms, significantly improving low-rank approximation accuracy compared with existing single-term schemes. We provide theoretical analysis of the approximation error bounds.
        \item We develop an accelerated multi-term randomized STP-SVD (MRSTP-SVD) algorithm by incorporating random projection and power iteration techniques, which achieves substantial speedup over deterministic counterparts with controllable approximation error. We derive expected error bounds and discuss the effects of key parameters.
        \item We conduct extensive experiments on image compression, video compression, and tensor completion tasks, demonstrating that the proposed method achieves superior or comparable performance against state-of-the-art baselines in both reconstruction quality and computational efficiency.
\end{itemize}
\par The remainder of this paper is organized as follows. Section \ref{sec:2} reviews preliminaries on tensor notations, t-product operations, and matrix semi-tensor product. In Section \ref{sec:3}, we present our novel tensor semi-tensor product definition along with its theoretical properties and connections to existing frameworks. Section \ref{sec:4} develops the MSTP-SVD decomposition model. The randomized MRSTP-SVD algorithm is deferred to Section \ref{sec:random}. 
Section \ref{sec:experiment} reports experimental results on various tasks, and Section \ref{sec:conclusion} concludes the paper with discussions on future work.

\section{Notations and Preliminaries}
\label{sec:2}
In this section, some notations and basic preliminaries adopted throughout this paper are summarized. Vectors are represented by bold lowercase letters, matrices by bold uppercase letters, and tensors by calligraphic letters. $\mathbb{R}$ and $\mathbb{C}$ denote the real and complex Euclidean spaces, respectively.
For a third-order tensor \(\mathcal{A} \in \mathbb{R}^{n_1 \times n_2 \times n_3}\), its \((i,j,k)\)-th entry is written as \(\mathcal{A}_{ijk}\).  
We adopt $\mathcal{A}^{(i)}$ to represent the $i$-th frontal slices. For two matrices \(\mathbf{A}, \mathbf{B} \in \mathbb{R}^{n_1 \times n_2}\), their inner product is defined as $\langle \mathbf{A}, \mathbf{B} \rangle = \mathrm{Tr}\left( \mathbf{A}^\top \mathbf{B} \right)$ , where $\mathbf{A}^\top$ denotes the transpose of $\mathbf{A}$ and $\mathrm{Tr}(\cdot)$
denotes the matrix trace operator. Additionally, $\mathbf{A}^\mathrm{H}$ denotes the conjugate transpose of matrix $\mathbf{A}$. For a arbitrary tensor $\mathcal{A}$, its Frobenius norm is defined as
$\|\mathcal{A}\|_F=\sqrt{\sum_{i,j.k} |\mathcal{A}_{ijk}|^2}$.

To facilitate the understanding of what follows, we first briefly review the definitions and key properties of vectors and the STP of matrices, 
which serve as the essential mathematical foundation for this study. Note that this paper considers only the left semi-tensor product. 
The right semi-tensor product and its formulation for general dimensions, as well as other related aspects, are discussed in detail in \cite{cheng2002matrix}. 
Hereafter, the term ``semi-tensor product'' refers to the left semi-tensor product.

\begin{definition}[Semi-tensor product of vectors \cite{cheng2007survey}]
        Let  $\mathbf{x} \in \mathbb{R}^{1 \times np}$  be a row vector and $\mathbf{y} \in \mathbb{R}^{p \times 1}$ be a column vector.  We split $\mathbf{x}$ into $p$ blocks as $\mathbf{x}_1, \mathbf{x}_2, \cdots, \mathbf{x}_p$, each block is a $1\times n$ row vector.
        Then we can define the STP of two vectors, denoted by $\ltimes$, as
\begin{equation*}
\mathbf{x} \ltimes \mathbf{y} = \sum_{i=1}^{p} \mathbf{x}_i y_i \in \mathbb{R}^{1\times n}, \quad \mathbf{y}^\top \ltimes \mathbf{x}^\top = \sum_{i=1}^{p} y_i (\mathbf{x}_i)^\top \in \mathbb{R}^{n\times 1}.
\end{equation*}
\end{definition}

\begin{definition}[Semi-tensor product of matrices \cite{cheng2007survey}]\label{def:matrix semi}
        Given two matrices $\mathbf{A}\in \mathbb{R}^{m \times n}$ and $\mathbf{B}\in \mathbb{R}^{p \times q}$. If $n=tp$ or $p=tn, t\in \mathbb{Z}^+$, then we can define the  semi-tensor product of $\mathbf{A}$ and $\mathbf{B}$, 
        denoted by $\mathbf{C}=\mathbf{A}\ltimes\mathbf{B}$. Here $\mathbf{C}$ is a block matrix which has $m \times q$ blocks, each block can be represented as 
        \[\mathbf{C}^{ij}= \mathbf{A}^i \ltimes \mathbf{B}_j,
        \]
        where $\mathbf{A}^i$ is the $i$-th row of $\mathbf{A}$ and $\mathbf{B}_j$ is the $j$-th column of $\mathbf{B}$.
\end{definition}

\begin{lemma}[The associative law of semi-tensor product of matrix \cite{cheng2012introduction}] \label{le:as law matrix}
        Let $\mathbf{A}, \mathbf{B}, \mathbf{C}$ be matrices of compatible dimensions,  The following properties of the semi-tensor product of matrices hold:

        \[(\mathbf{A}\ltimes \mathbf{B})\ltimes\mathbf{C} = \mathbf{A}\ltimes(\mathbf{B}\ltimes\mathbf{C}).
        \]
\end{lemma}

Let \(\mathbf{I}_n\) denote an \(n \times n\) identity matrix. Then, the matrix STP defined in Definition \ref{def:matrix semi} admits the following representation in terms of the Kronecker product (see Supplementary Material for its definition).
\begin{lemma}\label{lem:equiv_stp}\cite{cheng2012introduction}
Let $\mathbf{A} \in \mathbb{R}^{m \times n }$, and $\mathbf{B} \in \mathbb{R}^{p \times q }$. Then, the following properties hold:
\begin{equation*}
\mathbf{A} \ltimes \mathbf{B} = (\mathbf{A}\otimes\mathbf{I}_{t/n})(\mathbf{B}\otimes\mathbf{I}_{t/p})\in\mathbb{R}^{m(t/n) \times q(t/p)} ,    
\end{equation*} 
where $t$ is the least common multiple of $n$ and $p$, i.e., $t=\lcm(n,p)$.
\end{lemma}
\begin{remark}
        when $p = n$, $\mathbf{A} \ltimes \mathbf{B} =(\mathbf{A} \otimes \mathbf{I}_1)(\mathbf{B} \otimes \mathbf{I}_1)= \mathbf{A}\mathbf{B}$. That is, the matrix STP can degenerate into the standard matrix multiplication.
\end{remark}

Existing work in \cite{kernfeld2015tensor} proposes the generalized definition of the t-product under arbitrary invertible linear transforms. In our paper, we consider a linear transform $L:\mathbb{R}^{n_1\times n_2 \times n_3} \rightarrow \mathbb{R}^{n_1\times n_2 \times n_3}$, which is defined as
\begin{equation}\label{eq:transform}
        \bar{\mathcal{A}}= L(\mathcal{A})=\mathcal{A}\times_3\mathbf{L},
\end{equation}
where ``$\times_3 $'' represents mode-3 product (Definition 2.5 in \cite{kernfeld2015tensor}) and  $\mathbf{L}\in\mathbb{C}^{n_3\times n_3}$ can be any invertible transform matrix. Clearly, $L$ is invertible with the inverse transform defined as $\mathcal{A} = L^{-1}(\bar{\mathcal{A}}) = \bar{\mathcal{A}} \times_3 \mathbf{L}^{-1}$.
\par we construct a block-diagonal matrix as follows:
\[\bar{\mathbf{A}}=\mathrm{bdiag}(\bar{\mathcal{A}}):=
\begin{bmatrix}
\bar{\mathcal{A}}^{(1)} & & & \\
& \bar{\mathcal{A}}^{(2)} & & \\
& & \ddots & \\
& & & \bar{\mathcal{A}}^{(n_3)}
\end{bmatrix}.
\]
The block-diagonal matrix constructed via the block-diagonalization operation can be converted back into the original tensor by the \(\mathrm{fold}\) operator: $\mathrm{fold}(\mathrm{bdiag}(\bar{\mathcal{A}}))=\bar{\mathcal{A}}.$

Based on the definition of $\mathrm{fold}$ and $\mathrm{bdiag}$  operators and invertible linear transform $L$, the definition of t-product can be given as follows.
\begin{definition}[T-product \cite{kernfeld2015tensor}] \label{def:tproduct}
Let $L$ be any invertible linear transform in \eqref{eq:transform}, and $\mathcal{A}$ be an $n_1\times n_2\times n_3$ third-order tensor and $\mathcal{B} $ be an 
$n_2\times n_4\times n_3$ third-order tensor, then the product of $\mathcal{A}*_L\mathcal{B}$ is an $n_1\times
n_4\times n_3$ tensor $\mathcal{C}$ which can be represented as
\[\mathcal{C} = \mathcal{A} *_L \mathcal{B} =L^{-1}[\mathrm{fold}(\mathrm{bdiag}(\bar{\mathcal{A}})\times 
\mathrm{bdiag}(\bar{\mathcal{B}}))],
\]
where ``$\times$'' denotes the standard matrix product.
\end{definition}
\begin{definition}[Tensor transpose \cite{kernfeld2015tensor}] 
Let $L$ be any invertible linear transform in \eqref{eq:transform}. Let $\mathcal{A}$ be an $n_1\times n_2 \times n_3$ tensor, then the transpose of 
$\mathcal{A}$ under $L$, denoted as $\mathcal{A}^\top$, satisfies $L(\mathcal{A}^\top)^{(i)}=L(\mathcal{A}^{(i)})^\top$,
$i=1,\dots,n_3$.
\end{definition}

\begin{definition}[Identity tensor \cite{kernfeld2015tensor}] \label{def:id tensor}
Let $L$ be any invertible linear transform in \eqref{eq:transform}. Let $\mathcal{I}$ be an $n\times n\times n_3$ tensor so that each frontal slices of 
$L(\mathcal{I})=\bar{\mathcal{I}}$ is an $n\times n$ sized identity matrix. Then
$\mathcal{I}=L^{-1}(\bar{\mathcal{I}})$ gives the identity tensor under $L$.
\end{definition}

\begin{definition}[F-diagonal tensor \cite{kilmer2011factorization}] 
A third-order tensor tensor called f-diagonal tensor if  each of its frontal slices is a diagonal matrix.
\end{definition}

\begin{definition}[Orthogonal tensor \cite{kernfeld2015tensor}] 
Let $L$ be any invertible linear transform in \eqref{eq:transform}. $\mathcal{Q}$ is an $n\times n\times n_3$ orthogonal tensor if $\mathcal{Q}^\top *_L \mathcal{Q}=\mathcal{Q} *_L \mathcal{Q}^\top=\mathcal{I}.$
\end{definition}

\begin{theorem}[T-SVD \cite{kernfeld2015tensor}] \label{th:t-svd}
Let $L$ be any invertible linear transform in \eqref{eq:transform}, and $\mathcal{A}\in\mathbb{R}^{n_1\times n_2\times n_3}$, then it can be factorized as 
\[\mathcal{A}=\mathcal{U} *_L \mathcal{S} *_L \mathcal{V}^\top,
\]
where $\mathcal{U}\in \mathbb{R}^{n_1\times n_1\times n_3}$, $\mathcal{V} \in \mathbb{R}^{n_2\times n_2\times n_3}$ are
orthogonal, and $\mathcal{S} \in \mathbb{R}^{n_1\times n_2\times n_3}$ is an f-diagonal tensor.
\end{theorem}


\section{Semi-tensor product of tensors via arbitrary invertible linear transforms}
\label{sec:3}
\par In this section, we define a novel tensor STP induced by arbitrary invertible linear transforms and investigate several theoretical properties of this newly proposed multiplication. The definitions and properties regarding Kronecker product are provided in the Supplementary Material. Building on these auxiliary prerequisites, we now define the tensor STP via arbitrary invertible linear transforms.
\begin{definition}[Semi-tensor product of tensors]\label{de:tstp}
Let $L$ be any invertible linear transform in \eqref{eq:transform}. Suppose that $\mathcal{A} \in\mathbb{R}^{n_1\times n_2\times n_3}$ and $\mathcal{B} \in\mathbb{R}^{m_1\times m_2\times n_3}$. We have 
\begin{equation*}
        \begin{split}
                 \mathcal{A} \ltimes_L \mathcal{B} 
                                &= \bigl(\mathcal{A} \otimes \mathcal{I}_{t/n_2}\bigr)*_L\bigl(\mathcal{B} \otimes \mathcal{I}_{t/m_1}\bigr)\\  
                                 &=  L^{-1}\biggl[\mathrm{fold}\biggl(\mathrm{bdiag}\bigl(\overline{\mathcal{A} \otimes \mathcal{I}_{t/n_2}}\bigr)\times 
\mathrm{bdiag}\bigl(\overline{\mathcal{B} \otimes \mathcal{I}_{t/m_1}}\bigr)\biggr)\biggr]\in \mathbb{R}^{n_1(t/n_2) \times m_2(t/m_1) \times n_3},
        \end{split}
\end{equation*}
where $t=\lcm(n_2,m_1)$, $\mathcal{I}_{t/n_2} \in \mathbb{R}^{t/n_2 \times t/n_2 \times 1}$ and $\mathcal{I}_{t/m_1} \in \mathbb{R}^{t/m_1 \times t/m_1 \times 1}$ are identity tensors.
\end{definition}
\begin{remark}
 We can see that if $n_2=m_1$,  $\mathcal{A} \ltimes_L \mathcal{B} =(\mathcal{A} \otimes \mathcal{I}_1)*_L(\mathcal{B} \otimes \mathcal{I}_1)= \mathcal{A}*_L\mathcal{B}$. That is, the tensor STP can reduce to the t-product.        
\end{remark}
The following equivalent representation follows directly from Definition \ref{de:tstp}.

\begin{lemma}\label{lem:equiv_tstp}
Let $L$ be any invertible linear transform in \eqref{eq:transform}. Suppose that $\mathcal{A} \in\mathbb{R}^{n_1\times n_2\times n_3}$ and $\mathcal{B} \in\mathbb{R}^{m_1\times m_2\times n_3}$ are third-order tensors. The following property holds:
\begin{equation*}     
\mathcal{A} \ltimes_L \mathcal{B}=  L^{-1}\biggl[\mathrm{fold}\biggl(\mathrm{bdiag}\bigl(\bar{\mathcal{A}}\bigr) \ltimes \mathrm{bdiag}(\bar{\mathcal{B}})\biggr)\biggr].   
\end{equation*}

\begin{proof} Based on  Lemma \ref{lem:equiv_stp} and Definition \ref{de:tstp}, and the definitions of the $\mathrm{fold}$, $\mathrm{bdiag}$ operators as well as the  Kronecker product, we have
\begin{equation*}
\begin{split}
 \mathcal{A} \ltimes_L \mathcal{B} &= L^{-1}\Biggl[\mathrm{fold}\biggl(\mathrm{bdiag}\Bigl(\overline{\mathcal{A}\otimes\mathcal{I}_{t/n_2}}\Bigr)\times 
\mathrm{bdiag}\Bigl(\overline{\mathcal{B} \otimes \mathcal{I}_{t/m_1}}\Bigr)\biggr)\Biggr]   \\
                                &=L^{-1}\Biggl[\mathrm{fold}\biggl(\mathrm{bdiag}\Bigl(\bar{\mathcal{A} }\otimes \mathcal{I}_{t/n_2}\Bigr)\times 
\mathrm{bdiag}\Bigl(\bar{\mathcal{B} }\otimes \mathcal{I}_{t/m_1}\Bigr)\biggr)\Biggr]   \\
                                 &=L^{-1}\Biggl[\mathrm{fold}\biggl(\Bigl(\mathrm{bdiag}(\bar{\mathcal{A}}) \otimes \mathbf{I}_{t/n_2}\Bigr)\times 
\Bigl(\mathrm{bdiag}(\bar{\mathcal{B}}) \otimes \mathbf{I}_{t/m_1}\Bigr)\biggr)\Biggr]   \\
                                 &=  L^{-1}\Biggl[\mathrm{fold}\biggl(\mathrm{bdiag}\bigl(\bar{\mathcal{A}}\bigr) \ltimes \mathrm{bdiag}\bigl(\bar{\mathcal{B}}\bigr)\biggr)\Biggr].
        \end{split}
\end{equation*}
\end{proof}
\end{lemma}
\begin{remark}\label{rem:transform_domain_interpretation}
Lemma \ref{lem:equiv_tstp} establishes that the tensor STP is equivalent to the STP of block-diagonal matrices in the transform domain, i.e., $L(\mathcal{A} \ltimes_L \mathcal{B})=\operatorname{fold}\bigl(\operatorname{bdiag}(\bar{\mathcal{A}})\ltimes \operatorname{bdiag}(\bar{\mathcal{B}})\bigr)$.
Accordingly, $\mathcal{A} \ltimes_L \mathcal{B}$ can be computed via slice-wise matrix STP operations on $\bar{\mathcal{A}}$ and $\bar{\mathcal{B}}$,  followed by the inverse linear transform $L^{-1}$.
\end{remark}
\par Next, we establish a fundamental property pertaining to the STP of tensors.
\begin{theorem} [The associative law of semi-tensor product of tensors]
        Let $\mathcal{A},\mathcal{B},\mathcal{C}$ be  third-order tensors with compatible dimensions such that their semi-tensor products are well-defined. Then the following equation holds.
        \[(\mathcal{A}\ltimes_L \mathcal{B})\ltimes_L \mathcal{C}=\mathcal{A} \ltimes_L (\mathcal{B}\ltimes_L \mathcal{C}).
        \]
        \begin{proof}
By Lemma \ref{lem:equiv_tstp},
\[
\mathcal{A} \ltimes_L \mathcal{B}=L^{-1}\Biggl[\mathrm{fold}\biggl(\mathrm{bdiag}\bigl(\bar{\mathcal{A}}\bigr) \ltimes \mathrm{bdiag}\bigl(\bar{\mathcal{B}}\bigr)\biggr)\Biggr], \mathcal{B} \ltimes_L \mathcal{C}=L^{-1}\Biggl[\mathrm{fold}\biggl(\mathrm{bdiag}\bigl(\bar{\mathcal{B}}\bigr) \ltimes \mathrm{bdiag}\bigl(\bar{\mathcal{C}}\bigr)\biggr)\Biggr].
\]
Then,
\begin{equation*}
\begin{split}
\bigl(\mathcal{A}\ltimes_L \mathcal{B}\bigr)\ltimes_L \mathcal{C}
&=L^{-1}\Biggl[\mathrm{fold}\biggl(\mathrm{bdiag}\Bigl(L\Bigl(L^{-1}\Bigl(\mathrm{fold}\Bigl(\mathrm{bdiag}\bigl(\bar{\mathcal{A}}\bigr) \ltimes \mathrm{bdiag}\bigl(\bar{\mathcal{B}}\bigr)\Bigr)\Bigr)\Bigr)\Bigr) \ltimes \mathrm{bdiag}\bigl(\bar{\mathcal{C}}\bigr)\biggr)\Biggr] \\
&=L^{-1}\Biggl[\mathrm{fold}\biggl(\Bigl(\mathrm{bdiag}\bigl(\bar{\mathcal{A}}\bigr)\ltimes \mathrm{bdiag}\bigl(\bar{\mathcal{B}}\bigr)\Bigr)\ltimes \mathrm{bdiag}\bigl(\bar{\mathcal{C}}\bigr)\biggr)\Biggr],
\end{split}
\end{equation*}
Similarly,
\begin{equation*}
\mathcal{A} \ltimes_L \bigl(\mathcal{B}\ltimes_L \mathcal{C}\bigr)
=L^{-1}\Biggl[\mathrm{fold}\biggl(\mathrm{bdiag}\bigl(\bar{\mathcal{A}}\bigr)\ltimes \Bigl(\mathrm{bdiag}\bigl(\bar{\mathcal{B}}\bigr)\ltimes \mathrm{bdiag}\bigl(\bar{\mathcal{C}}\bigr)\Bigr)\biggr)\Biggr].
\end{equation*}
The conclusion holds trivially by Lemma \ref{le:as law matrix}.
        \end{proof}
\end{theorem}
\section{Tensor singular value decomposition via semi-tensor product}
\label{sec:4}
\subsection{Single-term STP-SVD of matrices}\label{subsec:matrix stpsvd}
\par Let 
\[
\mathbf{A} = 
\begin{bmatrix}
\mathbf{A}_{1,1} & \cdots & \mathbf{A}_{1,n_1} \\
\mathbf{A}_{2,1} & \cdots & \mathbf{A}_{2,n_1} \\
\vdots & \ddots & \vdots \\
\mathbf{A}_{m_1,1} & \cdots & \mathbf{A}_{m_1,n_1}
\end{bmatrix}
\in \mathbb{R}^{m_1 m_2 \times n_1 n_2},
\]
where each block $\mathbf{A}_{i,j}$ with $i=1,2,\cdots,m_1$ and $j=1,2,\cdots,n_1$ is an $m_2\times n_2$ matrix. Next, we define the rearrangement operator $\mathscr{R}$, applied to the matrix \(\mathbf{A}\), as follows:  
\begin{equation}\label{eq:R(A)}
        \mathscr{R}(\mathbf{A})=\bigl[\mathbf{A}_1\ \mathbf{A}_2\ \cdots\ \mathbf{A}_{n_1} \bigr]^\top \in \mathbb{R}^{m_1 n_1 \times m_2 n_2},
\end{equation}
where $\mathbf{A}_j = \bigl[\mathrm{vec}(\mathbf{A}_{1,j}),\mathrm{vec}(\mathbf{A}_{2,j}),\cdots,\mathrm{vec}(\mathbf{A}_{m_1,j})\bigr]^\top$.
 
From the definition of $\mathscr{R}$ in \eqref{eq:R(A)}, we obtain the following result on optimal Kronecker product approximation.
\begin{lemma}\label{lem:kpd}\cite{van1993approximation}
        Given $\mathbf{A} \in \mathbb{R}^{m_1m_2 \times n_1n_2}$, there exist matrices $\mathbf{B} \in \mathbb{R}^{m_1 \times n_1}$ and $\mathbf{C} \in \mathbb{R}^{m_2 \times n_2}$
        such that $\mathrm{vec}(\mathbf{B}) = \sqrt{\sigma_1}\mathbf{U}(:,1)$, $\mathrm{vec}(\mathbf{C}) = \sqrt{\sigma_1}\mathbf{V}(:,1)$
        which minimize 
        \begin{equation*}
        \left\| \mathbf{A}-\mathbf{B}\otimes \mathbf{C}\right\|_F,      
        \end{equation*}
where $\sigma_1$ denotes the largest singular value of $\mathscr{R}(\mathbf{A})$ given in \eqref{eq:R(A)}, and $\mathbf{U}(:,1)$ and $\mathbf{V}(:,1)$ are the corresponding left and right singular vectors, respectively.
\end{lemma}

\par By virtue of Lemma \ref{lem:kpd}, an SVD-like approximate matrix decomposition using the STP, which is referred to in the literature as STP-SVD.

\begin{theorem}\label{th:stpsvd} \cite{chen2023tensor}
Given $\mathbf{A} \in \mathbb{R}^{m_1 m_2 \times n_1 n_2}$. Then there exist  orthogonal matrices $\mathbf{U} \in \mathbb{R}^{m_1 \times m_1}$ and $\mathbf{V} \in \mathbb{R}^{n_1 \times n_1}$ such that
\begin{equation}
\mathbf{A}=\mathbf{U} \ltimes \mathbf{\Sigma} \ltimes \mathbf{V}^\top + \mathbf{E},
\end{equation}
where $\mathbf{\Sigma} = \mathrm{blockdiag}(\mathbf{S}_1, \mathbf{S}_2, \cdots \mathbf{S}_p) \in \mathbb{R}^{m_1 m_2 \times n_1 n_2}$ is a block-diagonal matrix with blocks $\mathbf{S}_i\in \mathbb{R}^{m_2 \times n_2}$ for $i=1, 2, \cdots, p$ and $p=\min\{m_1, n_1\}$,
satisfying $\|\mathbf{S}_1\|_F \ge \|\mathbf{S}_2\|_F \ge \cdots \ge \|\mathbf{S}_p\|_F $. The approximation error satisfies
\begin{equation*}
\|\mathbf{E}\|_F = \sqrt{\sum_{i=2}^{v} \sigma_i^2},
\end{equation*}
where $\sigma_2 \ge \dots \ge \sigma_v \ge 0$ are the singular values of 
$\mathscr{R}(\mathbf{A}) \in \mathbb{R}^{m_1 n_1 \times m_2 n_2}$ defined in \eqref{eq:R(A)}, 
and $v = \min\{m_1 n_1, m_2 n_2\}$.
\end{theorem}
Algorithm~\ref{alg:stpsvd} presents the complete procedure for Theorem~\ref{th:stpsvd}.
\begin{remark}
By applying truncated SVD to the matrix $\mathbf{B}$, a truncated STP-SVD algorithm applicable to general matrices can be derived. It follows the identical workflow except replacing full $\mathrm{svd}$ by rank-truncated $\mathrm{svds}$ with given rank parameter $r$. For brevity, its pseudocode is deferred to Supplementary Material.
Correspondingly, we can obtain the upper bound of the error matrix $\mathbf{E}$ in this case \cite{chen2023tensor},
\begin{equation*}
   \|\mathbf{E}\|_F \leq \sqrt{\sum_{i=2}^{v} \sigma_i^2} + \sqrt{\sum_{j=r+1}^{p} \|\mathbf{S}_j\|_F^2}.
\end{equation*}       
\end{remark}

\begin{algorithm}[ht]
\caption{STP-SVD of matrices \cite{chen2023tensor}}
\label{alg:stpsvd}
\KwIn{$\mathbf{A} \in \mathbb{R}^{m_1 m_2 \times n_1 n_2}$.}
\KwOut{$\mathbf{U}, \mathbf{\Sigma}, \mathbf{V}$.}
Calculate matrices $\mathbf{B} \in \mathbb{R}^{m_1 \times n_1}$ and $\mathbf{C} \in \mathbb{R}^{m_2 \times n_2}$ via Lemma \ref{lem:kpd}, such that $\mathbf{A} \approx \mathbf{B} \otimes \mathbf{C}$\;
Perform full SVD on $\mathbf{B}$: $[\mathbf{U}_\mathbf{B}, \mathbf{\Sigma}_{\mathbf{B}}, \mathbf{V}_\mathbf{B}] = \mathrm{svd}(\mathbf{B})$\;
\Return{$\mathbf{U} = \mathbf{U}_\mathbf{B}, \mathbf{\Sigma} = \mathbf{\Sigma_B} \otimes \mathbf{C}, \mathbf{V}=\mathbf{V}_\mathbf{B}$.}
\end{algorithm}
\subsection{Multi-term STP-SVD of matrices}\label{subsec: matrix MSTP-SVD}
While the single-term STP-SVD (Subsection \ref{subsec:matrix stpsvd}) establishes an STP-based matrix factorization framework, its approximation accuracy is inherently limited by a single Kronecker component. To address this bottleneck and achieve more flexible, precise approximation, we propose a multi-term STP decomposition scheme that integrates multiple orthogonal components. This approach substantially improves approximation fidelity while preserving the structural properties and computational efficiency of the original single-term formulation. We first present a supporting lemma before detailing the multi-term decomposition.
\begin{lemma}\label{lem:mkpd} \cite{van1993approximation}
Given $\mathbf{A} \in \mathbb{R}^{m_1 m_2 \times n_1 n_2}$ and let $k$ be a given positive integer. Then there exist matrices $\mathbf{B}_i \in \mathbb{R}^{m_1 \times n_1}$ and $\mathbf{C}_i \in \mathbb{R}^{m_2 \times n_2}$ ($i=1,2,\cdots,k$)
such that $\mathrm{vec}(\mathbf{B}_i) = \sqrt{\sigma_i}\mathbf{U}(:,i)$, $\mathrm{vec}(\mathbf{C}_i) = \sqrt{\sigma_i}\mathbf{V}(:,i)$
which minimize 
\begin{equation}\label{eq:multi_kpd}
\left\| \mathbf{A} - \sum_{i=1}^{k} \mathbf{B}_i \otimes \mathbf{C}_i \right\|_F,       
\end{equation}
where $\sigma_i$ denotes the $i$-th singular value of $\mathscr{R}(\mathbf{A})$ given in  \eqref{eq:R(A)}, and $\mathbf{U}(:,i)$ and $\mathbf{V}(:,i)$
are the corresponding left and right singular vectors, respectively.
\end{lemma}
We formally introduce this improved decomposition in Theorem \ref{th:mstpsvd}, along with a rigorous theoretical analysis of its approximation error.
\begin{theorem}[Multi-term STP-SVD of matrices]\label{th:mstpsvd}
         Let $\mathbf{A} \in \mathbb{R}^{m_1 m_2 \times n_1 n_2}$ and $k$ be a given positive integer. 
         Then $\mathbf{A}$ can be factorized as 
         \begin{equation}
                \mathbf{A}=\sum_{i=1}^{k}\mathbf{U}_i \ltimes \mathbf{\Sigma}_i \ltimes \mathbf{V}_i^\top + \mathbf{E}_k,
         \end{equation}
         where $\mathbf{U}_i \in \mathbb{R}^{m_1 \times m_1}$ and $\mathbf{V}_i \in \mathbb{R}^{n_1 \times n_1}$ $(i=1,2,\cdots,k)$ are orthogonal, 
         and each $\mathbf{\Sigma}_i  \in\mathbb{R}^{m_1 m_2\times n_1 n_2}(i=1,2,\cdots,k)$ is  block-diagonal  with  blocks 
         $\mathbf{S}_{ij} \in \mathbb{R}^{m_2 \times n_2}$ $( j=1,2,\cdots,p)$ and $p=\min \{m_1 , n_1\}$, satisfying $\|\mathbf{S}_{i1}\|_F \ge \|\mathbf{S}_{i2}\|_F \ge \cdots \ge \|\mathbf{S}_{ip}\|_F$.
        The approximation error satisfies
\begin{equation*}
\|\mathbf{E}_k\|_F^2 = \sum_{i=k+1}^{v} \sigma_i^2,
\end{equation*}
where $\sigma_{k+1} \ge \sigma_{k+2} \ge \dots \ge \sigma_v \ge 0$ are the singular values of 
$\mathscr{R}(\mathbf{A}) \in \mathbb{R}^{m_1 n_1 \times m_2 n_2}$ defined in \eqref{eq:R(A)}, 
and $v = \min\{m_1 n_1, m_2 n_2\}$.
\begin{proof}
The proof can be found in the Supplementary Material.
\end{proof} 
\end{theorem}
\begin{remark}
Compared with the STP-SVD (Theorem \ref{th:stpsvd}), the multi-term formulation in Theorem \ref{th:mstpsvd} provides a more accurate approximation of the target matrix. In particular, when $k=1$, the multi-term decomposition collapses to the single-term case. As $k$ increases, 
the approximation error $\|\mathbf{E}_k\|_F$ decreases monotonically, since additional Kronecker components are included to capture more information from the original matrix. We now present Algorithm~\ref{alg:mstpsvd of matrices} for the multi-term STP-SVD.
\end{remark}
\begin{algorithm}[!ht]
\caption{Multi-term STP-SVD of matrices}
\label{alg:mstpsvd of matrices}
\KwIn{$\mathbf{A} \in \mathbb{R}^{m_1 m_2 \times n_1 n_2}$, the number of terms $k$.}
\KwOut{$\mathbf{U}_i$, $\mathbf{\Sigma}_i$, $\mathbf{V}_i$.}
Calculate matrices $\mathbf{B}_i \in \mathbb{R}^{m_1 \times n_1}$ and $\mathbf{C}_i \in \mathbb{R}^{m_2 \times n_2}$ ($i=1,\ldots,k$) via Lemma \ref{lem:mkpd}, such that $\mathbf{A} \approx \sum_{i=1}^{k} \mathbf{B}_i \otimes \mathbf{C}_i$\;
\For{$i = 1$ \KwTo $k$}{
    Perform full SVD on $\mathbf{B}_i$: $[\mathbf{U}_i, \mathbf{\Sigma}_{\mathbf{B}_i}, \mathbf{V}_i] = \operatorname{svd}(\mathbf{B}_i)$\;
    $\mathbf{\Sigma}_i = \mathbf{\Sigma}_{\mathbf{B}_i} \otimes \mathbf{C}_i$\;
}
\Return{$\mathbf{U}_i, \mathbf{\Sigma}_i, \mathbf{V}_i$.}
\end{algorithm}
\begin{remark}
Theorem \ref{th:mstpsvd} gives the full multi-term STP-SVD of matrices. To reduce computational overhead, we develop a truncated variant that retains only the leading $r$ singular components of each $\mathbf{B}_i$. The resulting decomposition preserves the same factorization structure and admits the following error bound:
\begin{equation}\label{eq:matrix_truncated mstpsvd_error}
\|\mathbf{E}_{k,r}\|_F^2 \leq \sum_{i=k+1}^{v}\sigma_i^2 + \sum_{i=1}^{k}\sum_{j=r+1}^{p}\|\mathbf{S}_{ij}\|_F^2,        
\end{equation}
where \(\mathbf{E}_{k,r}\) denotes the approximation error induced by the truncated multi-term STP-SVD. The first term on the right-hand side originates from the rank-k truncated SVD over \(\mathscr{R}(\mathbf{A})\), while the second term arises from truncated SVD for each matrix \(\mathbf{B}_i\). Here, r stands for the number of retained singular components for truncation on each \(\mathbf{B}_i\). For brevity, its detailed computational procedure is deferred to Supplementary Material.
\end{remark}
\subsection{Multi-term STP-SVD of tensors}\label{subsec:mstpsvd}
\par Theorem \ref{th:mstpsvd} establishes the multi-term STP-SVD for general real matrices, which approximates a target matrix by summing multiple orthogonal STP factorization components. Benefiting from the tensor STP defined under arbitrary invertible linear transforms in Definition \ref{de:tstp}, this matrix decomposition paradigm can be naturally generalized to third-order tensors, giving rise to the tensor MSTP-SVD formulation shown in Theorem \ref{th:tmstpsvd} with structurally consistent factorization form.
\begin{theorem}[MSTP-SVD]\label{th:tmstpsvd}
Let $L$ be any invertible linear transform in \eqref{eq:transform} and the transform matrix $\mathbf{L}$ satisfies $\mathbf{L}^\mathrm{H}\mathbf{L} = \mathbf{L}\mathbf{L}^\mathrm{H} = \rho\mathbf{I}_{l}$ and $\mathbf{L}^{-1} =  \mathbf{L}^\mathrm{H}/\rho$ for some constant $\rho > 0$, and \(\mathcal{A} \in \mathbb{R}^{m_1 m_2 \times n_1 n_2 \times l}\). Then it can be factorized as 
\begin{equation}\label{eq:mstpsvd}
       \mathcal{A} = \sum_{i=1}^{k} \mathcal{U}_i \ltimes_L  \mathcal{S}_i \ltimes_L \mathcal{V}_i^\top + \mathcal{E}_k,
\end{equation}
where $\mathcal{U}_i \in \mathbb{R}^{m_1 \times m_1 \times l}$ and $\mathcal{V}_i \in \mathbb{R}^{n_1 \times n_1 \times l}$ ($i=1,2,\cdots,k$) are orthogonal, each frontal slice of $\mathcal{S}_i \in \mathbb{R}^{m_1 m_2 \times n_1 n_2 \times l}$ is a block-diagonal matrix, and $\mathcal{E}_k$ is an error tensor, its squared Frobenius norm satisfies
\begin{equation}\label{eq:mstpsvd_error}
\|\mathcal{E}_k\|_F^2 = \frac{1}{\rho} \sum_{j=1}^{l} \sum_{i=k+1}^v \bigl(\hat{\sigma}_i^{(j)}\bigr)^2,  
\end{equation}
where $\hat{\sigma}_i^{(j)}$ is the $i$-th singular value of $\mathscr{R}(\bar{\mathcal{A}}^{(j)})$ and $v = \min\{m_1 n_1, m_2 n_2\}$.
\begin{proof}
The proof is provided in Supplementary Material.
\end{proof}
\end{theorem}

\begin{remark}
As implied by the constructive proof of Theorem \ref{th:tmstpsvd}, the MSTP-SVD of tensor $\mathcal{A}$ can be computed by performing the matrix MSTP-SVD on each frontal slice of $\bar{\mathcal{A}}$. Algorithm~\ref{alg:mstpsvd} summarizes this procedure.
\end{remark}
\begin{remark}
In Subsection \ref{subsec: matrix MSTP-SVD}, we introduced the truncated MSTP-SVD for matrices, and the framework can be naturally generalized to third-order tensors herein. For brevity, its detailed computational procedure is deferred to Supplementary Material. 
\end{remark}
\begin{algorithm}[!ht]
\caption{MSTP-SVD method of tensors}
\label{alg:mstpsvd}
\KwIn{$\mathcal{A} \in \mathbb{R}^{m_1 m_2 \times n_1 n_2 \times l}$, the number of terms $k$.}
\KwOut{$\mathcal{U}_i$, $\mathcal{S}_i$, $\mathcal{V}_i$.}
Obtain $\bar{\mathcal{A}}$ by applying an invertible linear transform $L$ on $\mathcal{A}$\;
\For{$j = 1$ \KwTo $l$}{
    Approximate the $j$-th frontal slice by Lemma \ref{lem:mkpd}:$\bar{\mathcal{A}}^{(j)} \approx \sum_{i=1}^{k} \mathbf{B}_i^{(j)} \otimes \mathbf{C}_i^{(j)}$\;
    \For{$i = 1$ \KwTo $k$}{
        Compute the full SVD of $\mathbf{B}_i^{(j)}$: $[\mathbf{U}_i^{(j)}, \mathbf{\Sigma}_{\mathbf{B}_i}^{(j)}, \mathbf{V}_i^{(j)}] = \mathrm{svd}\left(\mathbf{B}_i^{(j)}\right)$\;
       $\mathbf{\Sigma}_i^{(j)} = \mathbf{\Sigma}_{\mathbf{B}_i}^{(j)} \otimes \mathbf{C}_i^{(j)}$\;
        Store $\mathbf{U}_i^{(j)}$, $\mathbf{\Sigma}_i^{(j)}$, $\mathbf{V}_i^{(j)}$ into $\mathcal{U}_i$, $\mathcal{S}_i$, $\mathcal{V}_i$, respectively\;
    }
}
\Return{$\mathcal{U}_i = L^{-1}(\mathcal{U}_i)$, $\mathcal{S}_i = L^{-1}(\mathcal{S}_i)$, $\mathcal{V}_i = L^{-1}(\mathcal{V}_i)$.}
\end{algorithm}
The core idea of truncated MSTP-SVD (TMSTP-SVD) is to perform a truncated matrix MSTP-SVD on each  $\mathbf{B}_i^{(j)}$ when decomposing the $\bar{\mathcal{A}}^{(j)}$. 
For this purpose, let $\mathbf{R} = [R_{ij}] \in \mathbb{N}_+^{k \times l}$ be a matrix of positive integers, where $R_{ij}$ denotes the truncation rank for the SVD of $\mathbf{B}_i^{(j)}$. For the tensor factorization \eqref{eq:mstpsvd},
the number of diagonal blocks within the $j$-th frontal slice $\mathcal{S}_i^{(j)}$ precisely equals $R_{ij}$. Based on this correspondence, we formally define the rank-truncation matrix $\mathbf{R} \in \mathbb{N}_+^{k \times l}$ below.
\begin{definition}
Consider $\mathcal{A} \in \mathbb{R}^{m_1 m_2 \times n_1 n_2 \times l}$ that obeys the full MSTP-SVD factorization given in \eqref{eq:mstpsvd}.
We define the truncation rank matrix $\mathbf{R} = [R_{ij}] \in \mathbb{N}_+^{k \times l}$, whose entry $R_{ij}$ stands for the retained singular rank adopted for matrix $\mathbf{B}_i^{(j)}$ during the decomposition of the $\bar{\mathcal{A}}^{(j)}$.
For the truncated variant with rank matrix $\mathbf{R}$, we denote the corresponding approximation error tensor as $\mathcal{E}_{k,\mathbf{R}}$.
\end{definition}
\begin{remark}
Suppose that the transform matrix $\mathbf{L}$ satisfies $\mathbf{L}^\mathrm{H}\mathbf{L} = \mathbf{L}\mathbf{L}^\mathrm{H} = \rho\mathbf{I}_{l}$ and $\mathbf{L}^{-1} =  \mathbf{L}^\mathrm{H}/\rho$ for some constant $\rho > 0$. For the TMSTP-SVD with $k$ multi-terms and truncation rank matrix $\mathbf{R}$, the approximation error tensor $\mathcal{E}_{k,\mathbf{R}}$ obeys the following squared Frobenius norm upper bound:
\begin{equation}\label{eq:truncated_mstpsvd_error}
        \begin{split}
\|\mathcal{E}_{k,\mathbf{R}}\|_F^2 &= \left(\frac{1}{\sqrt{\rho}} \|\mathrm{bdiag}(\bar{\mathcal{E}}_{k,\mathbf{R}})\|_F\right)^2 = \frac{1}{\rho} \sum_{j=1}^{l} \|\bar{\mathcal{E}}_{k,\mathbf{R}}^{(j)}\|_F^2\\
&\leq \frac{1}{\rho} \sum_{j=1}^{l} \Biggl(\sum_{i=k+1}^{v}\bigl(\hat{\sigma}_i^{(j)}\bigr)^2 + \sum_{i=1}^{k}\sum_{t=R_{ij}+1}^{p} \|\mathbf{S}_{it}^{(j)}\|_F^2\Biggr),        
        \end{split}
\end{equation}
where  the last inequality is from \eqref{eq:matrix_truncated mstpsvd_error}. The two summation terms inside the parentheses correspond to two independent sources of approximation error:
\begin{itemize}
    \item The term $\sum_{i=k+1}^{v}(\hat{\sigma}_i^{(j)})^2$ arises from truncating the trailing singular values of the rearranged matrix $\mathscr{R}(\bar{\mathcal{A}}^{(j)})$, where $\hat{\sigma}_i^{(j)}$ are the singular values of $\mathscr{R}(\bar{\mathcal{A}}^{(j)})$ and $v = \min\{m_1 n_1, m_2 n_2\}$.
    \item The term $\sum_{i=1}^{k}\sum_{t=R_{ij}+1}^{p}\|\mathbf{S}_{it}^{(j)}\|_F^2$ originates from the rank-$R_{ij}$ truncated SVD of each matrix $\mathbf{B}_i^{(j)}$, where $\mathbf{S}_{it}^{(j)}$ are the diagonal blocks located in the $j$-th frontal slice of $\bar{\mathcal{S}}_i$, and $p = \min\{m_1, n_1\}$.
\end{itemize}
\end{remark}

\section{Fast randomized tensor singular value decomposition via semi-tensor product}\label{sec:random}
\par Despite improved approximation accuracy, deterministic MSTP-SVD (Subsection \ref{subsec:mstpsvd}) incurs heavy computational cost dominated by SVD on the large rearranged matrix \(\mathscr{R}(\bar{\mathcal{A}}^{(j)})\), as verified in Fig.~\ref{fig:svd_time} with four high-resolution RGB images. We therefore adopt a randomized acceleration scheme, replacing the costly exact SVD with its randomized counterpart for notable speedup with negligible accuracy loss.
\begin{figure}[!t]
  \centering
  \includegraphics[width=\textwidth]{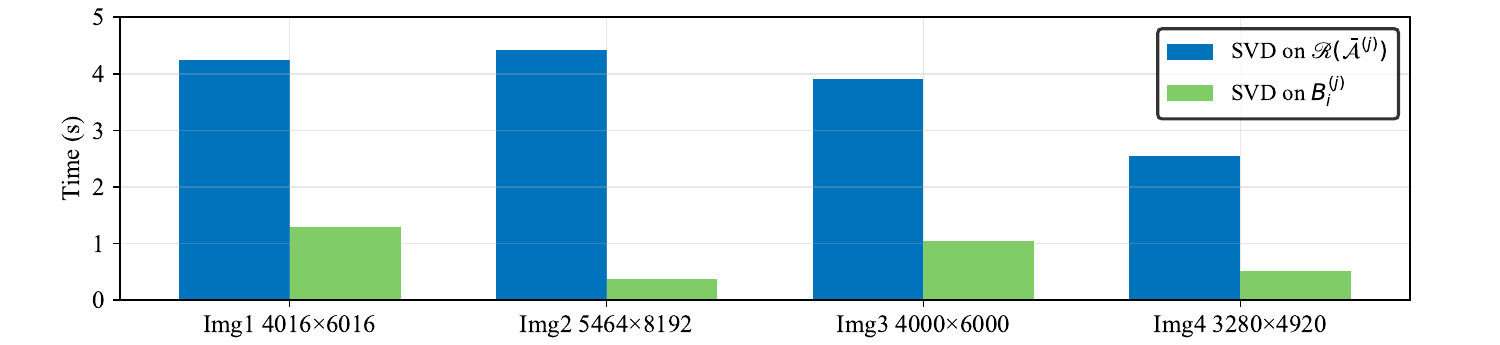}
  \caption{Runtime comparison of SVD computations for \(\mathscr{R}(\bar{\mathcal{A}}^{(j)})\) (Case 1) and submatrices \(\mathbf{B}_{i}^{(j)}\) (Case 2) under MSTP-SVD (\(k=1\)).}
  \label{fig:svd_time}
\end{figure}

Randomized T-SVD projects high-dimensional tensors onto a low-dimensional subspace via random projection, performs efficient T-SVD in the subspace, and reconstructs results in the original space, achieving significant complexity reduction with bounded approximation error.
\par As the t-product under any invertible linear transform  decouples into slice-wise matrix multiplications \cite{kernfeld2015tensor}, we  analyze randomized T-SVD entirely in the transform domain. 
We generate a Gaussian random tensor \(\mathcal{G}\) (each frontal slice $\bar{\mathcal{G}}^{(j)}$ has independent standard-normal entries \cite{qin2024nonconvex}) and apply  transform $L$ to both  \(\mathcal{A}\) and  \(\mathcal{G}\)  to obtain $\bar{\mathcal{A}}$ and $\bar{\mathcal{G}}$.
Direct full SVD on each large-scale frontal slice $\bar{\mathcal{A}}^{(j)}$ incurs prohibitive cost. Naive random projection \(\mathbf{Y}_0 = \bar{\mathcal{A}}^{(j)}\bar{\mathcal{G}}^{(j)}.\)
suffers from large approximation bias when the singular-value gap of \(\bar{\mathcal{A}}^{(j)}\) is narrow, requiring heavy oversampling \cite{halko2011finding}. We thus adopt power-enhanced random projection:
\begin{equation*}
\mathbf{Y} = \bigl(\bar{\mathcal{A}}^{(j)}(\bar{\mathcal{A}}^{(j)})^\top\bigr)^q \bar{\mathcal{A}}^{(j)}\bar{\mathcal{G}}^{(j)}.       
\end{equation*}
\begin{remark}
Substituting the SVD of \(\bar{\mathcal{A}}^{(j)}\) into the power term shows that singular values are raised to the $2q+1$. 
This widens the singular-value gap \(\tau_k^{(j)}=\hat{\sigma}_{k+1}^{(j)}/\hat{\sigma}_{k}^{(j)}\), suppresses trivial components, and mitigates oversampling-induced bias, yielding a tighter error bound with only marginal extra computational cost.
\end{remark}
After obtaining \(\mathbf{Y}\), thin QR factorization yields an orthonormal basis \(\mathbf{Q}_j\) for the principal subspace of \(\bar{\mathcal{A}}^{(j)}\). 
Projecting \(\bar{\mathcal{A}}^{(j)}\) onto \(\mathbf{Q}_j\) gives a compact matrix \(\mathbf{B} = \mathbf{Q}_j^\top \bar{\mathcal{A}}^{(j)}\), whose  SVD is far cheaper.
We retain the top-k singular components, store truncated factors \(\mathbf{U},\mathbf{S},\mathbf{V}^\top\) in the transform domain, and apply \(L^{-1}\) to recover T-SVD factors in the original space. The full procedure is outlined in Algorithm~\ref{alg:RTSVD}.
\begin{algorithm}[!ht]
\caption{Randomized T-SVD with power iteration (RT-SVD) \cite{zhang2018randomized}}
\label{alg:RTSVD}
\KwIn{$\mathcal{A} \in \mathbb{R}^{m \times n \times l}$, truncation term $k$, oversampling parameter $s \ge 0$, iteration parameter $q \ge 0$.}
\KwOut{$\mathcal{U}_k$, $\mathcal{S}_k$, $\mathcal{V}_k$.}
Generate a Gaussian random tensor $\mathcal{G} \in \mathbb{R}^{n \times (k+s) \times l}$\;
Compute $\bar{\mathcal{A}} = L(\mathcal{A})$ and $\bar{\mathcal{G}} = L(\mathcal{G})$\;
\For{$j = 1$ \KwTo $l$}{
    Compute $\mathbf{Y} = \bigl(\bar{\mathcal{A}}^{(j)} (\bar{\mathcal{A}}^{(j)})^\top\bigr)^q \bar{\mathcal{A}}^{(j)} \bar{\mathcal{G}}^{(j)}$\;
    Compute thin-QR factorization $\mathbf{Y} = \mathbf{Q}_j \mathbf{R}$\;
    Compute $\mathbf{B} = \mathbf{Q}_j^\top \bar{\mathcal{A}}^{(j)}$\;
    Compute the SVD of $\mathbf{B}$: $\mathbf{B} = \mathbf{U} \mathbf{S} \mathbf{V}^\top$\;
    Form $\mathbf{U}_k$, $\mathbf{V}_k$, $\mathbf{S}_k$ by truncating $\mathbf{Q}_j\mathbf{U}$, $\mathbf{V}$, $\mathbf{S}$ with $k$\;
    Set $\bar{\mathcal{U}_k}^{(j)} = \mathbf{U}_k$, $\bar{\mathcal{S}_k}^{(j)} = \mathbf{S}_k$, $\bar{\mathcal{V}_k}^{(j)} = \mathbf{V}_k$\;
}
\Return{$\mathcal{U}_k = L^{-1}(\bar{\mathcal{U}_k})$,\ $\mathcal{S}_k = L^{-1}(\bar{\mathcal{S}_k})$,\ $\mathcal{V}_k = L^{-1}(\bar{\mathcal{V}_k})$.}
\end{algorithm}

Based on the runtime results in Fig.~\ref{fig:svd_time}, which verifies that SVD on \(\mathscr{R}(\bar{\mathcal{A}}^{(j)})\) dominates the computational cost of deterministic MSTP-SVD, we embed the above randomized subspace extraction scheme into the MSTP-SVD framework. Concretely, we replace \(\bar{\mathcal{A}}^{(j)}\) adopted in standard randomized T-SVD with the rearranged matrix \(\mathscr{R}(\bar{\mathcal{A}}^{(j)})\) defined for multi-term semi-tensor modeling, leading to our accelerated MRSTP-SVD algorithm. Its complete pipeline is provided in Algorithm~\ref{alg:mrstpsvd}. 

We further derive the unified expected error bound for the proposed MRSTP-SVD algorithm, which is stated in the following theorem.
\begin{theorem}\label{th:MRSTP-SVD error}
Let $L$ be  any invertible linear transform in  \eqref{eq:transform}, and the transform matrix $\mathbf{L}$ satisfies $\mathbf{L}^\mathrm{H}\mathbf{L} = \mathbf{L}\mathbf{L}^\mathrm{H} = \rho\mathbf{I}_{l}$ and $\mathbf{L}^{-1} =  \mathbf{L}^\mathrm{H}/\rho$ for some constant $\rho > 0$. Suppose that $\mathcal{A} \in \mathbb{R}^{m_1 m_2 \times n_1 n_2\times l} $  and  $\mathcal{G} \in \mathbb{R}^{m_2 n_2 \times (k+s) \times l}$  is a Gaussian random tensor with $k+s \leq \min\{m_1 n_1,m_2 n_2\}$.
If $\widetilde{\mathcal{A}} = \sum_{i=1}^{k} \mathcal{U}_i \ltimes_L  \mathcal{S}_i \ltimes_L \mathcal{V}_i^\top $, where $\mathcal{U}_i$, $\mathcal{S}_i$, and $\mathcal{V}_i$ are obtained from  Algorithm~\ref{alg:mrstpsvd}, then,
\begin{equation}\label{eq:mrstpsvd_error}
        \mathbb{E}\|\mathcal{A}-\widetilde{\mathcal{A} }\|_F^2 \leq \frac{2}{\rho} \sum_{j=1}^{l} \left[\left(2+\frac{k}{s-1} (\tau_k^{(j)})^{4q} \right) \left(\sum_{i=k+1}^{v} (\hat{\sigma}_i^{(j)})^2 \right)  \right],
\end{equation}
where $v=\min\{m_1 n_1,m_2 n_2\}$, $k$ is the target truncation term, $s \ge 2 $ denotes the oversampling parameter , $q$ is the number of power iteration steps, $\hat{\sigma}_i^{(j)}$ denotes the $i$-th singular value of $\mathscr{R}(\bar{\mathcal{A}}^{(j)})$, $\tau_k^{(j)}=\hat{\sigma}_{k+1}^{(j)}/\hat{\sigma}_k^{(j)} \ll 1$ is the singular value gap.
\begin{proof}
The detailed proof  is provided in Supplementary Material.
\end{proof}
\end{theorem}
Theorem \ref{th:MRSTP-SVD error} reveals a delicate three-way trade-off among hyperparameters $k$, $s$, and $q$. Inceasing the truncation term $k$ reduces residual energy from truncated trailing singular components but amplifies random-sampling bias. 
We can mitigate this sampling error by adopting a larger oversampling parameter $s$.
Especially for slowly-decaying singular values of \(\mathscr{R}(\bar{\mathcal{A}}^{(j)})\), more power-iteration steps $q$ suppress the detrimental effect of the singular-value gap \(\tau_k^{(j)}\), albeit with additional computational overhead.
\begin{remark}
        The above error bound is derived based on full MSTP-SVD, if we adopt the TMSTP-SVD scheme defined in Supplementary Material, the deterministic residual \(\|\mathcal{A}-\mathcal{A}_{\text{MSTP}}\|_F^2\) is updated to \eqref{eq:truncated_mstpsvd_error}. Substituting this truncated deterministic error into the proof flow of Theorem \ref{th:MRSTP-SVD error}, the overall expected squared Frobenius error of truncated MRSTP-SVD (TMRSTP-SVD) reads
\begin{equation*}
        \mathbb{E}\|\mathcal{A}-\widetilde{\mathcal{A}}_{\text{Trunc}}\|_F^2 \leq \frac{2}{\rho} \sum_{j=1}^{l} \left[\left(2+\frac{k}{s-1} (\tau_k^{(j)})^{4q} \right) \left(\sum_{i=k+1}^{v} (\hat{\sigma}_i^{(j)})^2 \right)+\sum_{i=1}^{k} \sum_{t=R_{ij}+1}^{p}\|\mathbf{S}_{it}^{(j)}\|_F^2  \right].
\end{equation*}
Compared with \eqref{eq:mrstpsvd_error}, the extra  term \(\sum_{i=1}^{k}\sum_{t=R_{ij}+1}^{p} \left\|S_{it}^{(j)}\right\|_F^2\) originates from rank-\(R_{ij}\) truncated SVD for each block matrix \(B_i^{(j)}\), where \(S_{it}^{(j)}\) are diagonal blocks in the $j$-th frontal slice of \(\bar{\mathcal{S}}_i\), and \(p=\min\{m_1,n_1\}\). 
\end{remark}
Correspondingly, we outline the full procedure of TMRSTP-SVD, which serves as a simplified truncated counterpart to Algorithm~\ref{alg:mrstpsvd}. For brevity, its pseudocode is deferred to Supplementary Material.
\begin{algorithm}[!ht]
\caption{MRSTP-SVD method of tensors}
\label{alg:mrstpsvd}
\KwIn{$\mathcal{A} \in \mathbb{R}^{m_1 m_2 \times n_1 n_2 \times l}$, the number of terms $k$, oversampling parameter $s \ge 0$, iteration parameter $q \ge 0$, truncated rank matrix $\mathbf{R}$.}
\KwOut{$\mathcal{U}_i$, $\mathcal{S}_i$, $\mathcal{V}_i$.}
Generate a Gaussian random tensor $\mathcal{G} \in \mathbb{R}^{n_1 n_2 \times (k+s) \times l}$\;
Compute $\bar{\mathcal{A}} = L(\mathcal{A})$ and $\bar{\mathcal{G}} = L(\mathcal{G})$\;
\For{$j = 1$ \KwTo $l$}{
    Obtain $\mathscr{R}(\bar{\mathcal{A}}^{(j)})$ by reorganizing the blocks of $\bar{\mathcal{A}}^{(j)}$\;
    Compute $\mathbf{Y} = \bigl(\mathscr{R}(\bar{\mathcal{A}}^{(j)}) \mathscr{R}(\bar{\mathcal{A}}^{(j)})^\top\bigr)^q \mathscr{R}(\bar{\mathcal{A}}^{(j)}) \bar{\mathcal{G}}^{(j)}$\;
    Compute thin-QR factorization $\mathbf{Y} = \mathbf{Q}_j \mathbf{R}$\;
    Compute $\mathbf{B} = \mathbf{Q}_j^\top \mathscr{R}(\bar{\mathcal{A}}^{(j)})$\;
    Compute the SVD of $\mathbf{B}$: $\mathbf{B} = \mathbf{U} \mathbf{S} \mathbf{V}^\top$\;
    Form $\mathbf{U}_k$, $\mathbf{V}_k$, $\mathbf{S}_k$ by truncating $\mathbf{Q}_j\mathbf{U}$, $\mathbf{V}$, $\mathbf{S}$ with $k$\;
    \For{$i = 1$ \KwTo $k$}{
        $\operatorname{vec}(\mathbf{B}_i^{(j)}) = \sqrt{\mathbf{S}_k(i,i)} \, \mathbf{U}_k(:, i)$, \quad$\operatorname{vec}(\mathbf{C}_i^{(j)}) = \sqrt{\mathbf{S}_k(i,i)} \, \mathbf{V}_k(:, i)$
        such that $\bar{\mathcal{A}}^{(j)} \approx \sum_{i=1}^{k} \mathbf{B}_i^{(j)} \otimes \mathbf{C}_i^{(j)}$\;
        Compute the full SVD of $\mathbf{B}_i^{(j)}$: $[\mathbf{U}_i^{(j)}, \mathbf{\Sigma}_{B_i}^{(j)}, \mathbf{V}_i^{(j)}] = \operatorname{svd}\left(\mathbf{B}_i^{(j)}\right)$\;
        $\mathbf{\Sigma}_i^{(j)} = \mathbf{\Sigma}_{\mathbf{B}_i}^{(j)} \otimes \mathbf{C}_i^{(j)}$\;
        Store $\mathbf{U}_i^{(j)}$, $\mathbf{\Sigma}_i^{(j)}$, $\mathbf{V}_i^{(j)}$ into $\mathcal{U}_i$, $\mathcal{S}_i$, $\mathcal{V}_i$, respectively\;
    }
}
\Return{$\mathcal{U}_i = L^{-1}(\mathcal{U}_i)$, $\mathcal{S}_i = L^{-1}(\mathcal{S}_i)$, $\mathcal{V}_i = L^{-1}(\mathcal{V}_i)$.}
\end{algorithm}
\section{Numerical experiments}
\label{sec:experiment}
In this section, we conduct numerical experiments on real-world data and compare with other algorithms: truncated T-SVD (TT-SVD) \cite{kilmer2011factorization}, STP-SVD and  truncated STP-SVD (TSTP-SVD) \cite{chen2023tensor} to substantiate the superiority and effectiveness of our methods. All simulations are performed on a laptop computer with 2.50GHz Intel(R) Core(TM) i5-10300H CPU and 24GB memory.
The Peak Signal-to-Noise Ratio (PSNR), the structural similarity (SSIM), and the
CPU runtime are employed to evaluate the  performance of the proposed algorithm. Let \(\mathcal{X}\), \(\hat{\mathcal{X}} \in \mathbb{R}^{m \times n \times l}\) be the original tensor and its reconstruction. The PSNR and SSIM are defined as
\begin{equation*}
        \text{PSNR} = 10\log_{10}\left(\frac{\|\mathcal{X}\|_\infty^2}{\frac{1}{mnl}\|\mathcal{X}-\mathcal{\hat{X}}\|_F^2}\right), \text{SSIM} = \frac{\left(2\mu_{\mathcal{X}}\mu_{\hat{\mathcal{X}}}+C_1\right)\left(2\sigma_{\mathcal{X} \hat{\mathcal{X}}}+C_2\right)}{\left(\mu_{\mathcal{X}}^2+\mu_{\hat{\mathcal{X}}}^2+C_1\right)\left(\sigma_{\mathcal{X}}^2+\sigma_{\hat{\mathcal{X}}}^2+C_2\right)},
\end{equation*}
where $\|\mathcal{X}\|_\infty$ denotes the maximum absolute value of all entries in \(\mathcal{X}\), $\sigma_{\mathcal{X} \hat{\mathcal{X}}}$ is the cross-covariance between $\mathcal{X}$ and $\hat{\mathcal{X}}$, $\mu_{\mathcal{X}}$, $\mu_{\hat{\mathcal{X}}}$ represent the average values of \(\mathcal{X}\) and \(\hat{\mathcal{X}}\), $\sigma_{\mathcal{X}}$, $\sigma_{\hat{\mathcal{X}}}$ are the standard deviations, $C_1, C_2$ are constants.
In general, larger PSNR and SSIM values correspond to superior reconstruction quality.

\begin{figure}[!htbp]
  \centering
    \includegraphics[width=\textwidth]{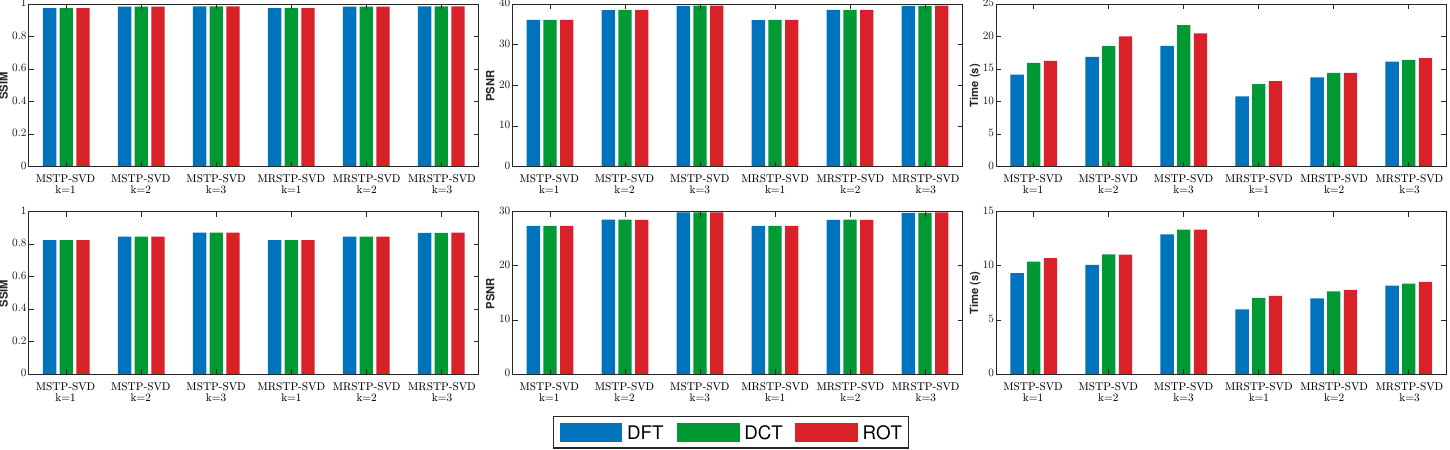}
    \caption{Quantitative metrics (PSNR, SSIM, runtime) of MSTP-SVD and randomized MRSTP-SVD for image compression under invertible transforms (DFT, DCT, ROT). Top: lake; Bottom: road.}
    \label{fig:conv_three_subplot}
  \end{figure}

\subsection{Compression on the image data}\label{subsec:image}
In this section, we employ the proposed algorithm for image compression. First, we select three distinct linear transforms: DFT, discrete cosine transform (DCT), and random orthogonal transform (ROT), and test their impacts on the performance of our proposed algorithm on four RGB benchmark images (Lake, Night, Road, and Fruit) downloaded from the ISO Republic website\footnote{https://isorepublic.com/}. Their corresponding resolutions are \(6016 \times 4016 \times 3\), \(4920 \times 3280 \times 3\), \(8192 \times 5464 \times 3\), and \(6000 \times 4000 \times 3\), respectively.
 Fig.~\ref{fig:conv_three_subplot}  reports PSNR, SSIM and runtime (in seconds) under different transform schemes, showcasing results for the lake and road images. The three invertible linear transforms deliver comparable PSNR and SSIM, while DFT exhibits persistently lower computational overhead. Averaged across two test images, DFT reduces total runtime by over 10\% compared with DCT and ROT. We therefore select DFT as the default linear transform for all subsequent experiments. Additional results for other test images in the Supplementary Material further validate this finding.
\begin{figure}[!ht]
\centering
\renewcommand{\arraystretch}{0.3}
\setlength\tabcolsep{0.1pt}
\begin{tabular}{@{}ccccccc@{}}

\tiny Original &\tiny TT-SVD & \tiny STP-SVD & \tiny TSTP-SVD &\tiny\makecell[c]{MSTP-SVD\\[-4pt](k=2)} & \tiny\makecell[c]{MSTP-SVD\\[-4pt](k=3)} &\tiny\makecell[c]{TMSTP-SVD\\[-4pt](k=2)} \\
\includegraphics[width=0.672in]{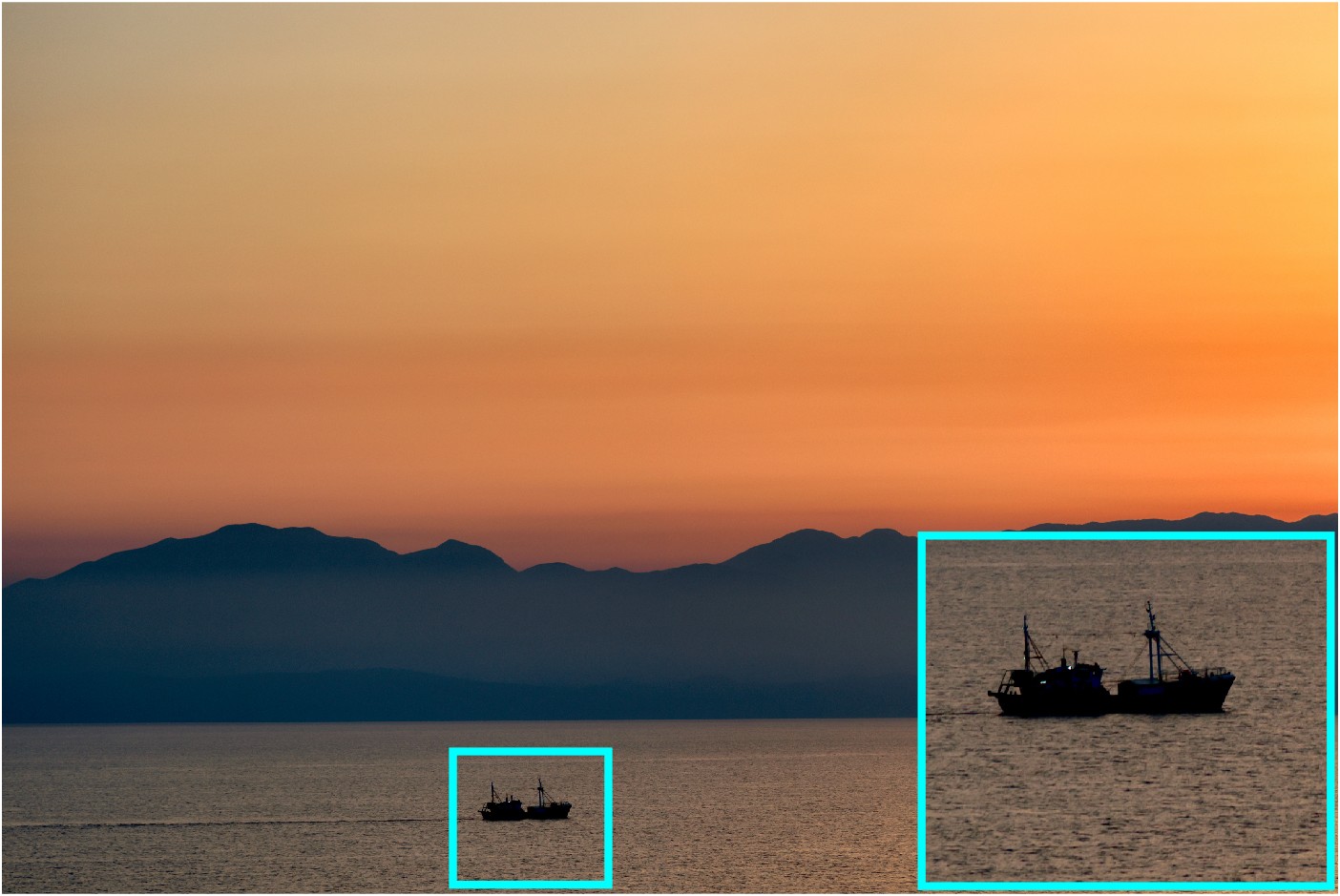} &
\includegraphics[width=0.672in]{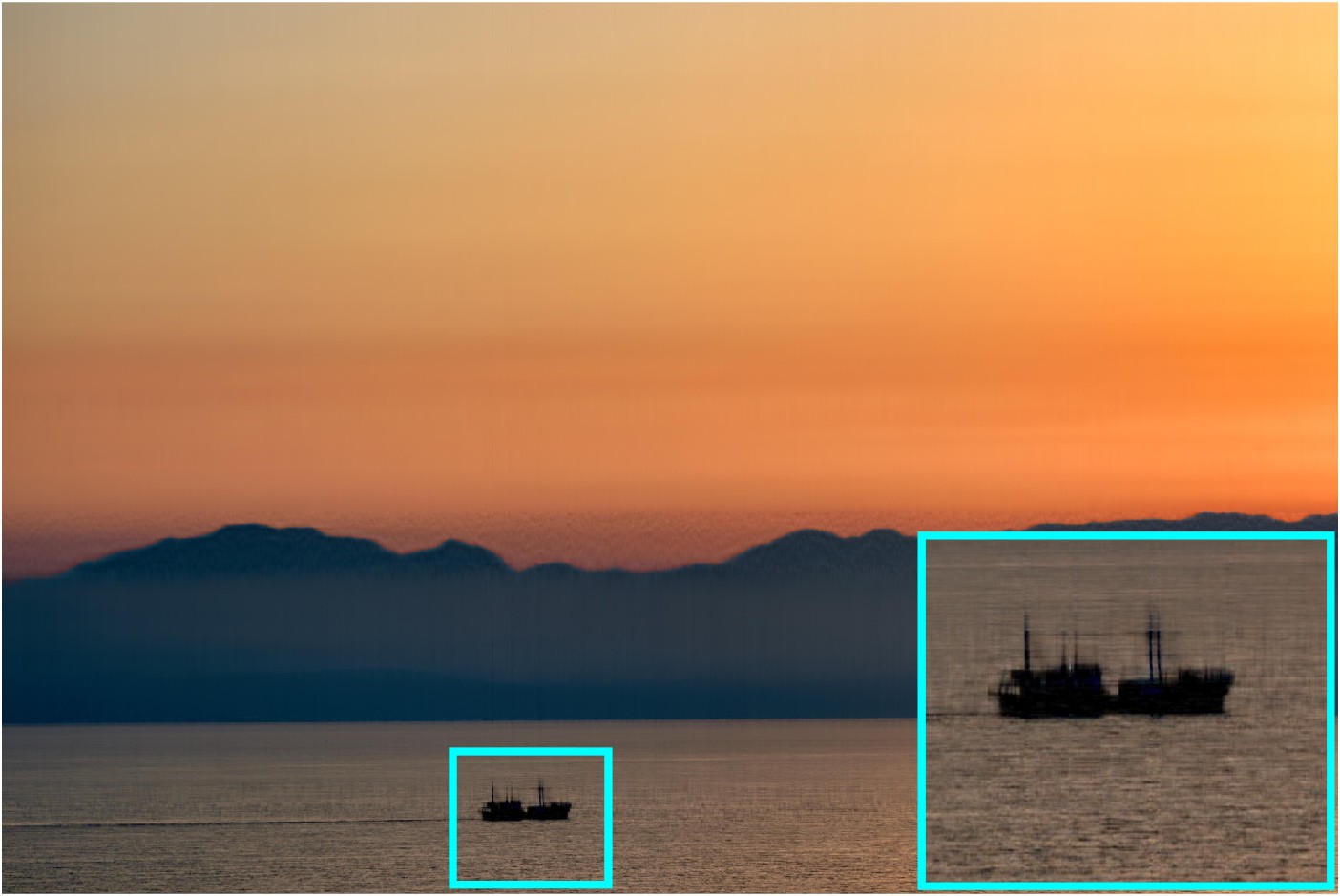} &
\includegraphics[width=0.672in]{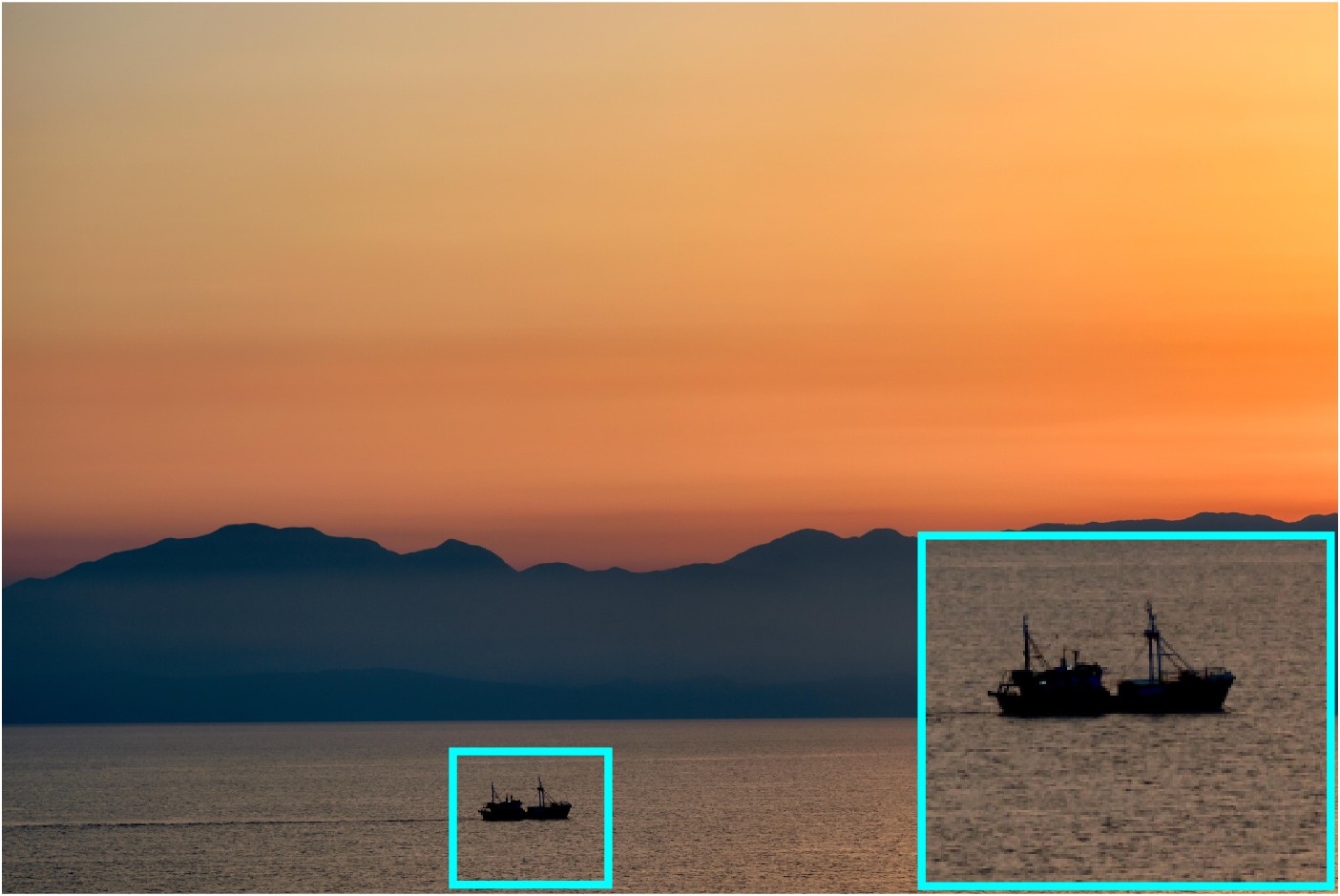} &
\includegraphics[width=0.672in]{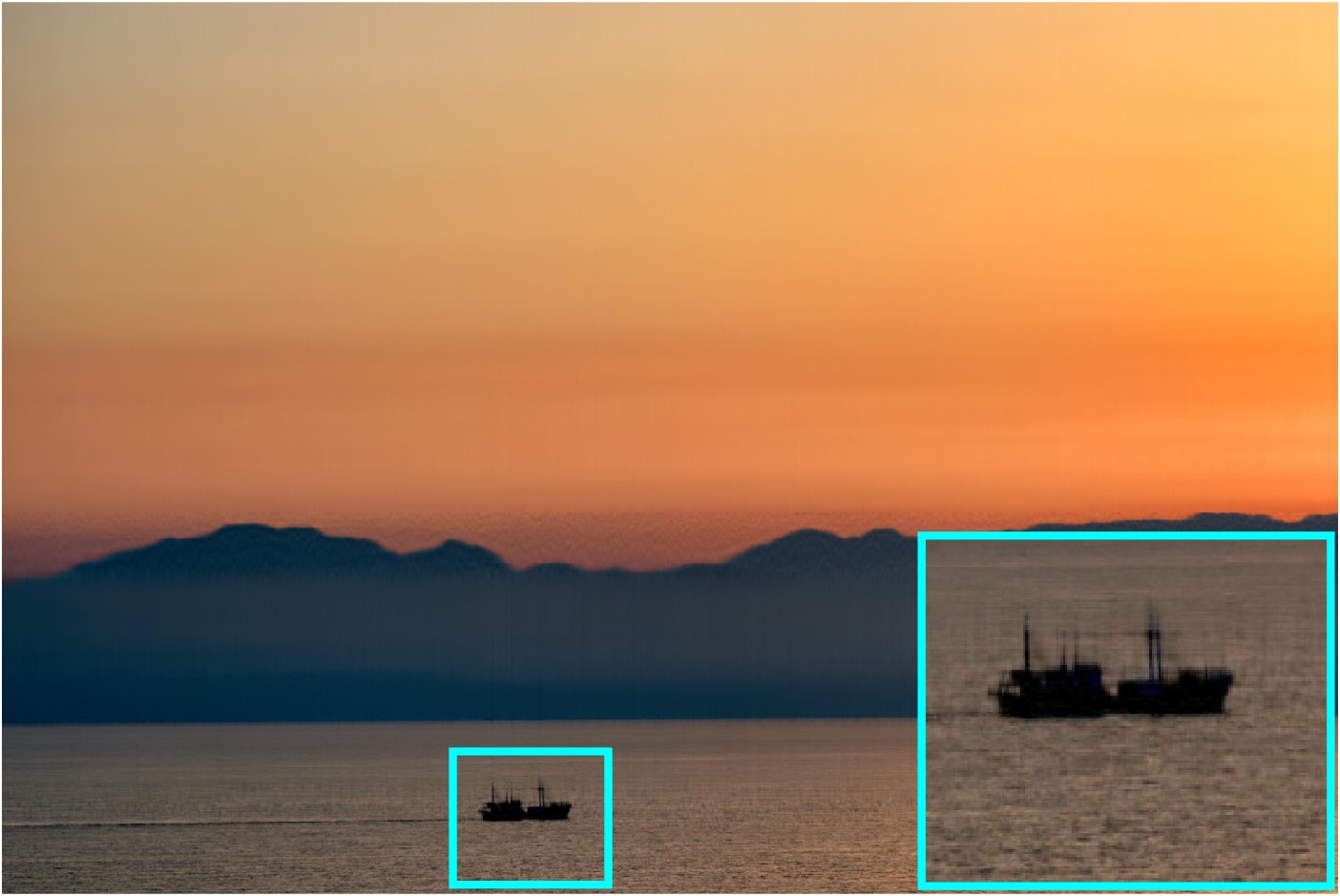} &
\includegraphics[width=0.672in]{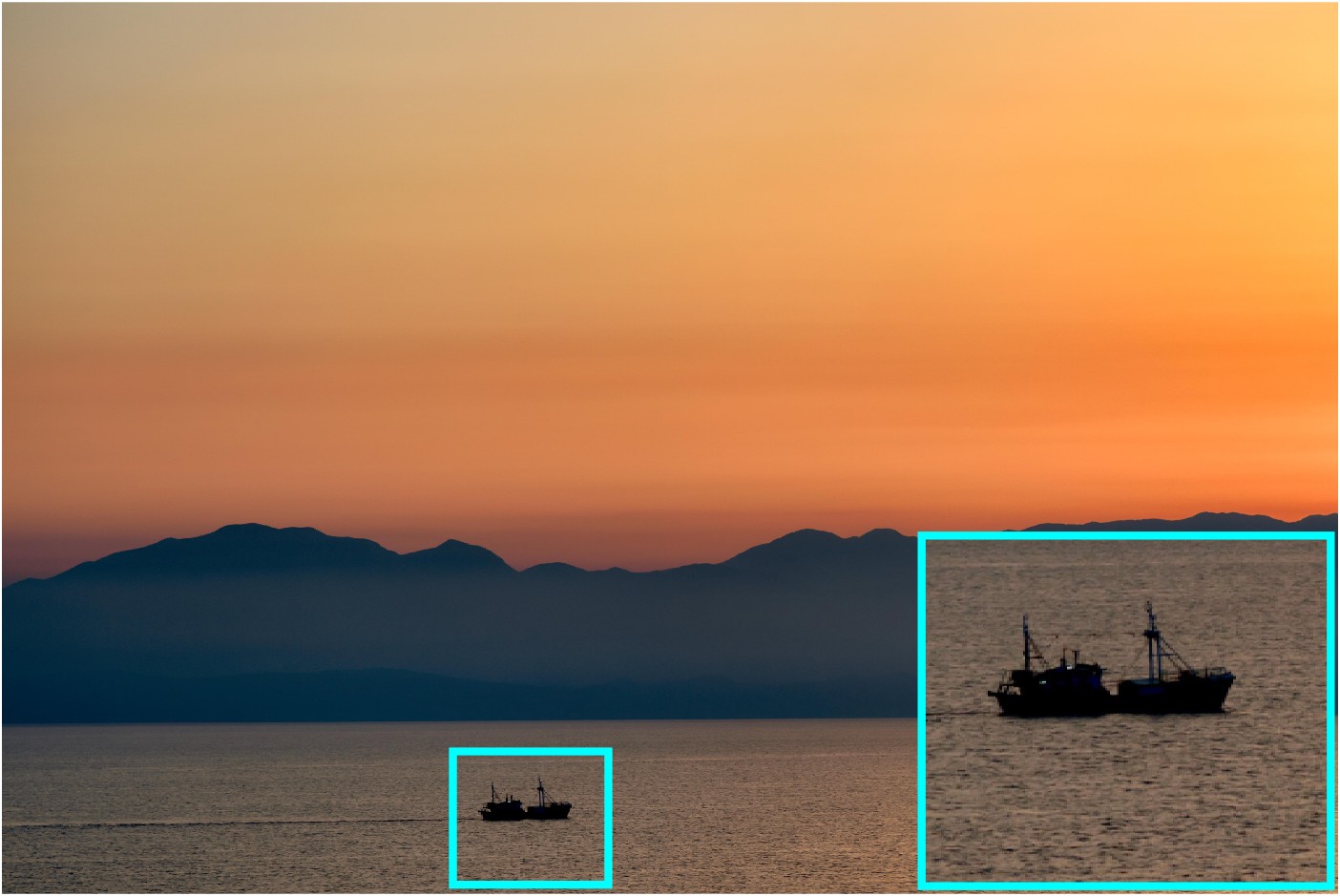} &
\includegraphics[width=0.672in]{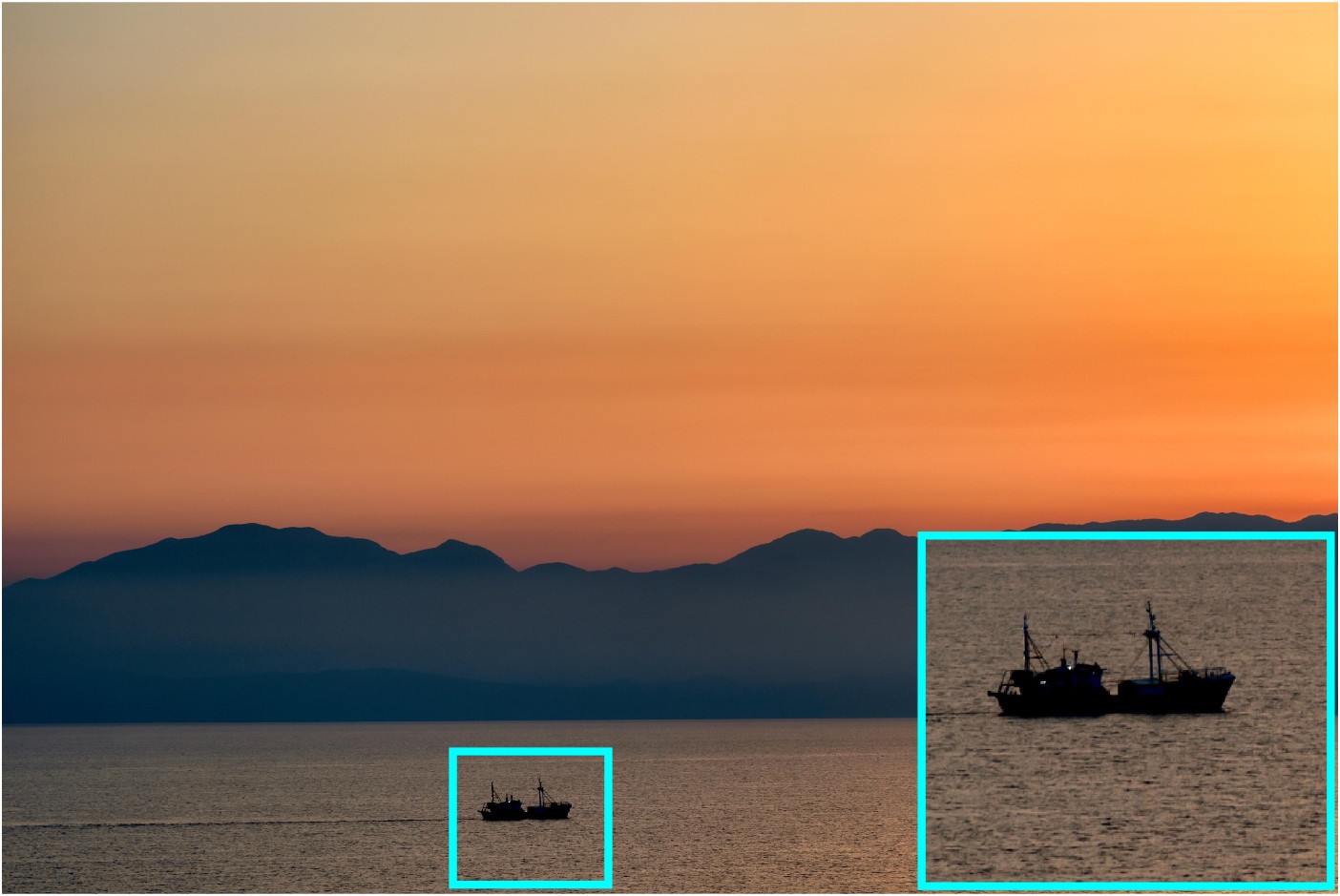} &
\includegraphics[width=0.672in]{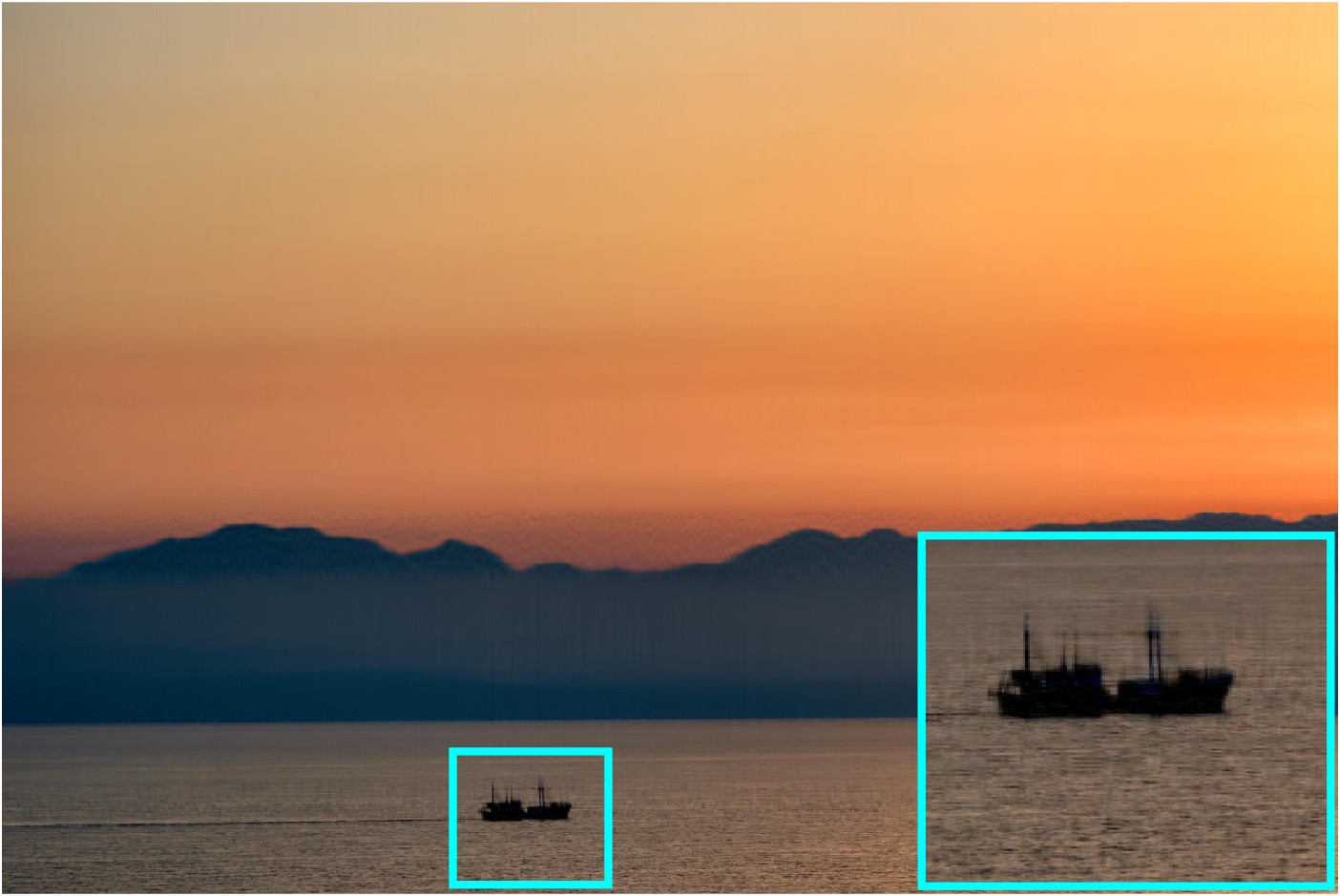} \\
 &\tiny PSNR:34.56 & \tiny PSNR:36.17 & \tiny PSNR:34.06 &\tiny PSNR:38.60 & \tiny PSNR:39.64 &\tiny PSNR:34.86 \\
 &\tiny Time:18.00s & \tiny Time:14.26s & \tiny Time:12.71s &  \tiny Time:16.79s & \tiny Time:22.37s &  \tiny Time:13.55s\\
 \tiny\makecell[c]{TMSTP-SVD\\[-4pt](k=3)} &\tiny\makecell[c]{MRSTP-SVD\\[-4pt](k=1)} & \tiny\makecell[c]{MRSTP-SVD\\[-4pt](k=2)} & \tiny\makecell[c]{MRSTP-SVD\\[-4pt](k=3)}& \tiny\makecell[c]{TMRSTP-SVD\\[-4pt](k=1)} & \tiny\makecell[c]{TMRSTP-SVD\\[-4pt](k=2)} & \tiny\makecell[c]{TMRSTP-SVD\\[-4pt](k=3)}\\
\includegraphics[width=0.672in]{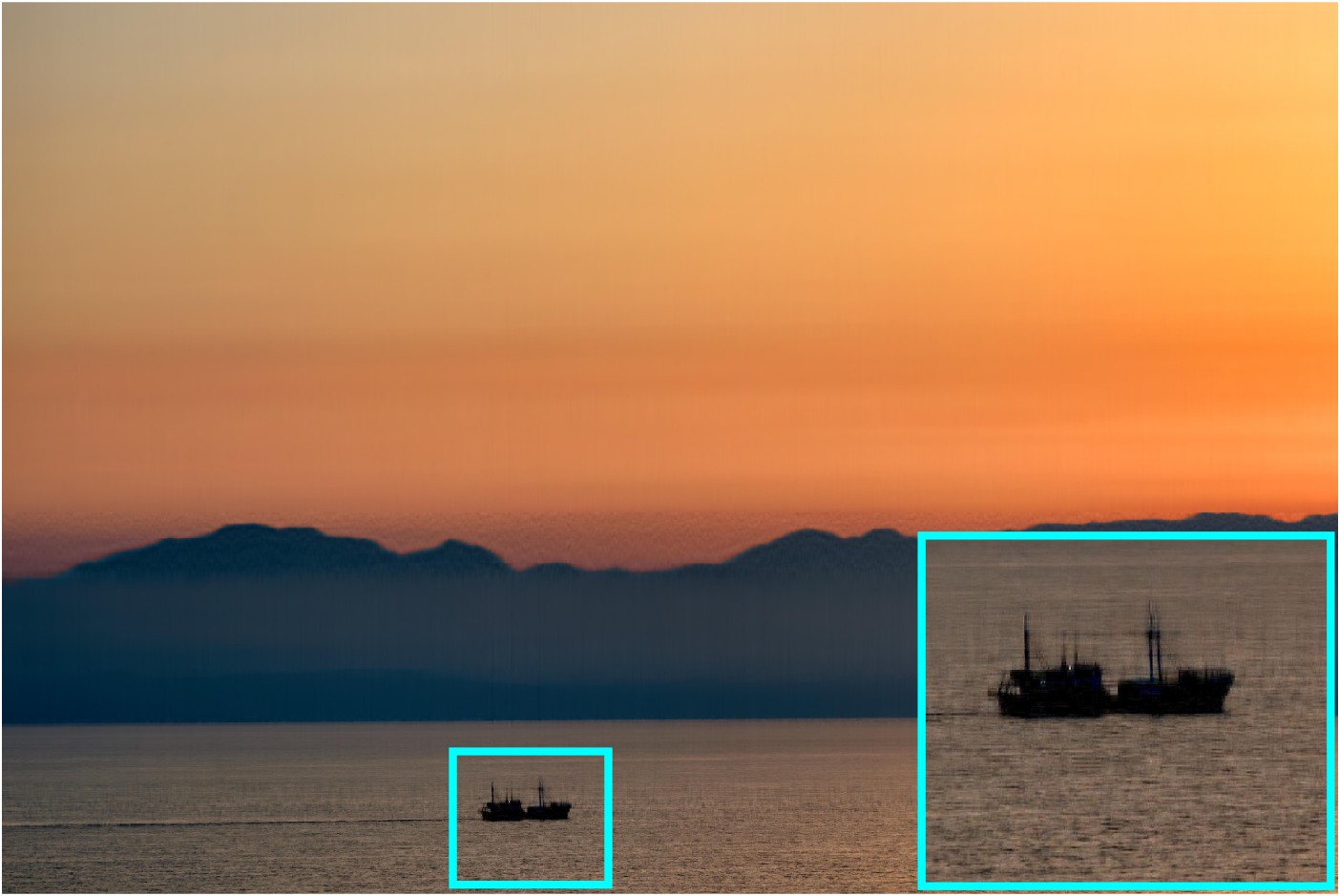} &
\includegraphics[width=0.672in]{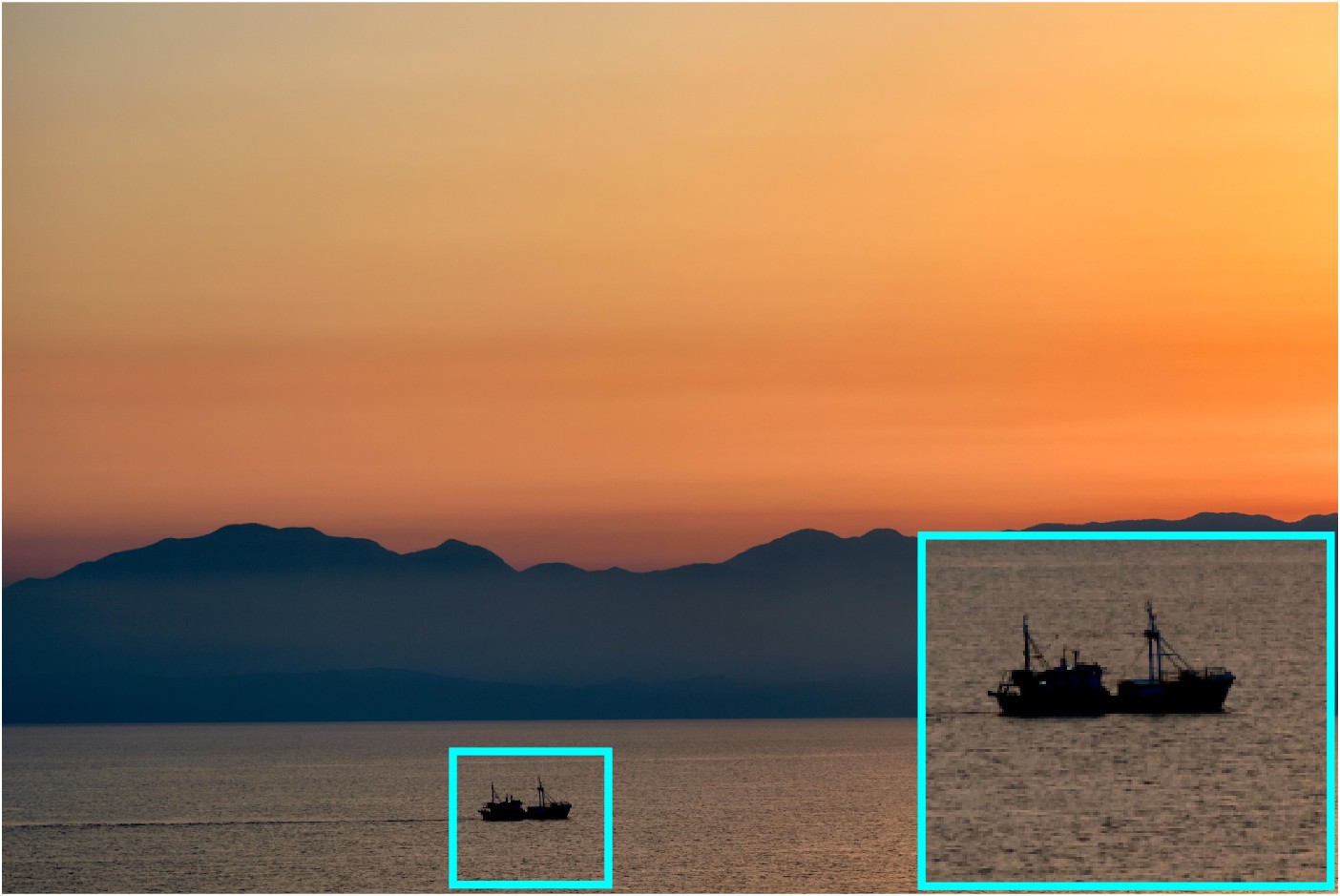} &
\includegraphics[width=0.672in]{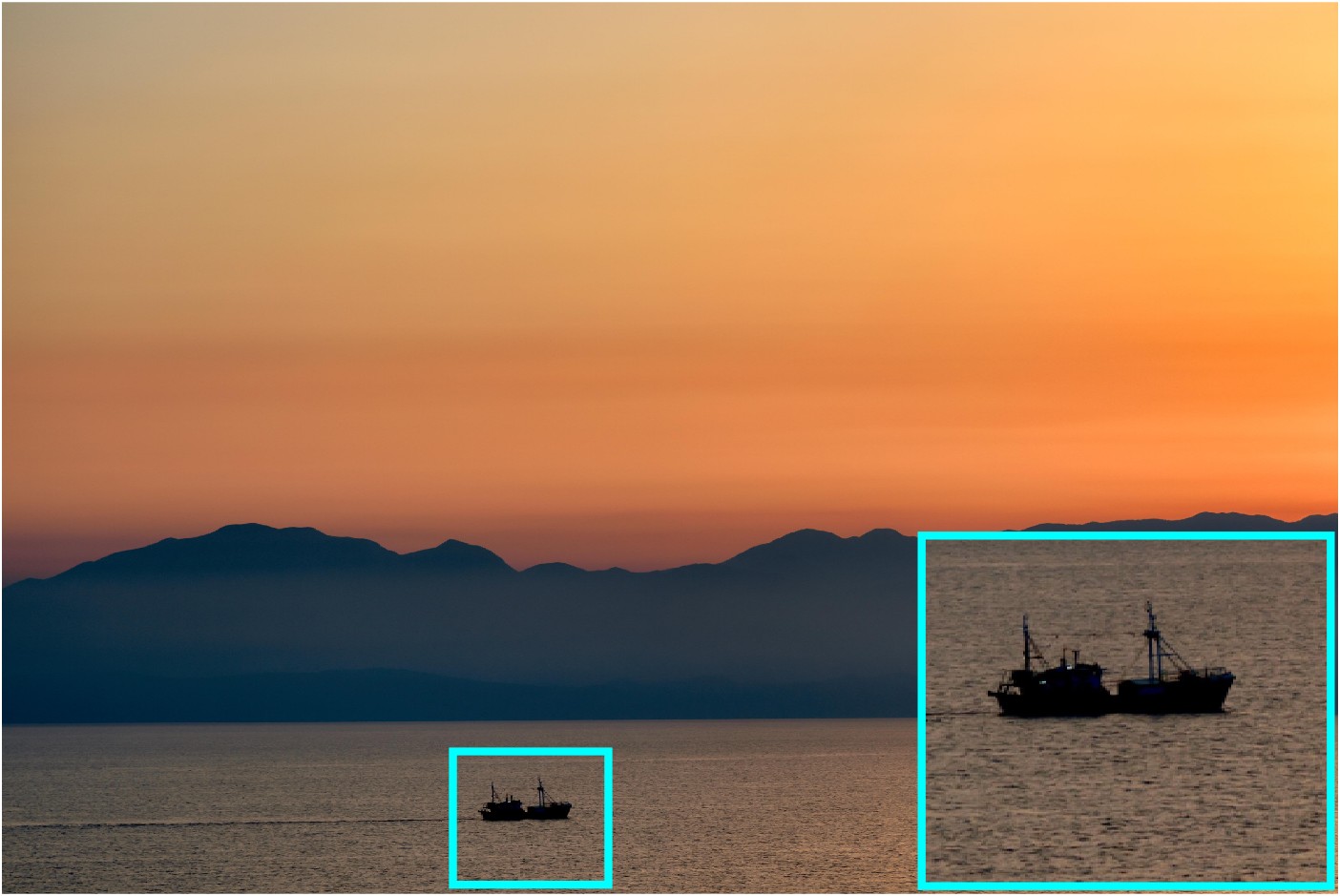} &
\includegraphics[width=0.672in]{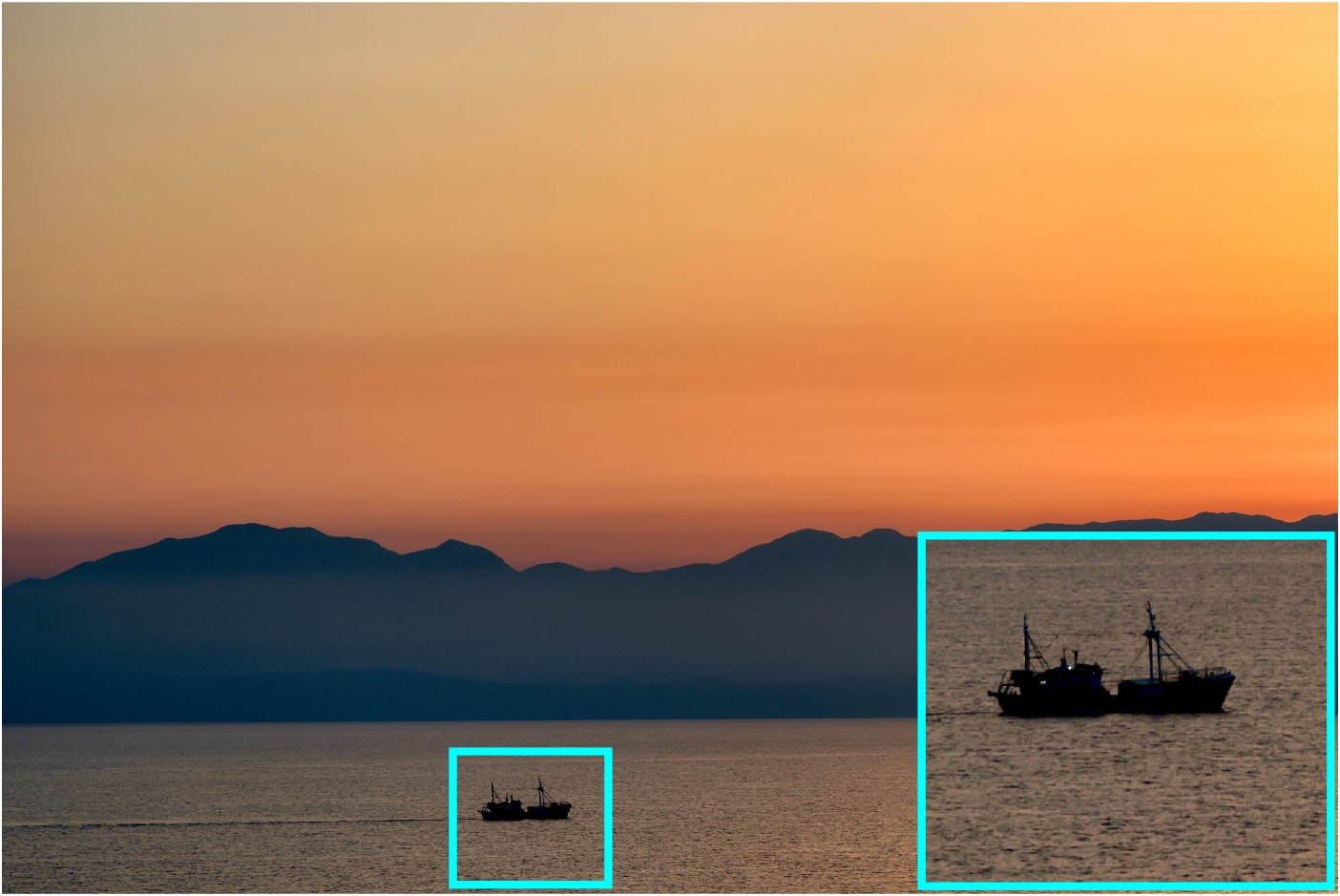}&
\includegraphics[width=0.672in]{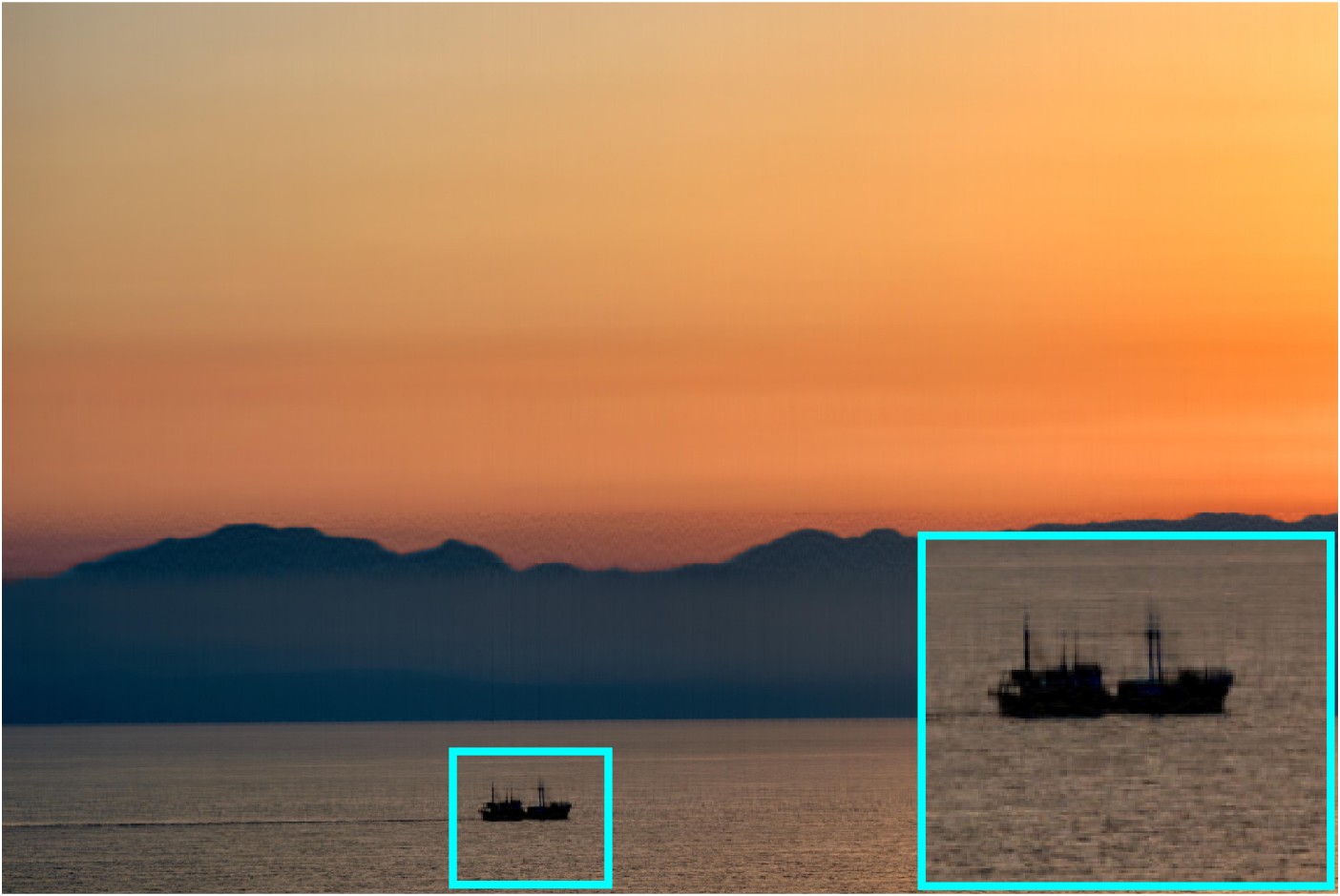} &
\includegraphics[width=0.672in]{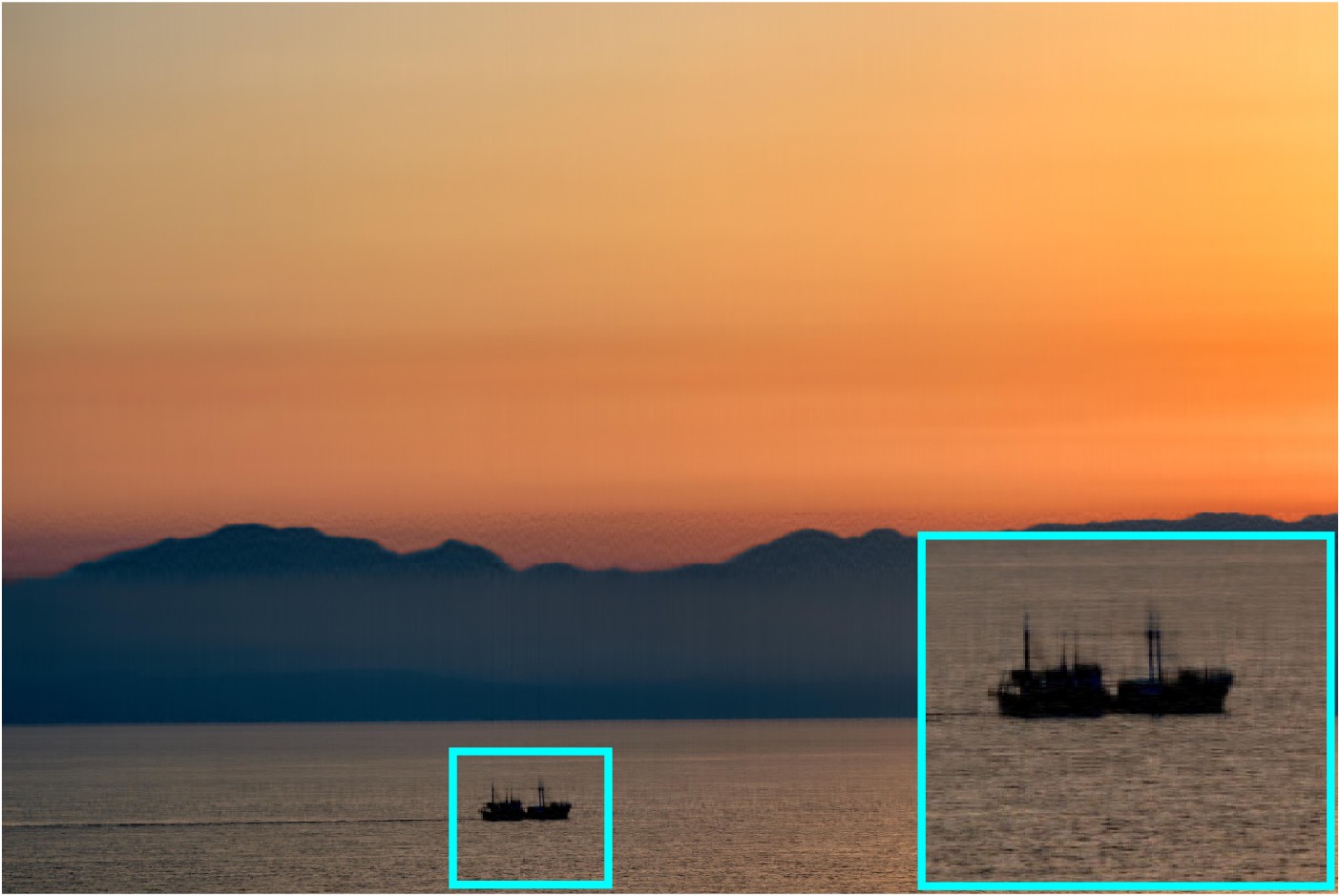} &
\includegraphics[width=0.672in]{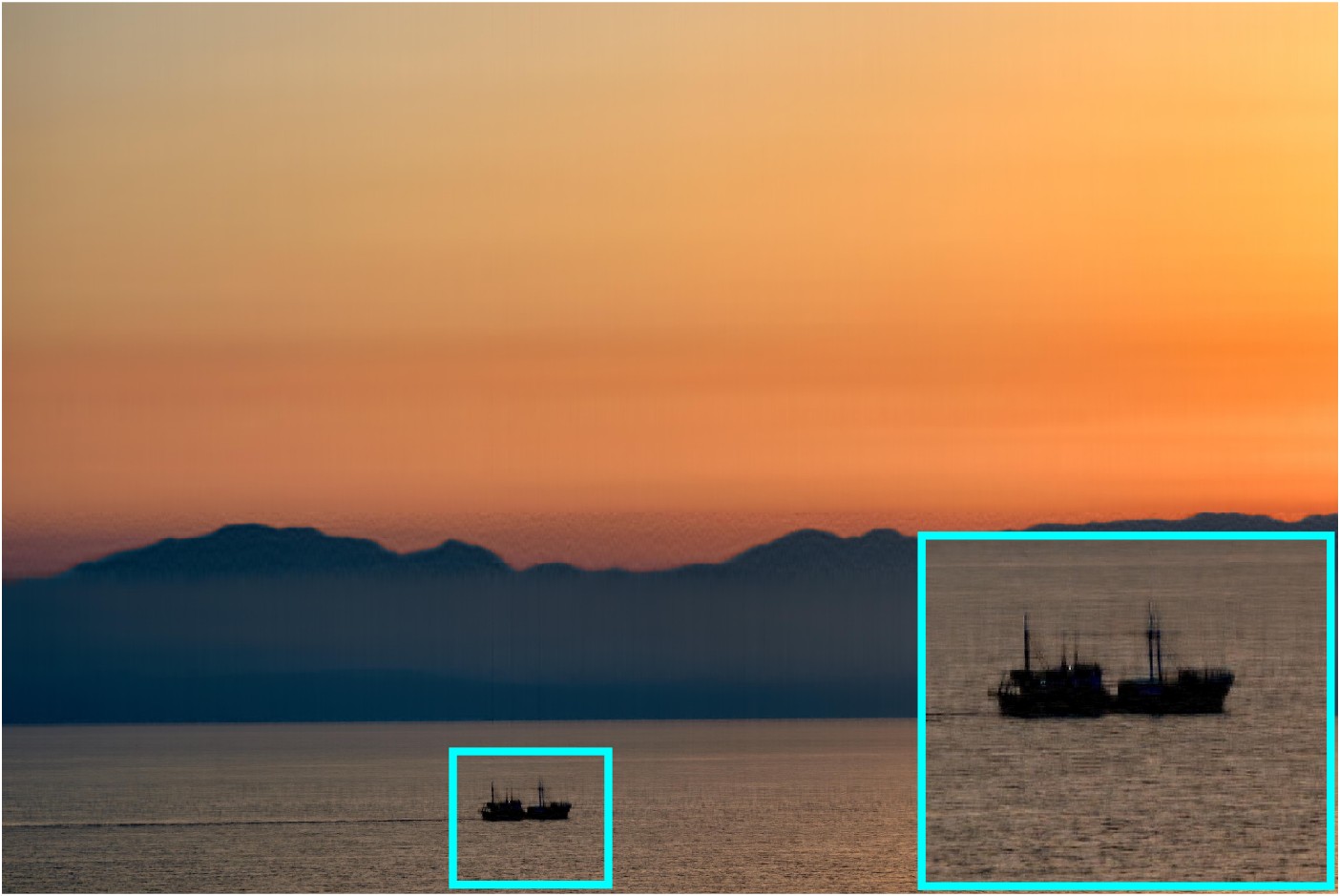} \\

\tiny PSNR:35.02 &\tiny PSNR:36.17 & \tiny PSNR:38.60 & \tiny PSNR:39.62 & \tiny PSNR:34.06 & \tiny PSNR:34.86 & \tiny PSNR:35.02\\
\tiny Time:17.45s & \tiny Time:10.70s & \tiny Time:14.26s & \tiny Time:15.95s & \tiny Time:9.33s & \tiny Time:10.03s& \tiny Time:10.21s\\

\tiny Original &\tiny TT-SVD & \tiny STP-SVD & \tiny TSTP-SVD &\tiny\makecell[c]{MSTP-SVD\\[-4pt](k=2)} & \tiny\makecell[c]{MSTP-SVD\\[-4pt](k=3)} &\tiny\makecell[c]{TMSTP-SVD\\[-4pt](k=2)} \\
\includegraphics[width=0.672in]{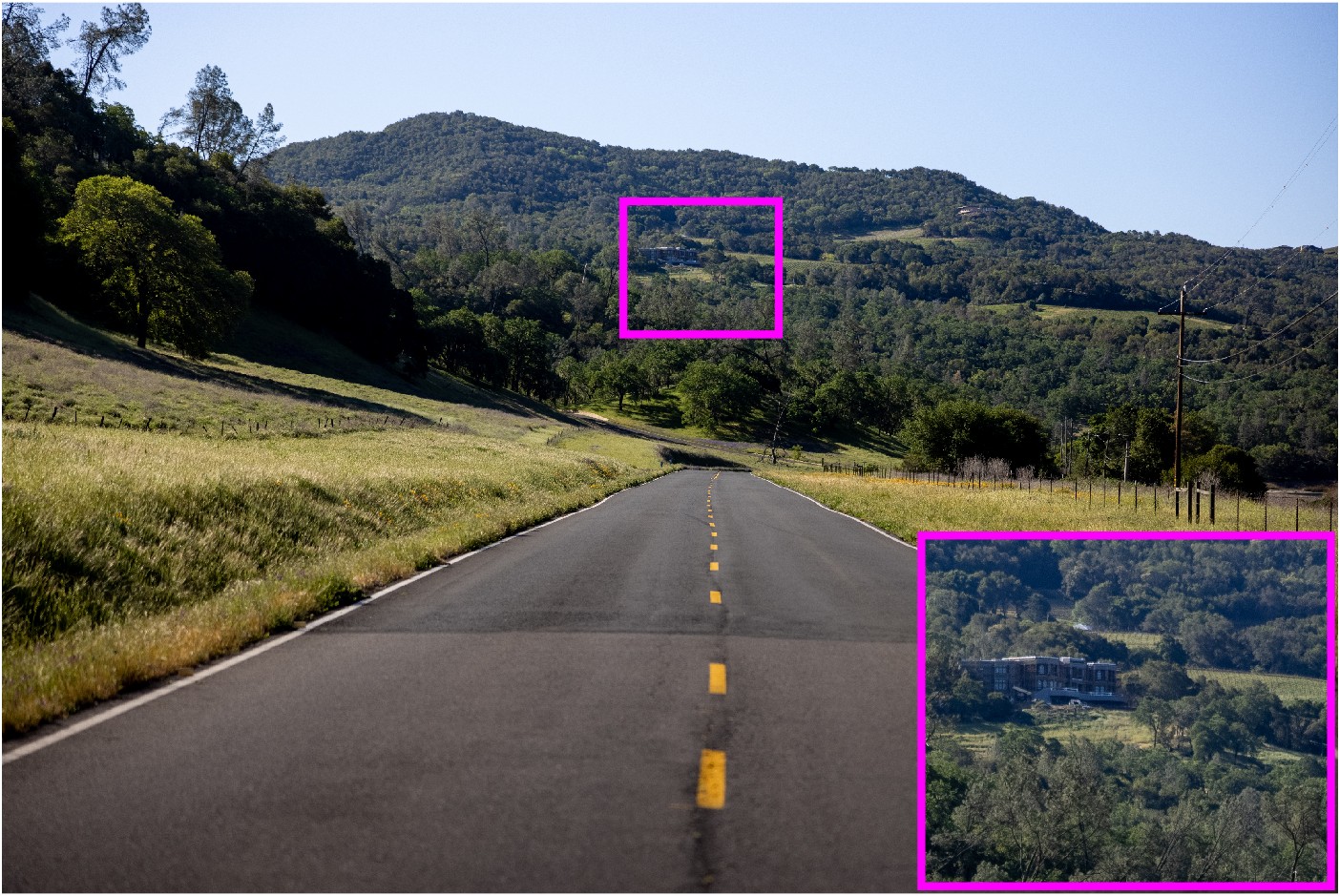} &
\includegraphics[width=0.672in]{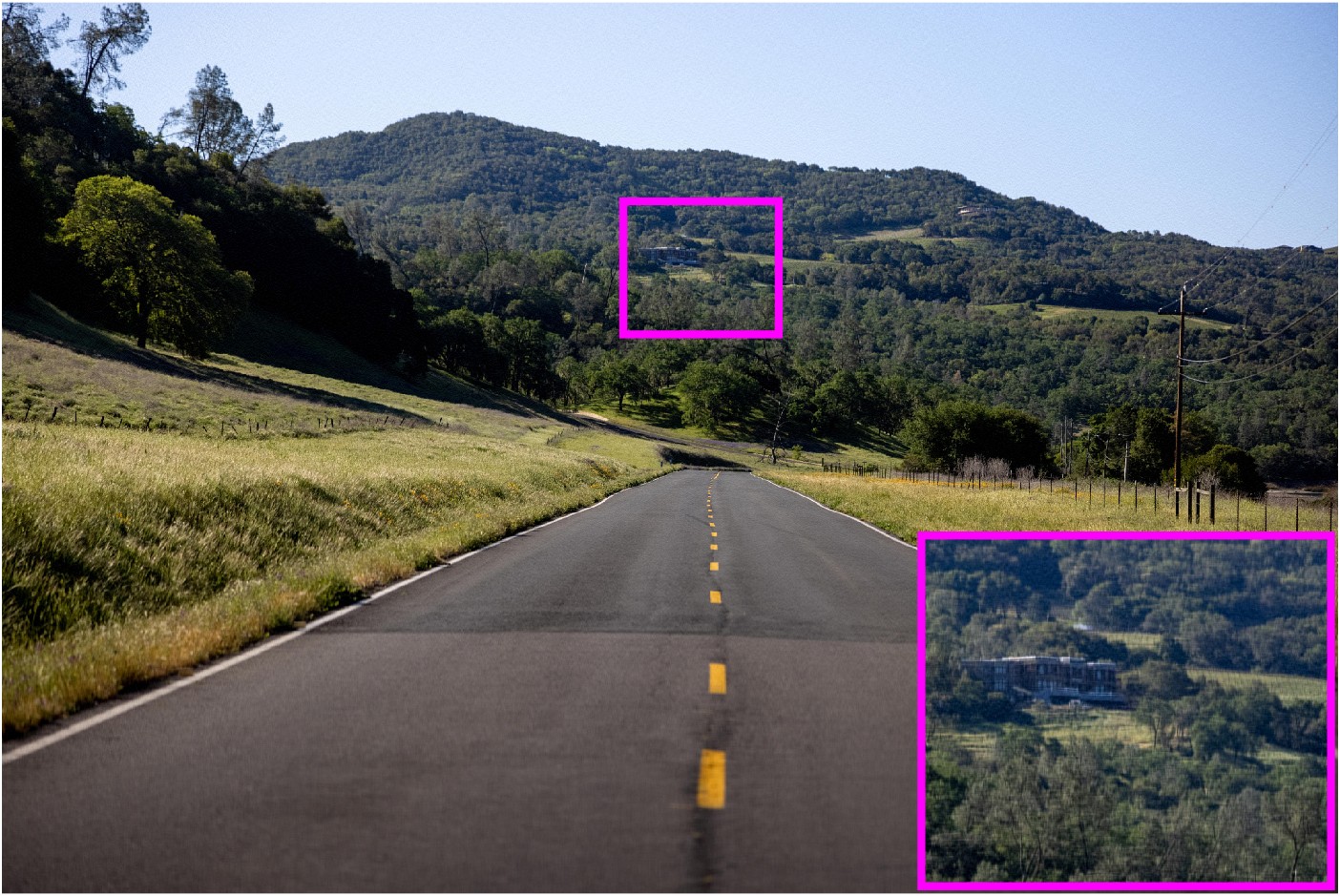} &
\includegraphics[width=0.672in]{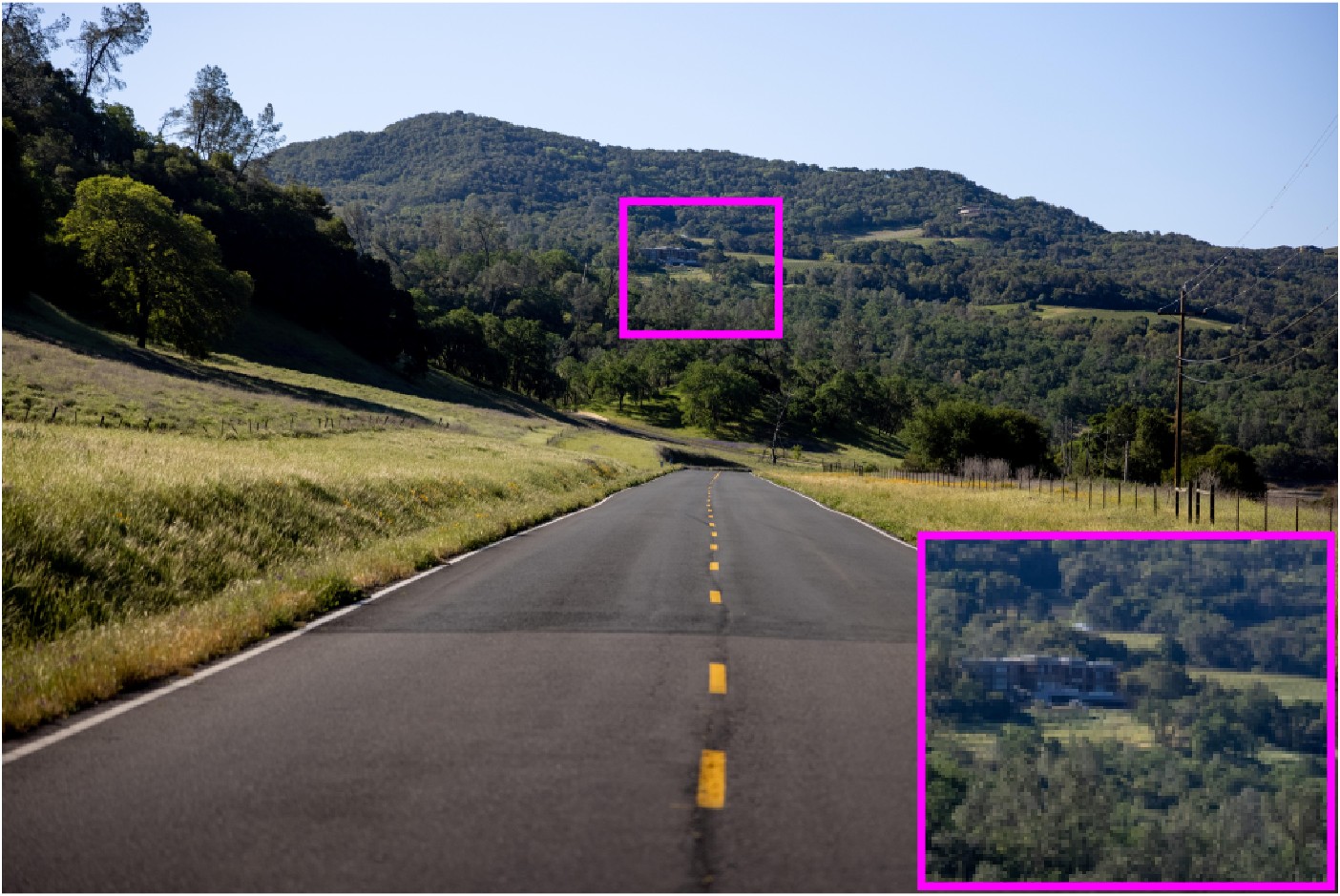} &
\includegraphics[width=0.672in]{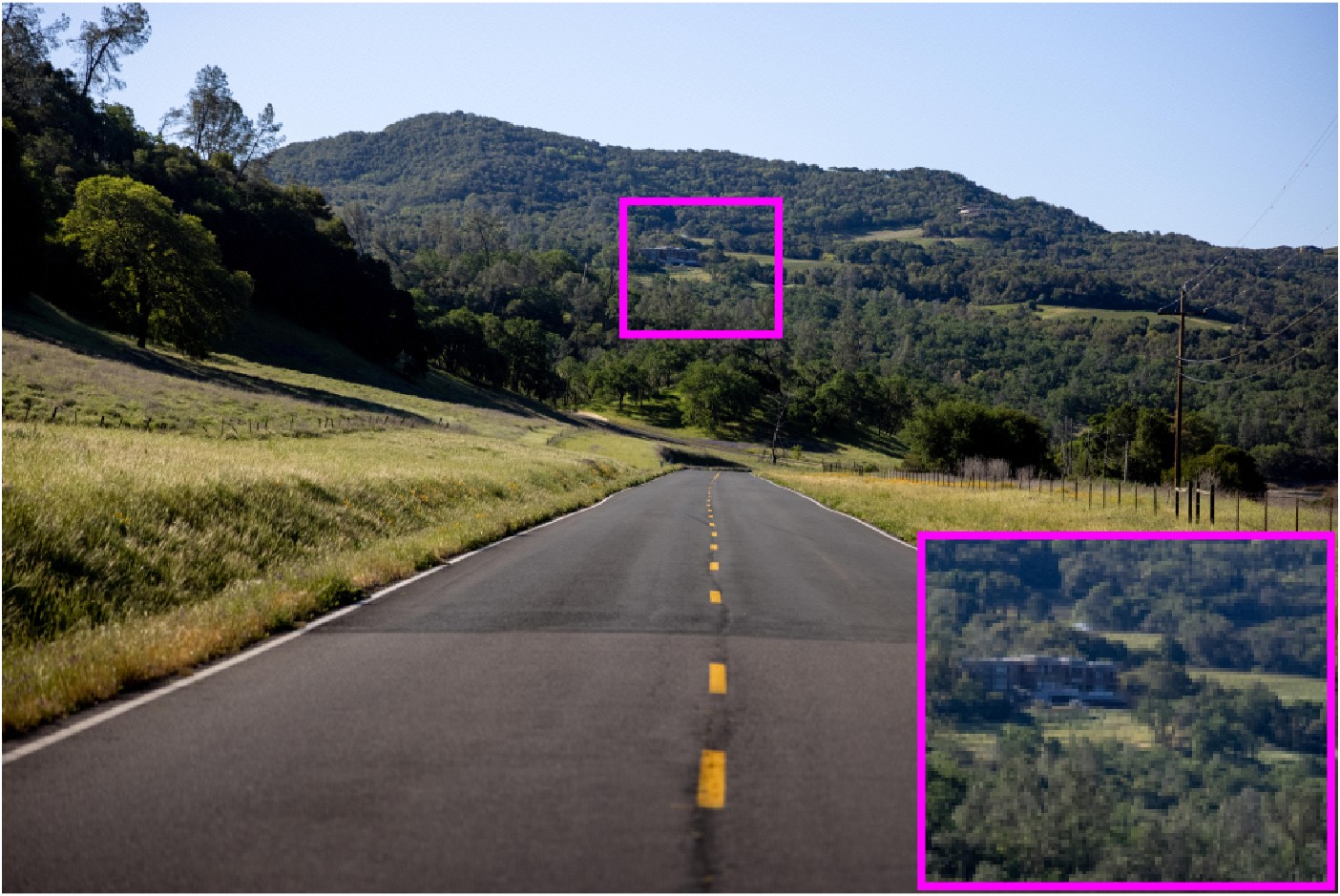} &
\includegraphics[width=0.672in]{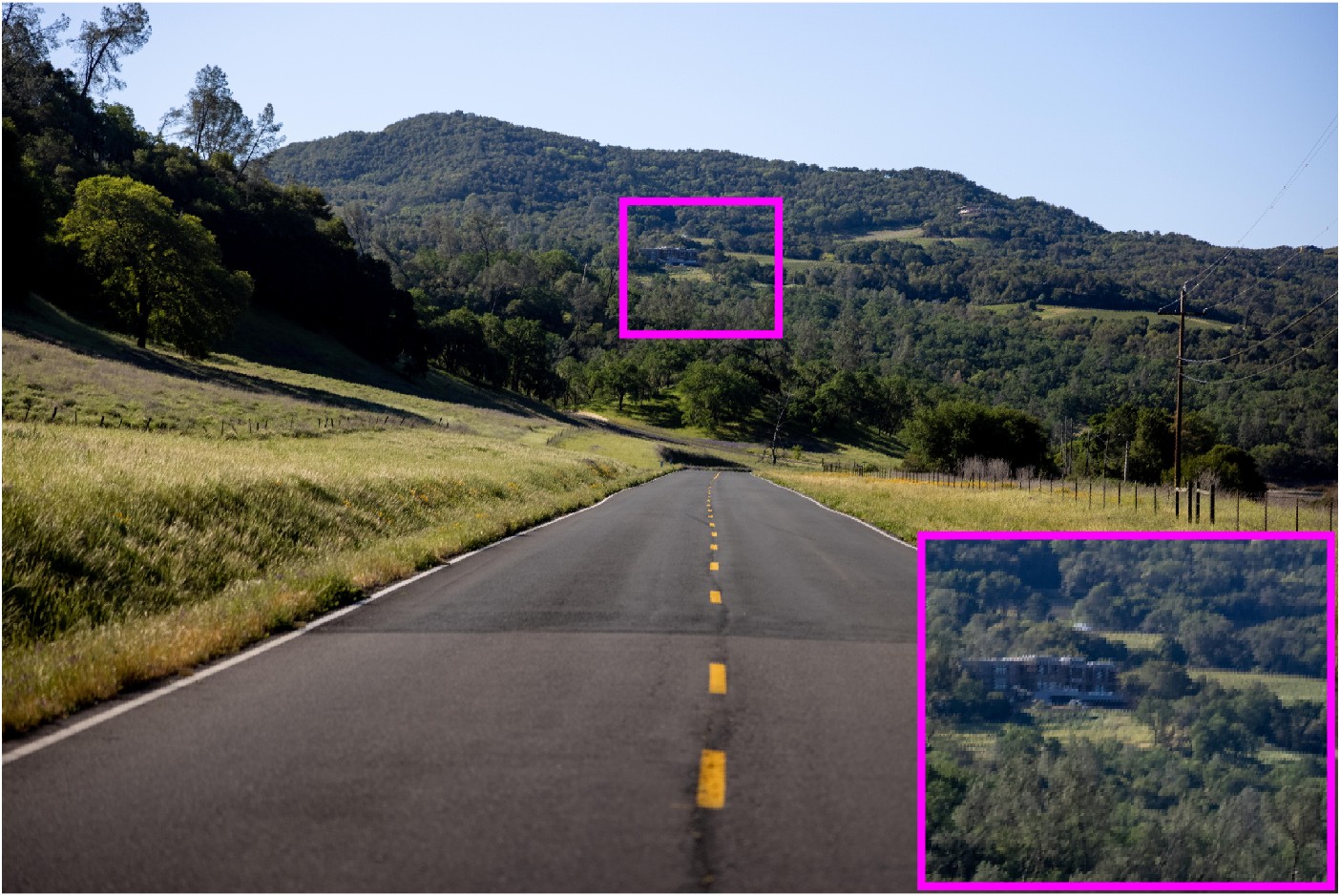} &
\includegraphics[width=0.672in]{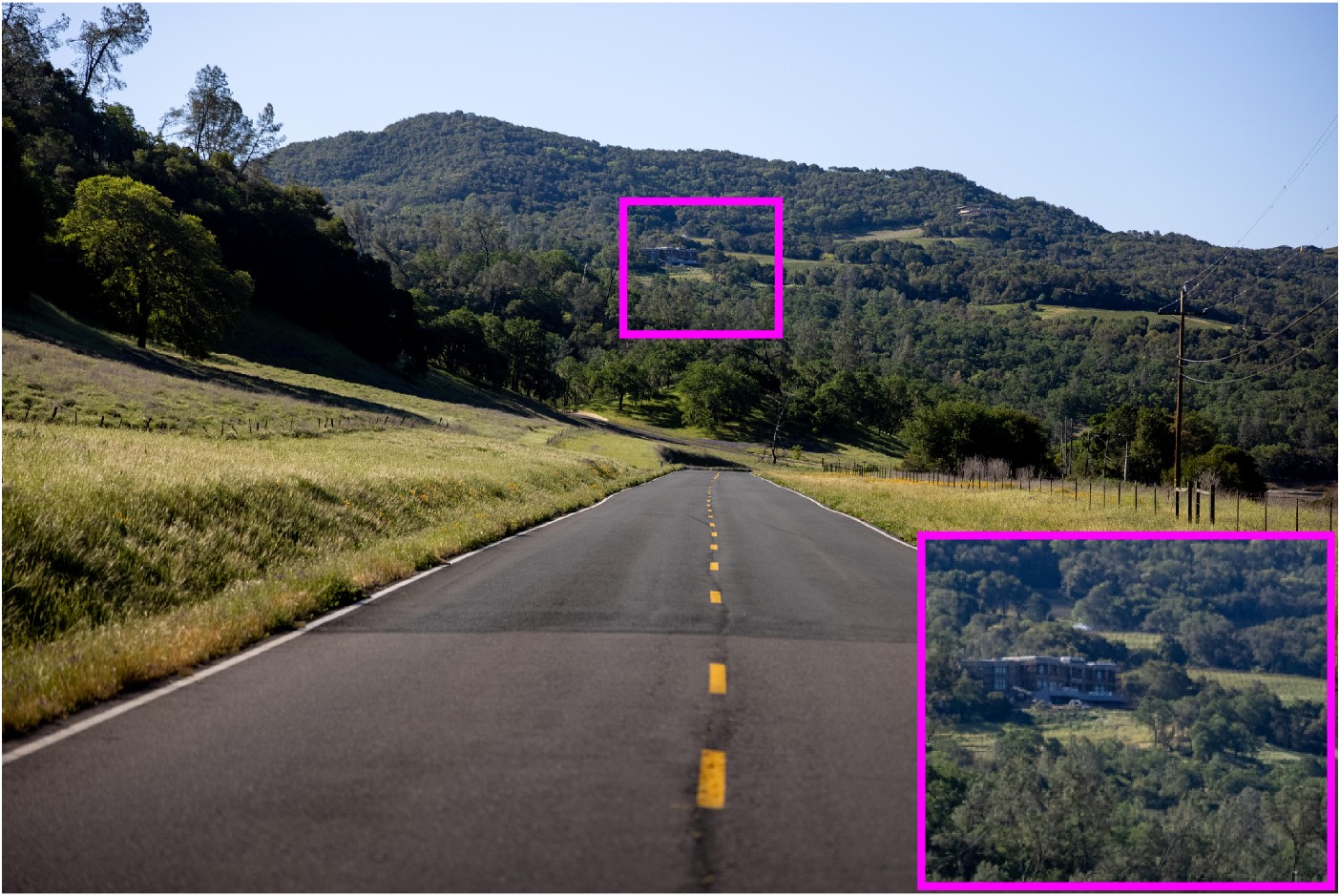} &
\includegraphics[width=0.672in]{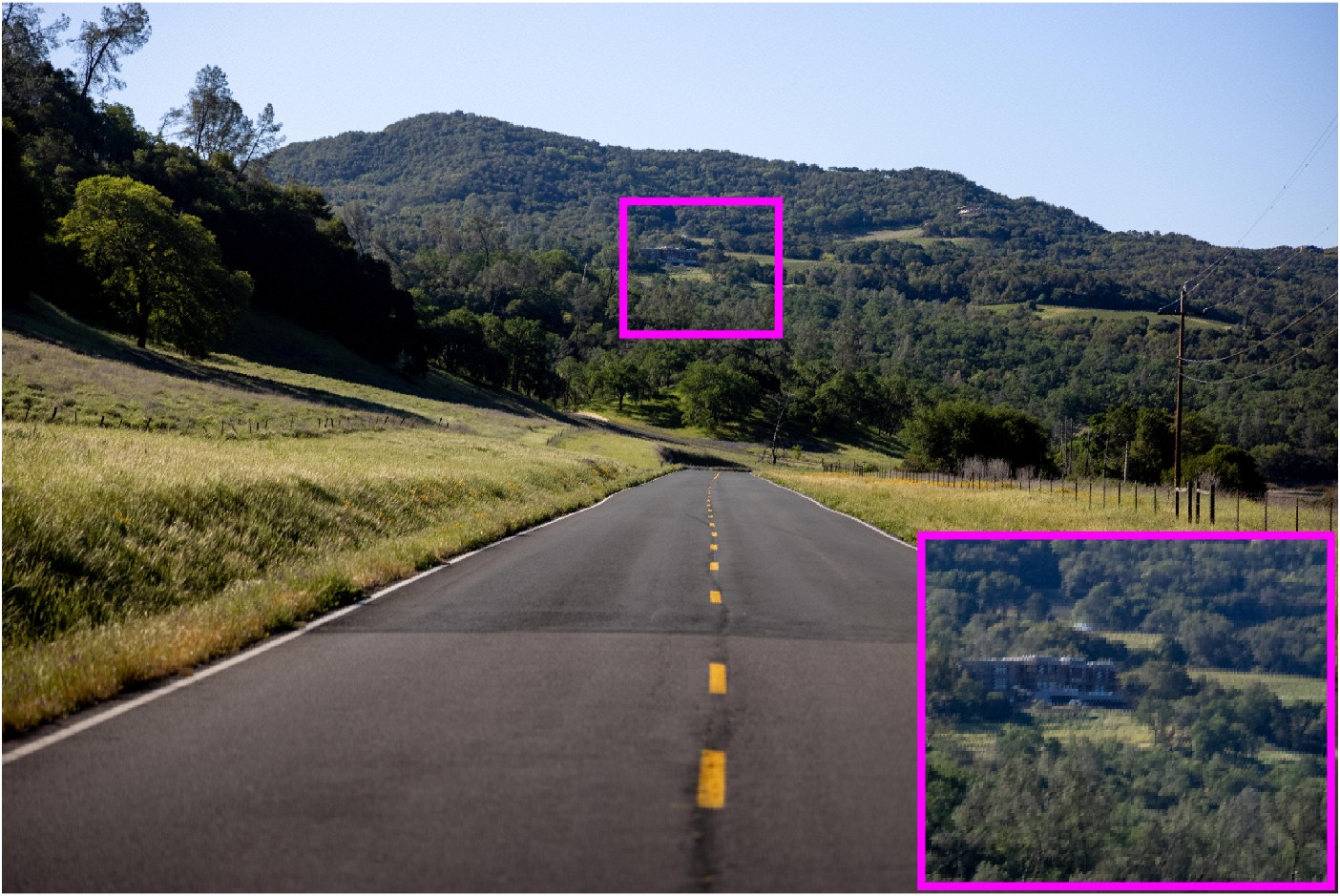} \\
&\tiny PSNR:29.08 & \tiny PSNR:27.40 & \tiny PSNR:26.98 &\tiny PSNR:28.53& \tiny PSNR:29.90 & \tiny PSNR:27.83\\
&\tiny Time:129.61s & \tiny Time:11.93s& \tiny Time:9.71s& \tiny Time:13.43s & \tiny Time:16.15s & \tiny Time:10.54s\\
\tiny\makecell[c]{TMSTP-SVD\\[-4pt](k=3)} &\tiny\makecell[c]{MRSTP-SVD\\[-4pt](k=1)} & \tiny\makecell[c]{MRSTP-SVD\\[-4pt](k=2)} & \tiny\makecell[c]{MRSTP-SVD\\[-4pt](k=3)}& \tiny\makecell[c]{TMRSTP-SVD\\[-4pt](k=1)} & \tiny\makecell[c]{TMRSTP-SVD\\[-4pt](k=2)} & \tiny\makecell[c]{TMRSTP-SVD\\[-4pt](k=3)}\\
\includegraphics[width=0.672in]{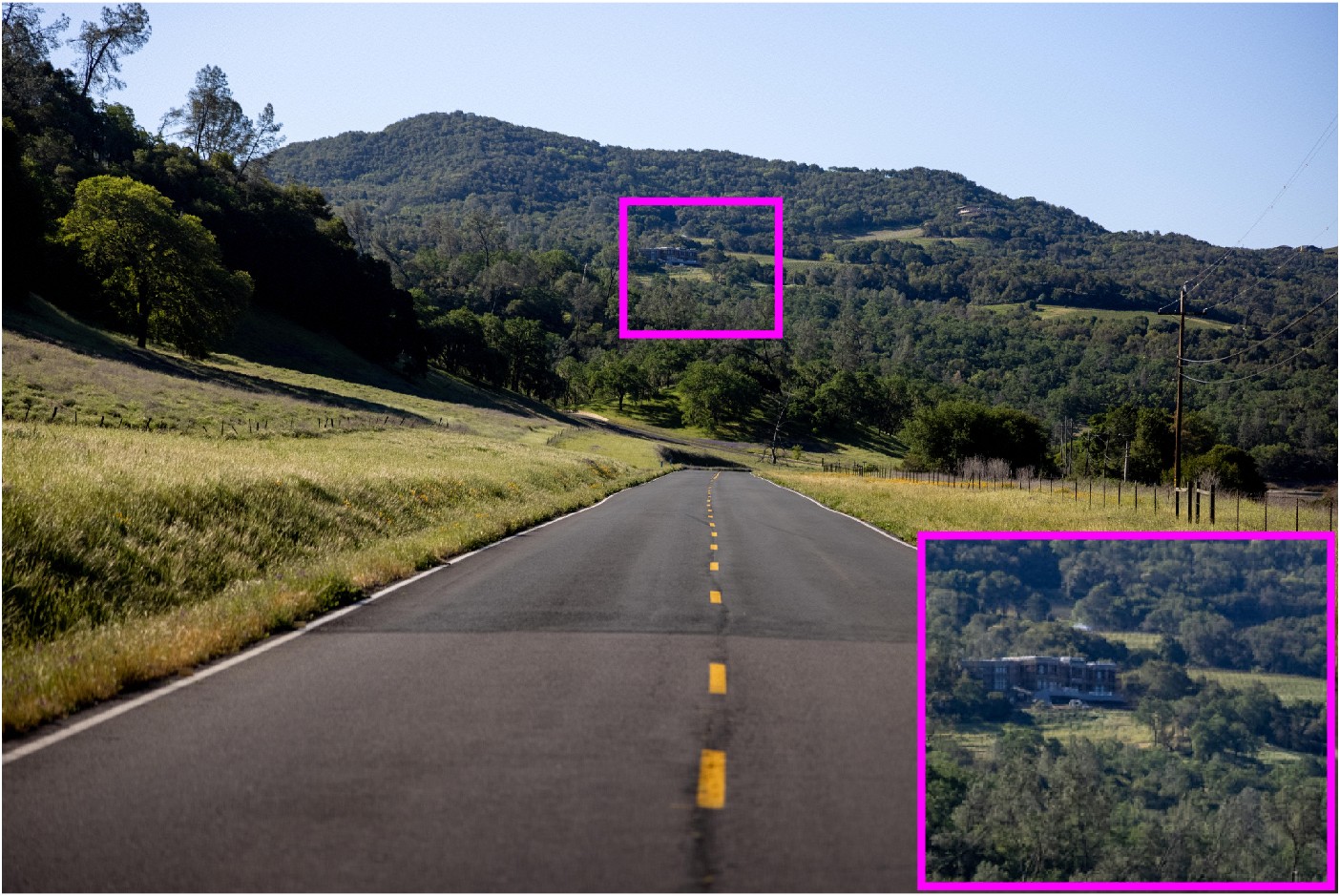} &
\includegraphics[width=0.672in]{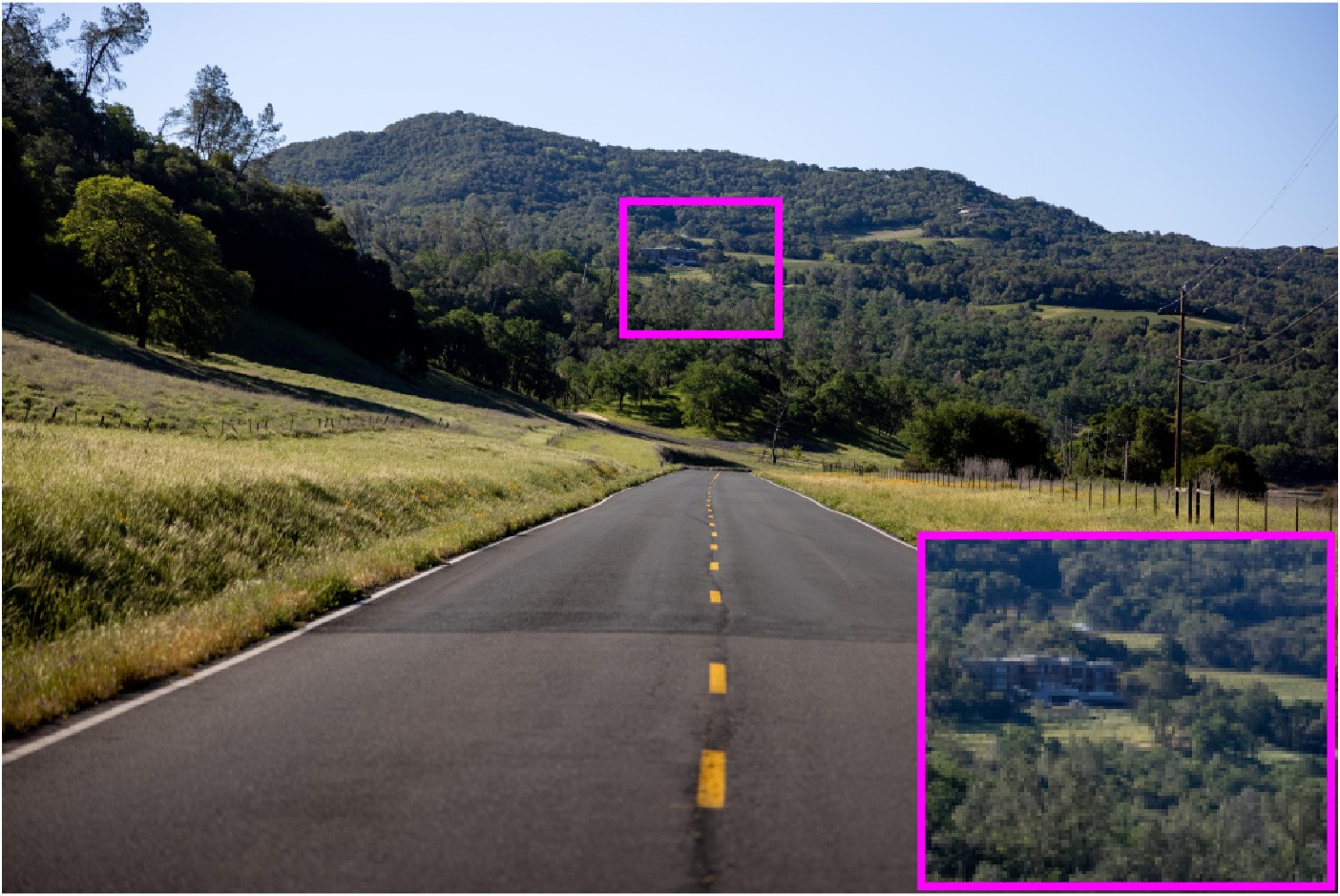} &
\includegraphics[width=0.672in]{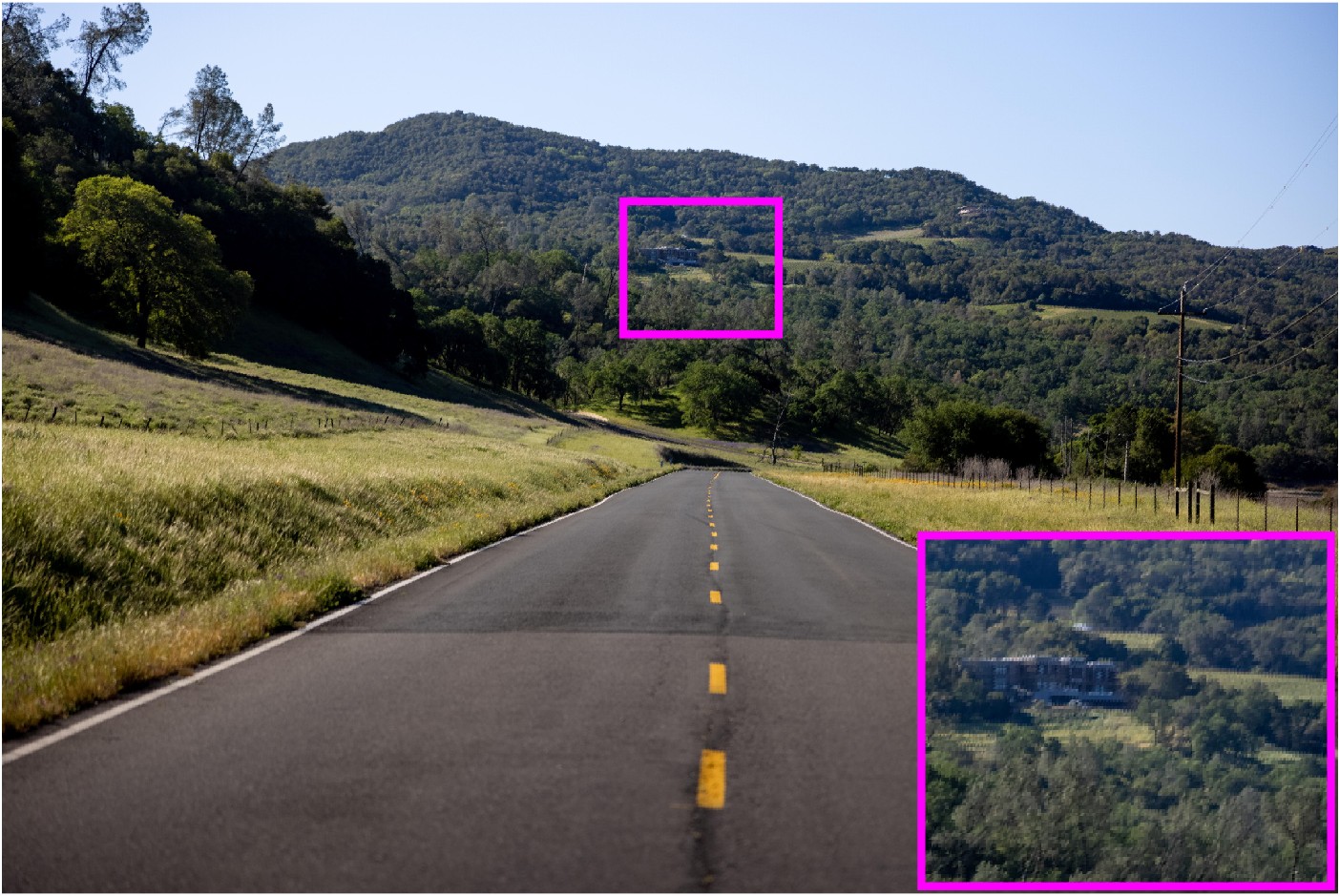} &
\includegraphics[width=0.672in]{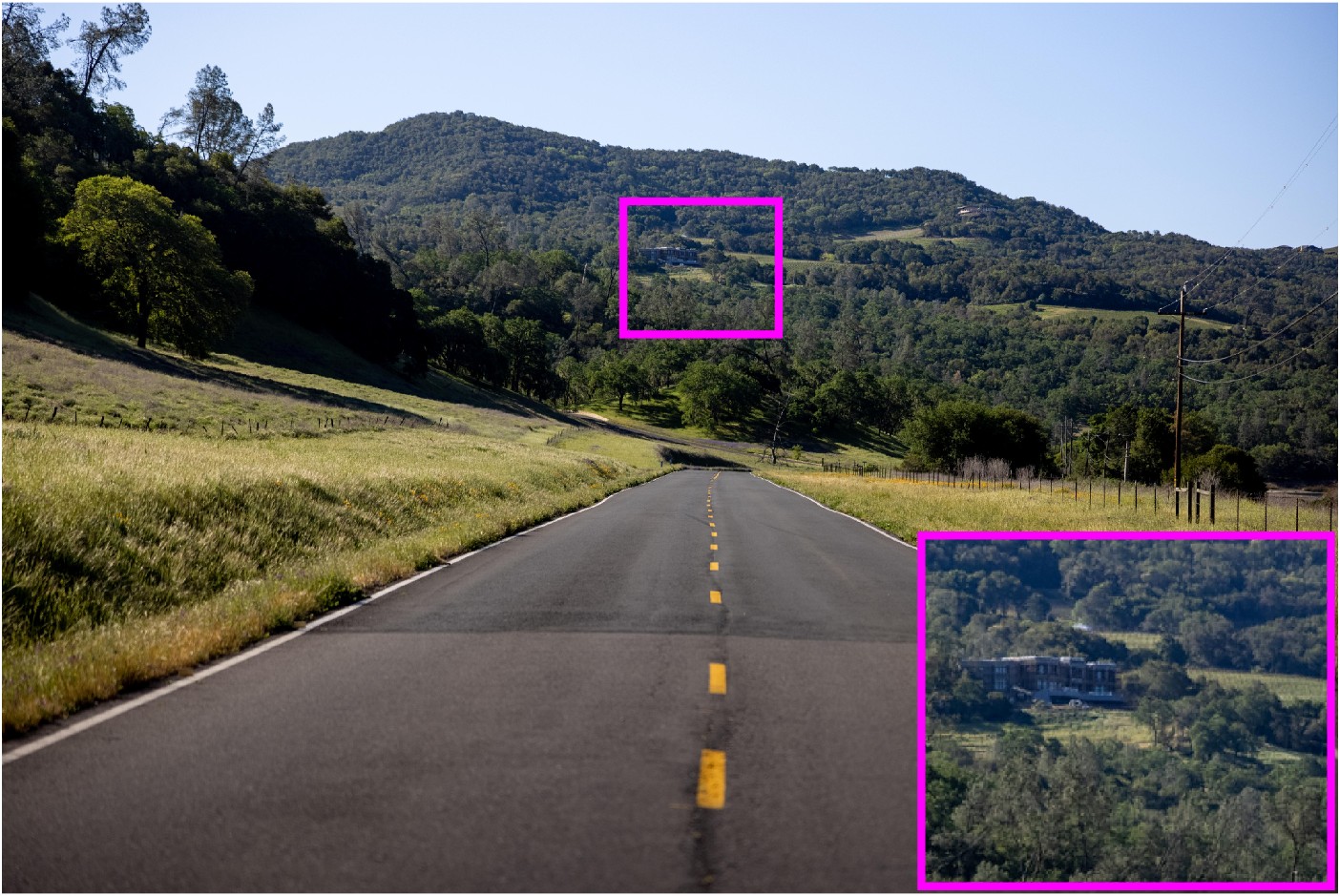}&
\includegraphics[width=0.672in]{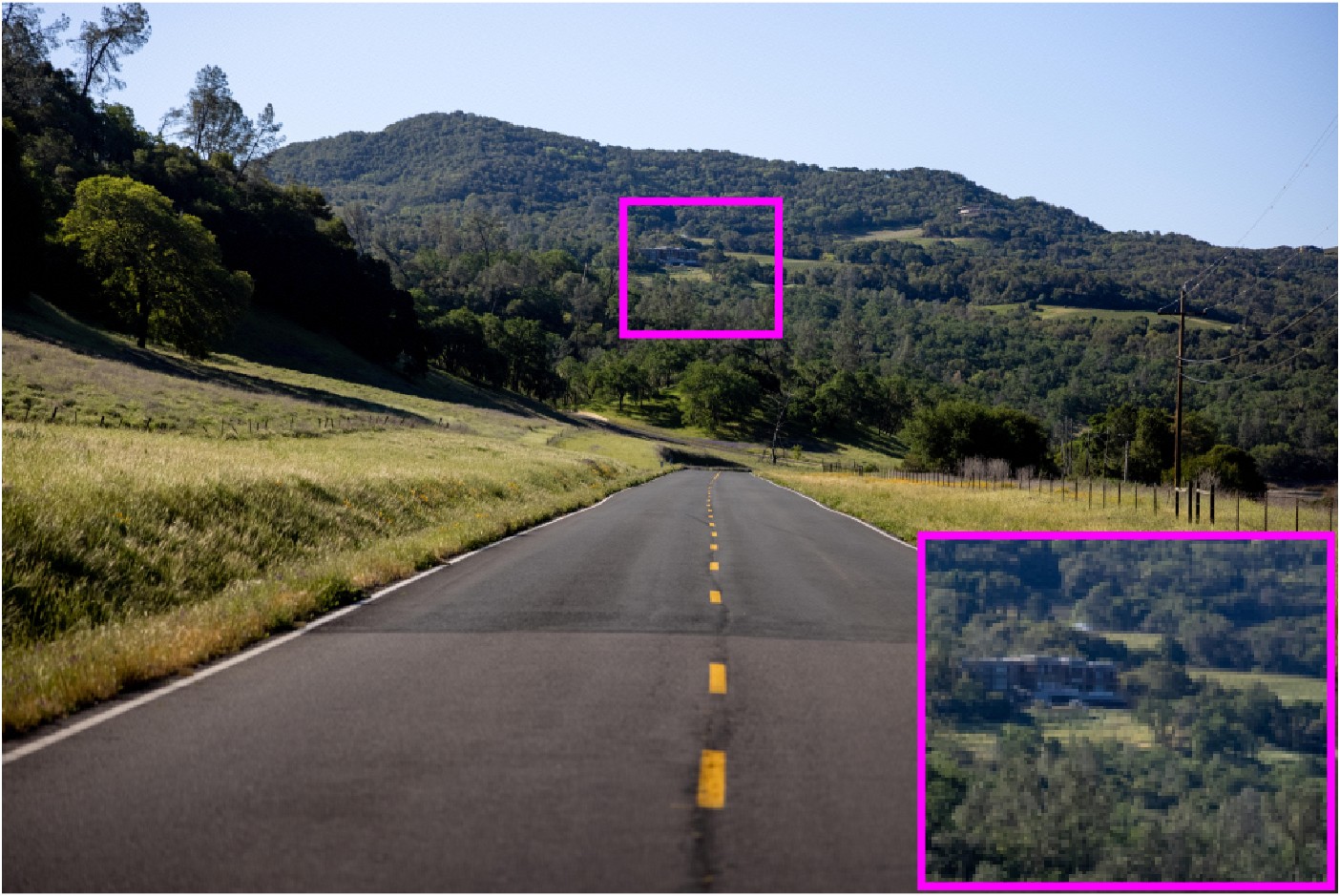} &
\includegraphics[width=0.672in]{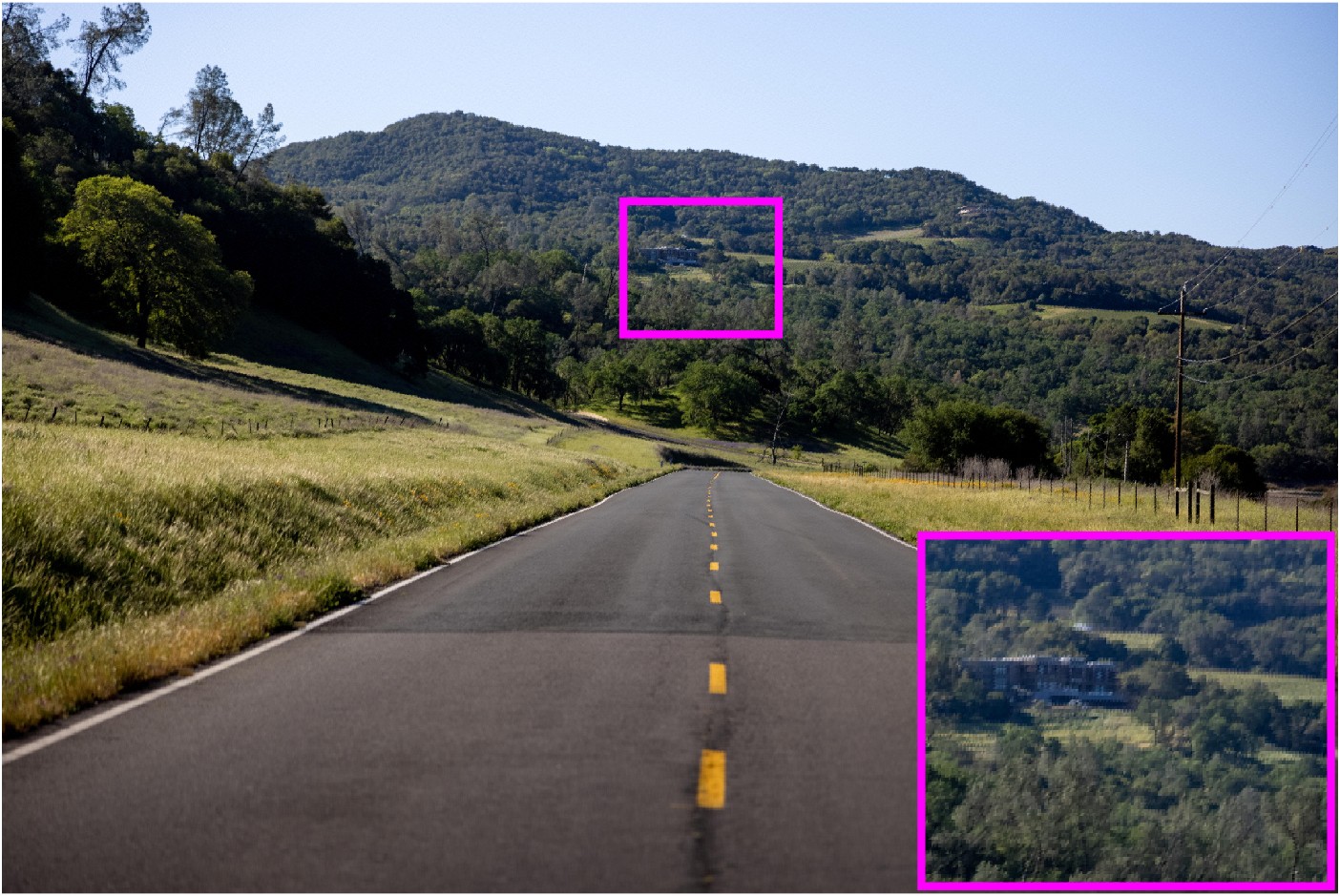} &
\includegraphics[width=0.672in]{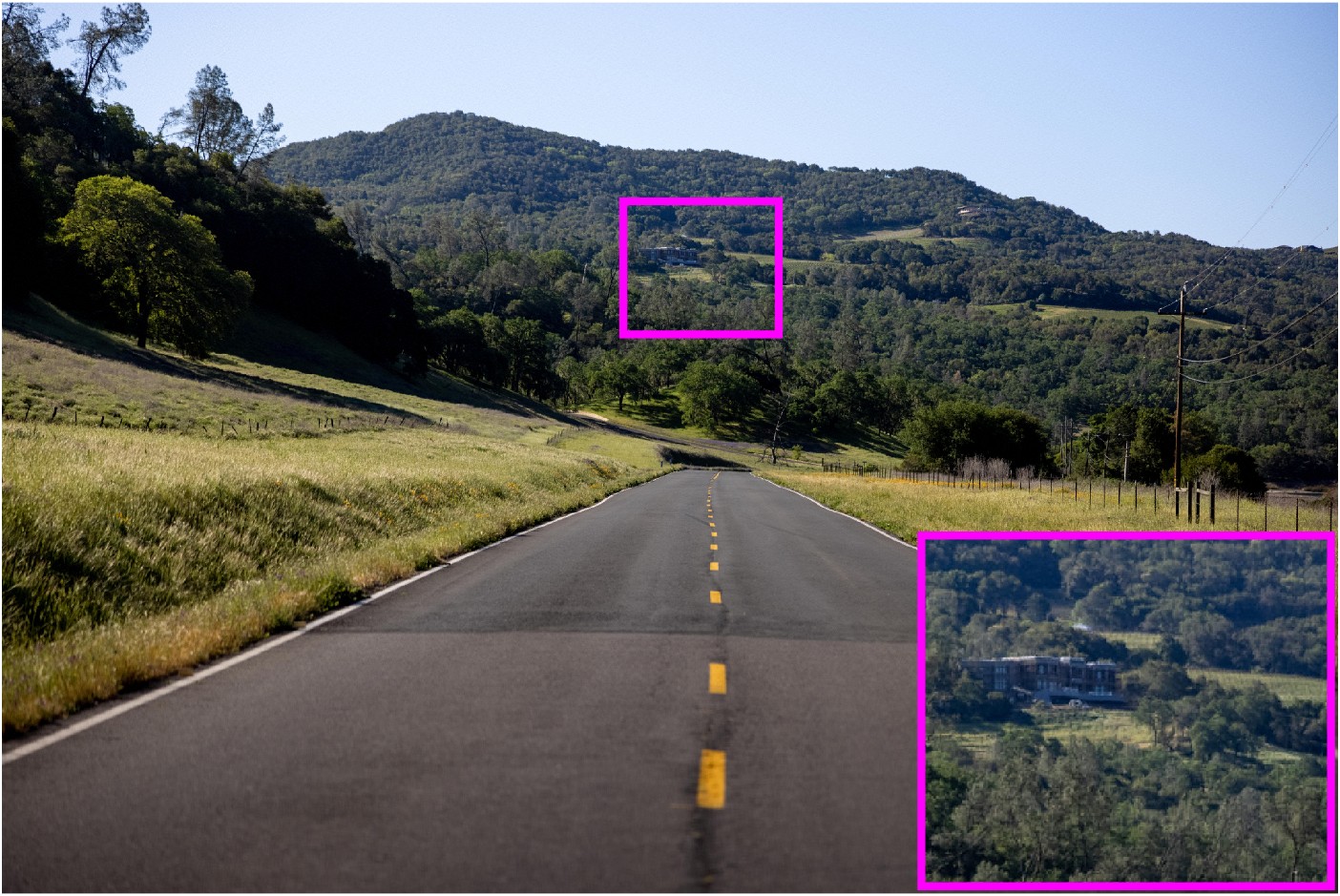}  \\
\tiny PSNR:28.78& \tiny PSNR:27.39 & \tiny PSNR:28.52 & \tiny PSNR:29.88 & \tiny PSNR:26.98 & \tiny PSNR:27.83 & \tiny PSNR:28.76\\
\tiny Time:12.74s & \tiny Time:6.99s & \tiny Time:9.17s & \tiny Time:10.29s & \tiny Time:5.47s & \tiny Time:6.03s & \tiny Time:7.06s\\

\end{tabular}
\caption{Visual reconstruction performance, PSNR and runtime of competing and proposed methods on two representative test images.}
\label{fig:image_compression}
\end{figure}

Visual outputs together with per-image PSNR and runtime measurements for Lake and Road are presented in Fig.~\ref{fig:image_compression}, whereas the corresponding results for Night and Fruit are provided in the Supplementary Material.
As shown in Fig.~\ref{fig:image_compression}, our proposed methods can better preserve local structural details of original images. By contrast, baseline methods tend to produce over-smoothed outputs and lose subtle image contents.
For the four test images, STP block partition sizes $(m_2, n_2)$ are adaptively assigned according to spatial resolution to balance approximation flexibility and computational efficiency. Specifically, $(m_2, n_2)=(4,4)$ for Lake and Night, $(4,6)$ for Fruit, and $(8,8)$ for the higher-resolution Road image.
Frontal-slice truncation ranks of TT-SVD are 50, 100, 100 and 250 for Lake, Night, Fruit and Road, respectively. Within MSTP-SVD and MRSTP-SVD, the truncation rank matrix is $\mathbf{R}=c\mathbf{J}_{k\times l}$, where $c$ denotes a uniform rank shared over all terms and frontal slices, and $\mathbf{J}_{k\times l}$ is the $k\times l$ all-ones matrix. We set $c$ equal to the TT-SVD truncation rank for each image to guarantee fair comparisons under identical rank budgets.
For the randomized MRSTP-SVD variant, we use oversampling $s=5$ and power iterations $q=1$, achieving a desirable trade-off between reconstruction accuracy and computational overhead.

We further conduct quantitative image compression experiments on 20 RGB test images. Fig.~\ref{fig:image20_compare} presents per-image PSNR, SSIM and runtime, and Table~\ref{tab:image compression} summarizes the averaged performance. As shown in Table~\ref{tab:image compression}, deterministic MSTP-SVD (\(k=3\)) achieves 33.67 dB average PSNR and 0.959 average SSIM, outperforming single-term STP-SVD by 4.96 dB and baseline TT-SVD by 8.07 dB in PSNR, validating the superior representation capability of multi-term semi-tensor decomposition. The randomized MRSTP-SVD (\(k=3\)) delivers nearly identical reconstruction quality (33.64 dB PSNR) while cutting average runtime from 7.57 s to 5.31 s, achieving a favorable accuracy-efficiency trade-off.
\begin{figure}[!ht]
  \centering
\includegraphics[width=\textwidth]{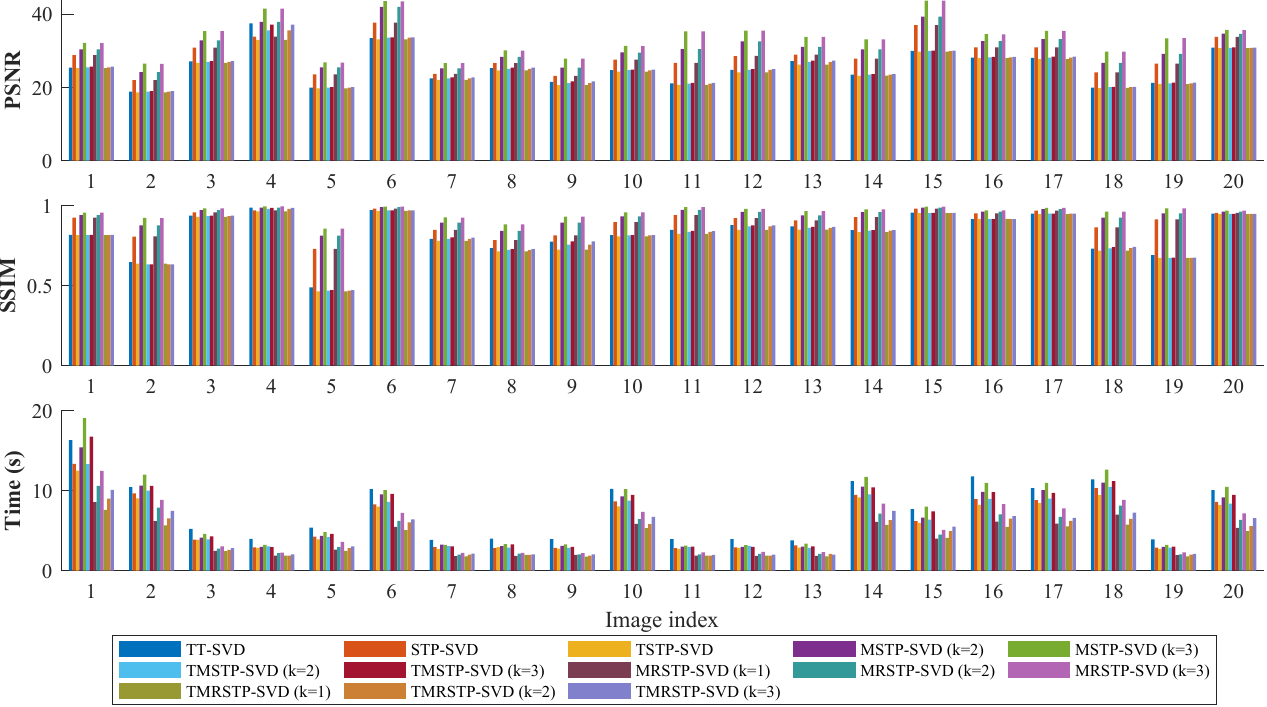}
    \caption{Per-image quantitative comparisons of PSNR, SSIM and runtime over twenty test images.}
\label{fig:image20_compare}  
\end{figure}

\begin{table}[!htbp]
\centering
\caption{Average PSNR, SSIM and runtime comparison of competing and proposed algorithms over twenty test images.}
\label{tab:image compression}
\footnotesize
\setlength{\tabcolsep}{1.0pt}
\renewcommand{\arraystretch}{0.8}
\begin{tabular*}{\linewidth}{l @{\extracolsep{\fill}} c c c *{2}{c} *{2}{c} *{3}{c} *{3}{c}}
\hline
 \multirow{2}{*}{Metric}
& \multirow{2}{*}{\scriptsize TT-SVD}
& \multirow{2}{*}{\scriptsize STP-SVD}
& \multirow{2}{*}{\scriptsize TSTP-SVD}
& \multicolumn{2}{c}{\shortstack{\scriptsize MSTP-SVD}}
& \multicolumn{2}{c}{\shortstack{\scriptsize TMSTP-SVD}}
& \multicolumn{3}{c}{\shortstack{\scriptsize MRSTP-SVD}}
& \multicolumn{3}{c}{\shortstack{\scriptsize TMRSTP-SVD}} \\
\cline{5-6} \cline{7-8} \cline{9-11} \cline{12-14}
& & &
& \scriptsize$k=2$ & \scriptsize$k=3$
& \scriptsize$k=2$ & \scriptsize$k=3$
& \scriptsize$k=1$ & \scriptsize$k=2$ & \scriptsize$k=3$
& \scriptsize$k=1$ & \scriptsize$k=2$ & \scriptsize$k=3$ \\
\hline
PSNR 
& 25.60 & 28.71 & 25.01
& 31.13 & \textbf{33.67}
& 25.47 & 25.74
& 28.71 & 31.13 & 33.64
& 25.01 & 25.47 & 25.74 \\
SSIM
& 0.830 & 0.902 & 0.816
& 0.937 & \textbf{0.959}
& 0.824 & 0.828
& 0.902 & 0.937 & 0.959
& 0.816 & 0.824 & 0.828 \\
Time (s)
& 7.58 & 6.18 & 5.86
& 6.74 & 7.57
& 6.25 & 6.88
& 4.03 & 4.66 & 5.31
& \textbf{3.74} & 4.22 & 4.65 \\
\hline
\end{tabular*}
\end{table}

\subsection{Compression on the video data}\label{video}
For video compression experiments, we validate the proposed algorithm on four representative test sequences from the derf video dataset\footnote{https://media.xiph.org/video/derf/}: Crosswalk, Market, Narrator, and Aerial. DFT is employed for its excellent reconstruction performance.
Owing to computational constraints, we extract the first $40$ frames from each sequence, yielding third-order tensors of size $2160 \times 4096 \times 40$. For the baseline TT-SVD, the frontal-slice truncation rank is fixed at \(r=50\). To ensure a fair comparison, the rank parameter $c$ in our proposed MSTP-SVD and MRSTP-SVD frameworks is set identically, i.e., all entries of the truncation rank matrix $\mathbf{R}$ take the value $50$.
The block partition sizes are  $(m_2,n_2) =(4,4)$ for Market, Narrator, and Aerial, and $(8,8)$ for the more spatially complex Crosswalk sequence. Detailed configurations of the randomized algorithms, including  oversampling parameter $s$ and  power iteration count $q$, are summarized in Table~\ref{tab:algorithm_parameters}.

\begin{table}[!htbp]
\centering
\caption{Hyperparameter configurations (power iteration $q$, oversampling $s$) for MRSTP-SVD and TMRSTP-SVD on video sequences.}
\label{tab:algorithm_parameters}
\setlength{\tabcolsep}{10.5pt}
\renewcommand{\arraystretch}{0.8}
\begin{tabular}{ccc ccc}
\hline
\multirow{2}{*}{Video} & \multirow{2}{*}{Algorithm} & \multirow{2}{*}{$q$} & \multicolumn{3}{c}{$s$} \\
\cline{4-6}
& & & $k=1$ & $k=2$ & $k=3$ \\
\hline
Crosswalk    & (truncated) MRSTP-SVD & 1 & 7 & 6 & 5 \\
Market   & (truncated) MRSTP-SVD & 1 & 5 & 4 & 3 \\
Narrator & (truncated) MRSTP-SVD & 1 & 5 & 4 & 3 \\
Aerial   & (truncated) MRSTP-SVD & 1 & 5 & 4 & 3 \\
\hline
\end{tabular}
\end{table}

\begin{figure}[!ht]
\centering
\includegraphics[width=\textwidth]{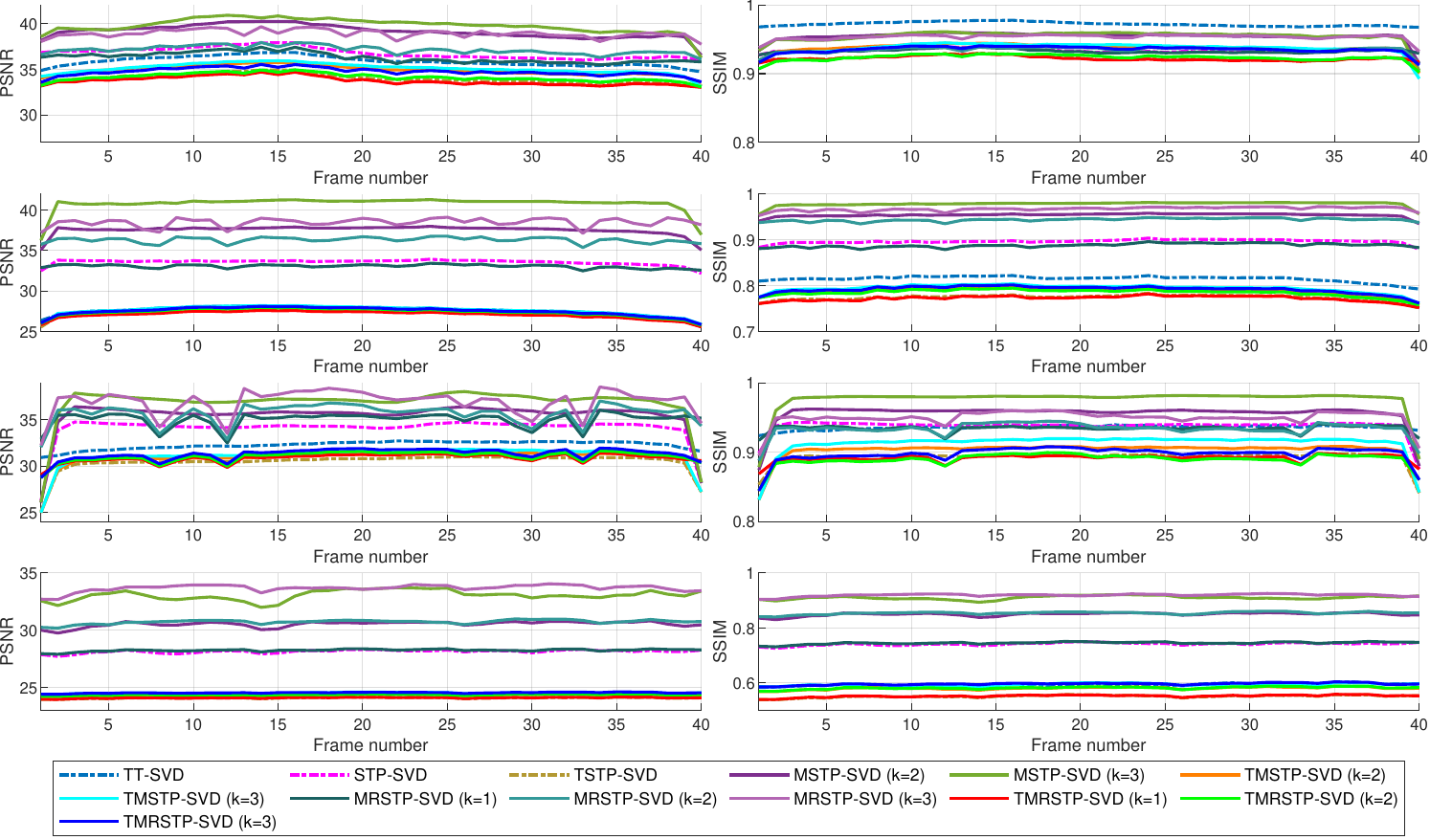}
\caption{ Frame-wise PSNR and SSIM curves over the first 40 frames of four benchmark videos for competing baselines and our approaches (original and truncated variants). From top to bottom: Crosswalk, Market, Narrator, Aerial.}
\label{fig:psnr_ssim_all}
\end{figure}
Fig.~\ref{fig:psnr_ssim_all} plots frame-wise PSNR and SSIM over the first 40 frames for all methods. Our multi-term schemes consistently outperform baselines, and higher $k$ yields steady improvements. The randomized variant achieves accuracy comparable to its deterministic counterpart with substantial acceleration and negligible performance loss.
Fig.~\ref{fig:video compression} presents visual and quantitative comparisons on sampled frames from Crosswalk and Market (results for Narrator and Aerial are in Supplementary Material). Our methods recover richer textures and finer local details, while baselines produce over-smoothed outputs. MRSTP-SVD notably reduces runtime across all sequences with negligible accuracy degradation relative to deterministic MSTP-SVD.
Table \ref{tab:video compression} summarizes average PSNR, SSIM and runtime. MRSTP-SVD (\(k=2,3\)) substantially outperforms TT-SVD and TSTP-SVD on all high-resolution videos, with over 5 dB PSNR gain on Market, nearly 4 dB on Aerial, and over 5 s runtime reduction. TMRSTP-SVD achieves further speedup with only minor acceptable accuracy loss.
Overall, the proposed framework strikes a favorable accuracy-efficiency trade-off, offering a practical solution for high-resolution video compression.

\begin{figure}[!ht]
\centering
\renewcommand{\arraystretch}{0.3}
\setlength\tabcolsep{0.1pt}
\begin{tabular}{@{}ccccccc@{}}

\tiny Original &\tiny TT-SVD & \tiny STP-SVD & \tiny TSTP-SVD &\tiny\makecell[c]{MSTP-SVD\\[-4pt](k=2)} & \tiny\makecell[c]{MSTP-SVD\\[-4pt](k=3)} &\tiny\makecell[c]{TMSTP-SVD\\[-4pt](k=2)} \\
\includegraphics[width=0.672in]{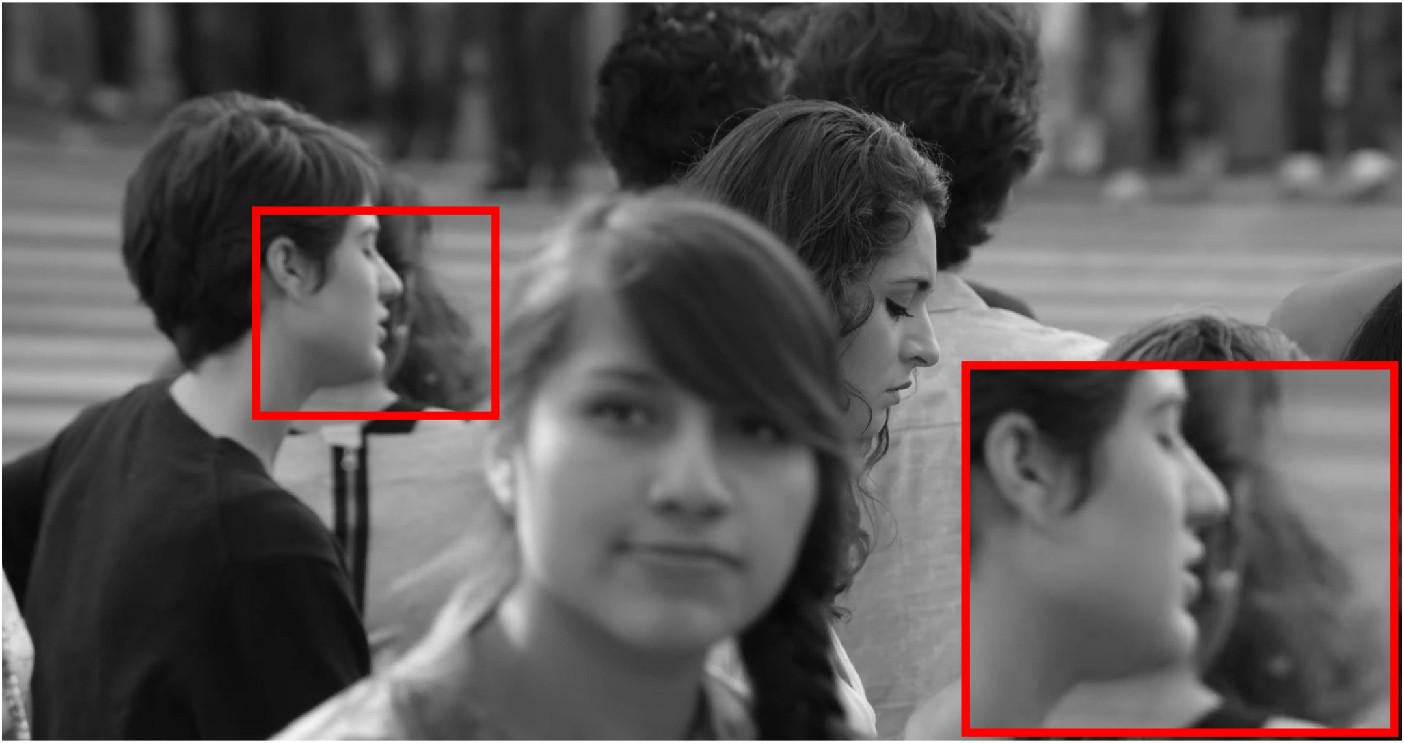} &
\includegraphics[width=0.672in]{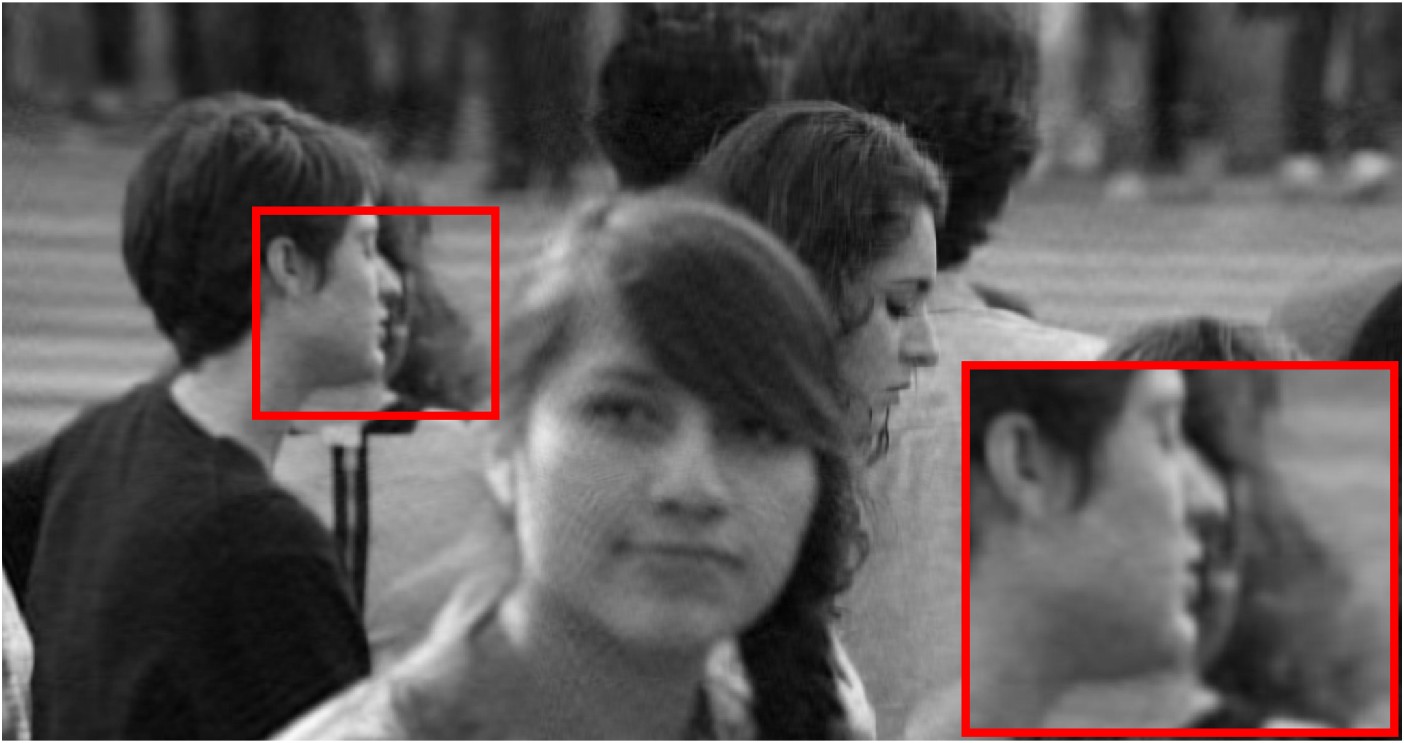} &
\includegraphics[width=0.672in]{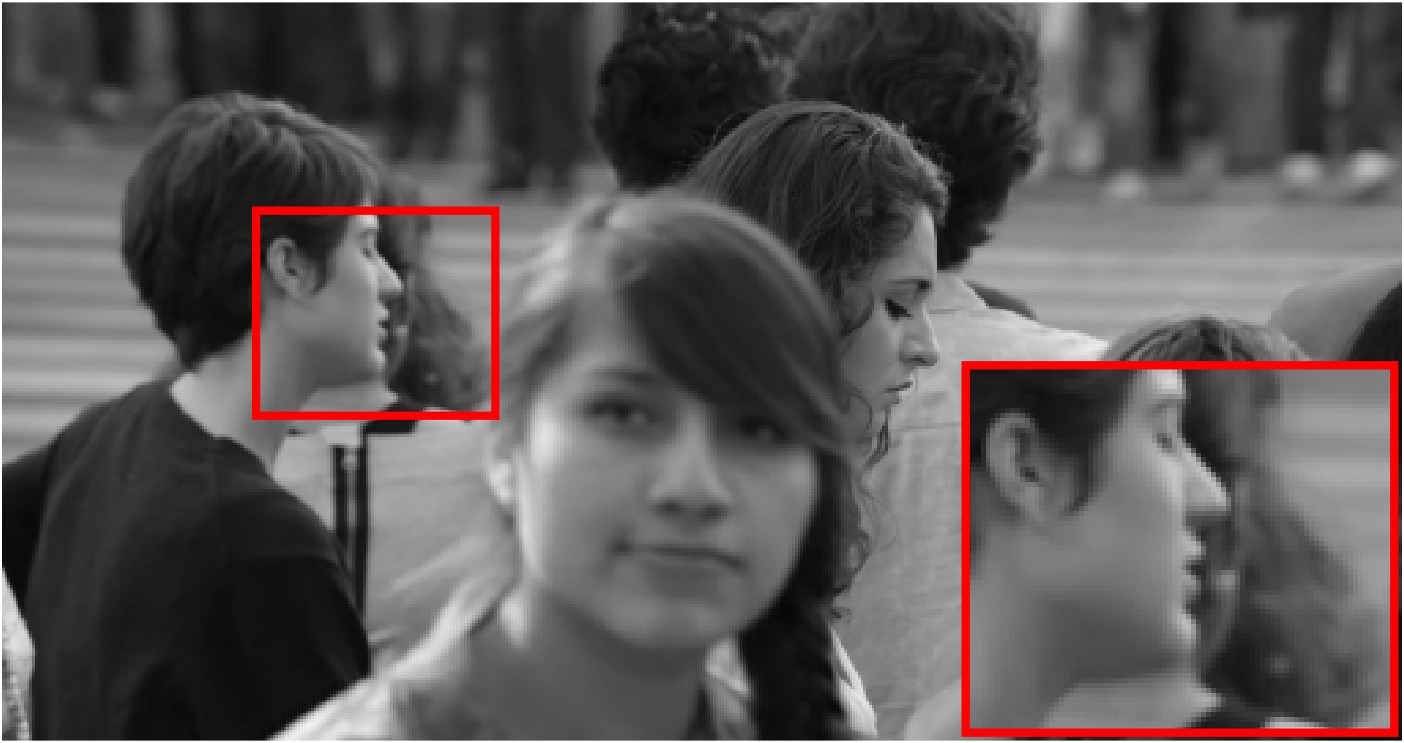} &
\includegraphics[width=0.672in]{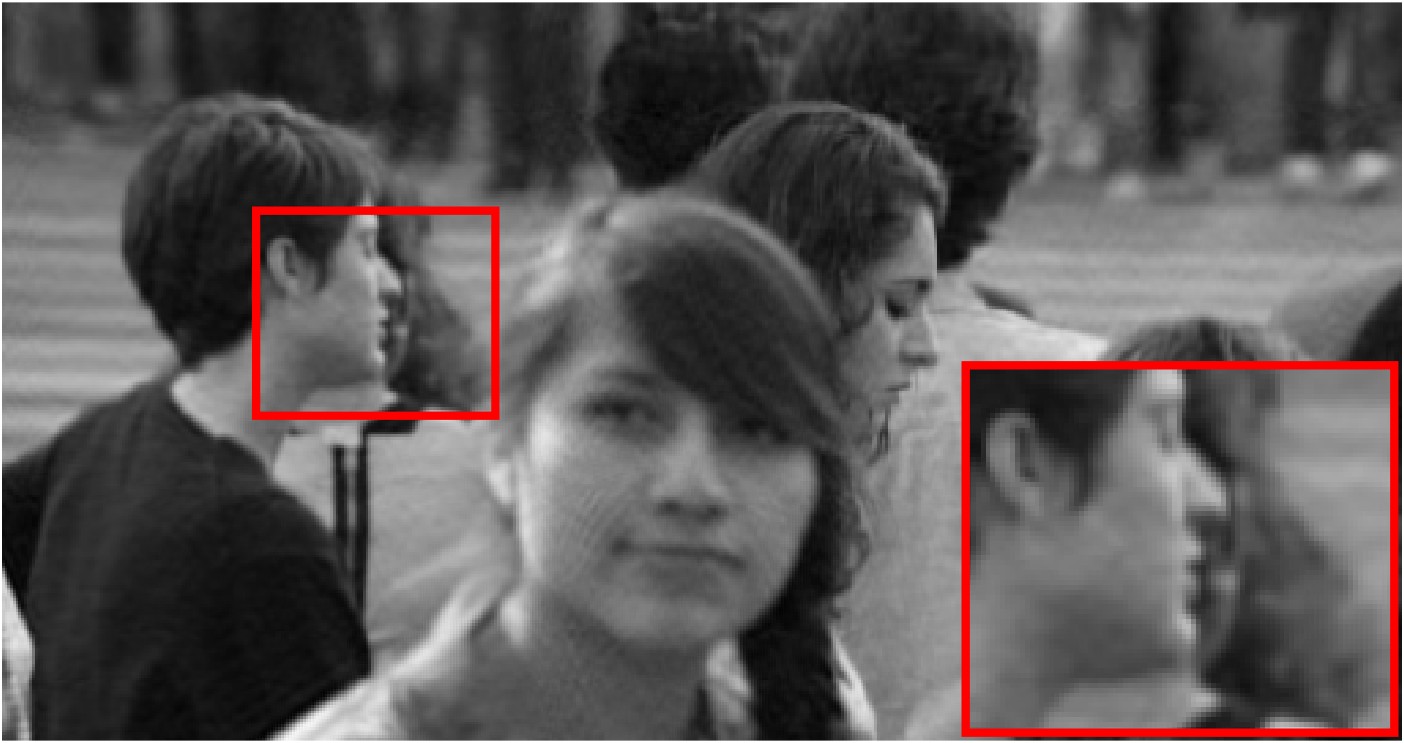} &
\includegraphics[width=0.672in]{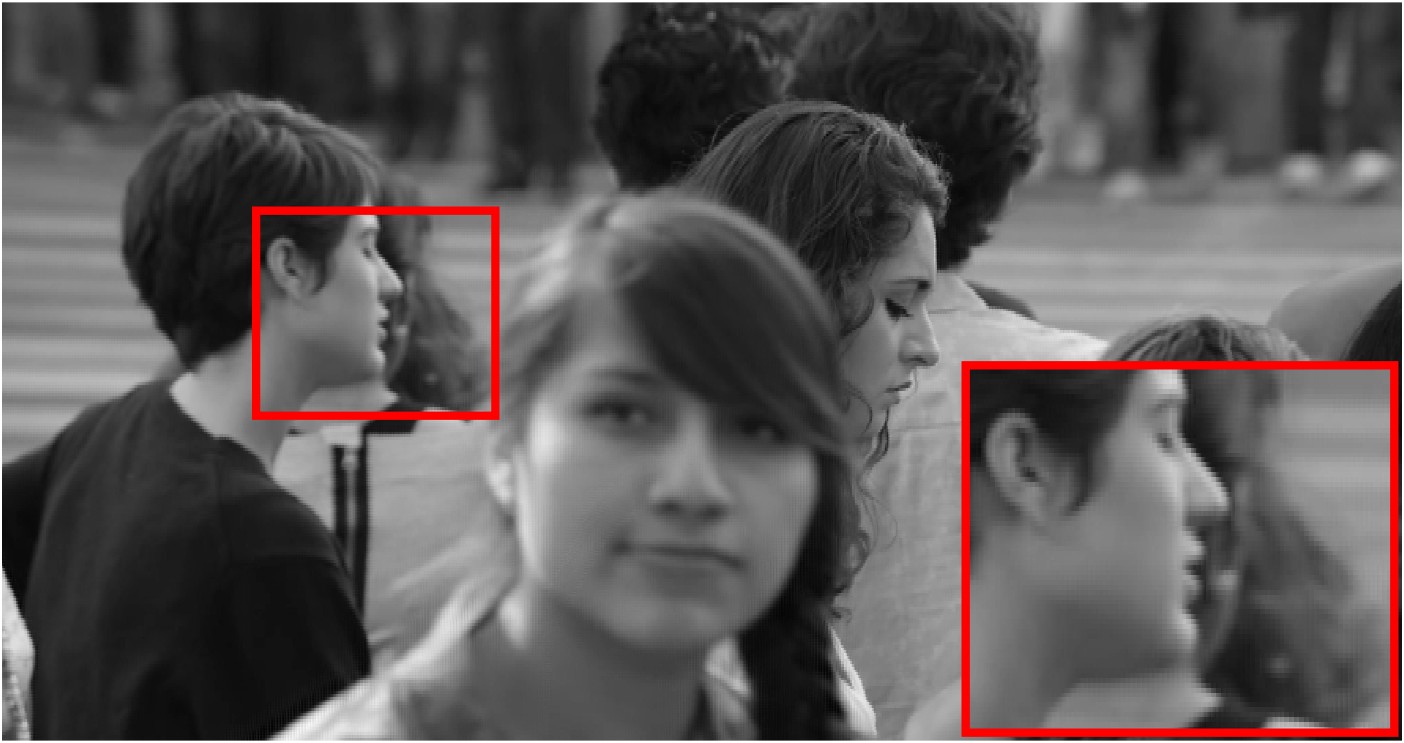} &
\includegraphics[width=0.672in]{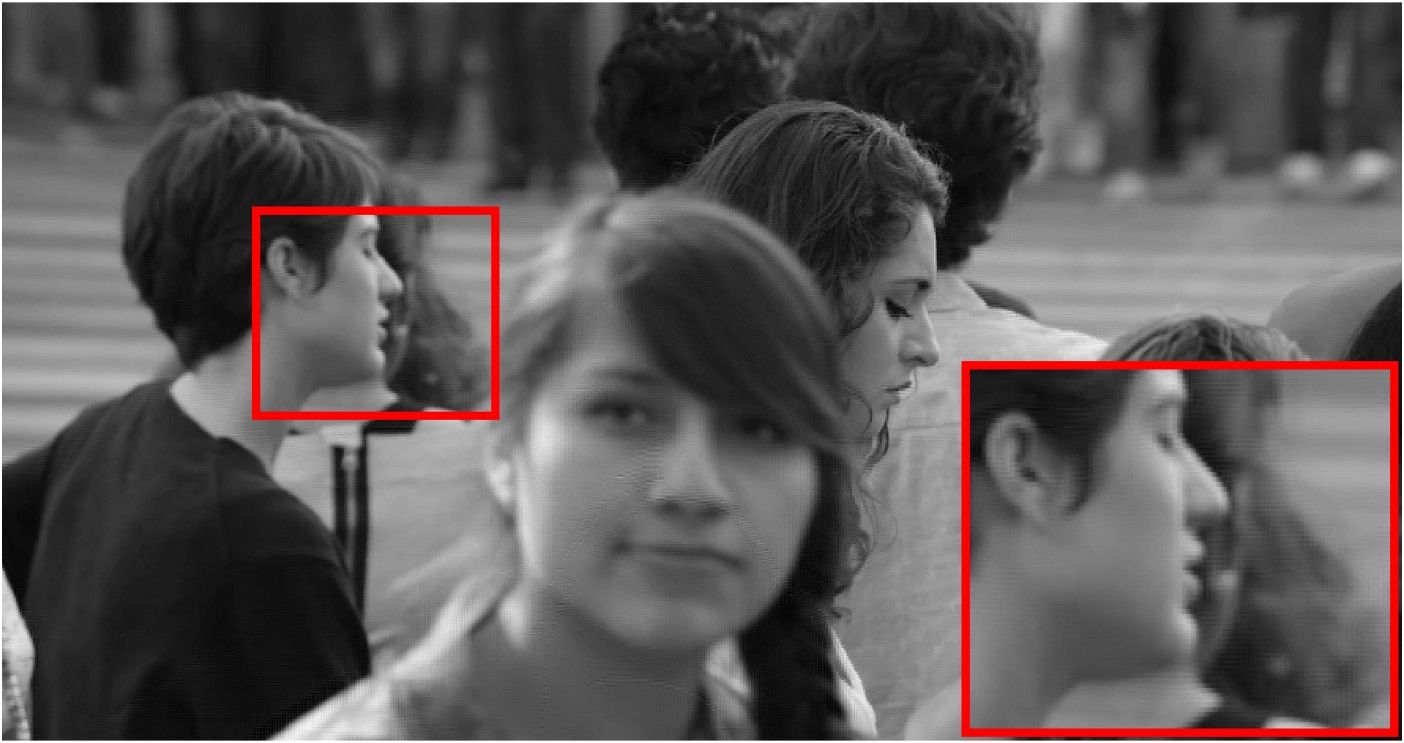} &
\includegraphics[width=0.672in]{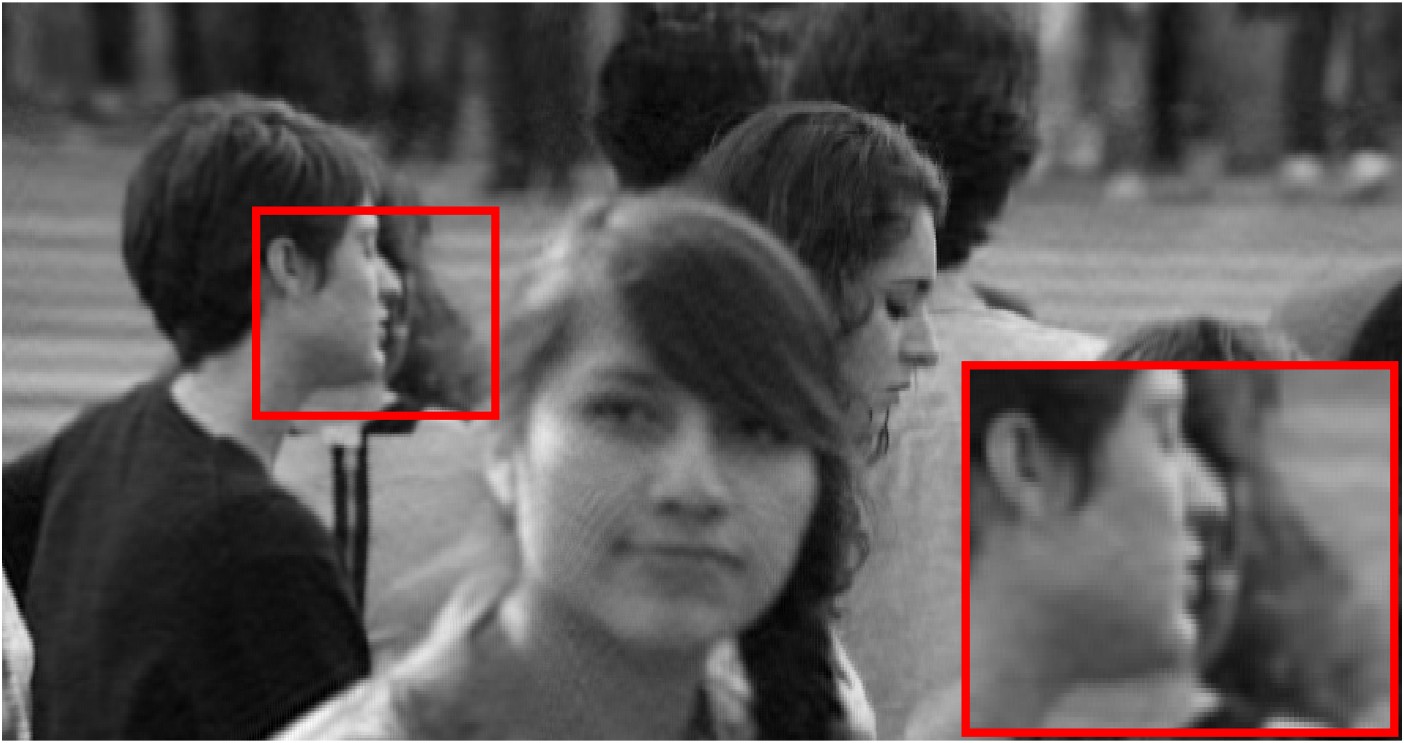} \\
 &\tiny PSNR:36.47 & \tiny PSNR:37.47 & \tiny PSNR:34.58 &\tiny PSNR:39.96 & \tiny PSNR:40.73 &\tiny PSNR:35.44 \\
 &\tiny SSIM:0.976 & \tiny SSIM:0.937 & \tiny SSIM:0.928 &  \tiny SSIM:0.959 & \tiny SSIM:0.960 &  \tiny SSIM:0.943\\
 \tiny\makecell[c]{TMSTP-SVD\\[-4pt](k=3)} &\tiny\makecell[c]{MRSTP-SVD\\[-4pt](k=1)} & \tiny\makecell[c]{MRSTP-SVD\\[-4pt](k=2)} & \tiny\makecell[c]{MRSTP-SVD\\[-4pt](k=3)}& \tiny\makecell[c]{TMRSTP-SVD\\[-4pt](k=1)} & \tiny\makecell[c]{TMRSTP-SVD\\[-4pt](k=2)} & \tiny\makecell[c]{TMRSTP-SVD\\[-4pt](k=3)}\\
\includegraphics[width=0.672in]{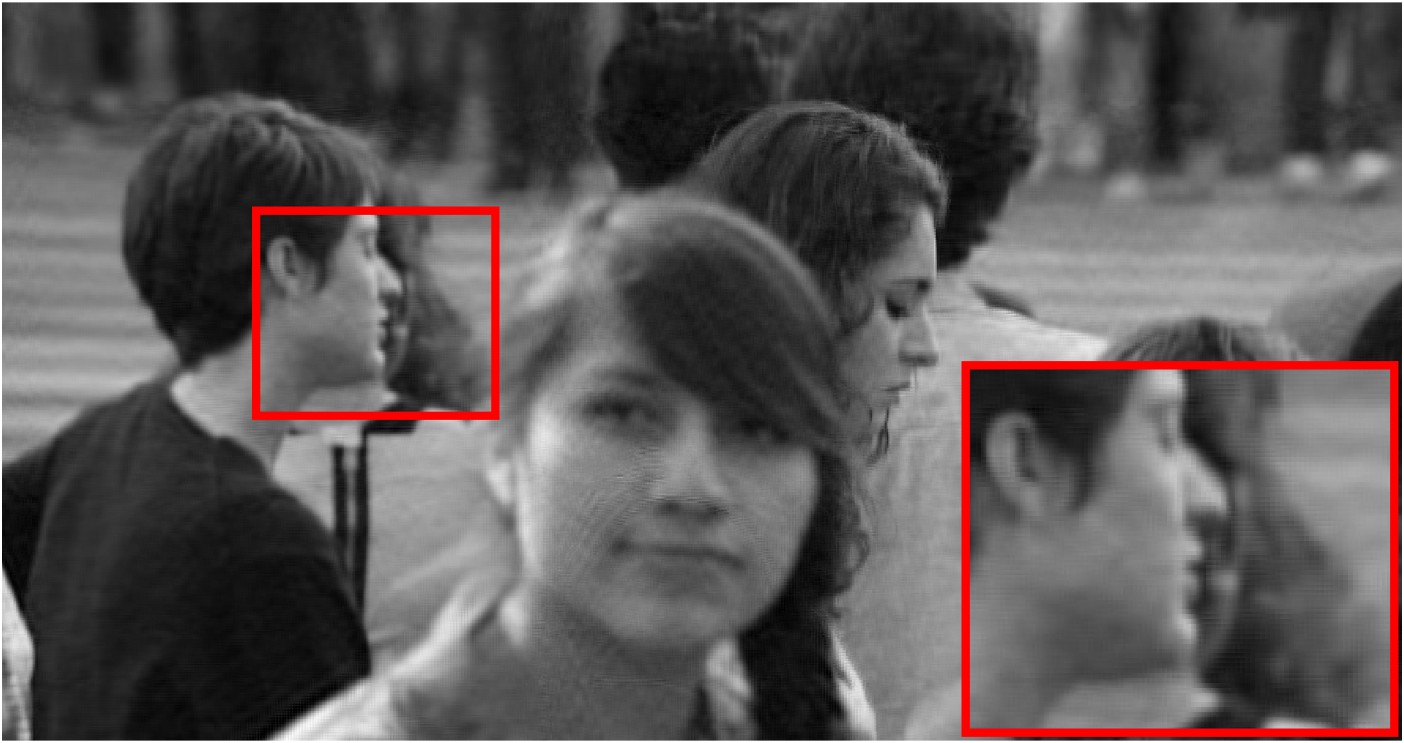} &
\includegraphics[width=0.672in]{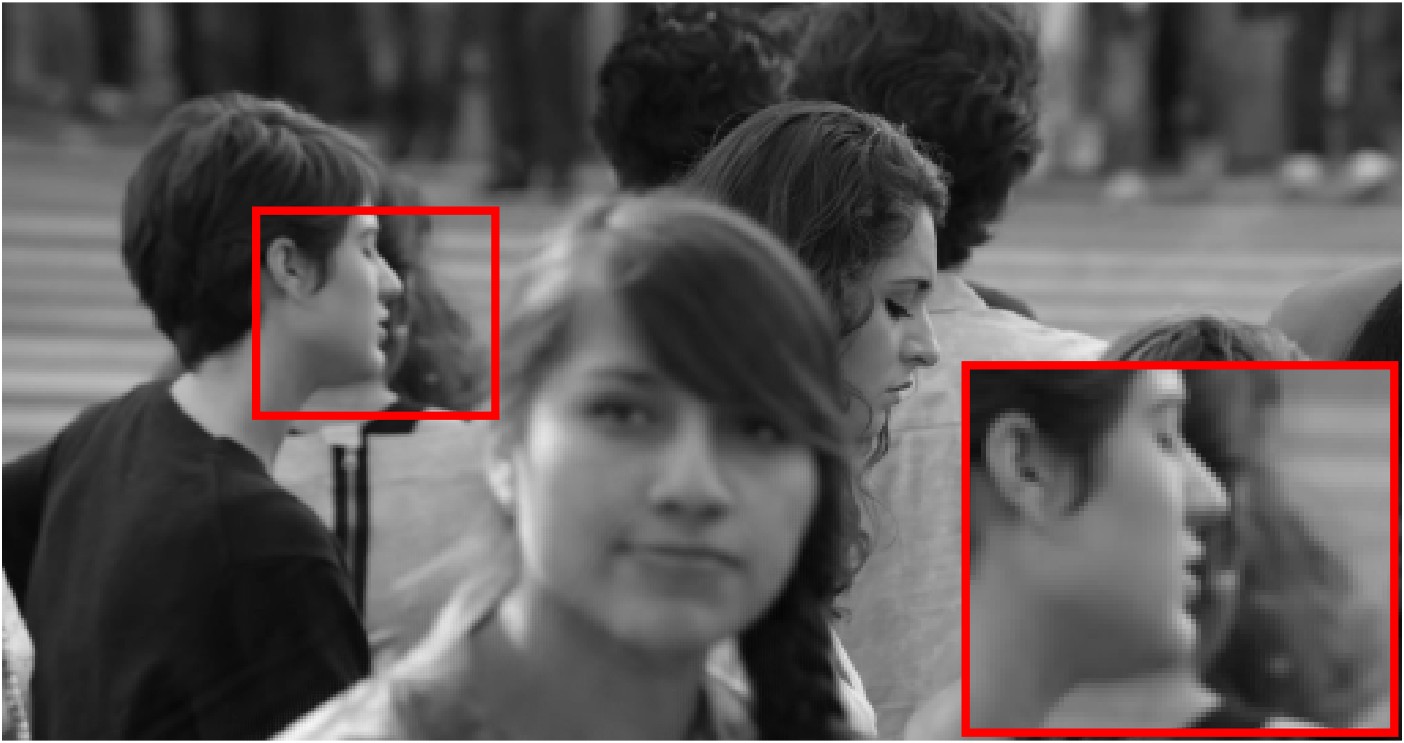} &
\includegraphics[width=0.672in]{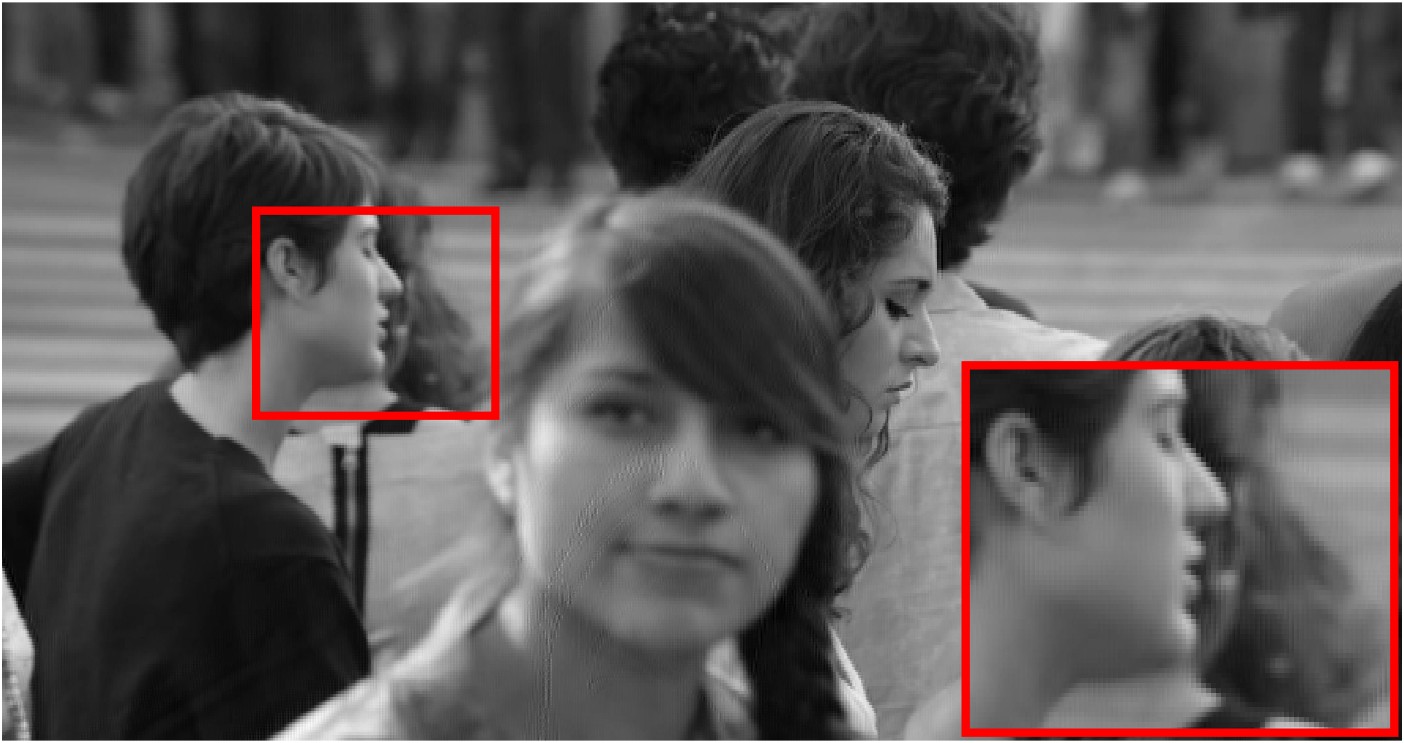} &
\includegraphics[width=0.672in]{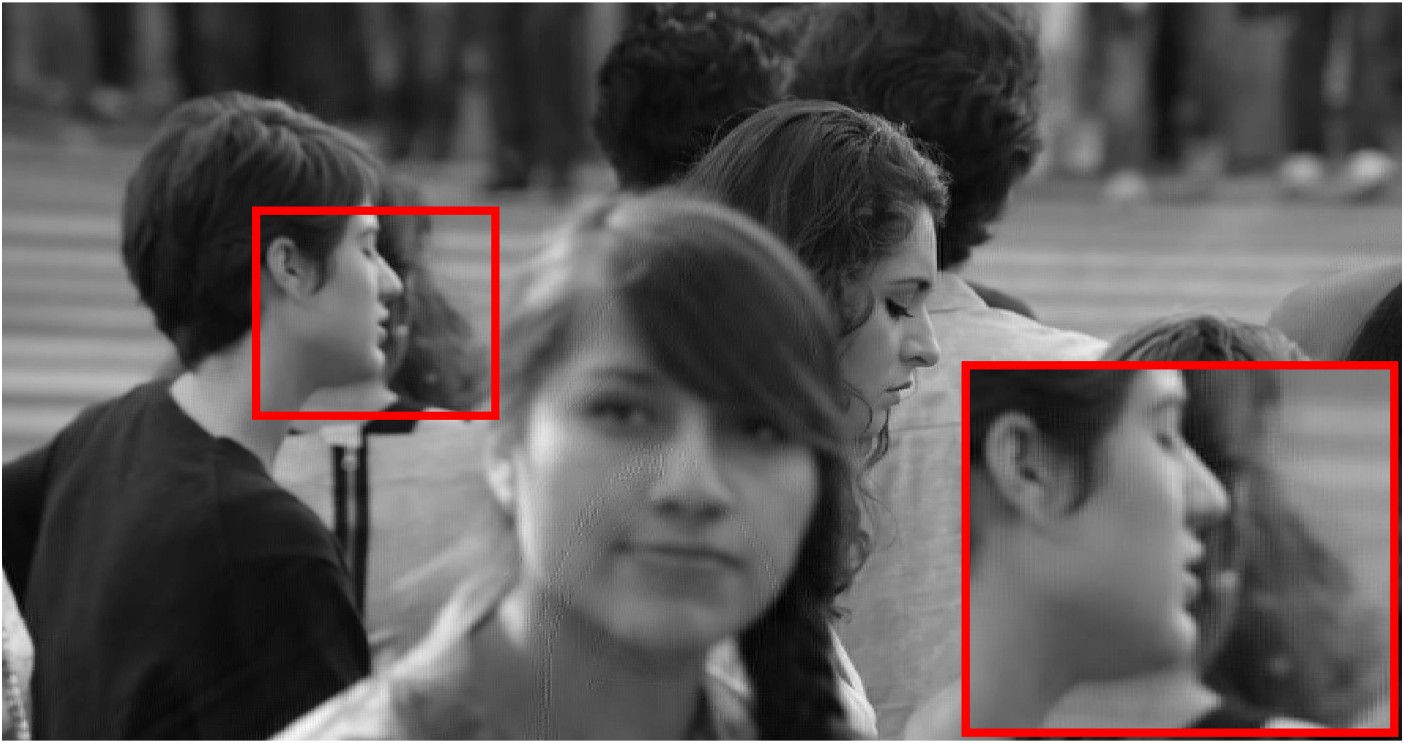}&
\includegraphics[width=0.672in]{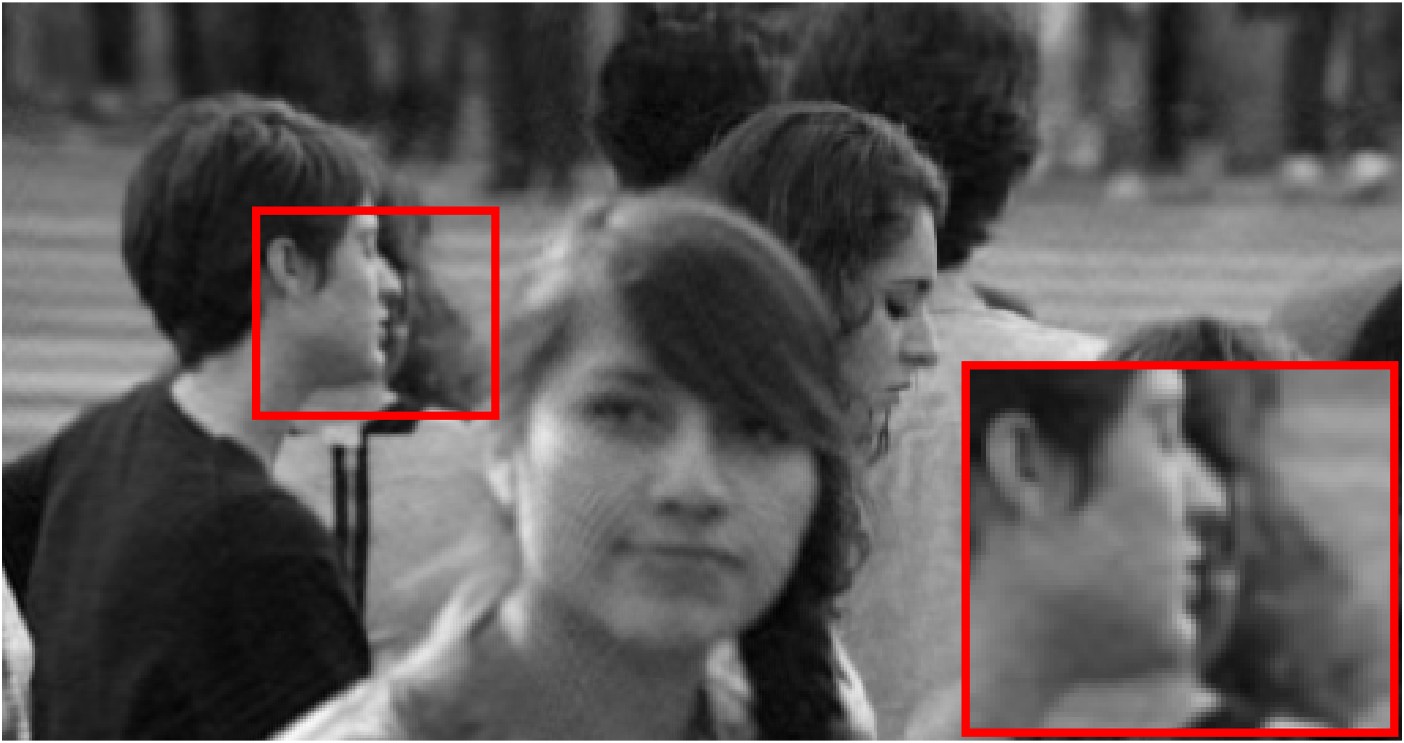} &
\includegraphics[width=0.672in]{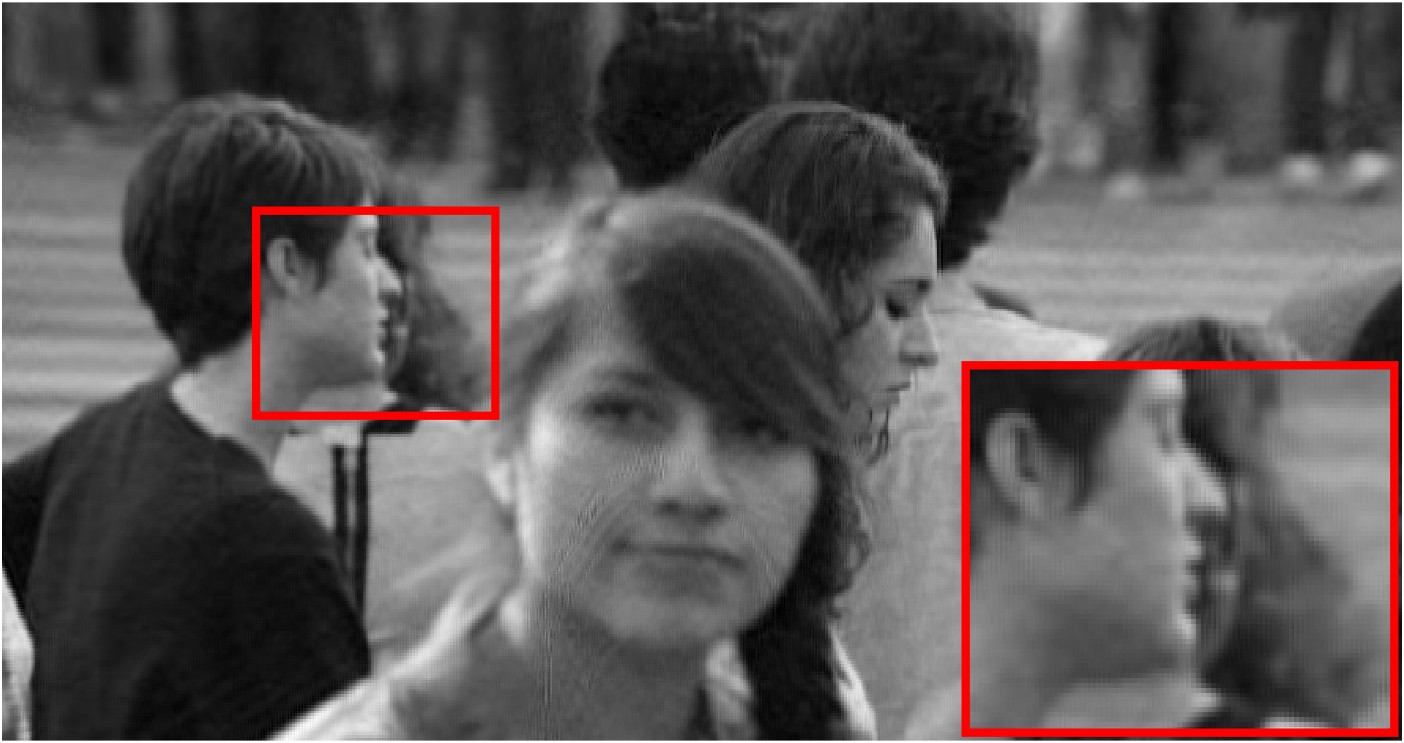} &
\includegraphics[width=0.672in]{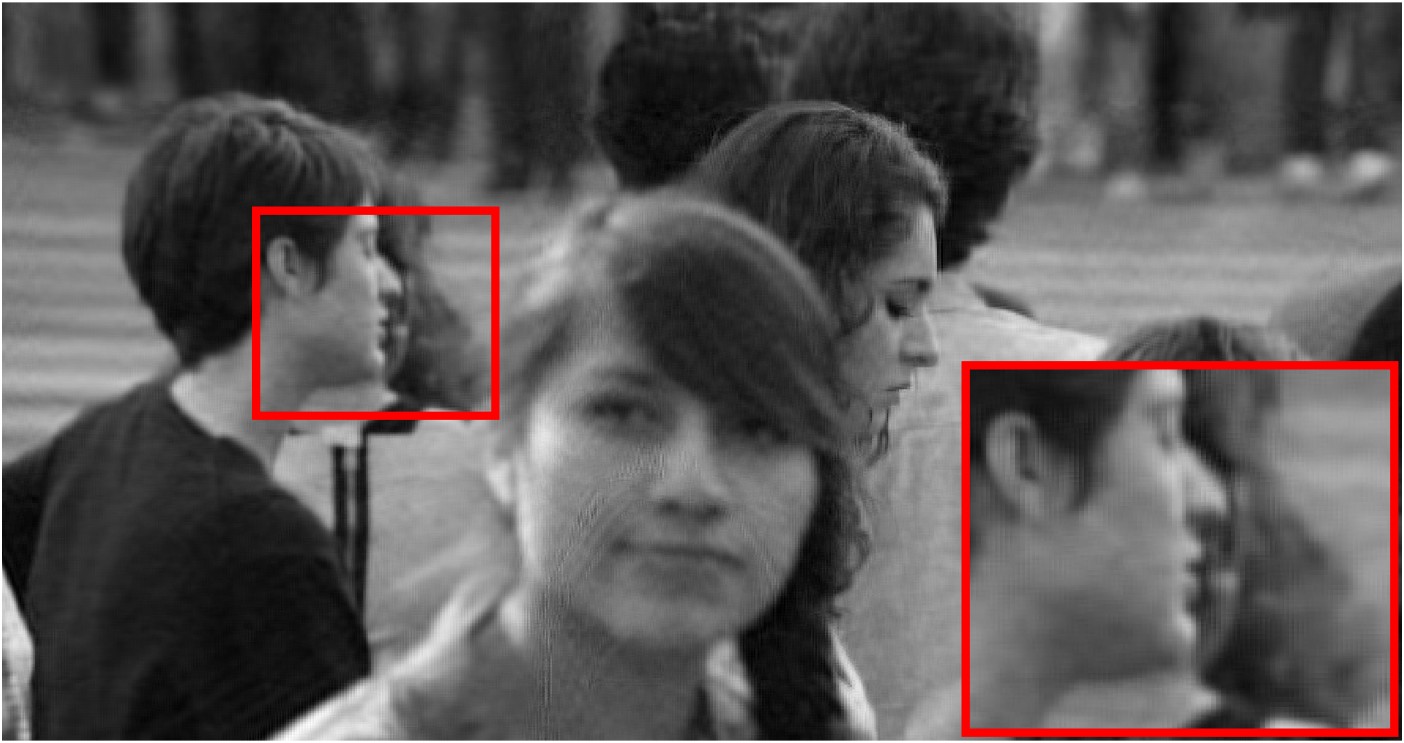} \\

\tiny PSNR:35.66 &\tiny PSNR:37.08 & \tiny PSNR:37.74 & \tiny PSNR:39.54 & \tiny PSNR:34.38 & \tiny PSNR:34.67 & \tiny PSNR:35.29\\
\tiny SSIM:0.944 & \tiny SSIM:0.937 & \tiny SSIM:0.938 & \tiny SSIM:0.956 & \tiny SSIM:0.927 & \tiny SSIM:0.928& \tiny SSIM:0.940\\

\tiny Original& \tiny TT-SVD & \tiny STP-SVD & \tiny TSTP-SVD &\tiny\makecell[c]{MSTP-SVD\\[-4pt](k=2)} & \tiny\makecell[c]{MSTP-SVD\\[-4pt](k=3)} &\tiny\makecell[c]{TMSTP-SVD\\[-4pt](k=2)} \\
\includegraphics[width=0.672in]{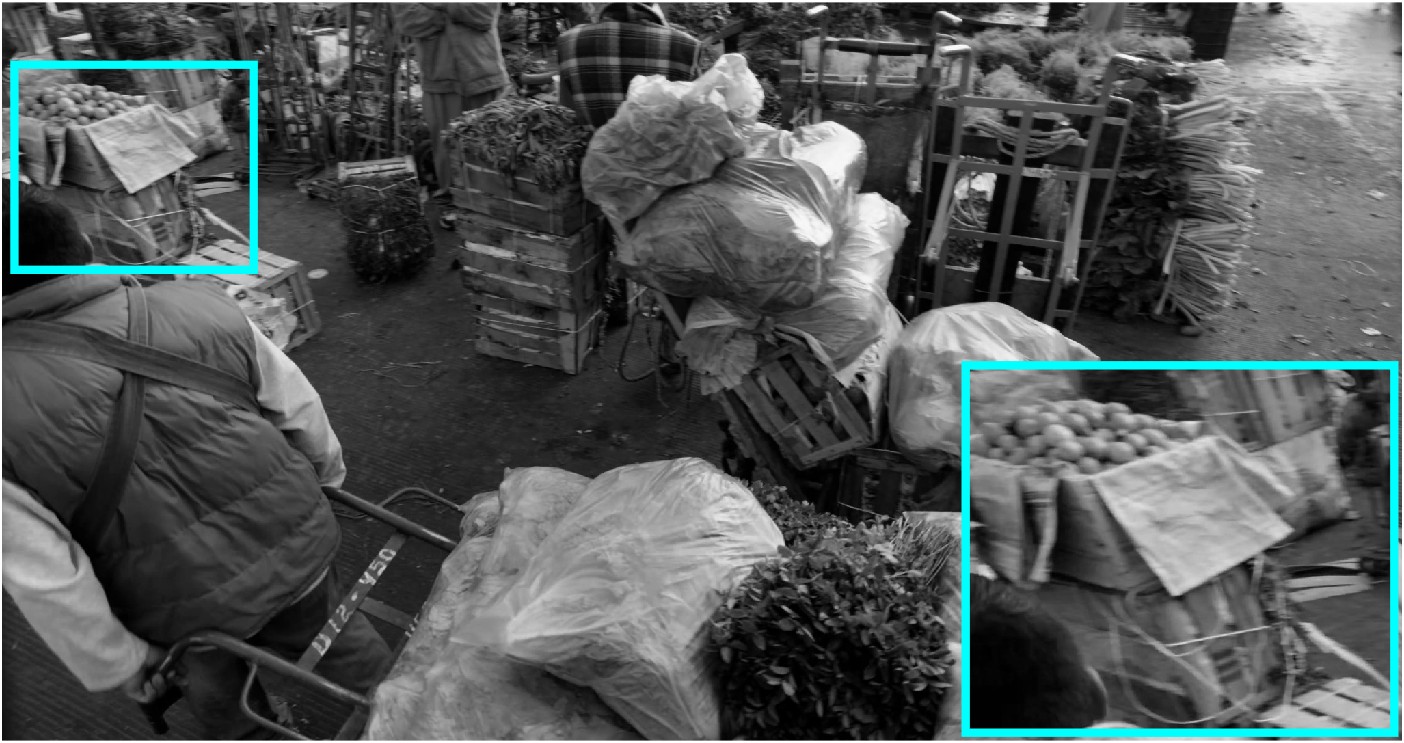} &
\includegraphics[width=0.672in]{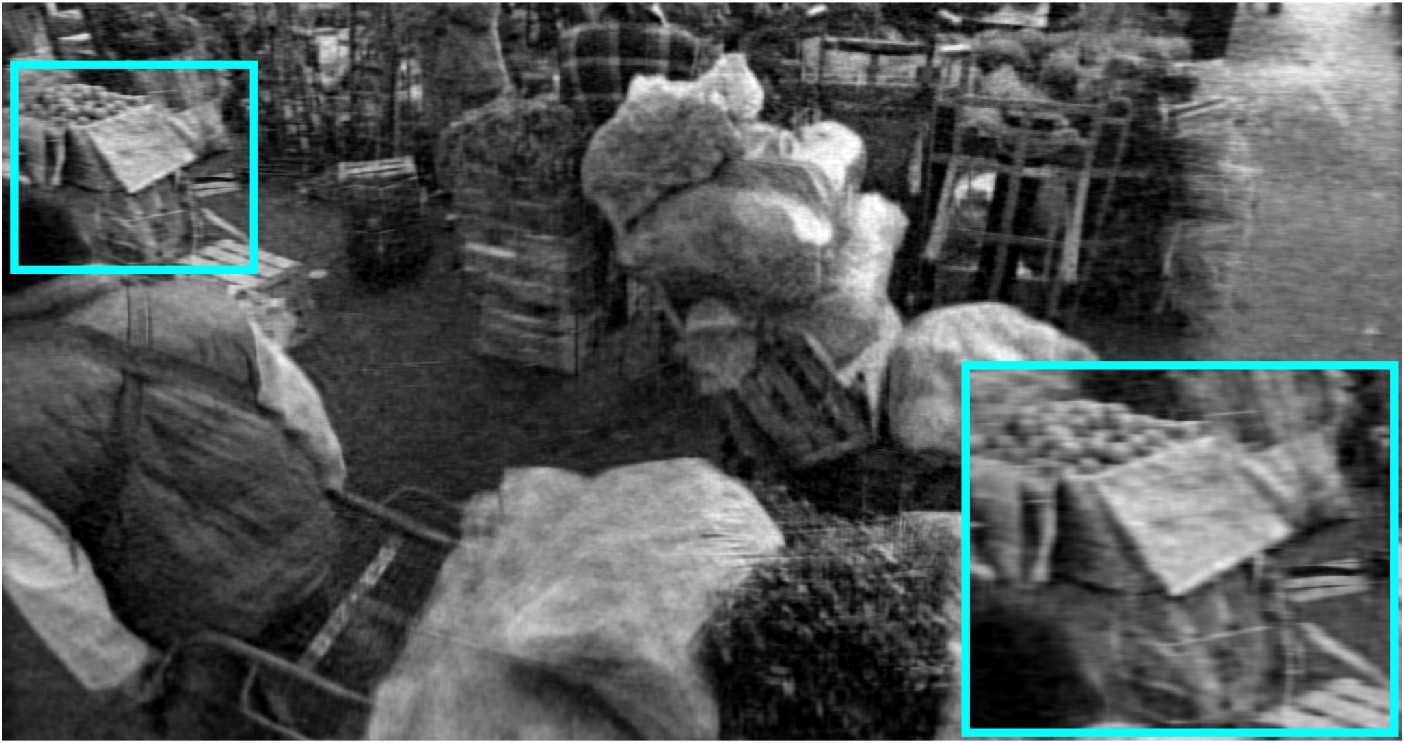} &
\includegraphics[width=0.672in]{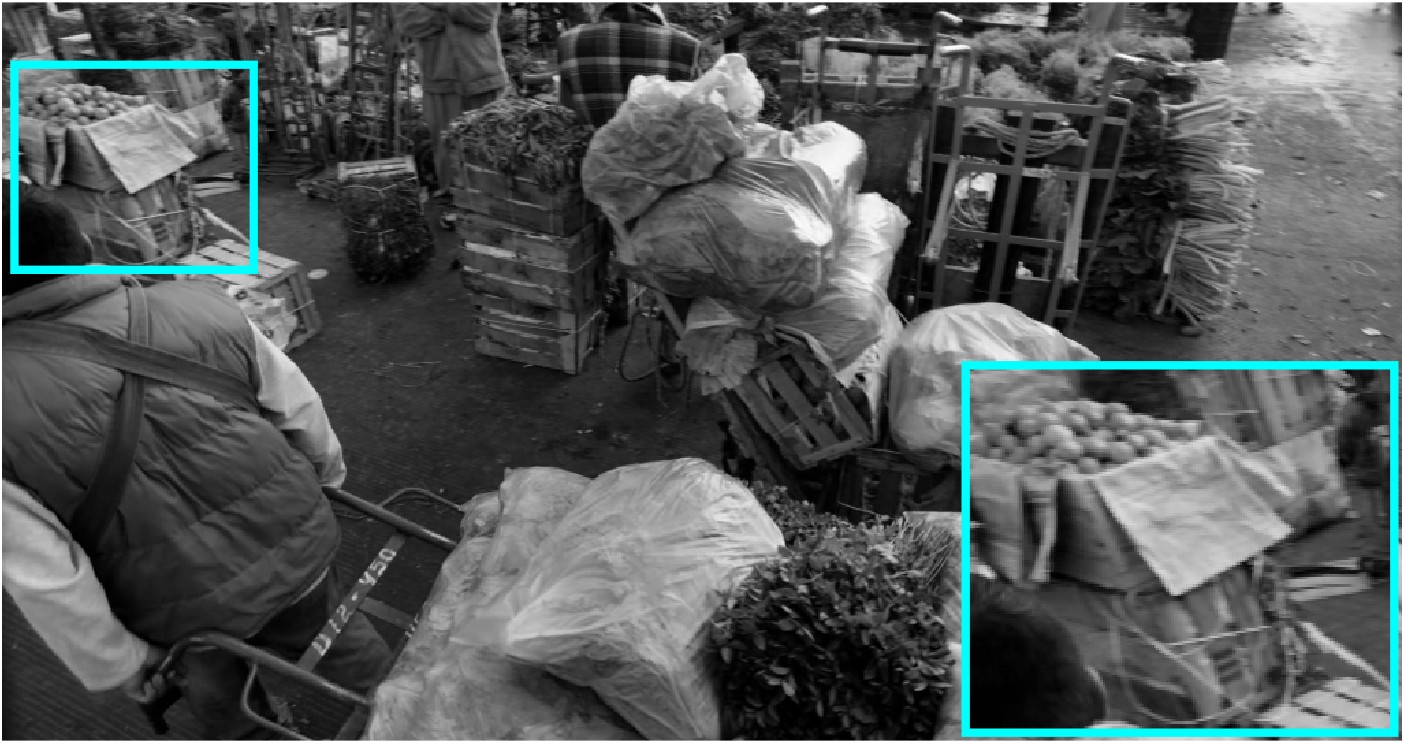} &
\includegraphics[width=0.672in]{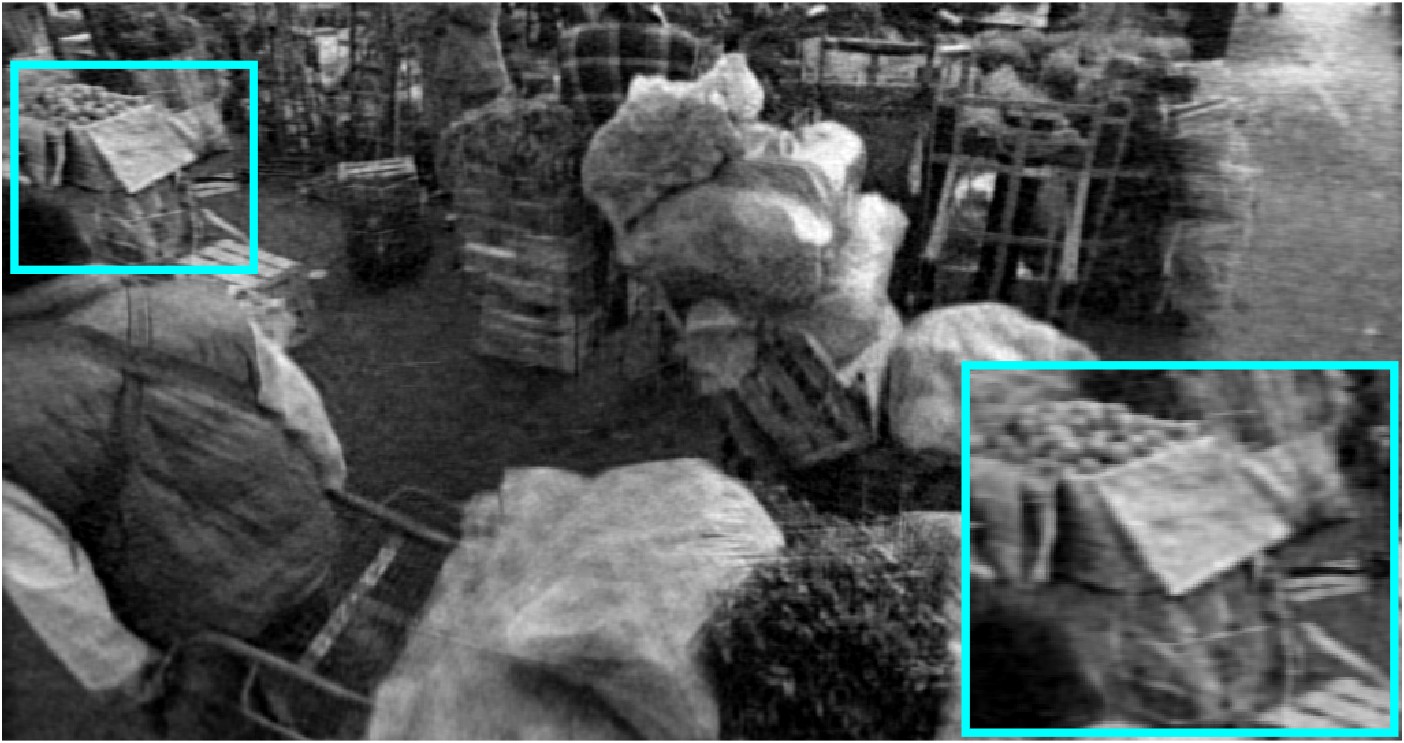} &
\includegraphics[width=0.672in]{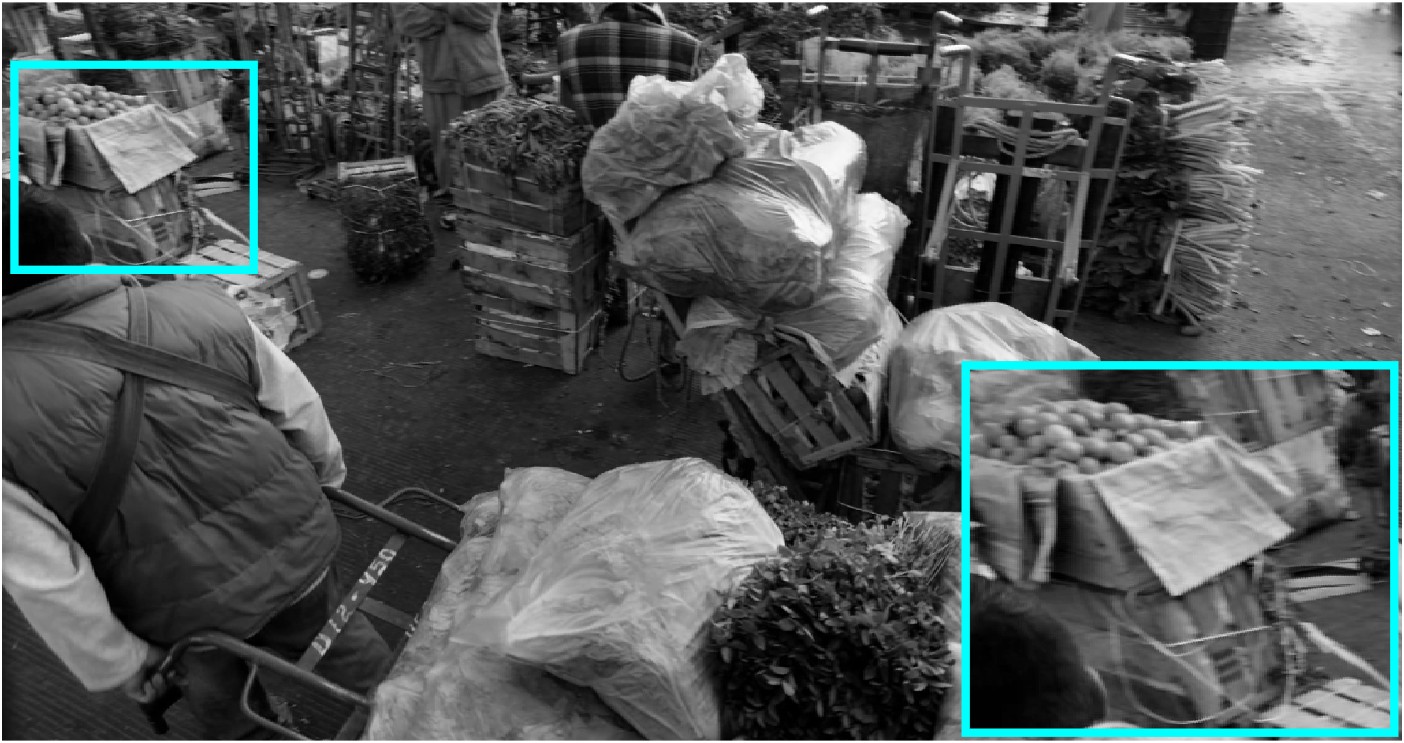} &
\includegraphics[width=0.672in]{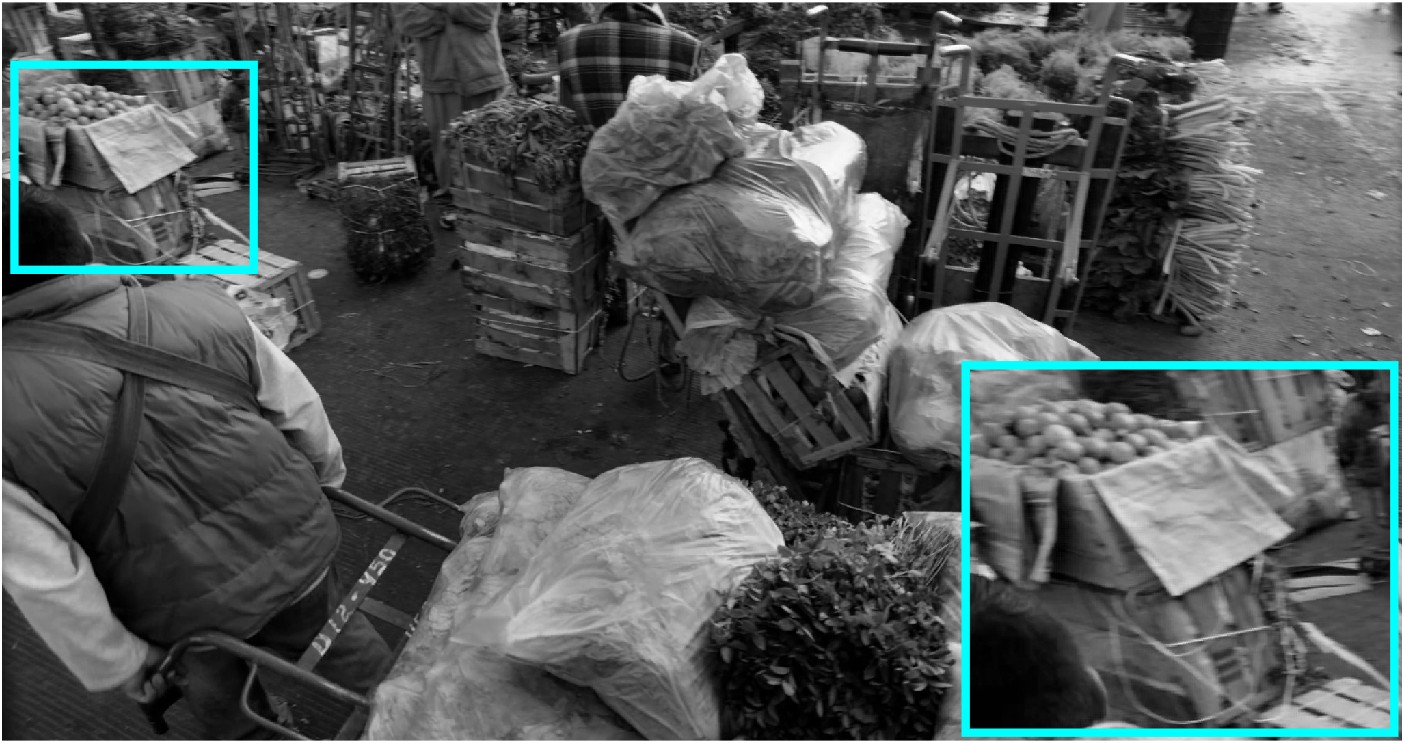} &
\includegraphics[width=0.672in]{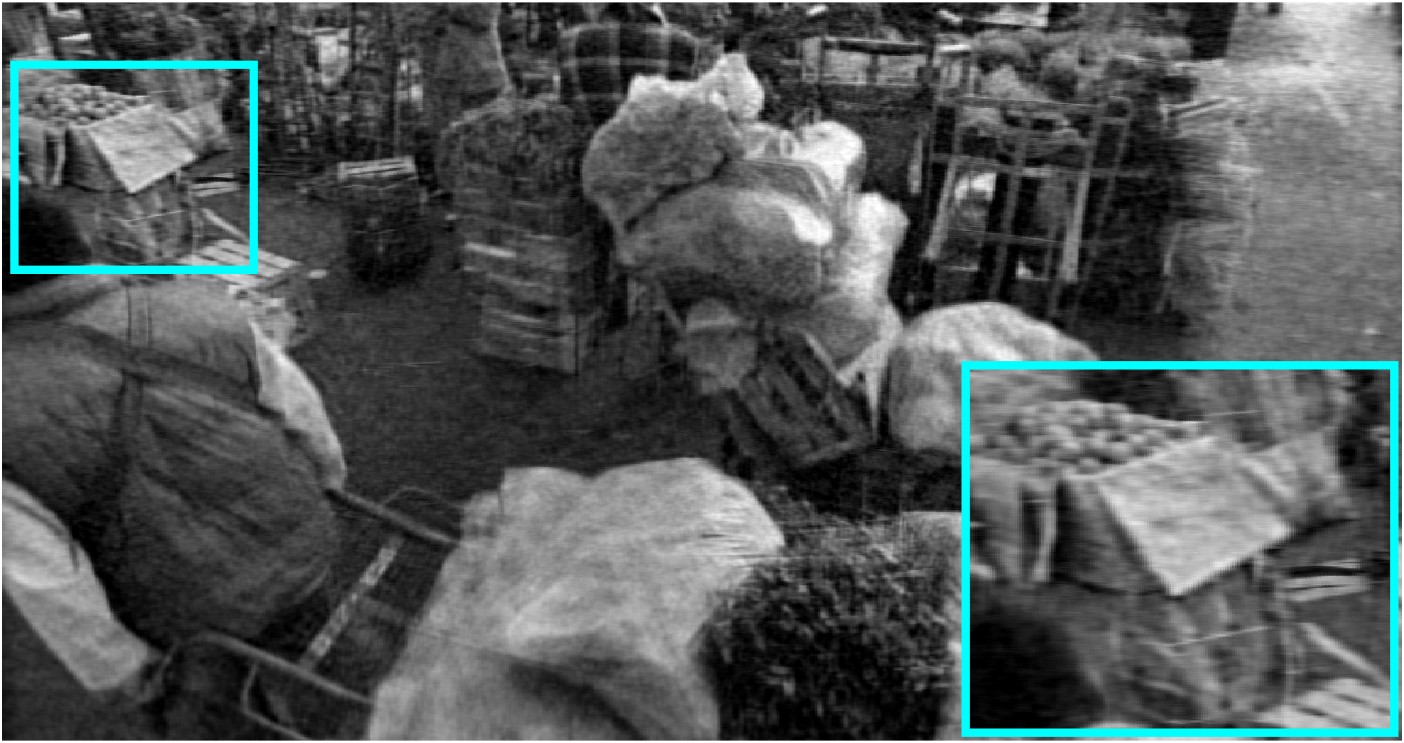} \\
&\tiny PSNR:27.39 & \tiny PSNR:33.42 & \tiny PSNR:26.96& \tiny PSNR:37.49 & \tiny PSNR:40.92 & \tiny PSNR:27.27\\ 
&\tiny SSIM:0.814 & \tiny SSIM:0.898 & \tiny SSIM:0.775 & \tiny SSIM:0.955 & \tiny SSIM:0.981 & \tiny SSIM:0.785\\ 
\tiny\makecell[c]{TMSTP-SVD\\[-4pt](k=3)} &\tiny\makecell[c]{MRSTP-SVD\\[-4pt](k=1)} & \tiny\makecell[c]{MRSTP-SVD\\[-4pt](k=2)} & \tiny\makecell[c]{MRSTP-SVD\\[-4pt](k=3)}& \tiny\makecell[c]{TMRSTP-SVD\\[-4pt](k=1)} & \tiny\makecell[c]{TMRSTP-SVD\\[-4pt](k=2)} & \tiny\makecell[c]{TMRSTP-SVD\\[-4pt](k=3)}\\
\includegraphics[width=0.672in]{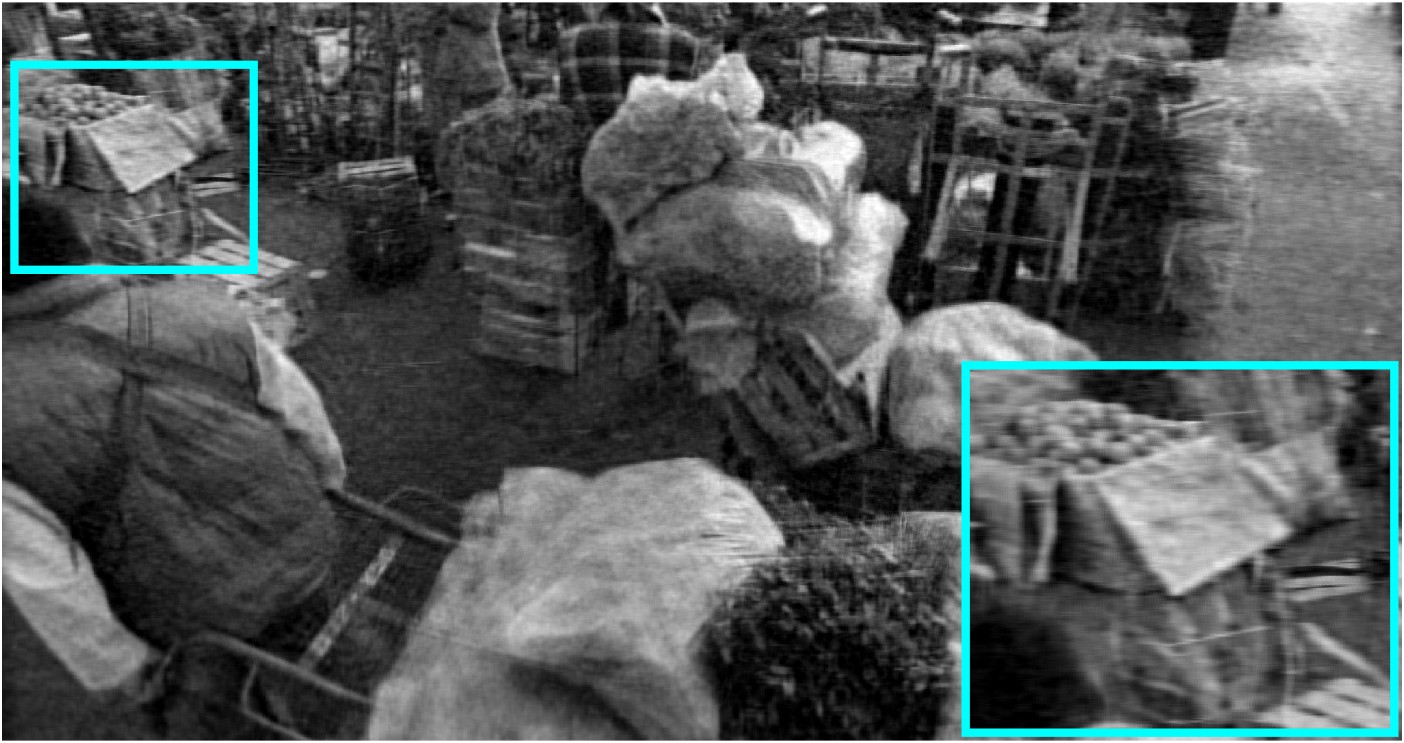} &
\includegraphics[width=0.672in]{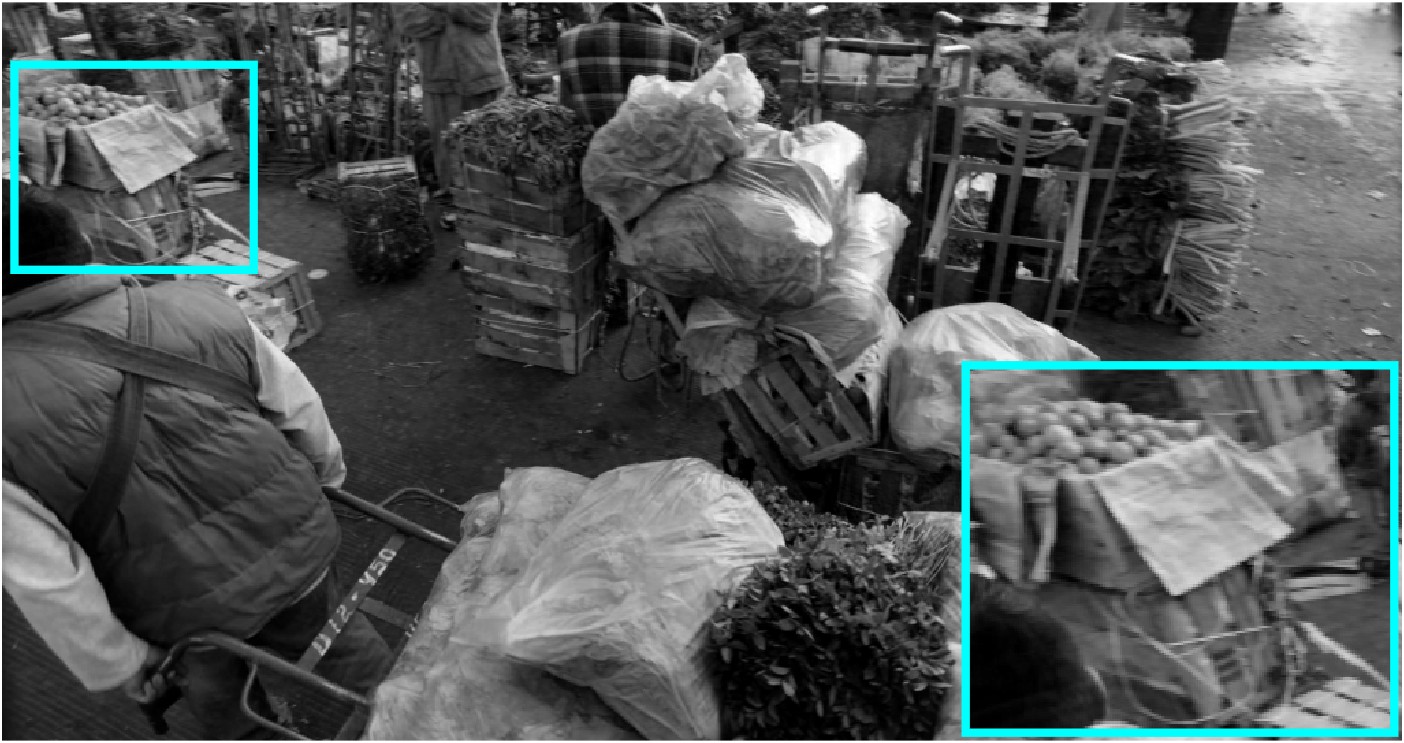} &
\includegraphics[width=0.672in]{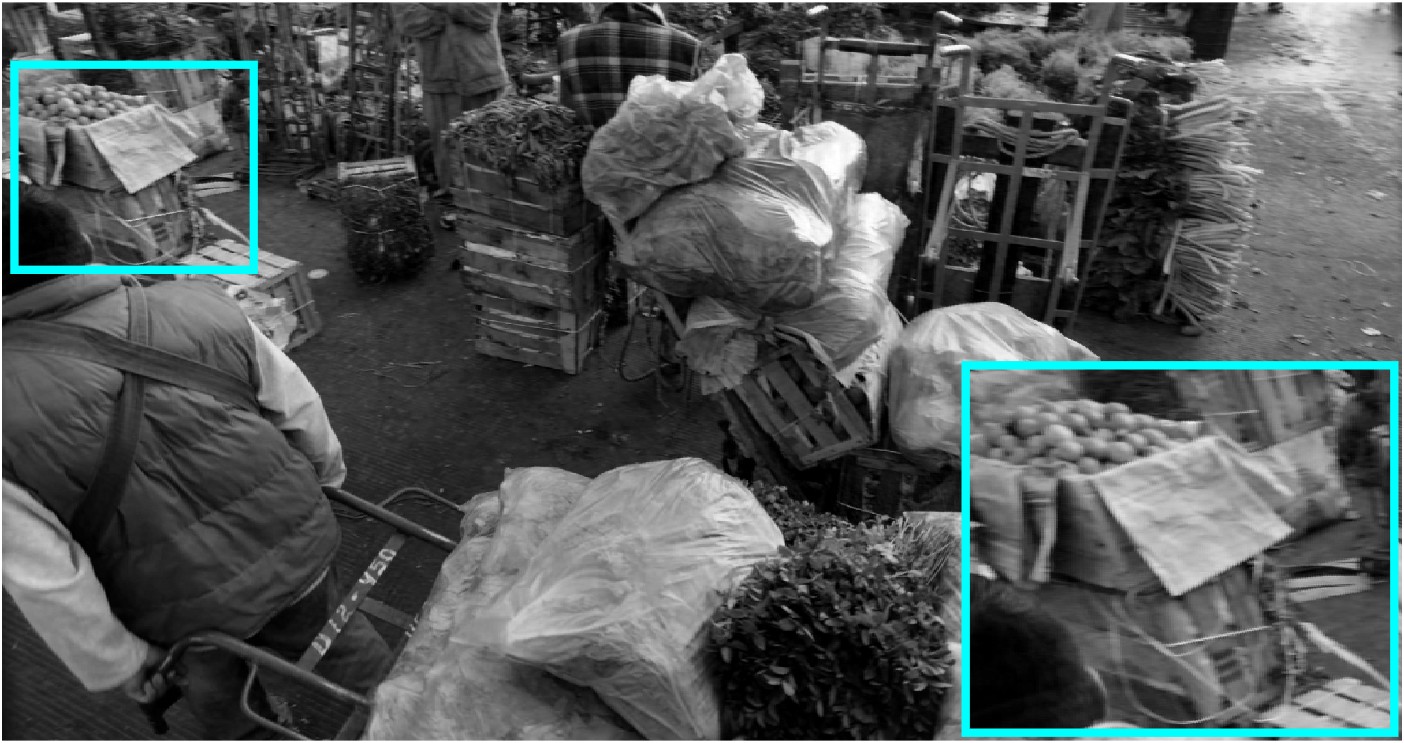} &
\includegraphics[width=0.672in]{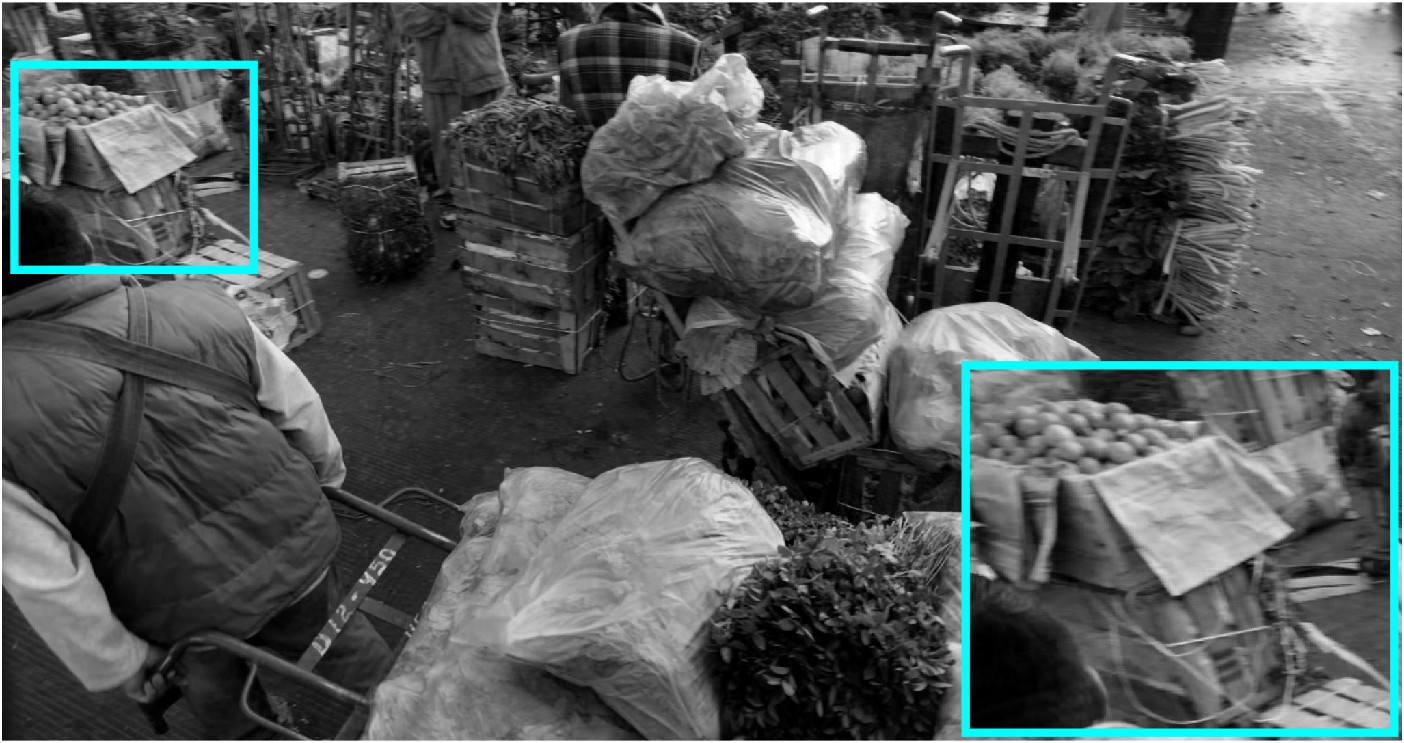}&
\includegraphics[width=0.672in]{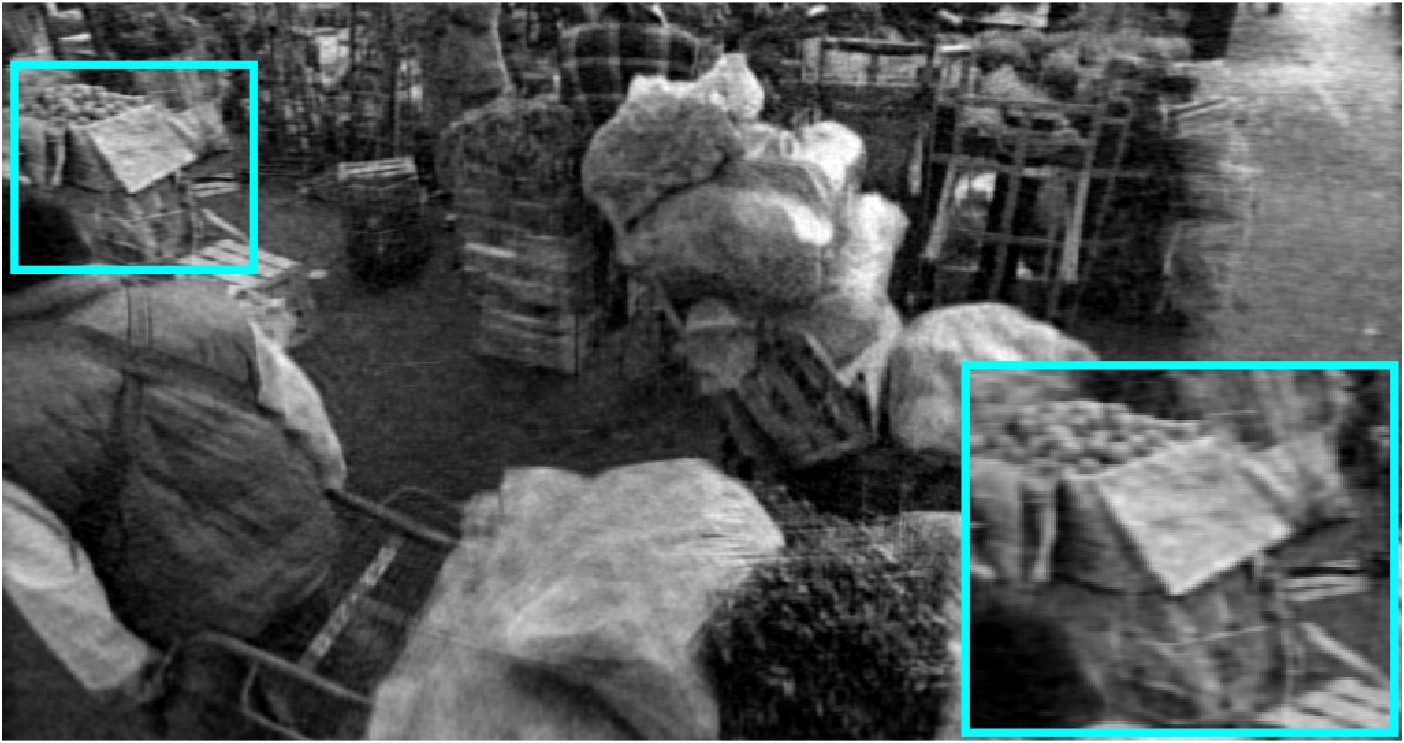} &
\includegraphics[width=0.672in]{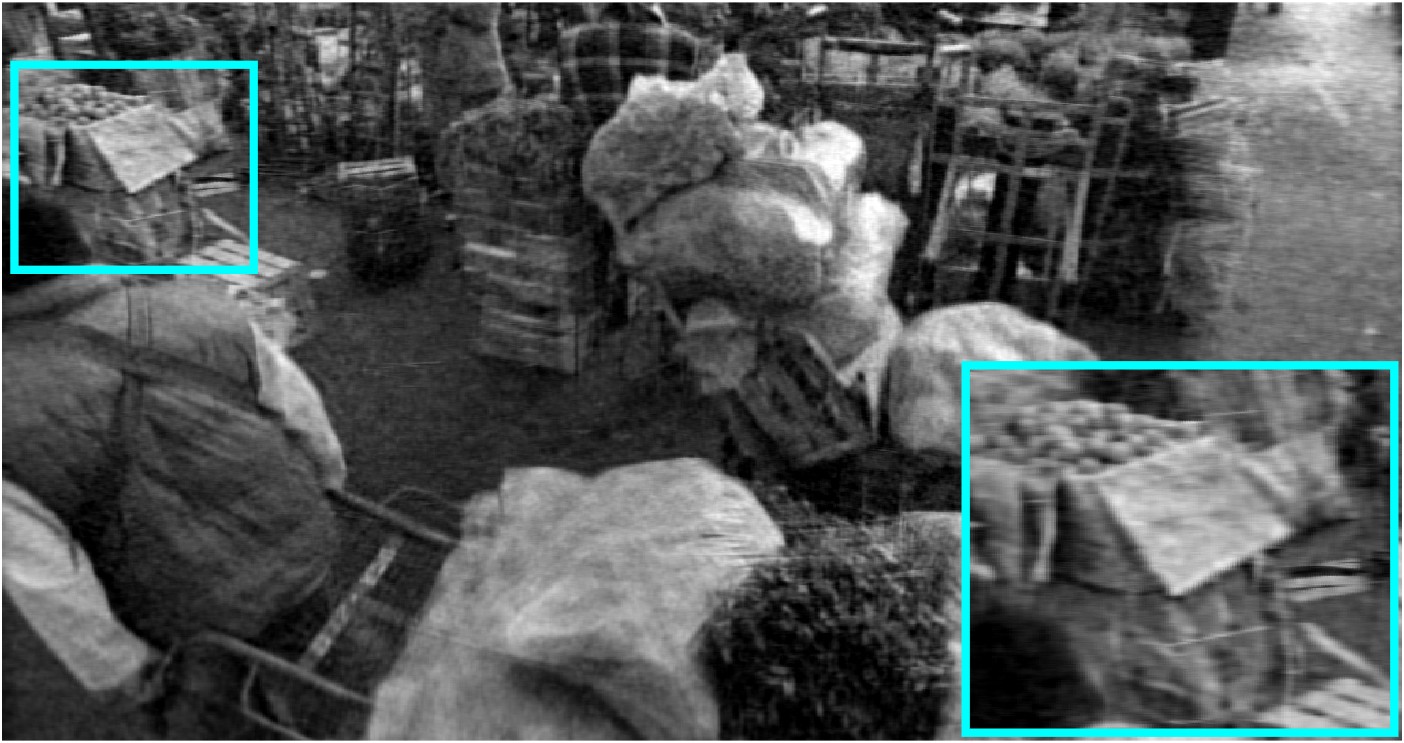} &
\includegraphics[width=0.672in]{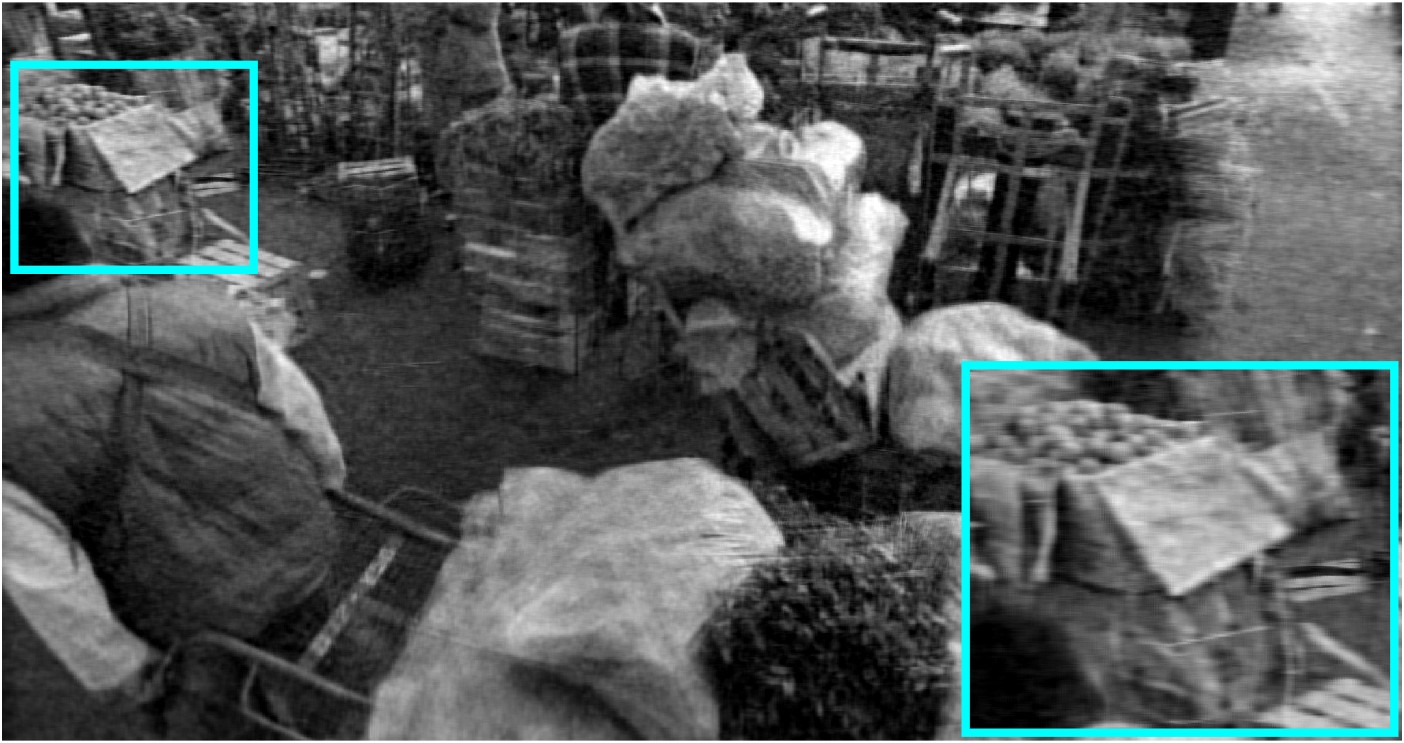} \\
\tiny PSNR:27.4& \tiny PSNR:32.48 & \tiny PSNR:35.37 & \tiny PSNR:37.13 & \tiny PSNR:26.78 & \tiny PSNR:27.08 & \tiny PSNR:33.79\\
\tiny SSIM: 0.793 & \tiny SSIM:0.890 & \tiny SSIM:0.942 & \tiny SSIM:0.966 & \tiny SSIM:0.771 & \tiny SSIM:0.781 & \tiny SSIM:9.01s \\

\end{tabular}

\caption{Visual reconstruction examples and quantitative PSNR-SSIM comparisons of competing baselines and our approaches (original and truncated variants) on randomly sampled frames from two test video sequences.}
\label{fig:video compression}
\end{figure}

\begin{table}[!htbp]
\footnotesize
\setlength{\tabcolsep}{1.0pt}
\renewcommand{\arraystretch}{0.8}
\caption{Average PSNR, SSIM and runtime comparison of competing baselines and our approaches (original and truncated multi-term variants) over all frames of four test videos.}
\centering

\begin{tabular*}{\linewidth}{l @{\extracolsep{\fill}} c c c *{2}{c} *{2}{c} *{3}{c} *{3}{c}}
\hline
 \multirow{2}{*}{Metric}
& \multirow{2}{*}{\scriptsize TT-SVD}
& \multirow{2}{*}{\scriptsize STP-SVD}
& \multirow{2}{*}{\scriptsize TSTP-SVD}
& \multicolumn{2}{c}{\scriptsize MSTP-SVD}
& \multicolumn{2}{c}{\scriptsize TMSTP-SVD}
& \multicolumn{3}{c}{\scriptsize MRSTP-SVD}
& \multicolumn{3}{c}{\scriptsize TMRSTP-SVD} \\
\cline{5-6} \cline{7-8} \cline{9-11} \cline{12-14}
& & &
& \scriptsize$k=2$ & \scriptsize$k=3$
& \scriptsize$k=2$ & \scriptsize$k=3$
& \scriptsize$k=1$ & \scriptsize$k=2$ & \scriptsize$k=3$
& \scriptsize$k=1$ & \scriptsize$k=2$ & \scriptsize$k=3$ \\
\hline\hline
\multicolumn{14}{c}{\textbf{Crosswalk}} \\ \hline\hline
PSNR 
& 35.84& 36.76 & 33.98
& 39.10 & \textbf{39.73}
& 34.86 & 35.06
& 36.27 & 37.09 & 38.87
& 33.72 & 34.10 & 34.70\\
SSIM
& 0.972 & 0.934 & 0.923
& \textbf{0.954} & \textbf{0.954}
& 0.937 & 0.937
& 0.933& 0.935 & 0.953
& 0.922 & 0.923 & 0.935 \\
Time (s)
& 63.45 & 22.26 & 21.05
& 26.78 & 30.86
& 22.42 & 26.18
& 13.42 & 16.42 & 18.52
& \textbf{12.52} & 13.51 & 14.73 \\
\hline\hline
\multicolumn{14}{c}{\textbf{Market}} \\ \hline\hline
PSNR 
& 27.58 & 33.55 & 27.14
& 37.52 & \textbf{40.77}
& 27.46 & 27.58
& 33.05 & 36.30 & 38.48
& 27.05 & 27.36 & 27.48 \\
SSIM
& 0.815 & 0.896 & 0.774
& 0.954 & \textbf{0.978}
& 0.787 & 0.794
& 0.888 & 0.943 & 0.967
& 0.772 & 0.784 & 0.791 \\
Time (s)
& 63.36 & 61.84 & 54.77
& 68.49 & 76.43
& 60.41 & 64.48
& 40.49 & 48.63 & 61.84
& \textbf{36.55} & 43.06 & 47.79 \\
\hline\hline
\multicolumn{14}{c}{\textbf{Narrator}} \\ \hline\hline
PSNR 
& 32.26 & 34.01 & 30.43
& 35.39 & 36.72
& 30.82& 31.10
&35.02 & 35.66 & \textbf{36.90}
& 30.86 & 31.03 & 31.29 \\
SSIM
& 0.936 & 0.937 & 0.893
& 0.955 & \textbf{0.975}
& 0.902 & 0.912
& 0.934 & 0.933 & 0.949
& 0.892 & 0.889 & 0.897 \\
Time (s)
& 63.12 & 57.38 & 53.62
& 70.87 & 78.73
& 60.85 & 63.68
& 40.25 & 47.08 & 60.53
& \textbf{38.84} & 42.41 & 47.78 \\
\hline\hline
\multicolumn{14}{c}{\textbf{Aerial}} \\ \hline\hline
PSNR 
& 24.39 & 28.14 & 24.09
& 30.54 &33.06
& 24.36 & 24.52
& 28.25 & 30.73 & \textbf{33.68}
& 24.11 & 24.37 & 24.54 \\
SSIM
& 0.594 & 0.741 & 0.551
& 0.850 & 0.910
& 0.581 & 0.594
& 0.745 & 0.854 & \textbf{0.919}
& 0.552 & 0.582 & 0.596\\
Time (s)
& 87.37 & 62.53 & 58.48
& 70.81 & 81.27
& 67.46 & 74.90
& 40.05 & 53.00 & 59.49
& \textbf{39.56} & 48.05 & 52.54 \\
\hline
\end{tabular*}
\label{tab:video compression}
\end{table}

\subsection{Image and video completion}
For image and video completion tasks, we adopt DFT as the invertible linear transform for its best reconstruction performance.
Consider $\mathcal{X} \in \mathbb{R}^{n_1 \times n_2\times n_3}$ that contains missing entries, and $\mathrm{\Omega}$
the index set of observed elements. The tensor completion problem is formulated as 
\begin{equation}\label{eq:tensor completion1}
      \min_{\mathcal{X}}  \big\|\boldsymbol{P}_{\mathrm{\Omega}} {(\mathcal{X})} -\boldsymbol{P}_{\mathrm{\Omega}} {(\mathcal{M})}  \big\|_F^2,\ \mathrm{s.t.} \  \mathrm{Rank}(\mathcal{X}) = \mathbf{R},
\end{equation}
where $\mathcal{M}$ is the observed  tensor and $\mathbf{R}$ controls truncation ranks across terms and slices. Following \cite{AHMADIASL2023109121}, we introduce an auxiliary variable \(\mathcal{Z}\) to reformulate the  problem \eqref{eq:tensor completion1} as
\begin{equation}\label{eq:tensor completion2}
      \min_{\mathcal{X}} \ \big\| \mathcal{X} -\mathcal{Z} \big\|_F^2,\ \mathrm{s.t.}  \ \mathrm{Rank}(\mathcal{X}) = \mathbf{R}, \ \boldsymbol{P}_{\mathrm{\Omega}}{(\mathcal{Z})}=\boldsymbol{P}_{\mathrm{\Omega}}{(\mathcal{M})}.
\end{equation}
The solution is approximated iteratively by
\begin{equation}\label{eq: lr_appro}
        \mathcal{X}^{(n)} \leftarrow \pi_r (\mathcal{Z}^{(n)}),
\end{equation}
\begin{equation}\label{eq:solve_Z}
        \mathcal{Z}^{(n+1)} \leftarrow   \mathcal{M}_{\mathrm{\Omega}} + \mathcal{X}^{(n)}_{\mathrm{\Omega}^c},
\end{equation}
where \(\pi_r(\cdot)\) returns the rank-$r$ low-rank approximation.
We initialize with the incomplete tensor \(\mathcal{X}^{(0)}\) and iterate until convergence. Note that \(\mathcal{M}_{\mathrm{\Omega}}\) equals \(\mathcal{X}^{(0)}\) and need not be recomputed each iteration.
To further boost performance, we apply smoothing to \(\mathcal{Z}^{(n+1)}\) before the low-rank approximation step. Since low-rank approximation dominates computational cost for large-scale data and many iterations, we replace the TMSTP-SVD  with TMRSTP-SVD in each iterative, fixing the term parameter \(k=2\) throughout all completion experiments. This randomized substitution yields nearly equivalent reconstruction quality with substantially lower per-iteration overhead.
\begin{figure}[!ht]
\centering
\renewcommand{\arraystretch}{0.5} 
\setlength{\tabcolsep}{0.3pt}       
\begin{tabular}{@{}cccccc@{}}
\scriptsize\makecell{ Original image } & \scriptsize\makecell{Observed} & \scriptsize\makecell{TT-SVD}&\scriptsize\makecell{TSTP-SVD}&\scriptsize\makecell[c]{TMSTP-SVD\\[-4pt] (k=2)} & \scriptsize\makecell[c]{TMRSTP-SVD\\[-4pt] (k=2)} \\
\includegraphics[width=0.77in]{image/lake/fig_output/lake.jpg} &
\includegraphics[width=0.77in]{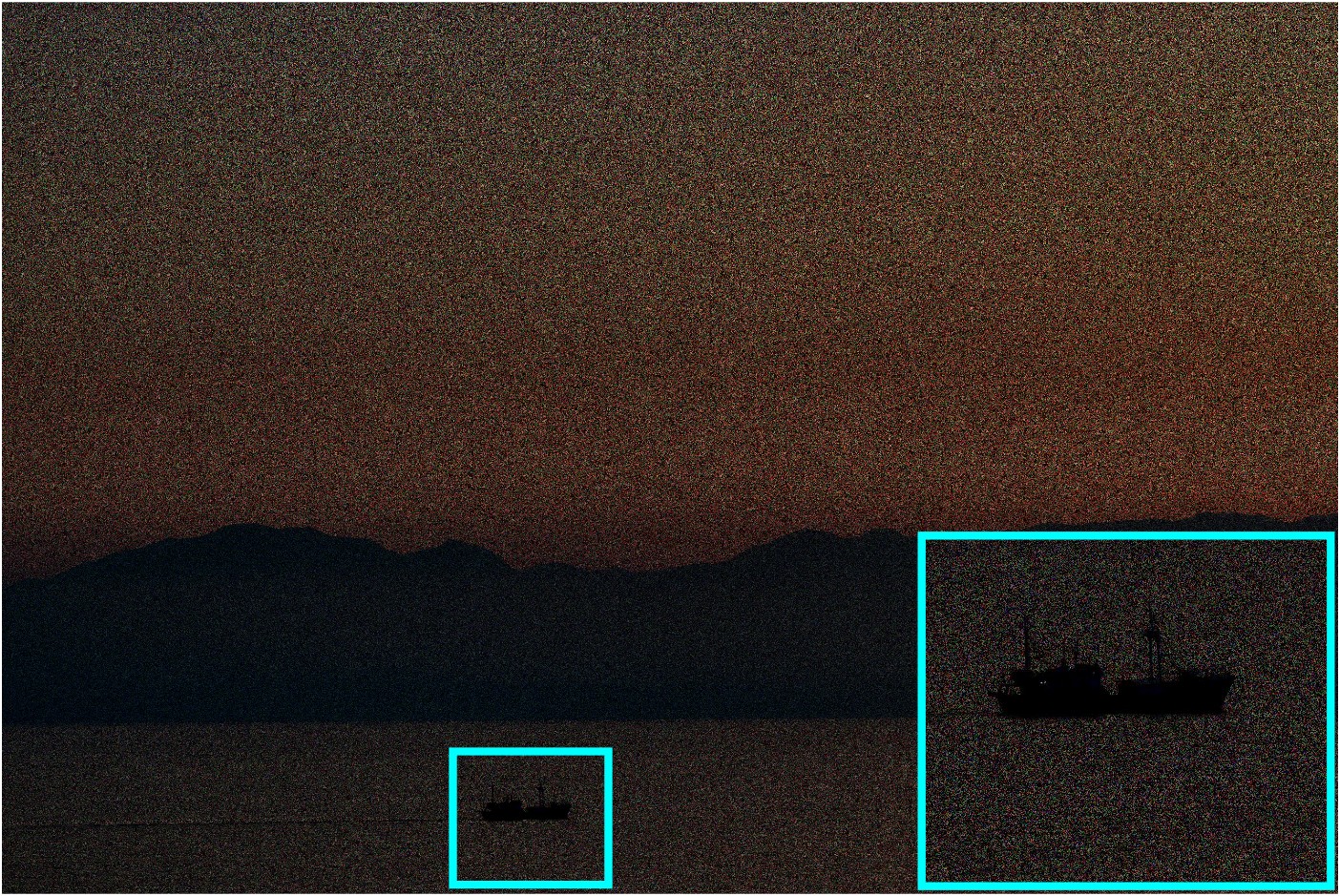} &
\includegraphics[width=0.77in]{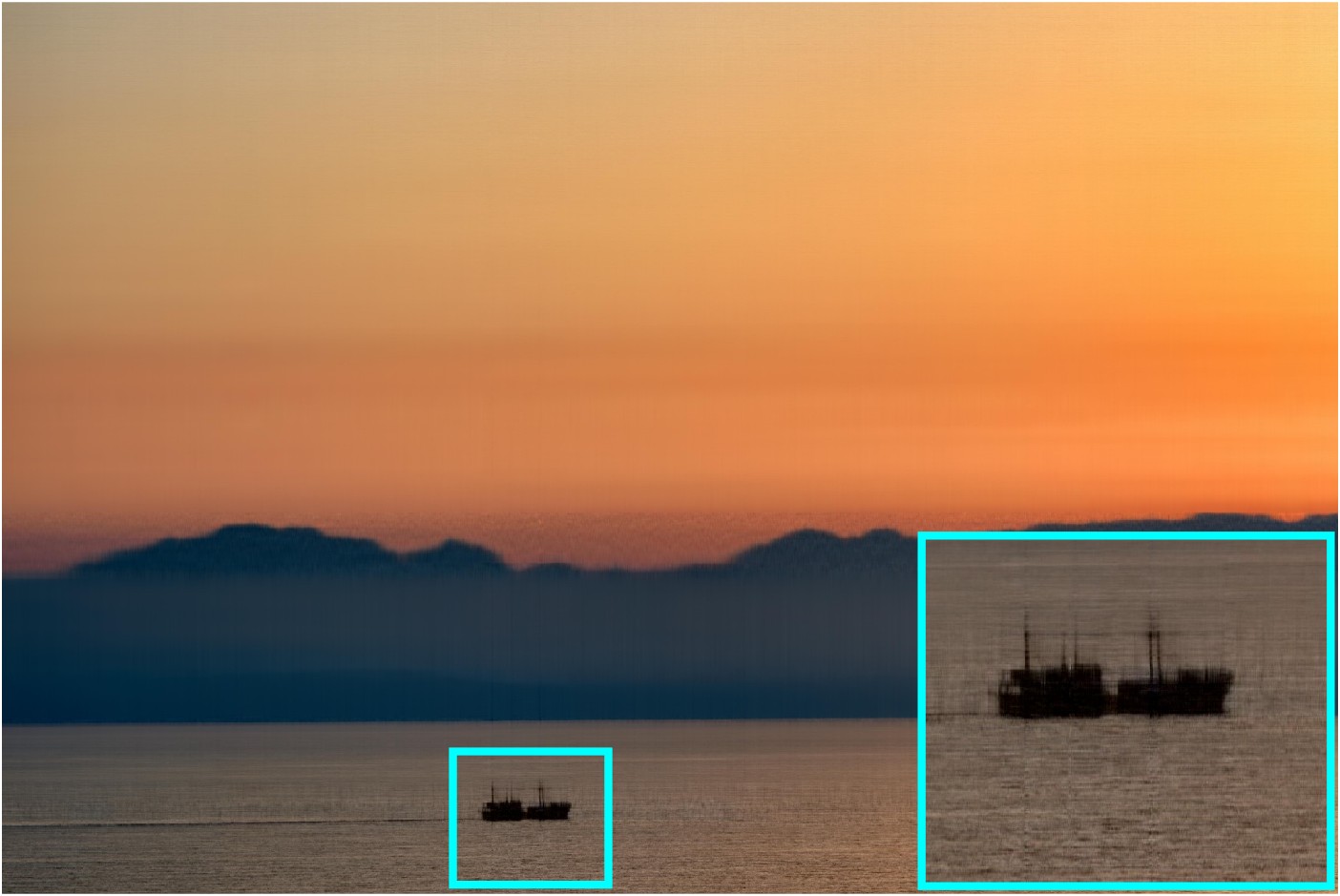} &
\includegraphics[width=0.77in]{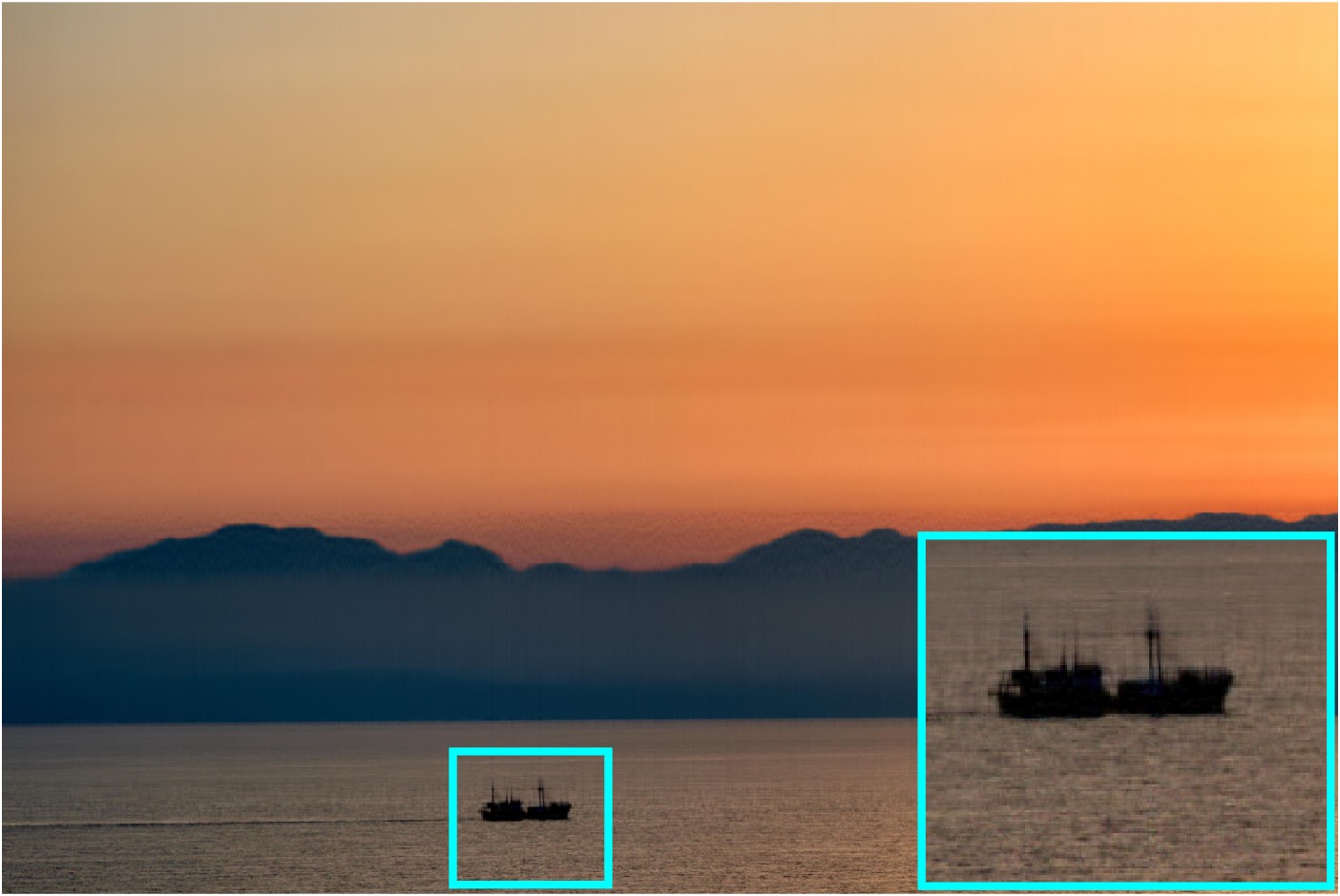} &
\includegraphics[width=0.77in]{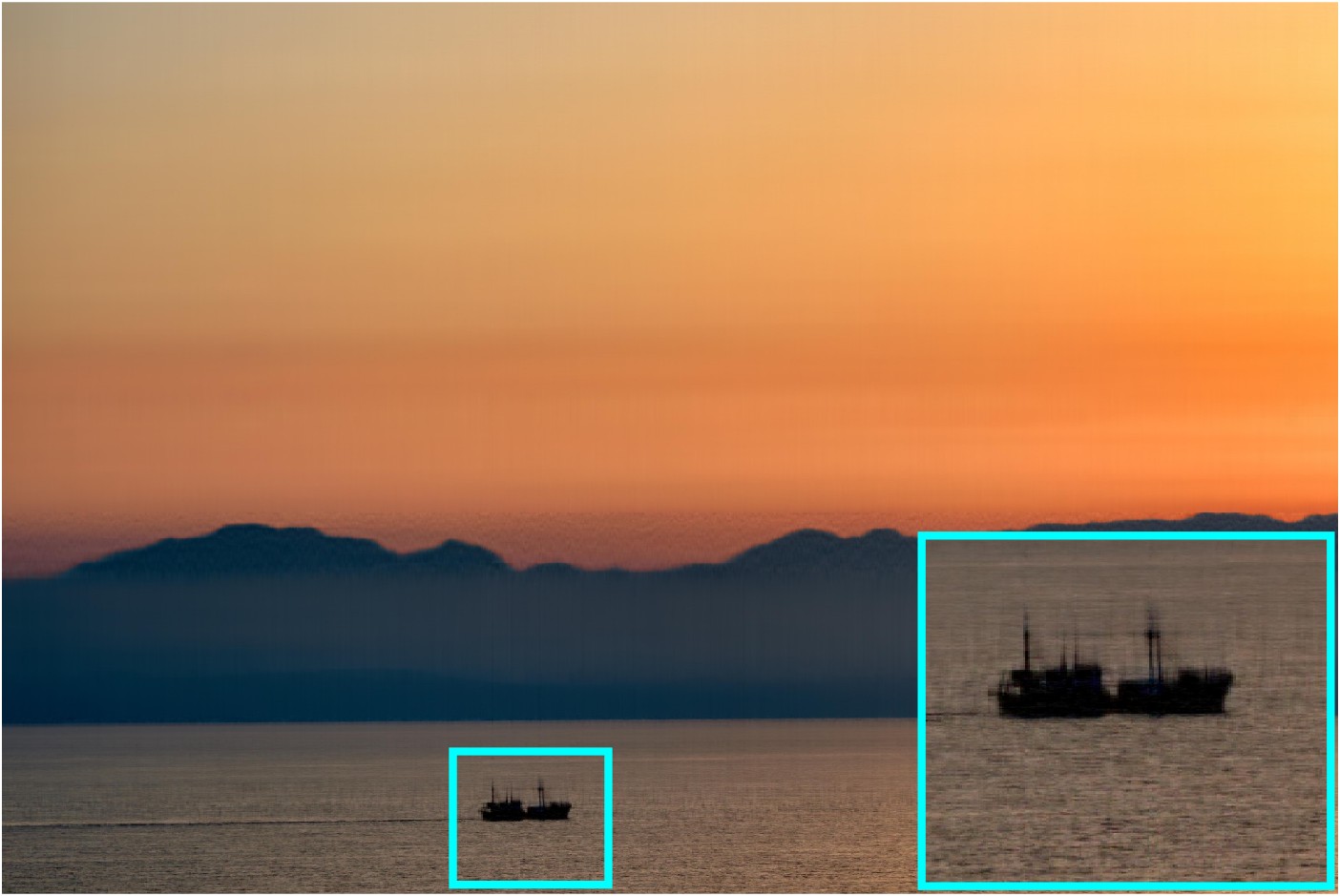} &
\includegraphics[width=0.77in]{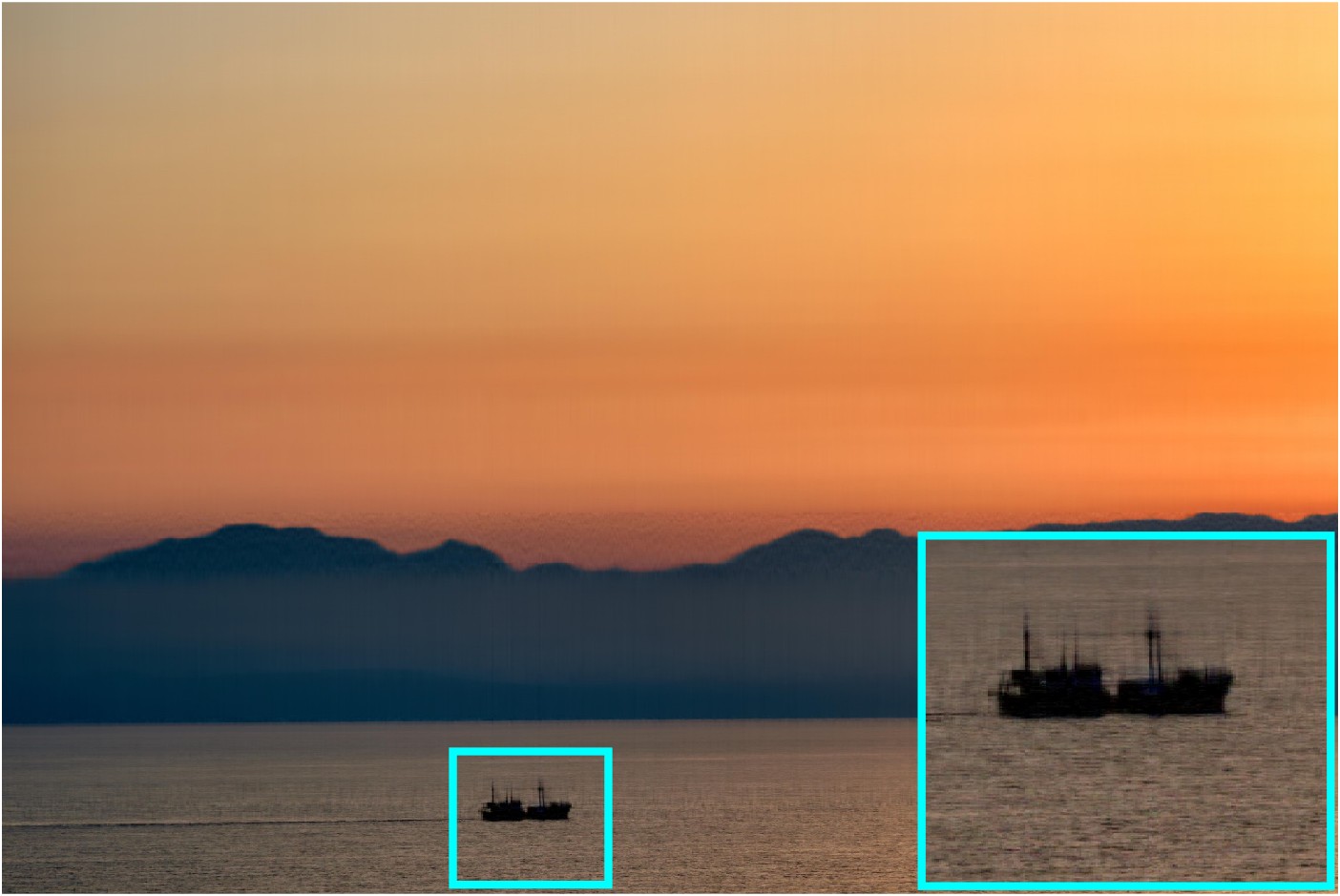} \\

& &\scriptsize PSNR:33.05 & \scriptsize PSNR:33.82&\scriptsize PSNR:34.35 & \scriptsize PSNR:34.45  \\
& &\scriptsize Time:331.92s& \scriptsize Time:308.55s&\scriptsize Time:345.20s& \scriptsize Time:246.06s \\

\includegraphics[width=0.77in]{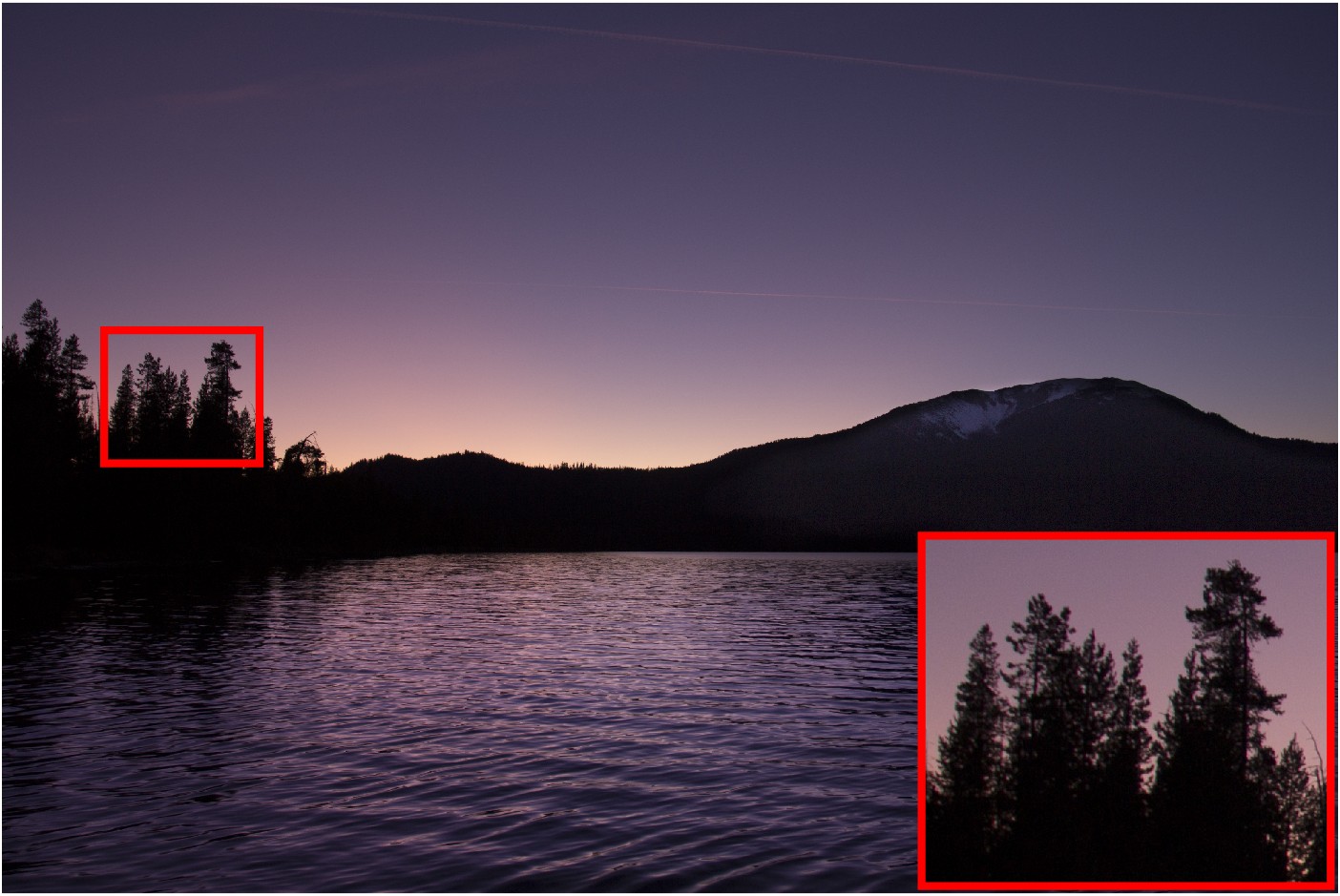} &
\includegraphics[width=0.77in]{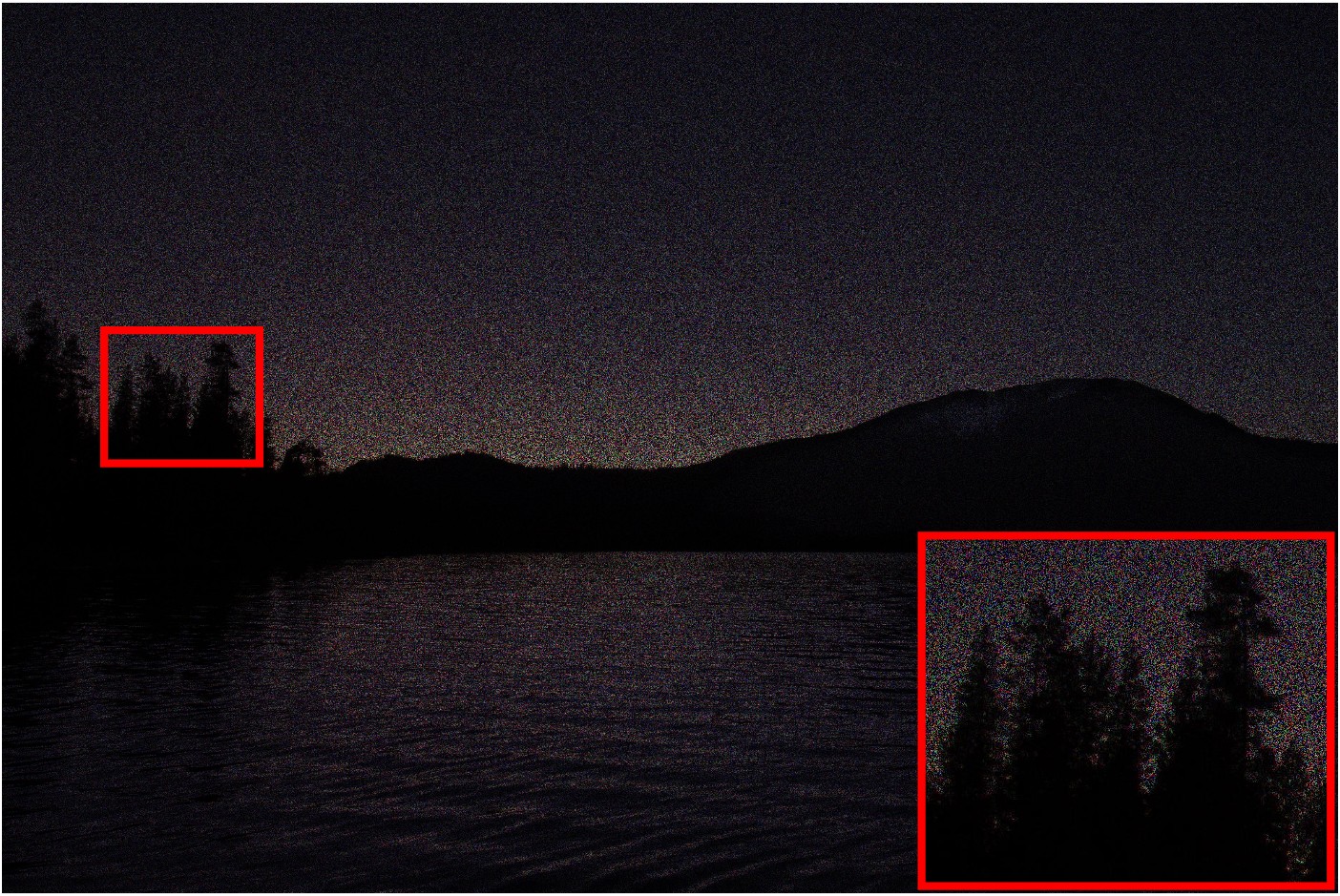} &
\includegraphics[width=0.77in]{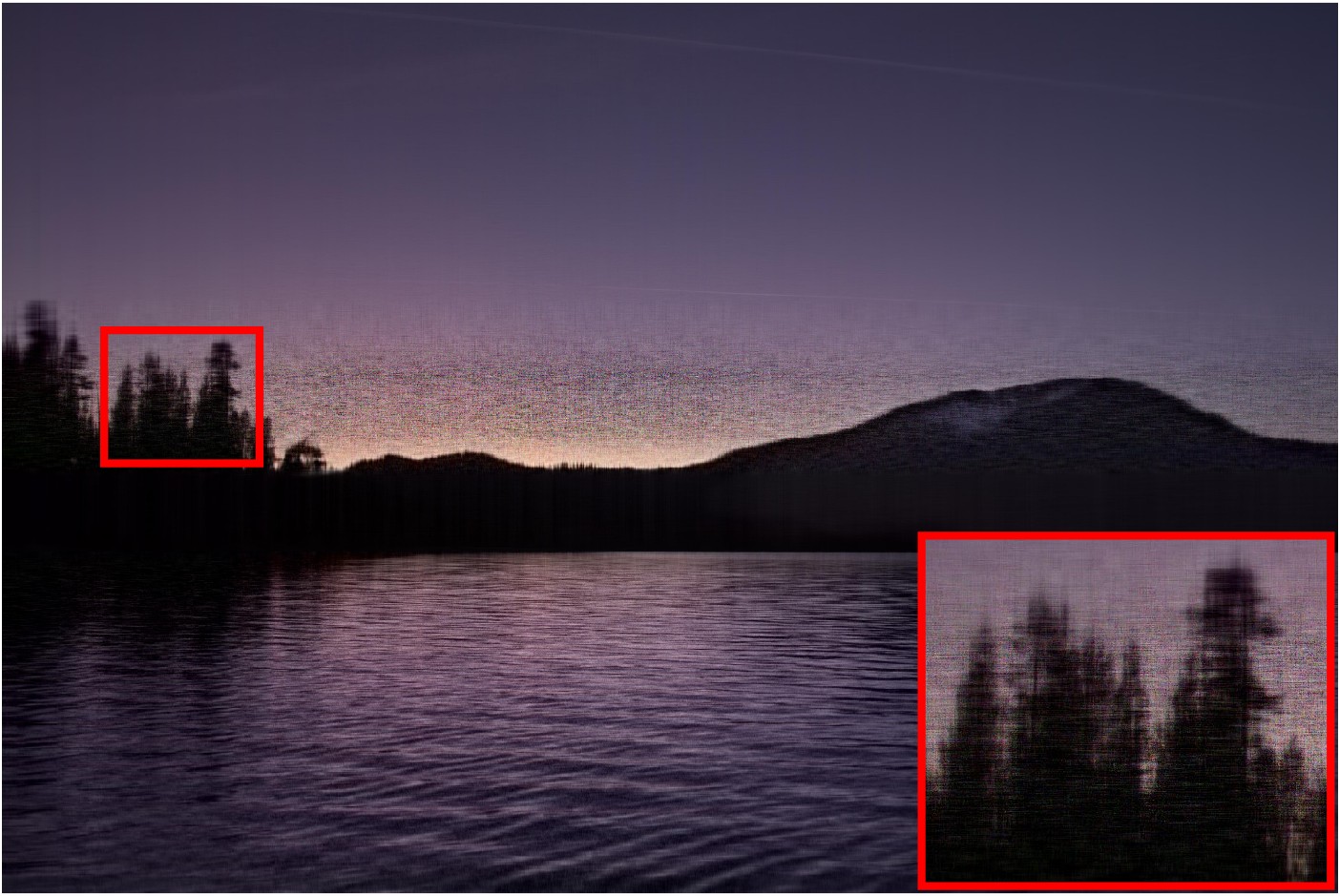} &
\includegraphics[width=0.77in]{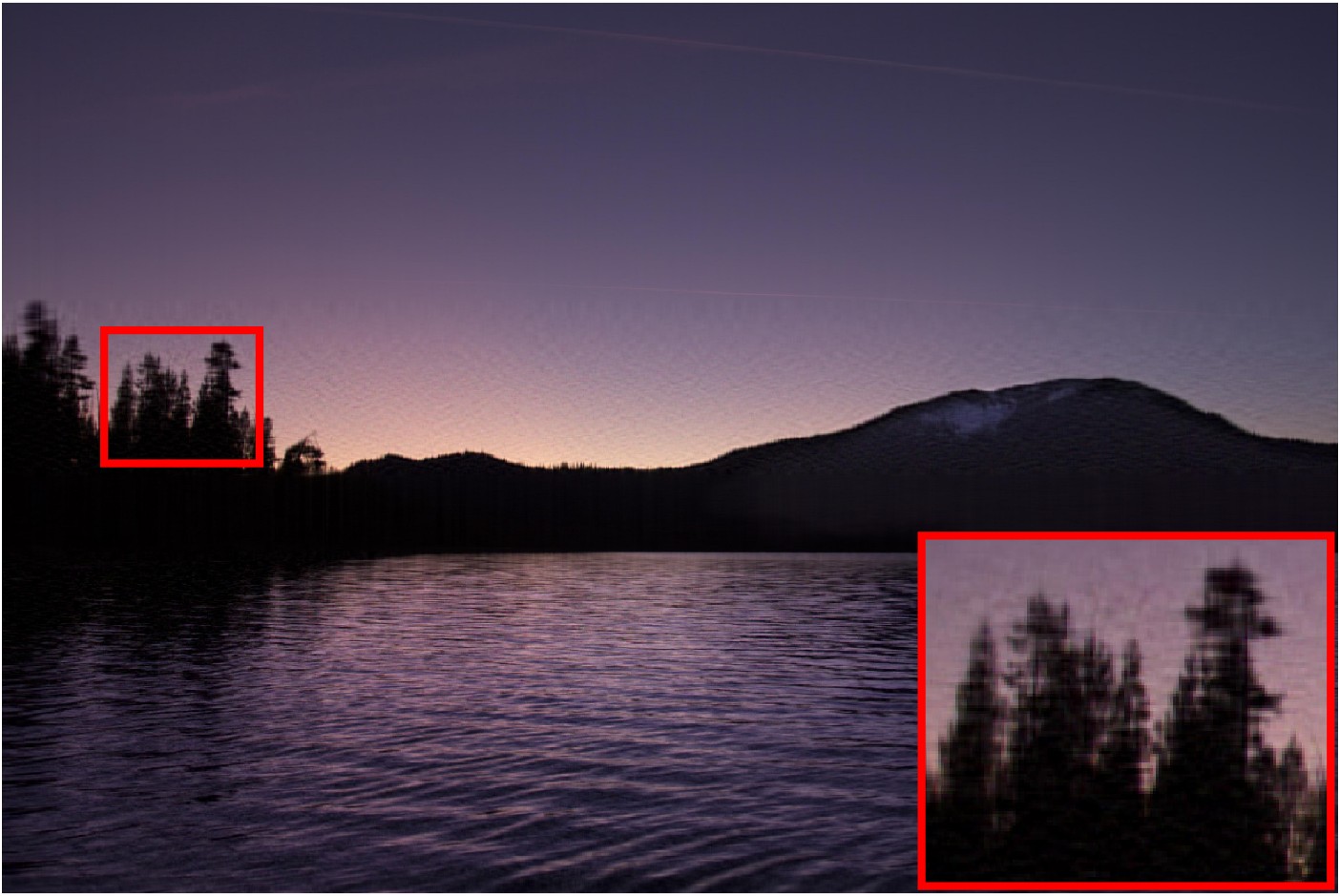} &
\includegraphics[width=0.77in]{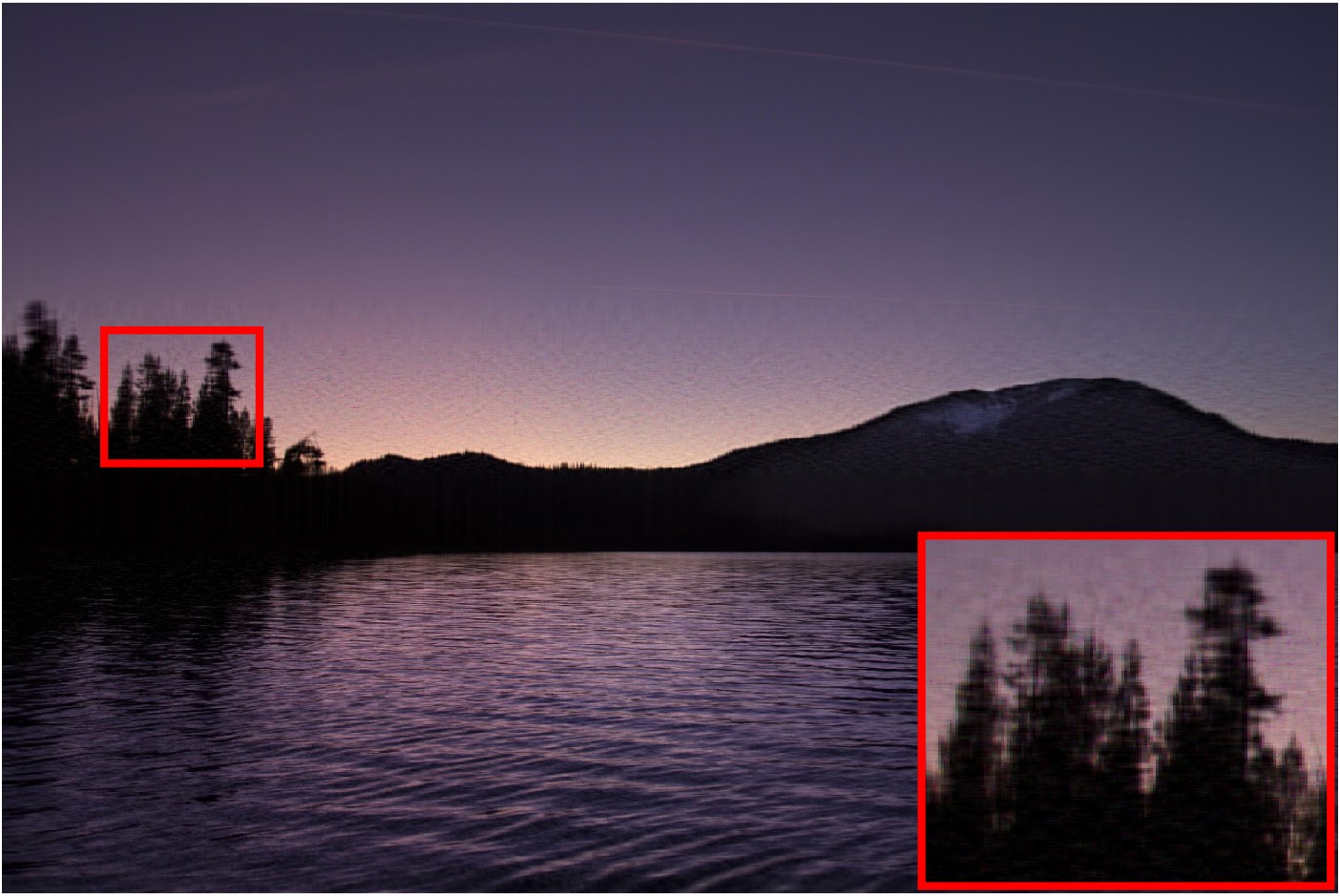} &
\includegraphics[width=0.77in]{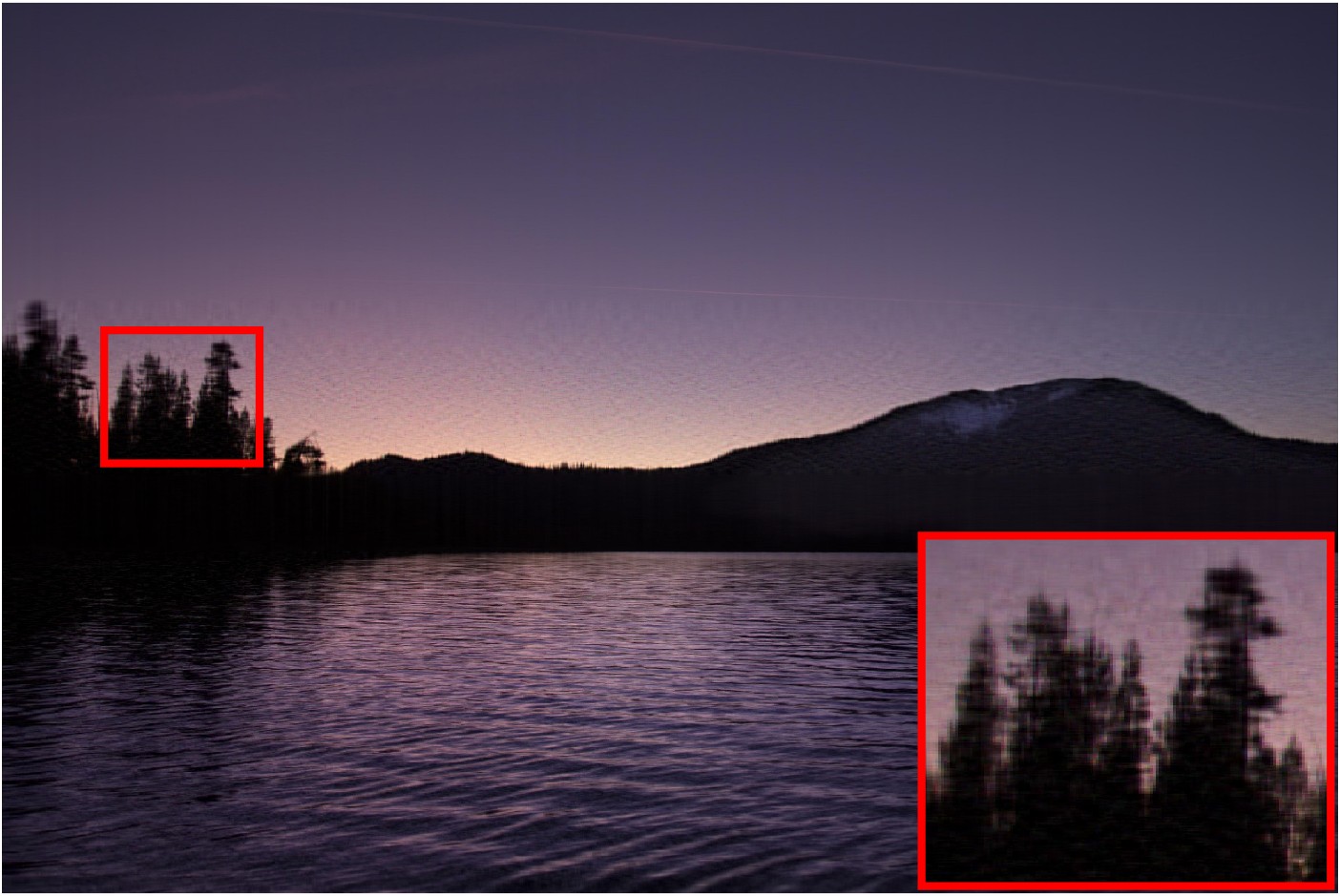} \\

& &\scriptsize PSNR:25.06 & \scriptsize PSNR:30.56&\scriptsize PSNR:31.79 & \scriptsize PSNR:32.06 \\
& &\scriptsize Time:387.35s & \scriptsize Time:233.50s&\scriptsize Time:264.79s & \scriptsize Time:198.24s \\

\includegraphics[width=0.77in]{image/road/fig_output/road.jpg} &
\includegraphics[width=0.77in]{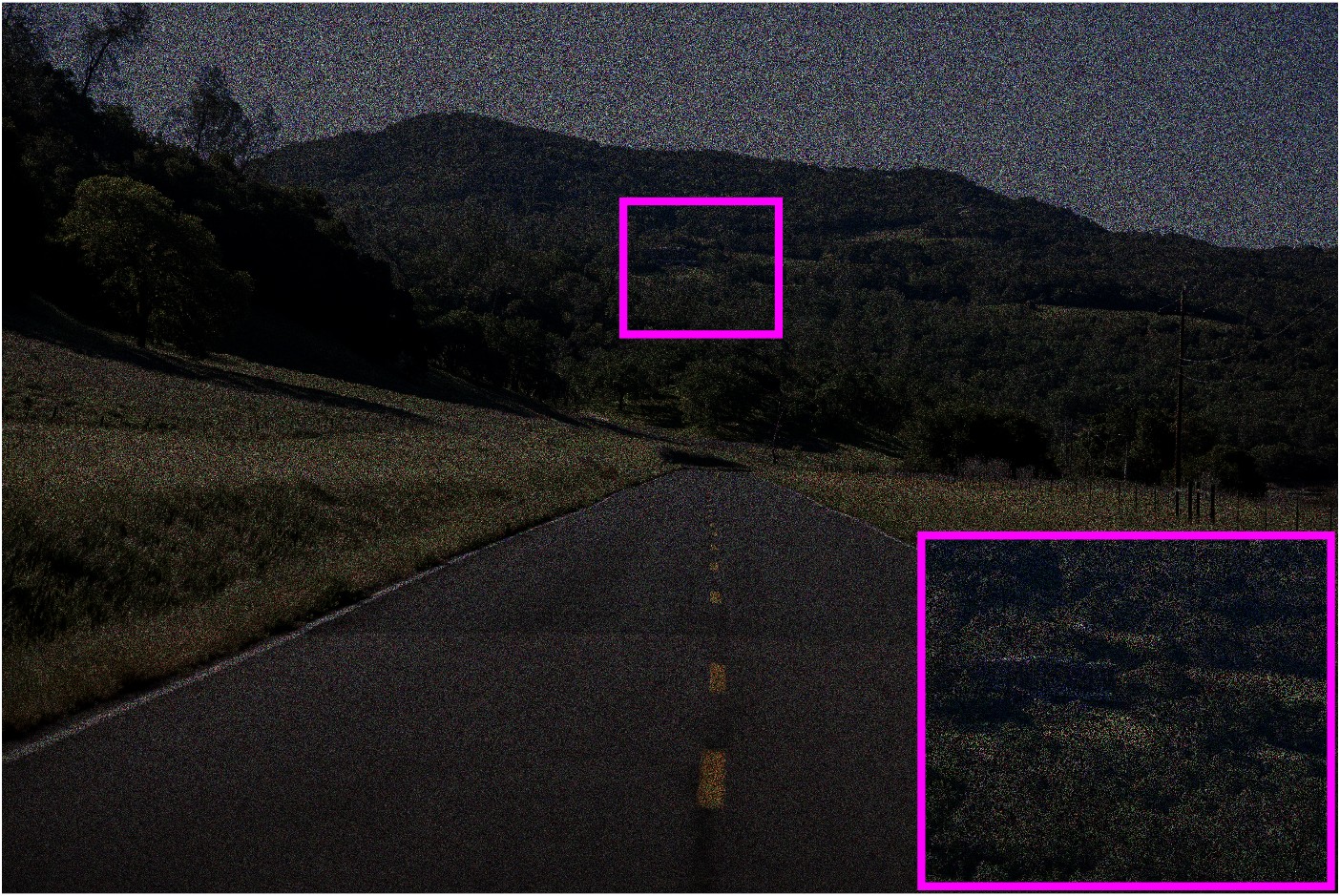} &
\includegraphics[width=0.77in]{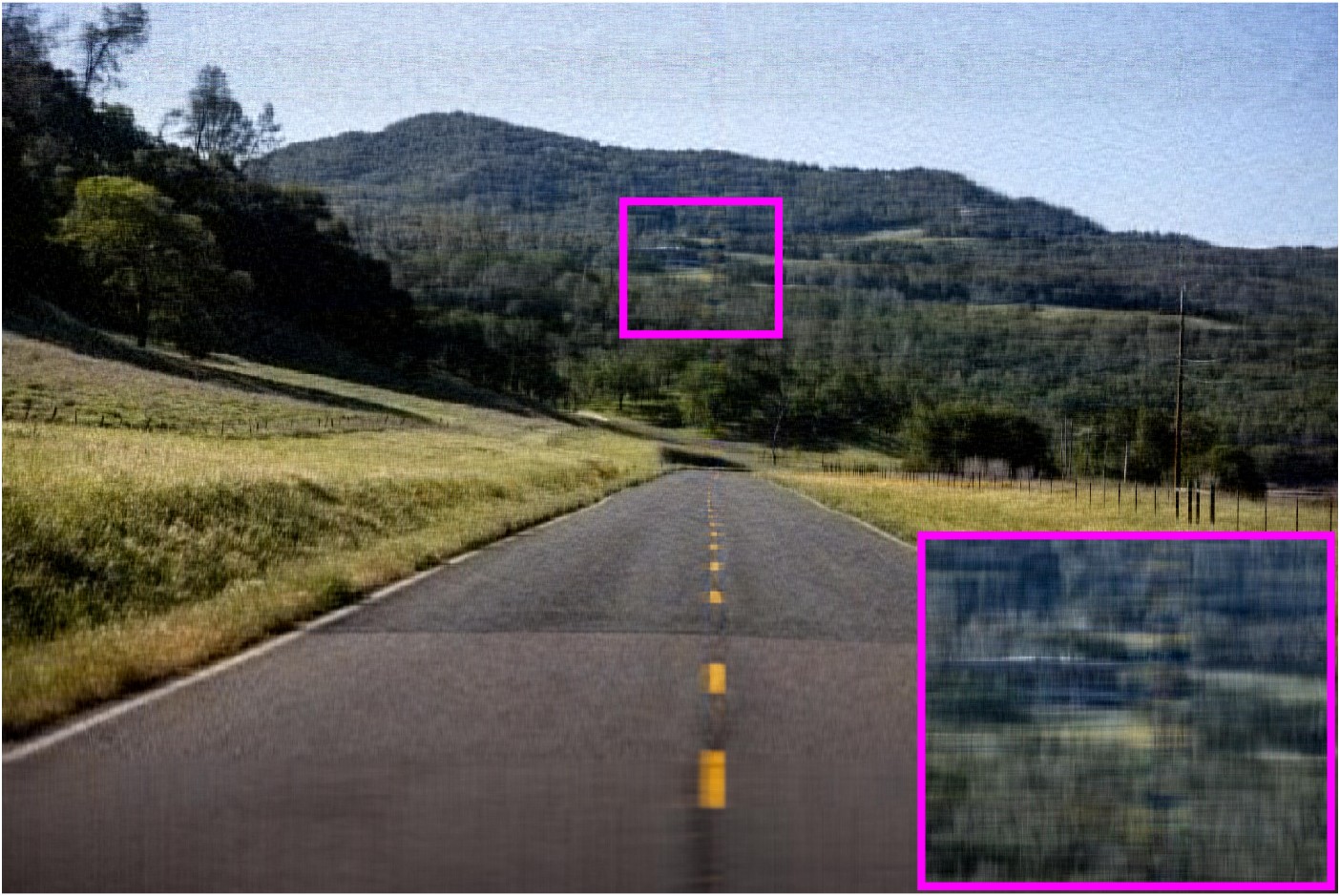} &
\includegraphics[width=0.77in]{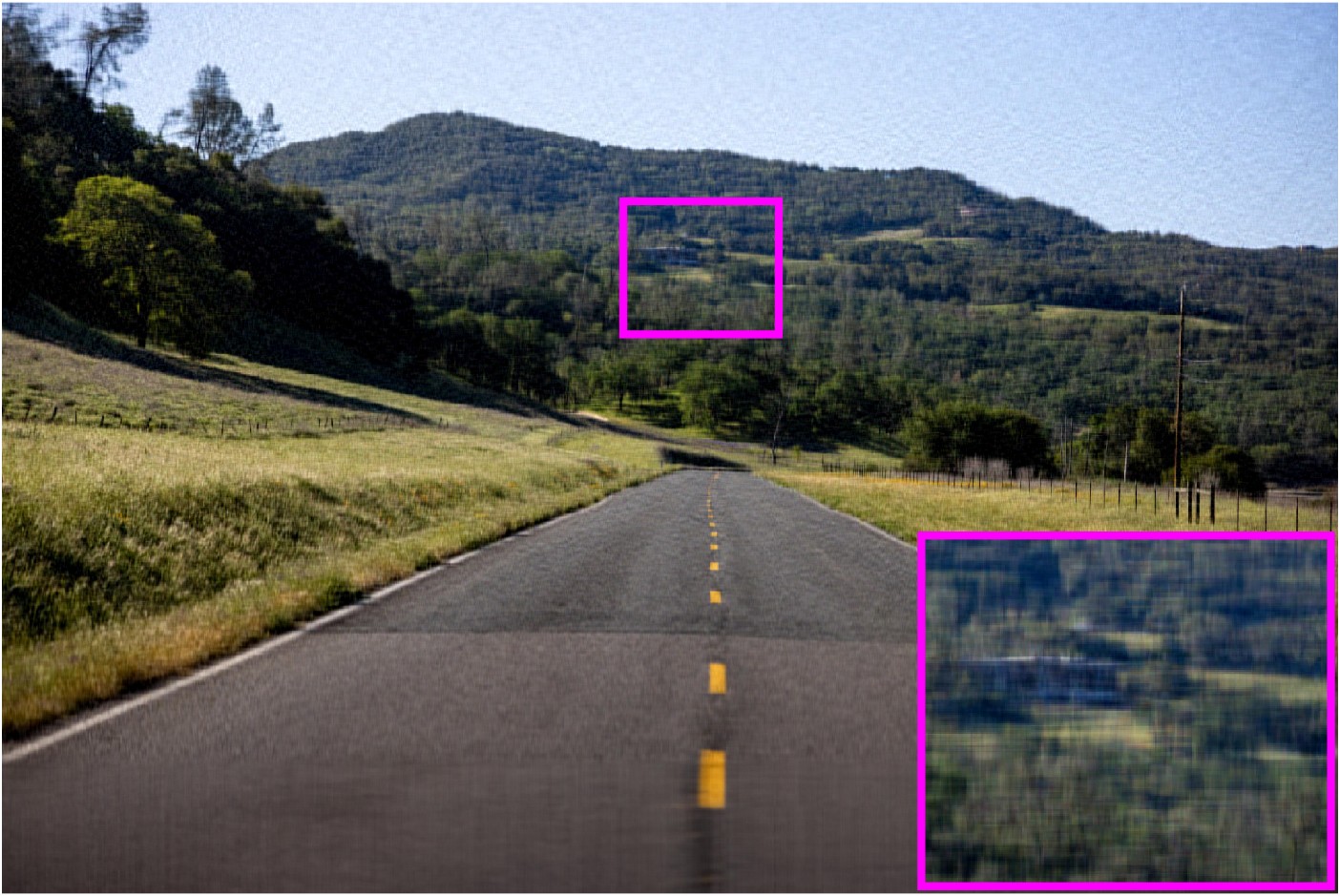} &
\includegraphics[width=0.77in]{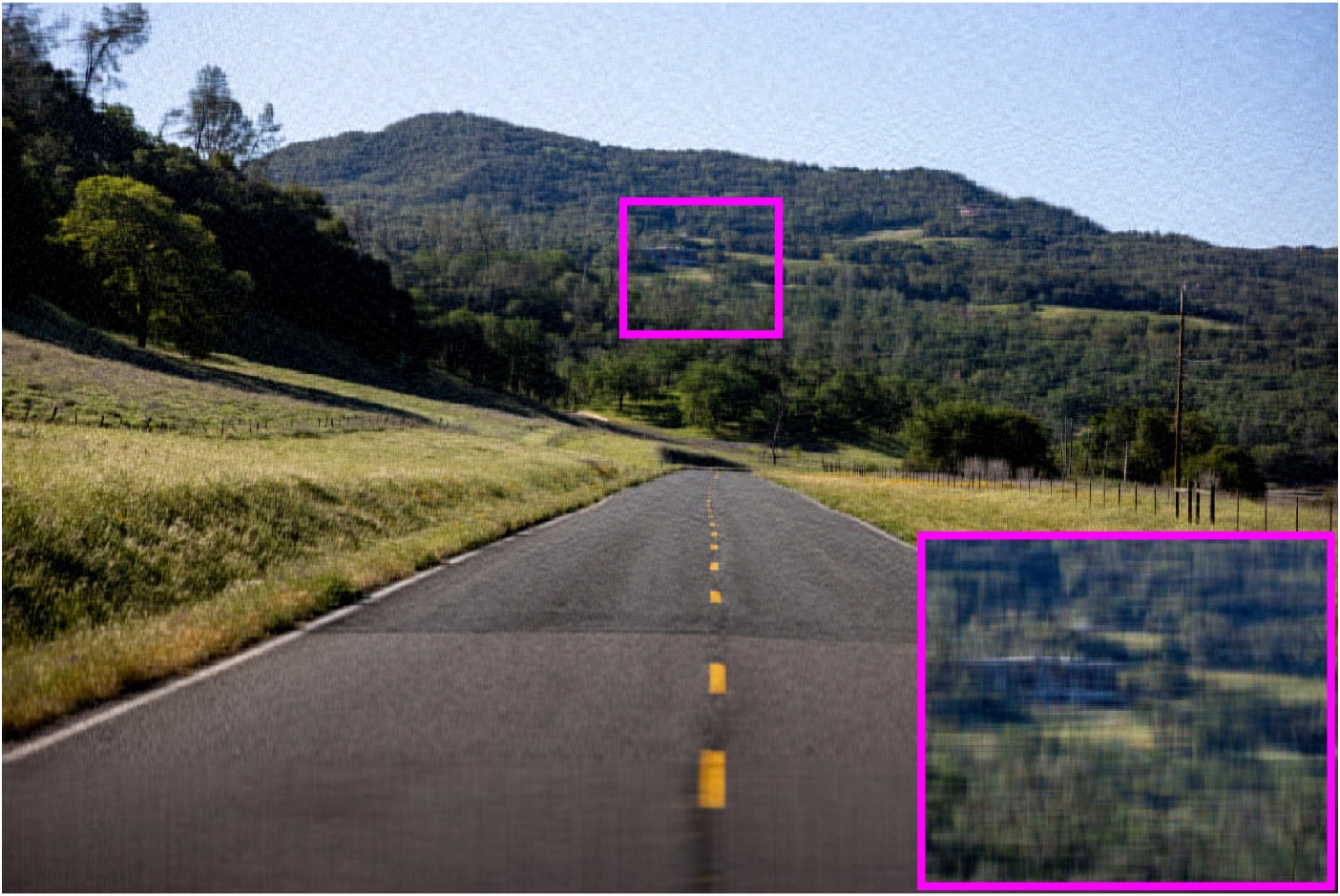} &
\includegraphics[width=0.77in]{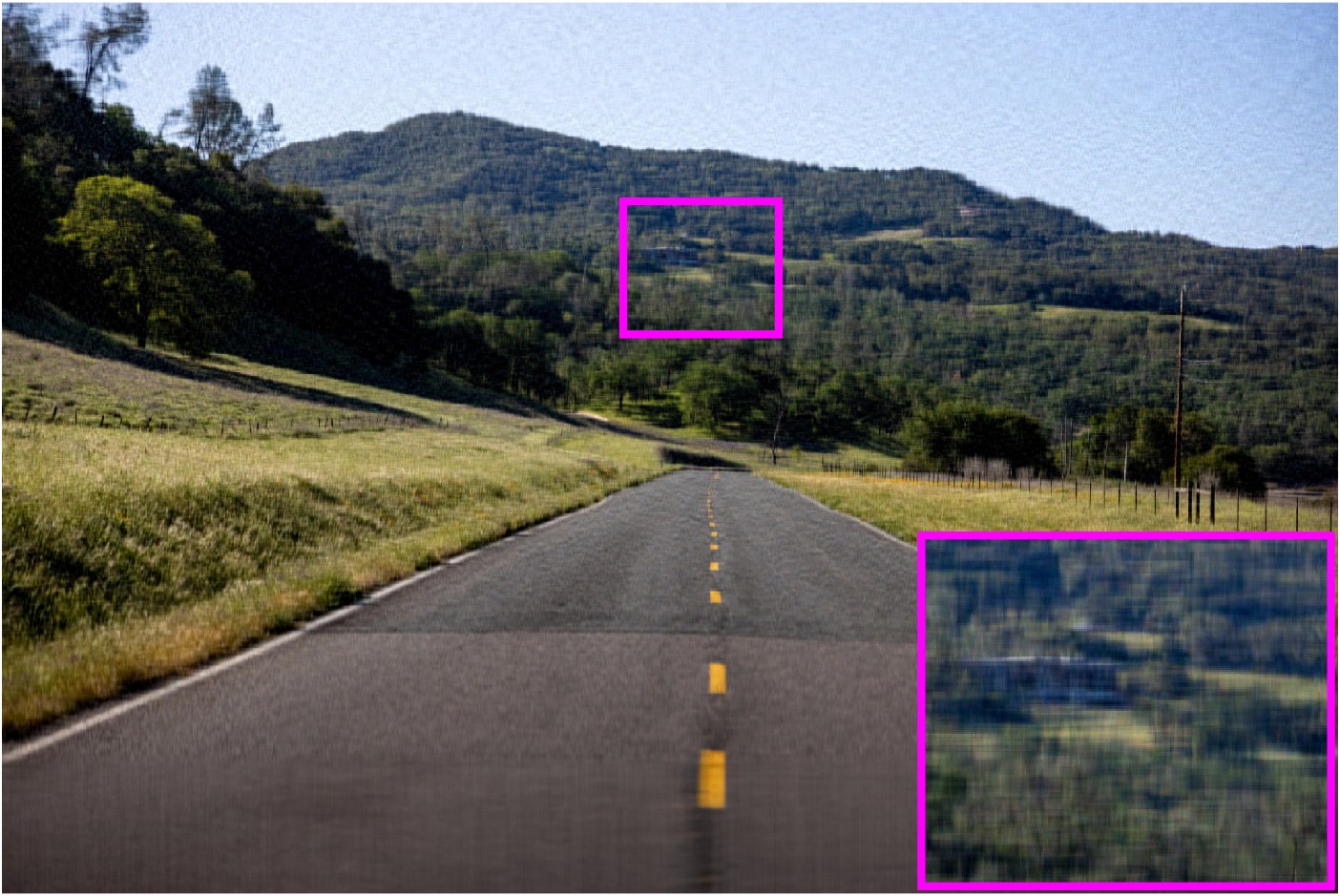} \\

& &\scriptsize PSNR:24.08 & \scriptsize PSNR:25.85 &\scriptsize PSNR:25.89 & \scriptsize PSNR:25.93  \\
& &\scriptsize Time:1028.57s & \scriptsize Time:607.52s&\scriptsize Time:687.87s & \scriptsize Time:533.98s \\

\includegraphics[width=0.77in]{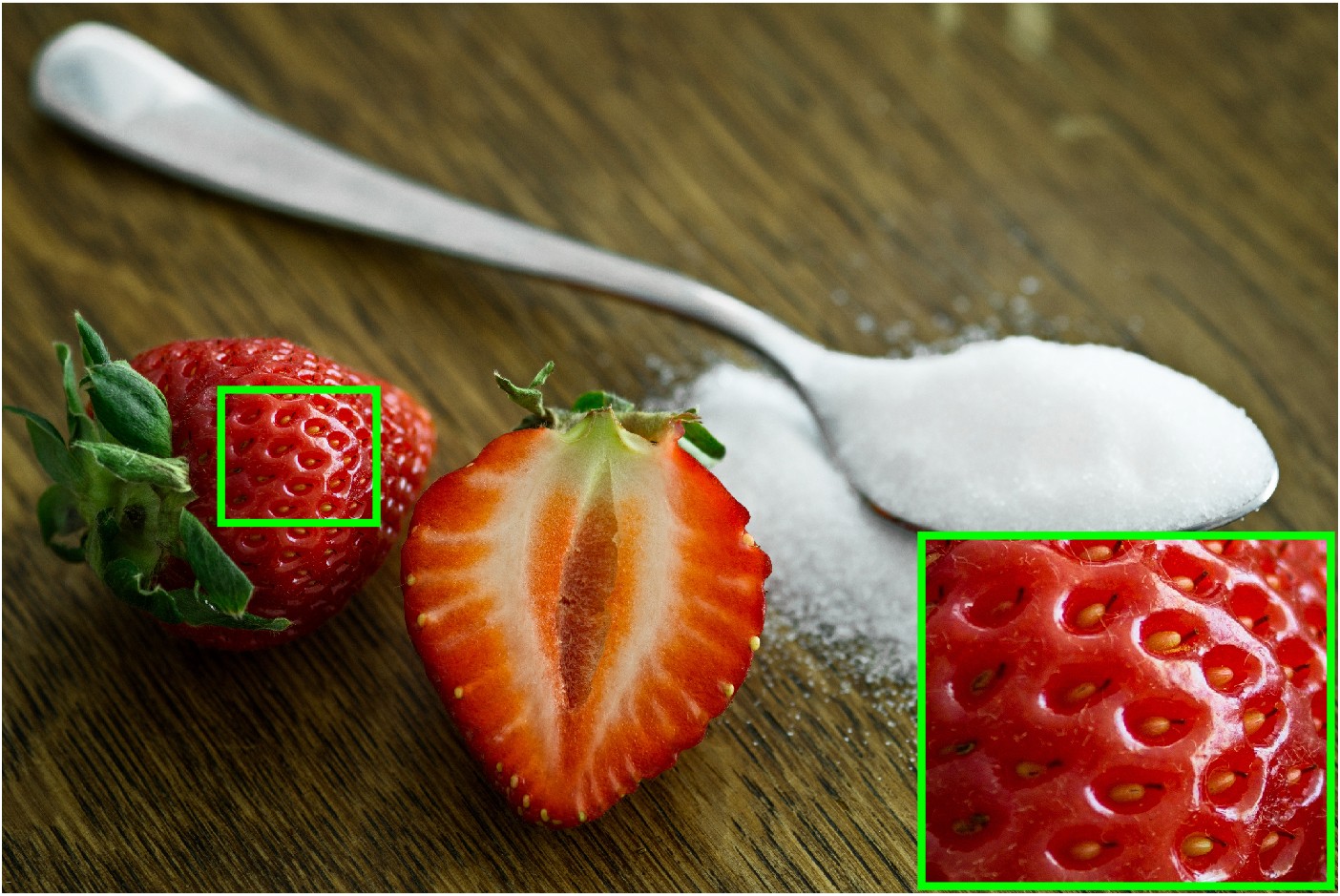} &
\includegraphics[width=0.77in]{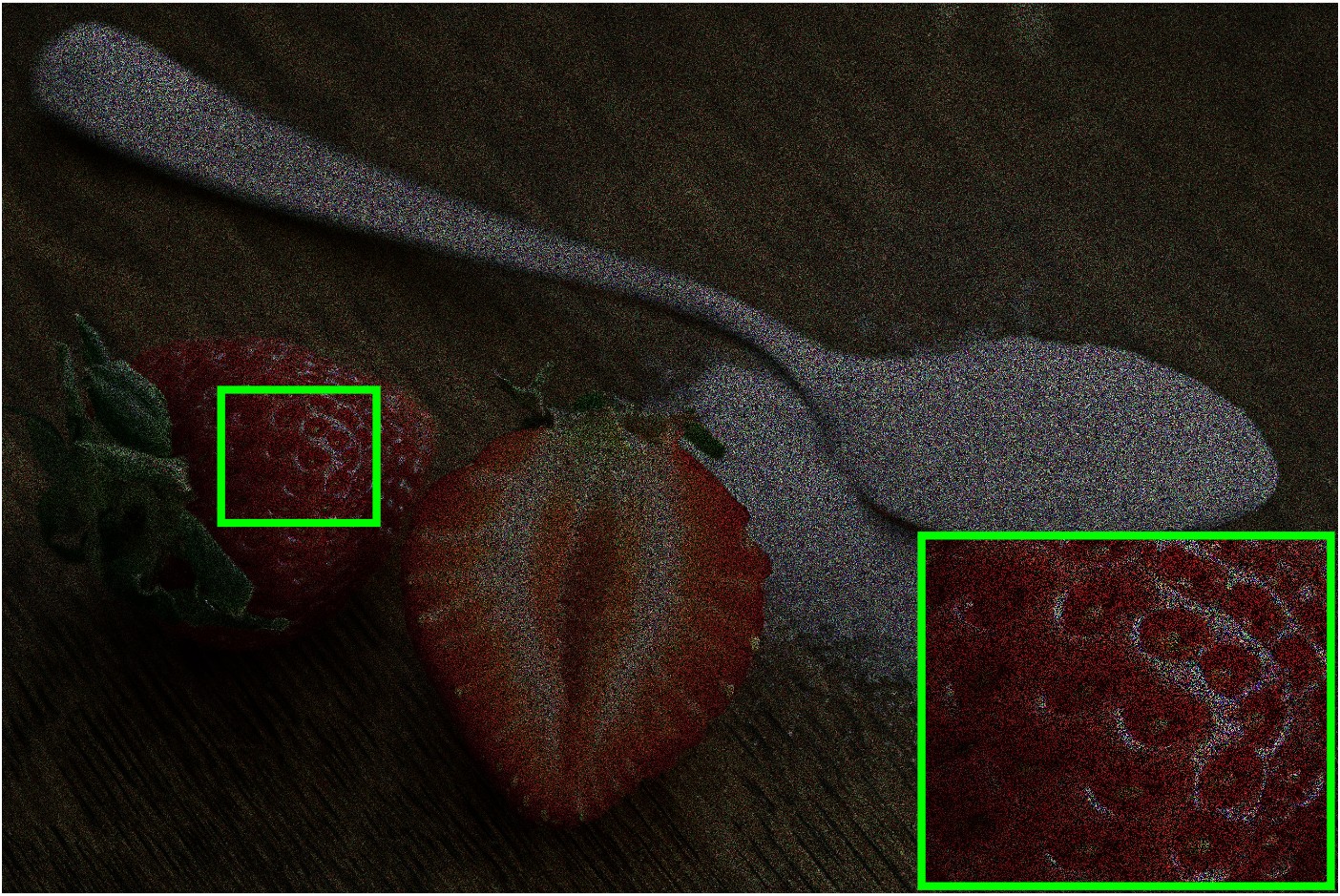} &
\includegraphics[width=0.77in]{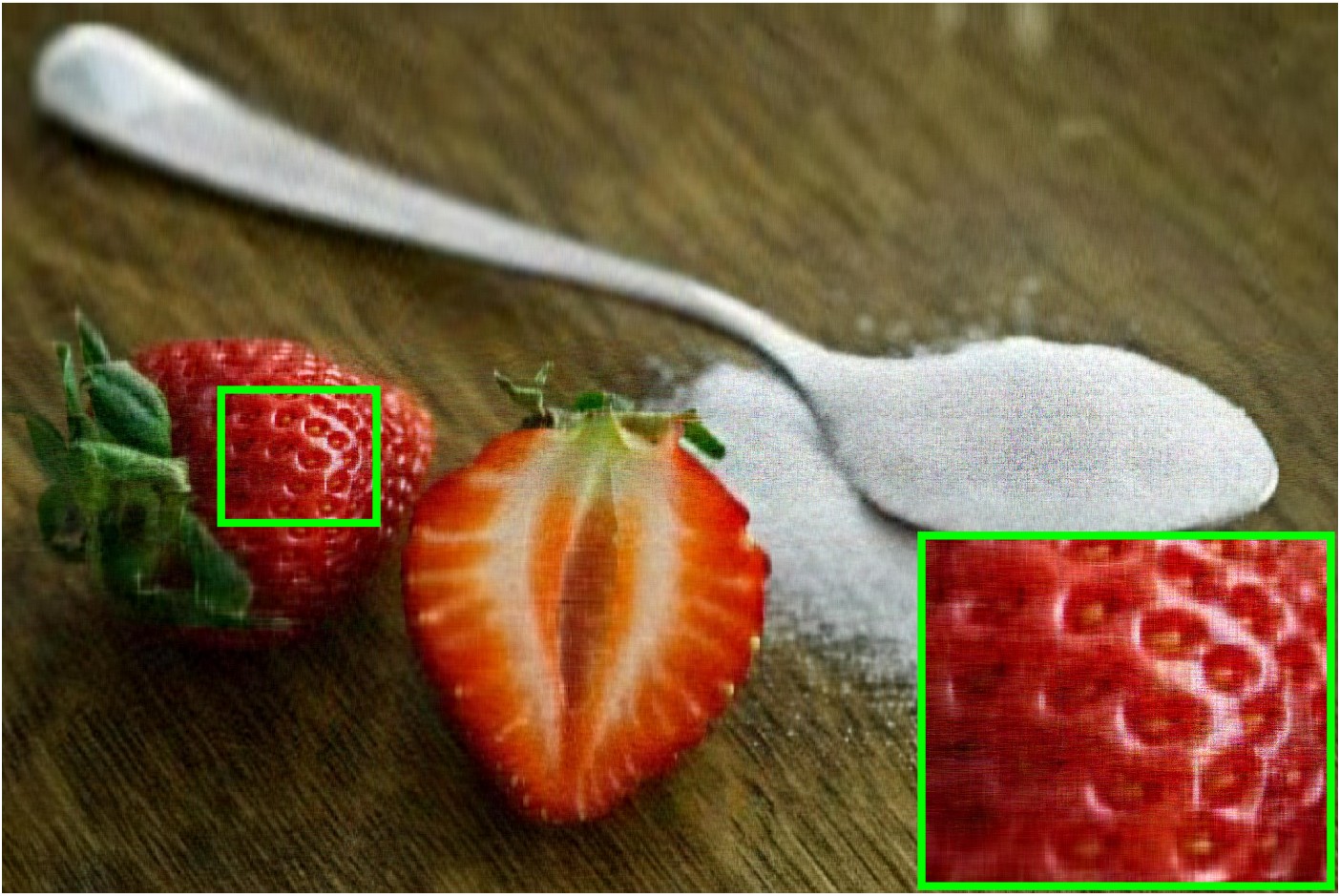} &
\includegraphics[width=0.77in]{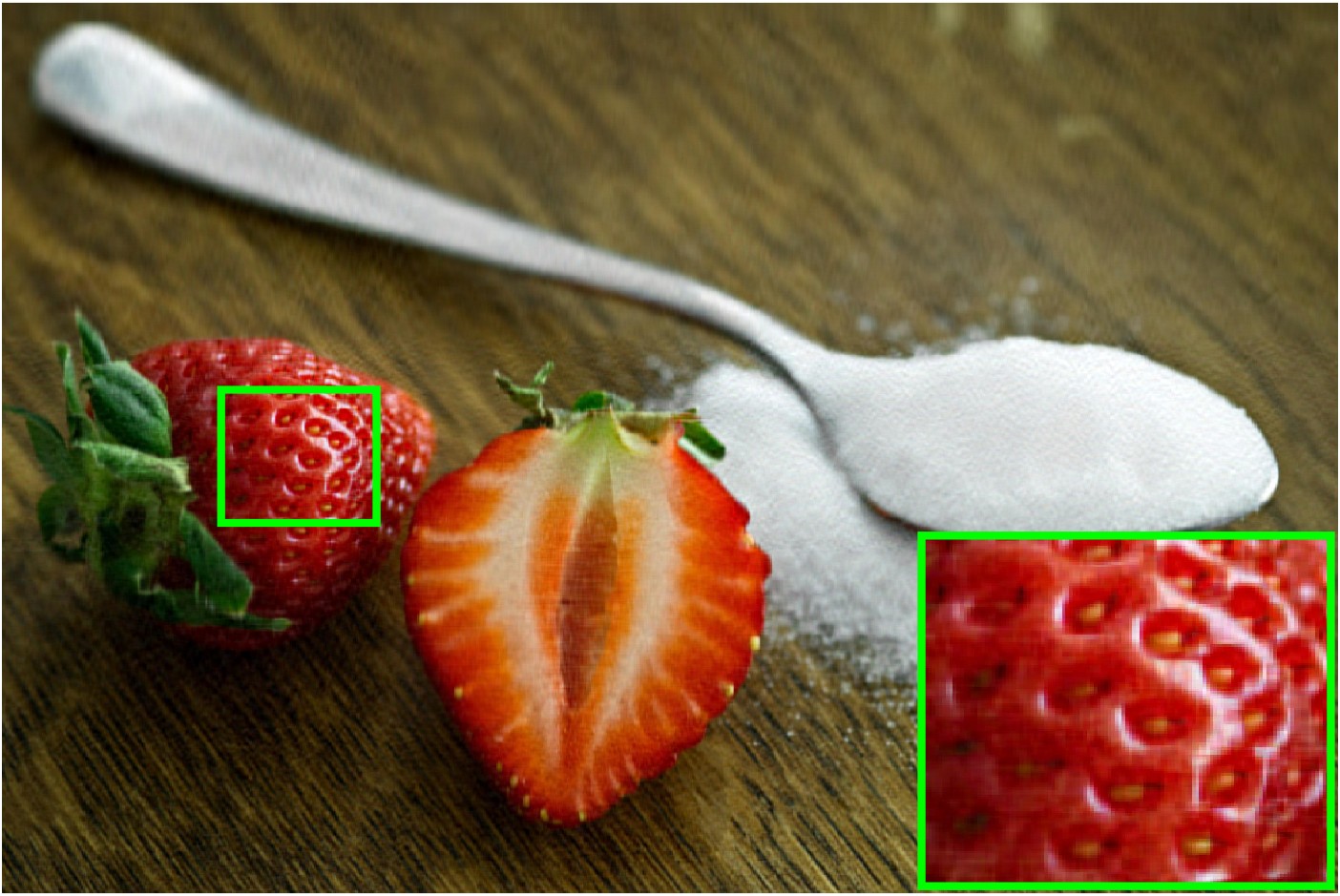} &
\includegraphics[width=0.77in]{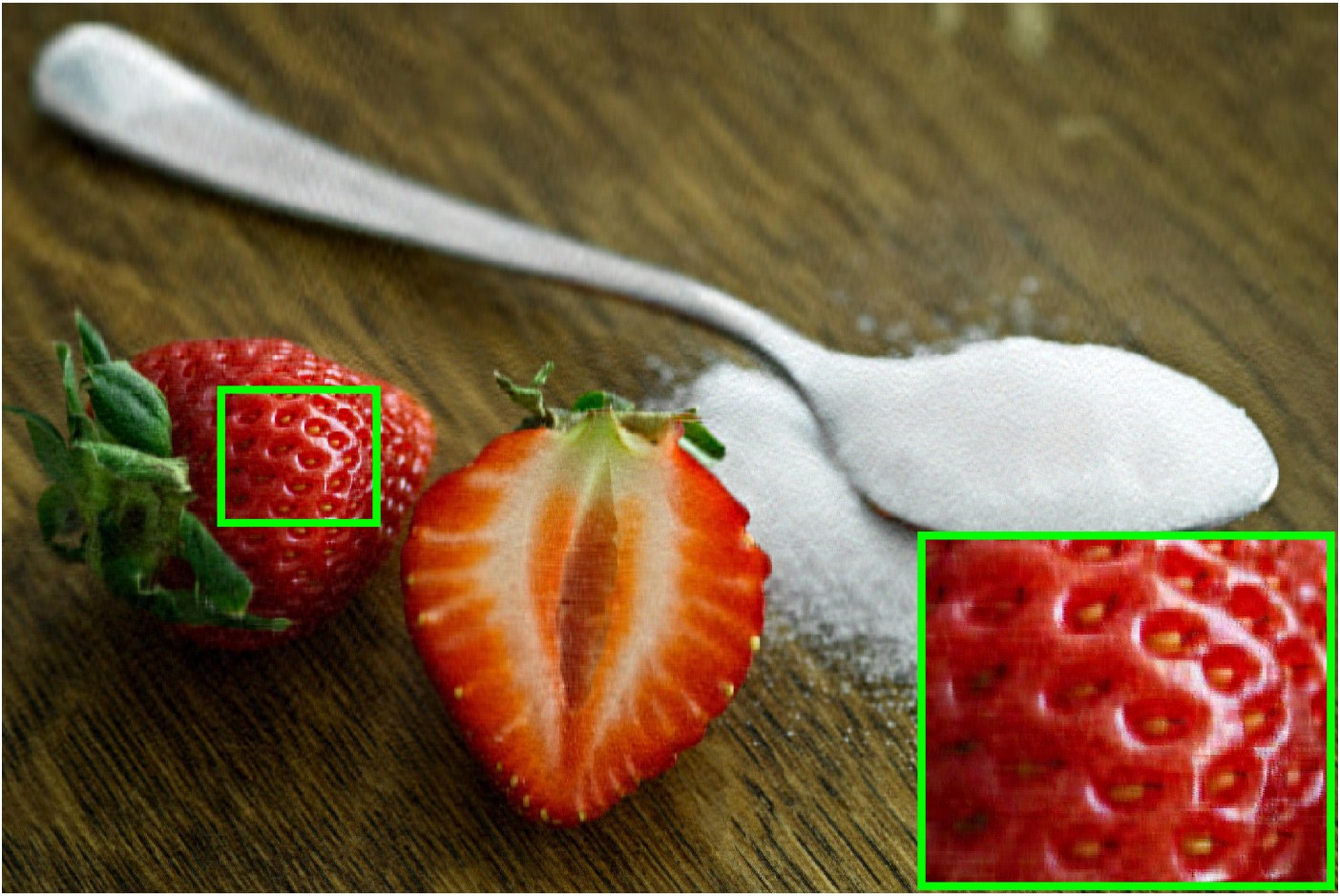} &
\includegraphics[width=0.77in]{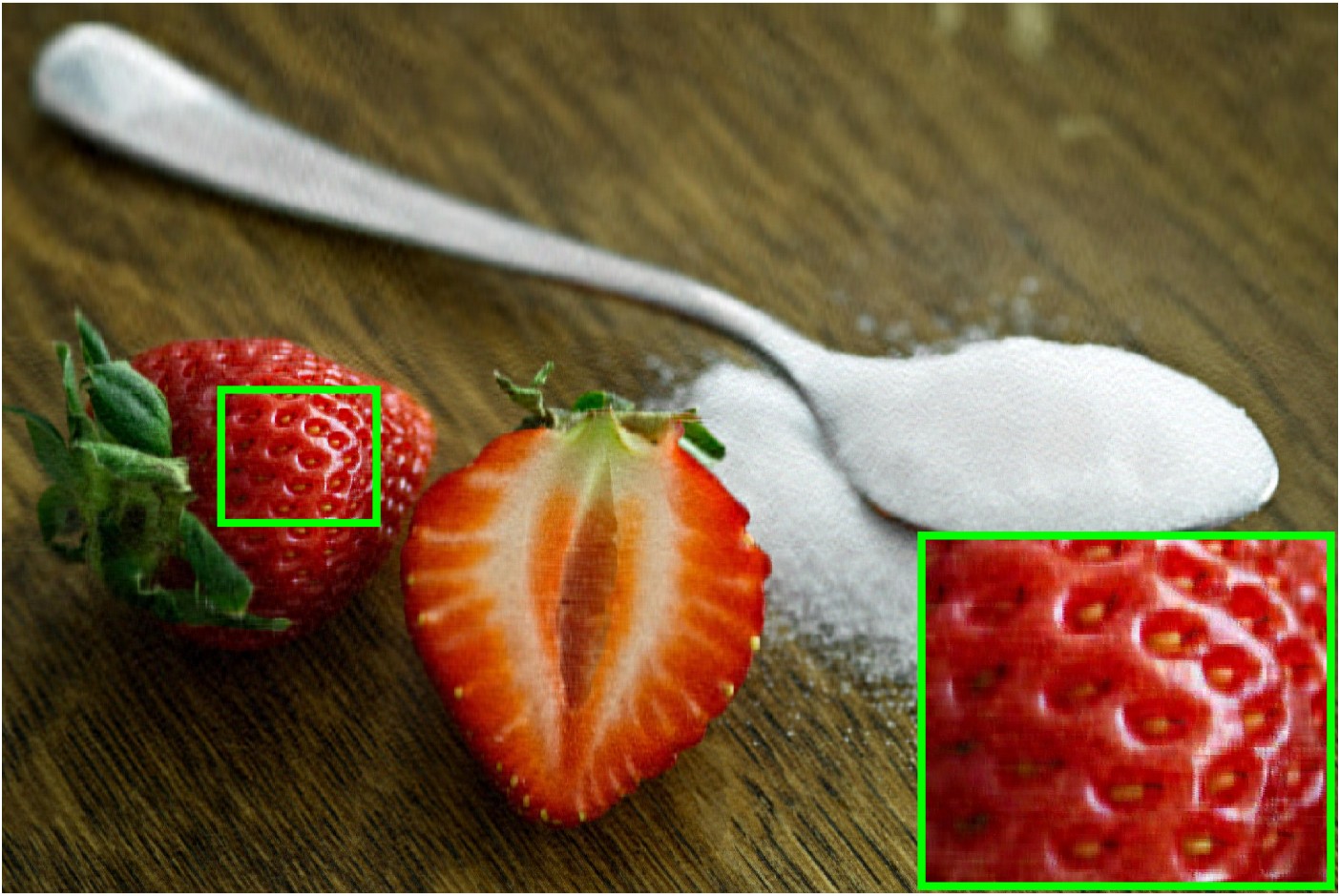} \\

& &\scriptsize PSNR:25.18 & \scriptsize PSNR:29.78&\scriptsize PSNR:29.94 & \scriptsize PSNR:30.05  \\
& &\scriptsize Time:605.71s & \scriptsize Time:276.99s&\scriptsize Time:336.44s & \scriptsize Time:239.84s \\

\end{tabular}

\caption{Visual reconstruction examples, PSNR and runtime comparisons of competing baselines and our proposed approaches on four representative test images for image completion.}
\label{fig:images completion}
\end{figure}

\begin{figure}[!ht]
  \centering
\includegraphics[width=\textwidth]{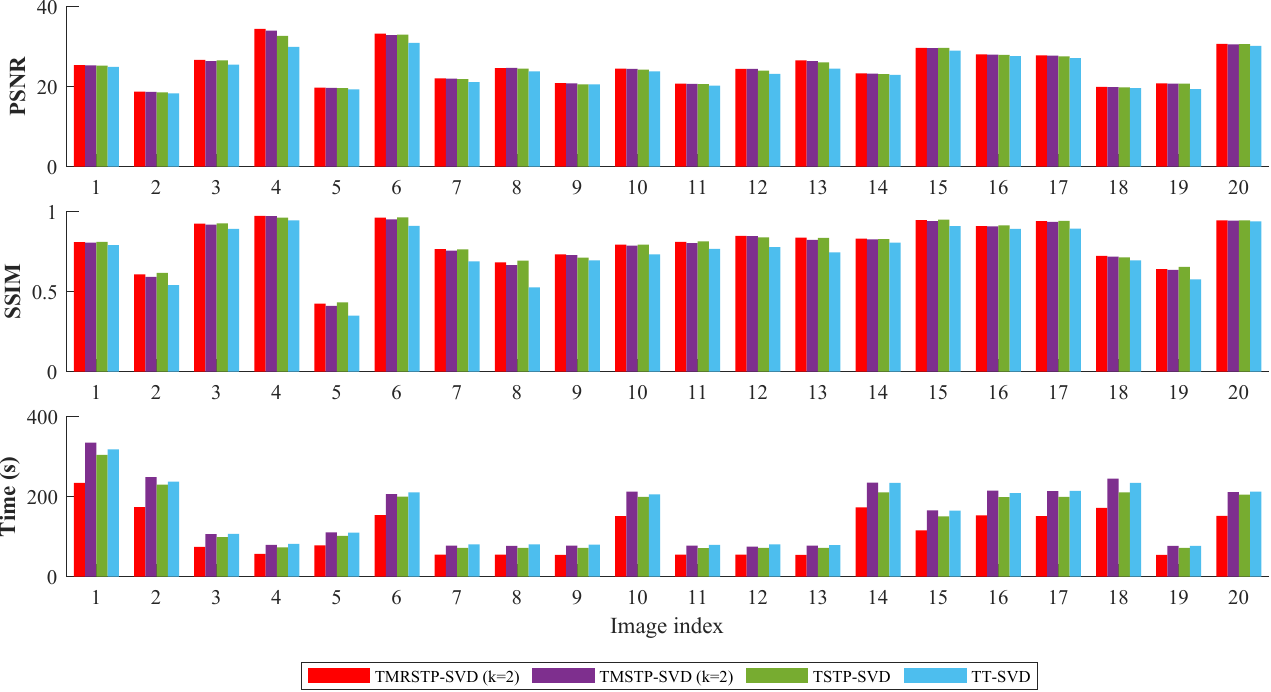}
\caption{Per-image PSNR, SSIM and runtime comparisons of competing baselines and our proposed approaches over twenty test images for image completion.}
\label{fig:image20_complete}  
\end{figure}

We evaluate image completion on 20 test images (four selected for visual comparison), with 70\% of pixels randomly removed. The oversampling parameter is set to $3$, and all other settings follow Subsection \ref{subsec:image}.
Fig.~\ref{fig:images completion} presents visual reconstruction results along with PSNR and runtime metrics obtained by competing baselines and our proposed approaches on these four representative images. Visually, our approaches recover richer textures and finer local structural details. Baseline methods, in contrast, tend to leave visible artifacts in the recovered regions. In terms of quantitative metrics, our approaches yield higher PSNR values. The randomized scheme maintains competitive reconstruction quality. It greatly reduces computational burden for iterative completion procedures.
Fig.~\ref{fig:image20_complete} summarizes per-image PSNR, SSIM and runtime comparisons across all twenty test images. The overall results further confirm the superiority of our framework over competing baselines in both reconstruction fidelity and computational efficiency. It renders our approach well-suited for large-scale image completion tasks requiring iterative optimization.

\begin{figure}[!ht]  
  \centering

  \includegraphics[width=0.49\linewidth]{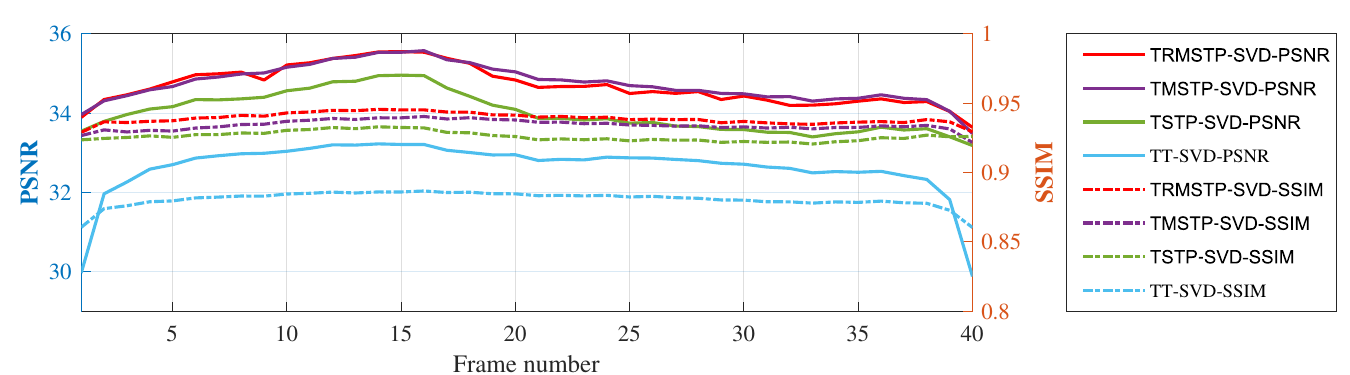}
\hfill
  \includegraphics[width=0.49\linewidth]{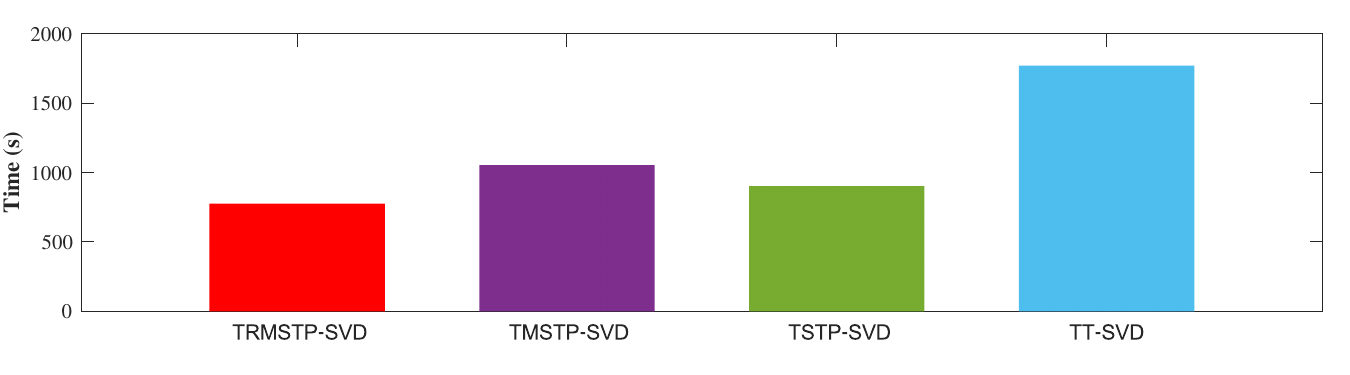}

  \renewcommand{\arraystretch}{0.5}
  \setlength{\tabcolsep}{0.03pt}       

  \begin{tabular}{@{}cccccc@{}}

    \includegraphics[width=0.77in]{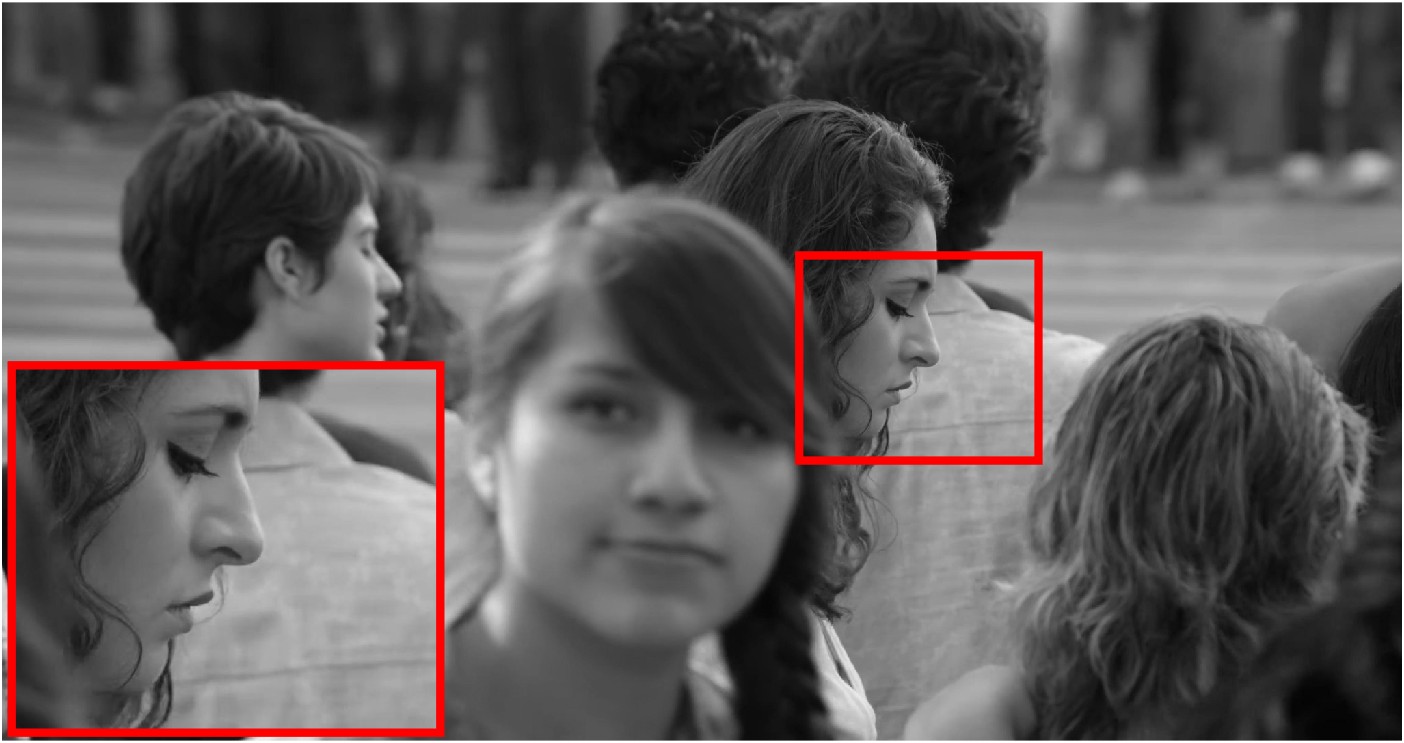} &
    \includegraphics[width=0.77in]{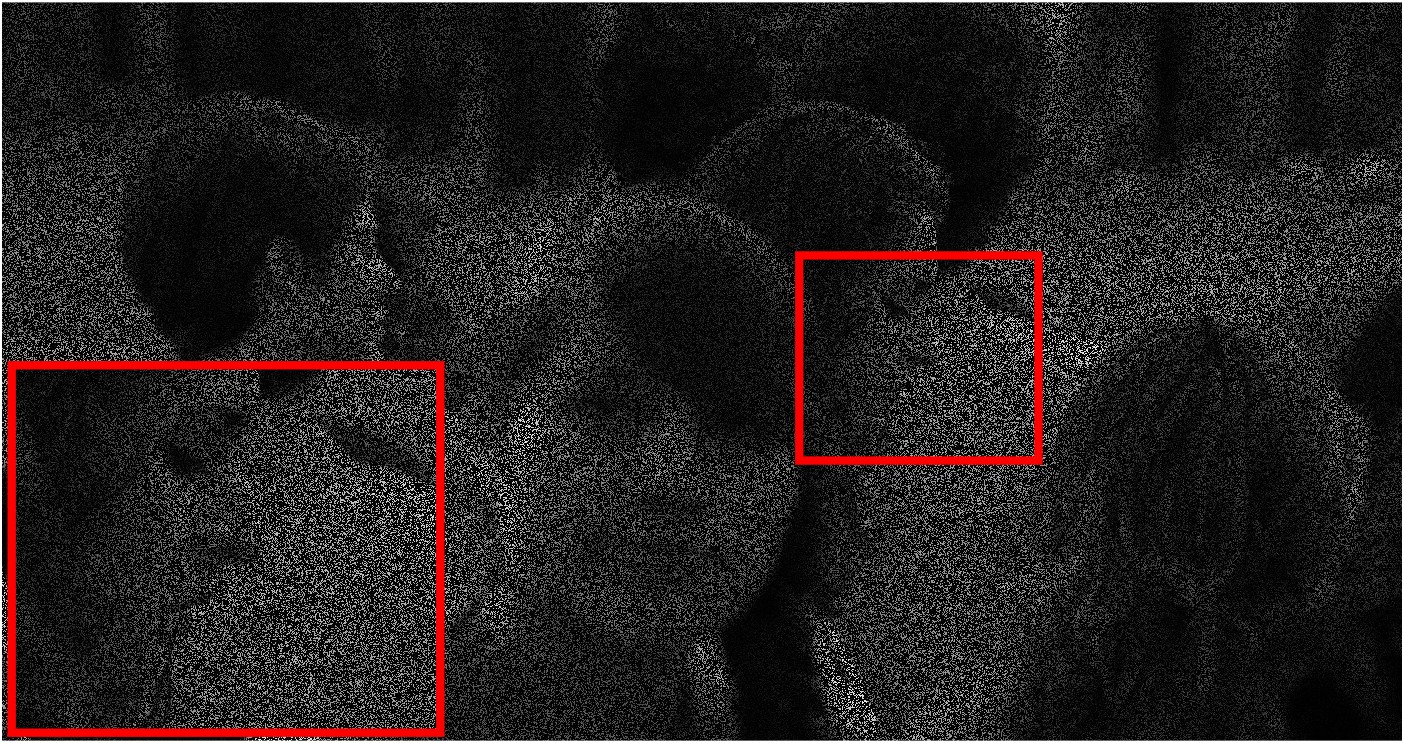} &
    \includegraphics[width=0.77in]{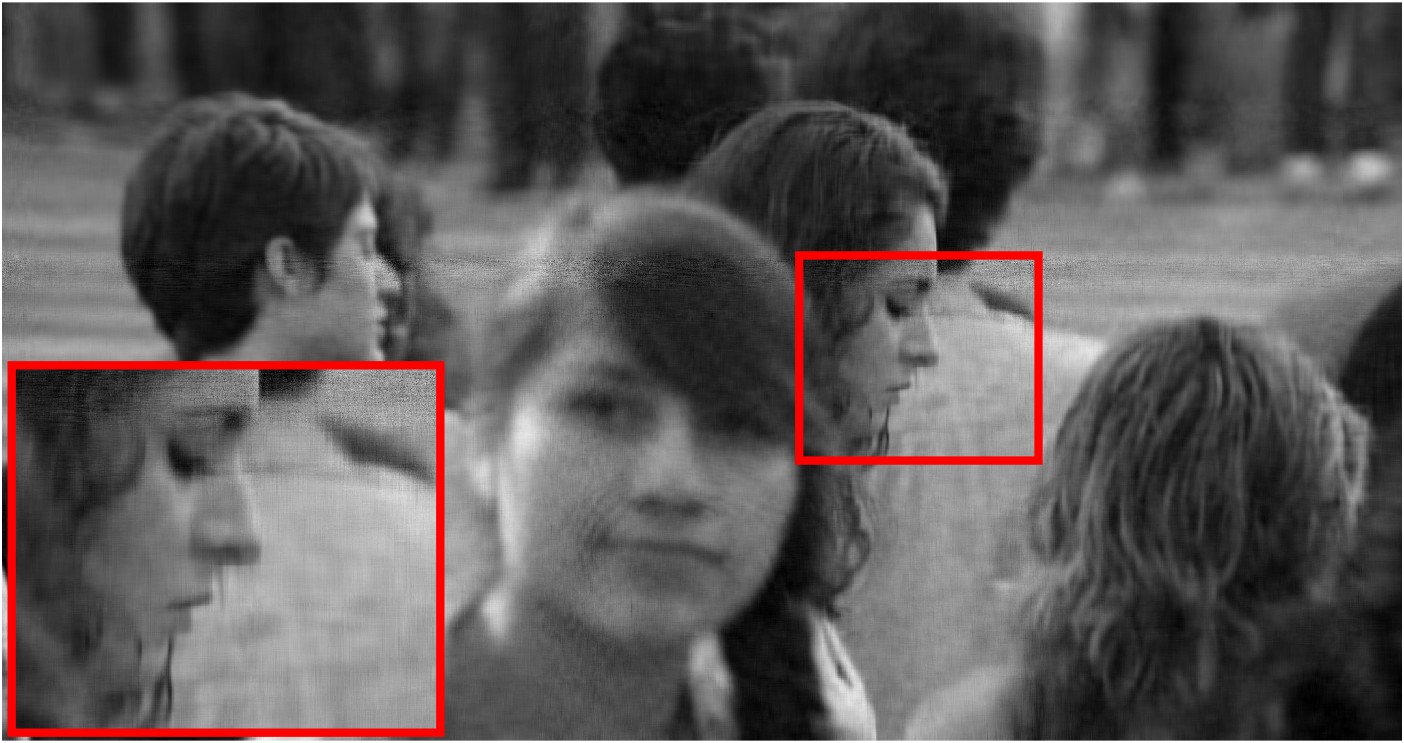} &
    \includegraphics[width=0.77in]{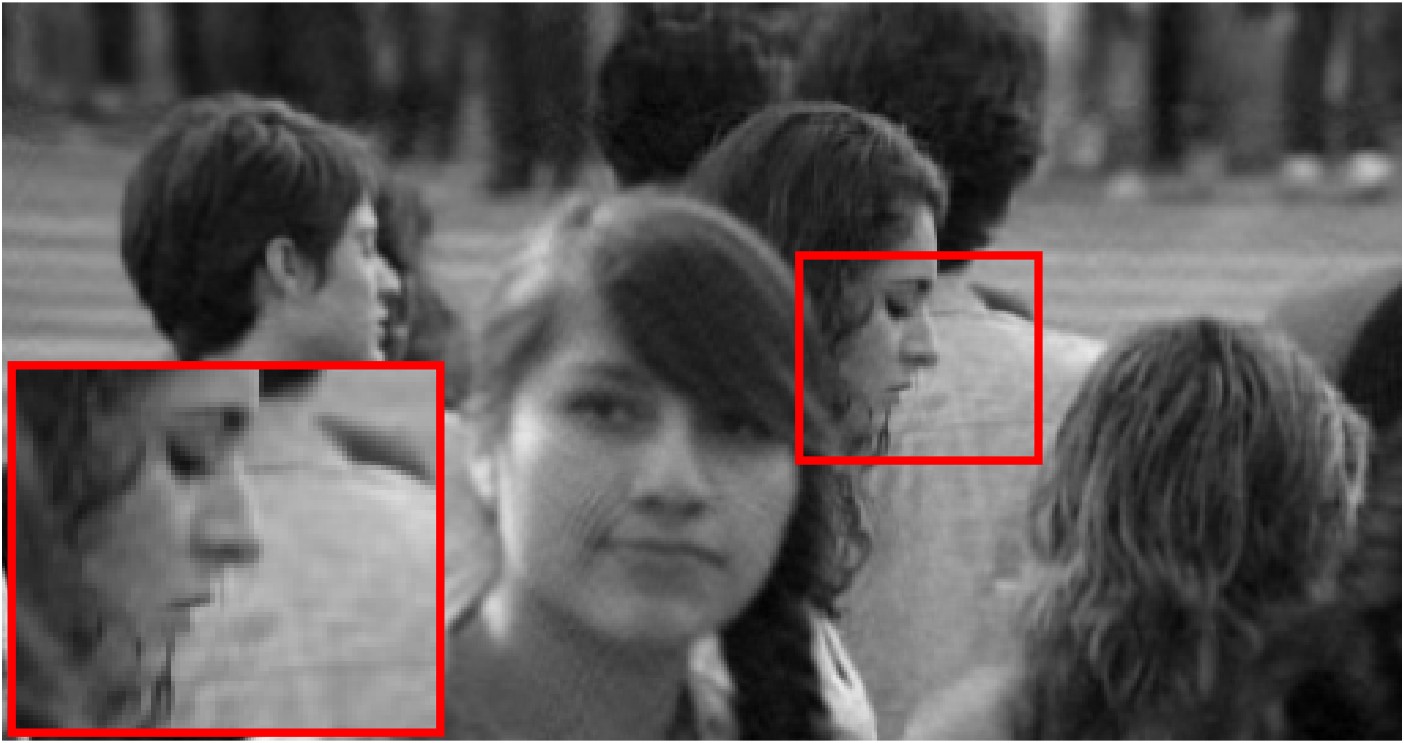}&
    \includegraphics[width=0.77in]{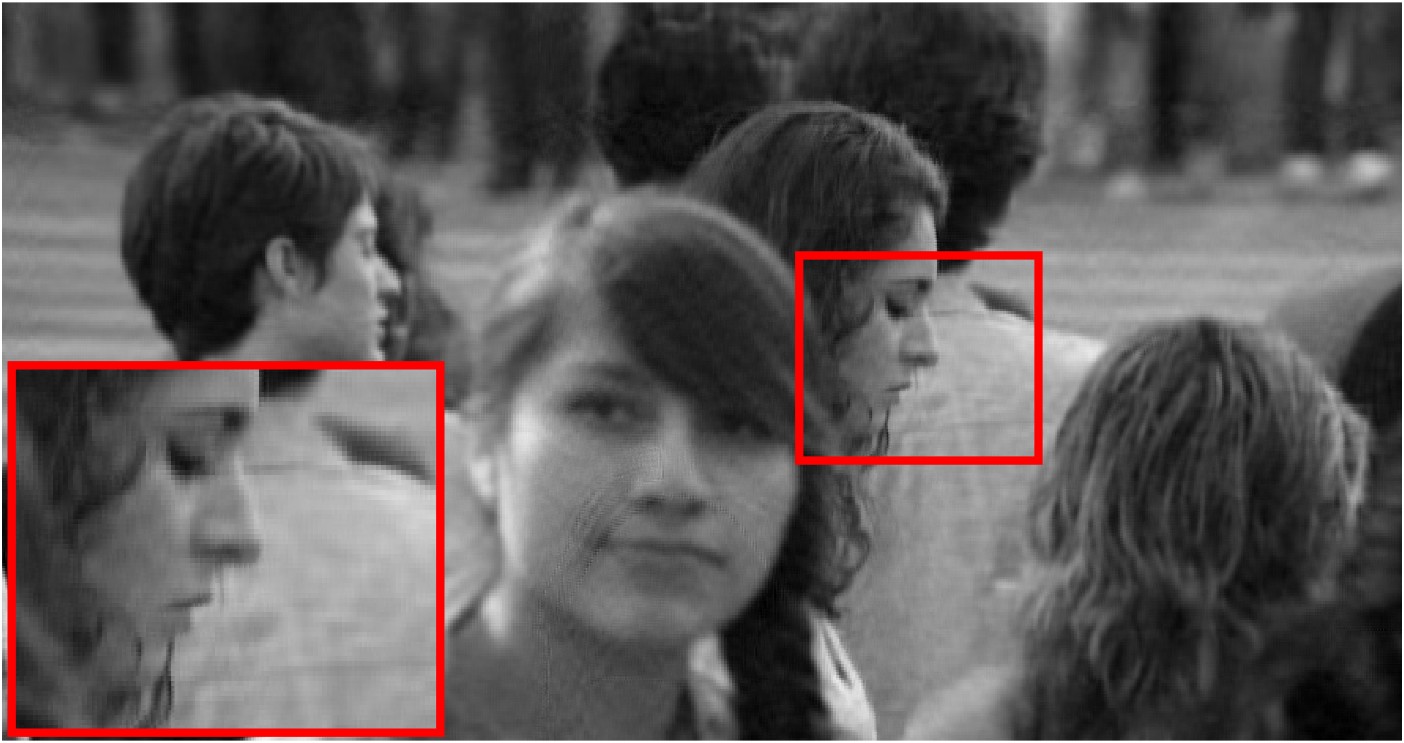} &
    \includegraphics[width=0.77in]{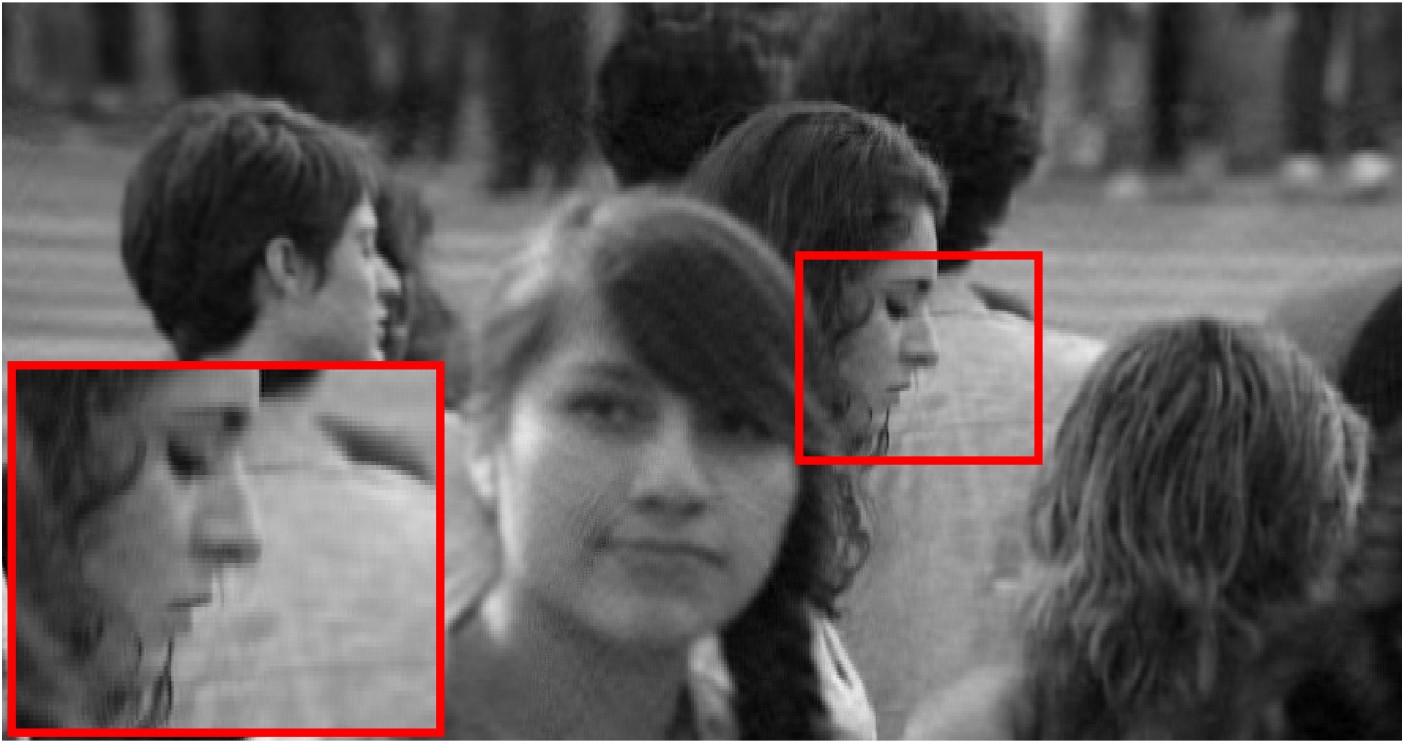} \\
    
    \includegraphics[width=0.77in]{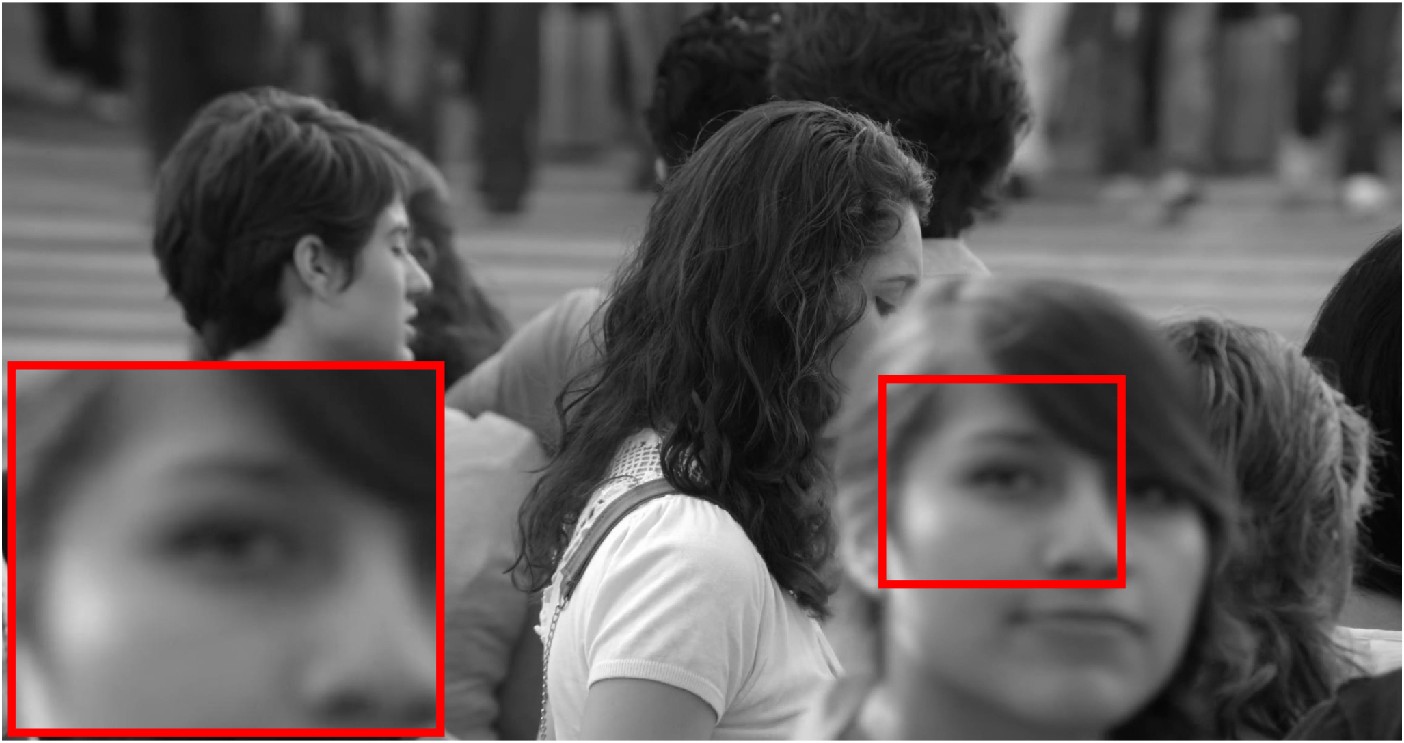} &
    \includegraphics[width=0.77in]{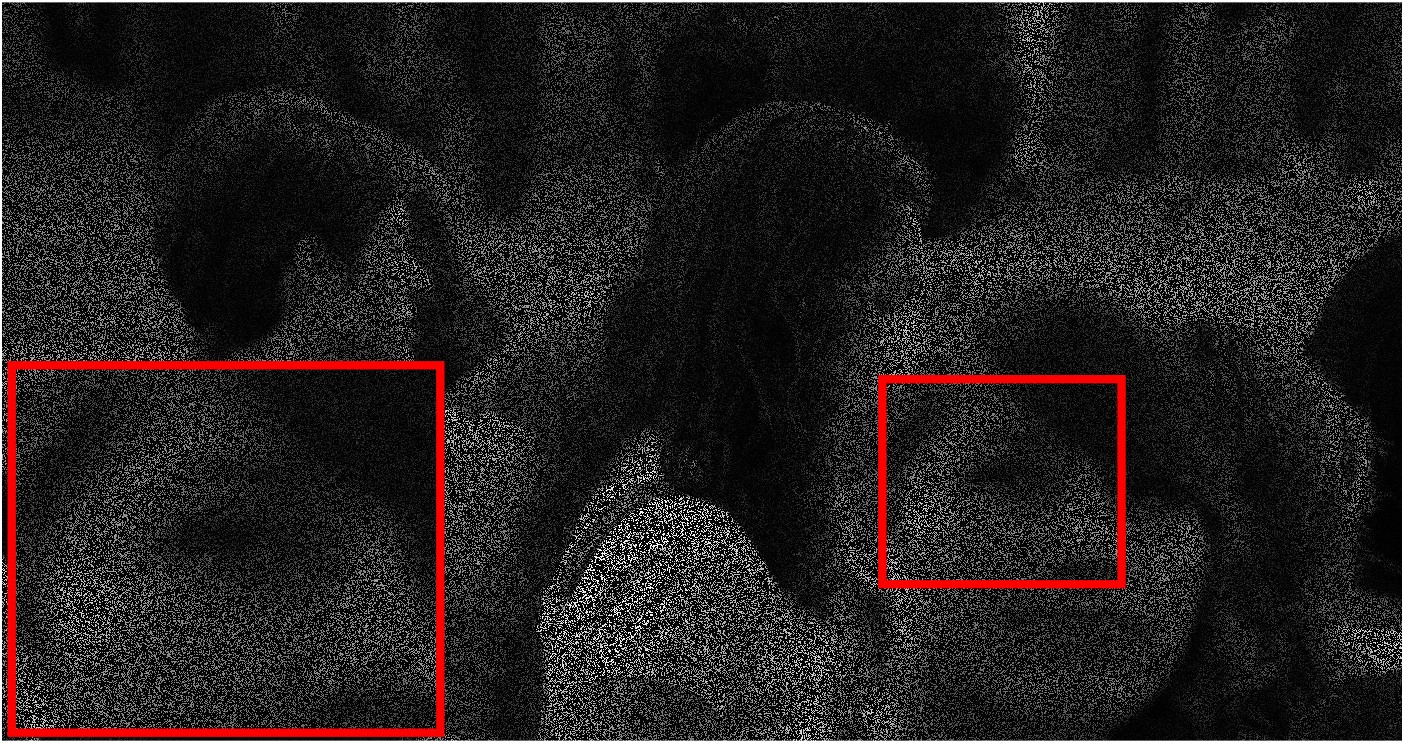} &
    \includegraphics[width=0.77in]{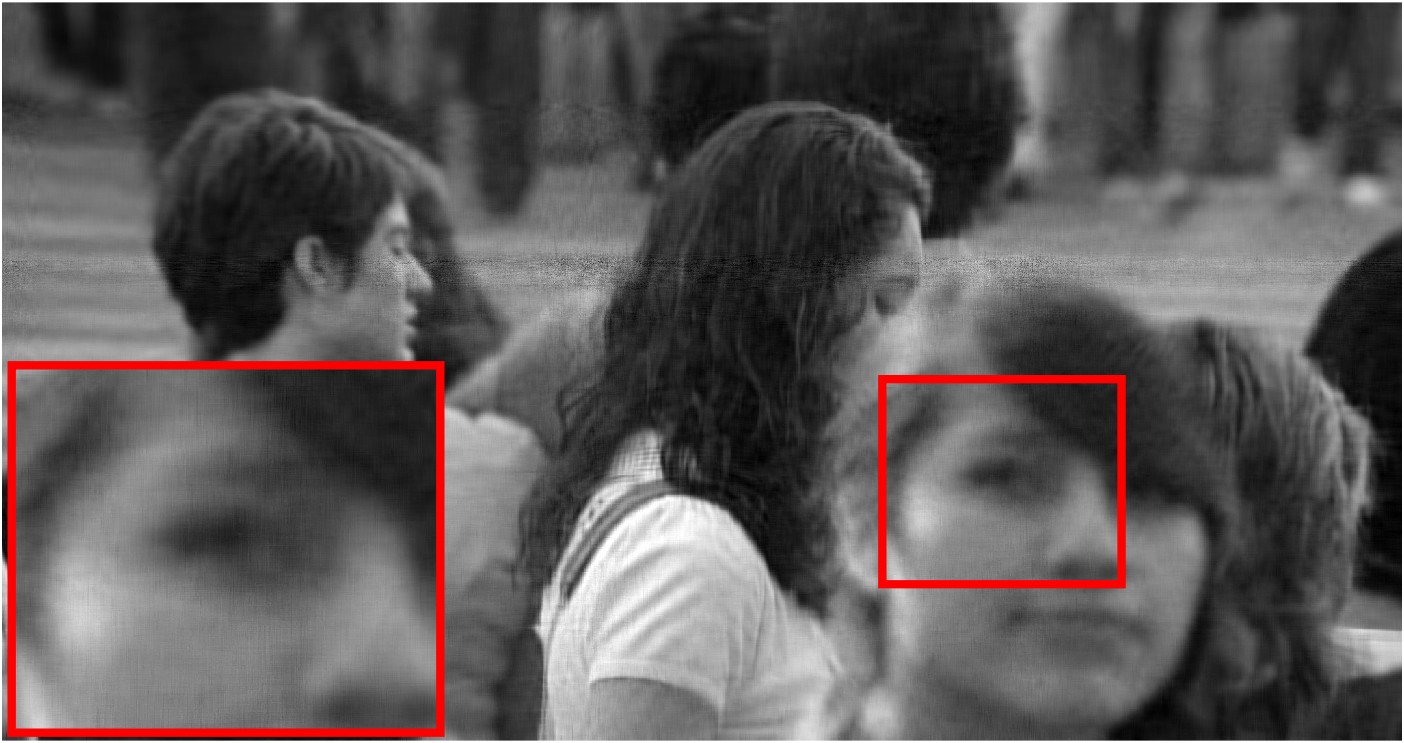} &
    \includegraphics[width=0.77in]{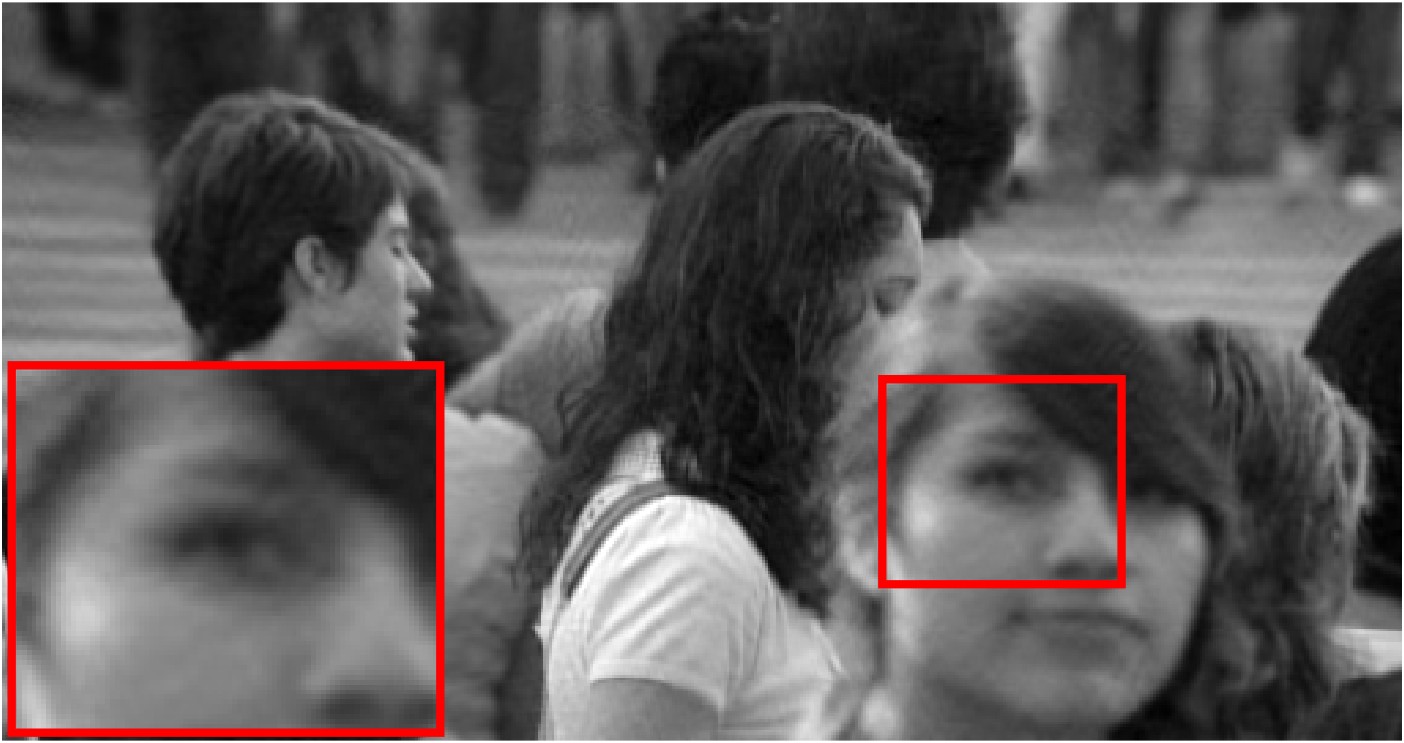}&
    \includegraphics[width=0.77in]{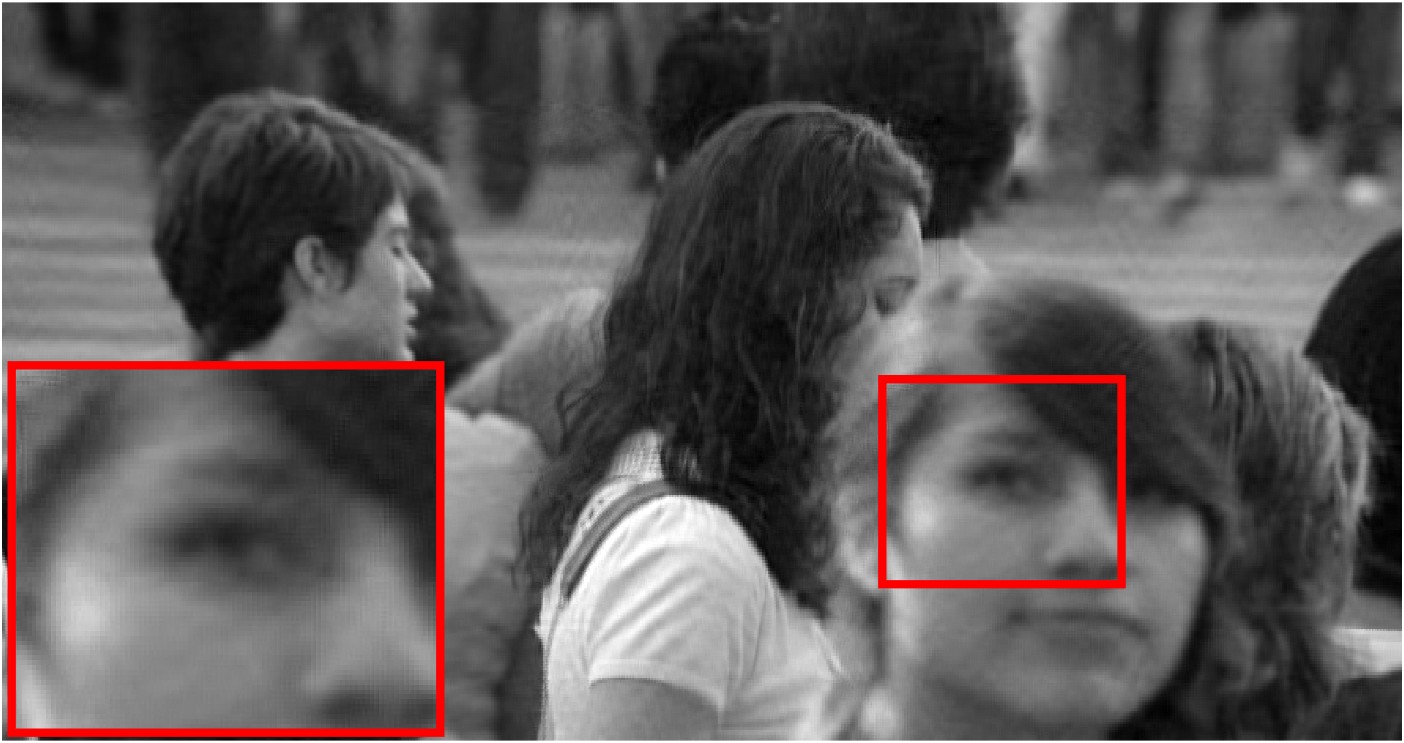} &
    \includegraphics[width=0.77in]{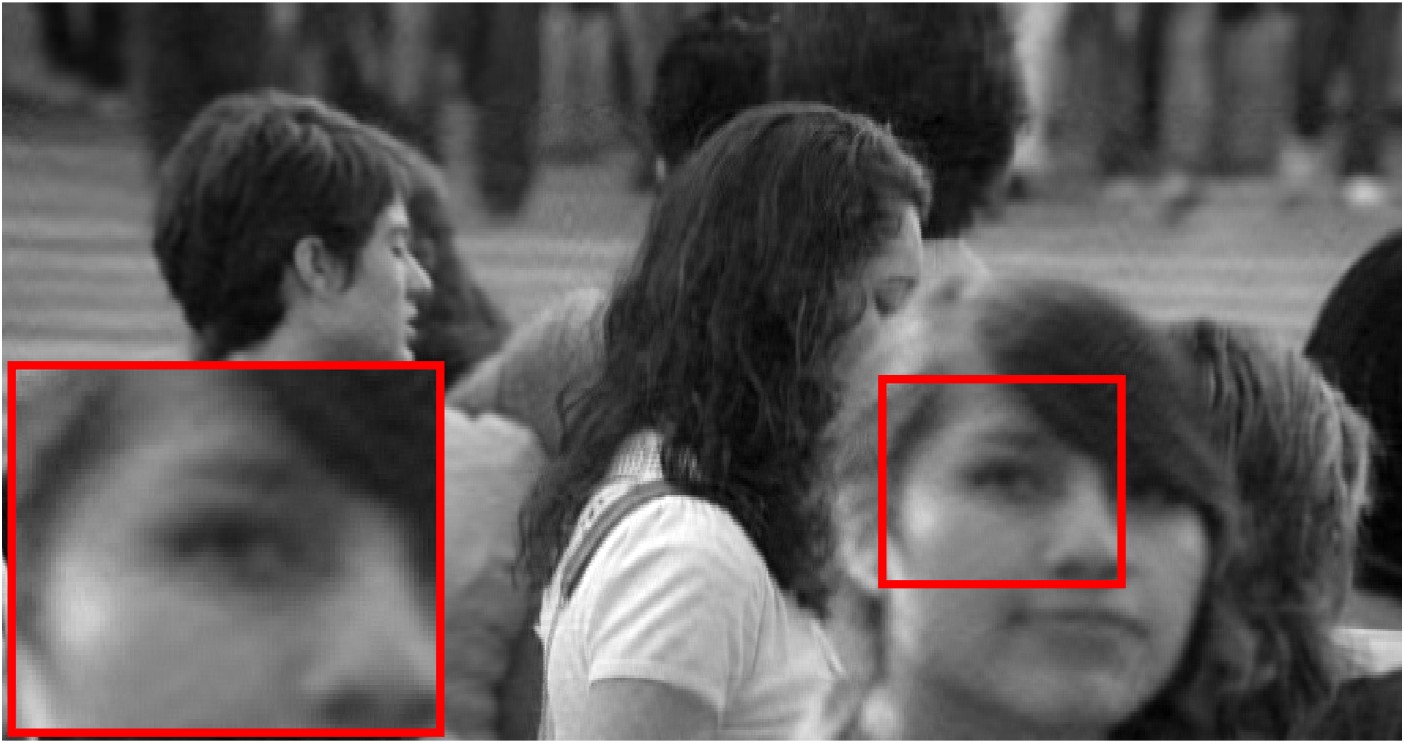} \\
    \scriptsize Original  &
    \scriptsize 70\% missing &
    \scriptsize  TT-SVD &
    \scriptsize  TSTP-SVD&
    \scriptsize  TMSTP-SVD &
    \scriptsize  TMRSTP-SVD
  \end{tabular}
  \includegraphics[width=0.49\linewidth]{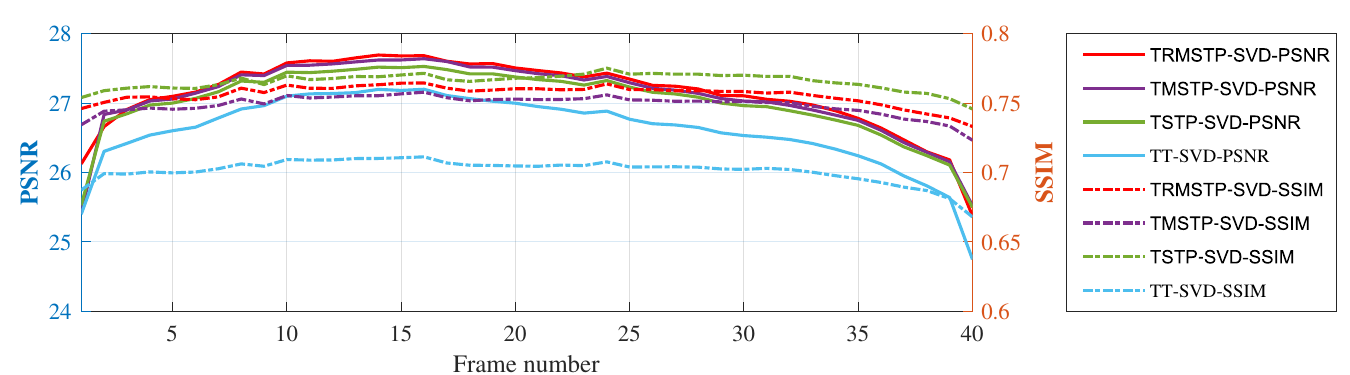}
  \hfill
  \includegraphics[width=0.49\linewidth]{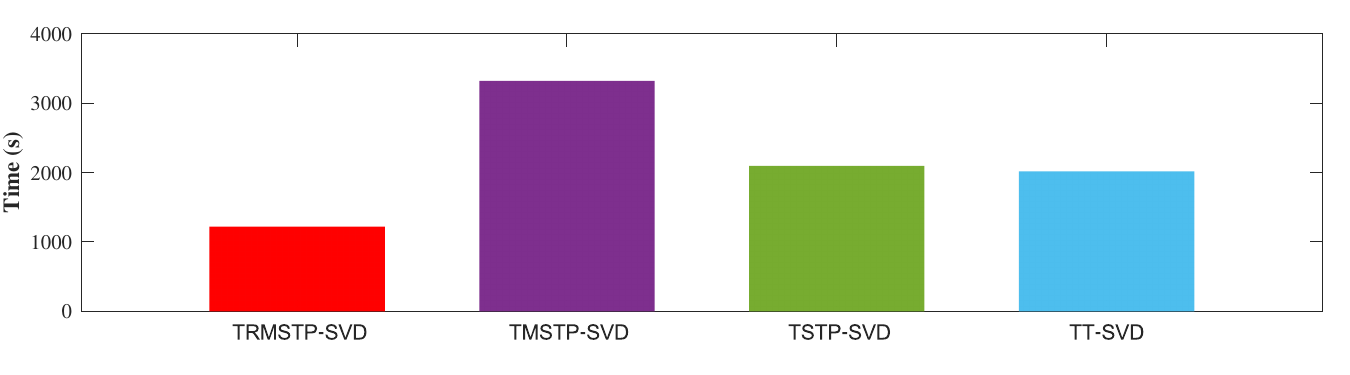}
\renewcommand{\arraystretch}{0.5} 
  \setlength{\tabcolsep}{0.3pt}      

  \begin{tabular}{@{}cccccc@{}}

    \includegraphics[width=0.77in]{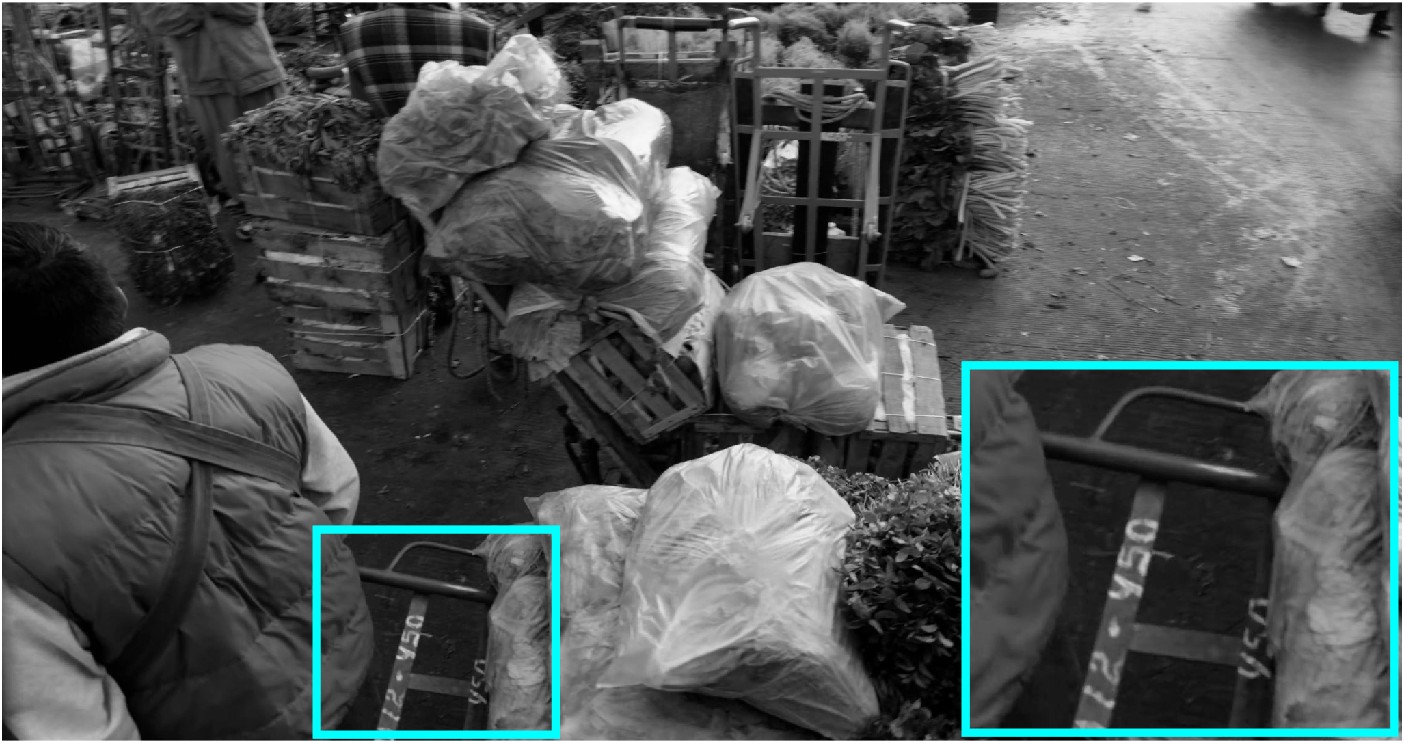} &
    \includegraphics[width=0.77in]{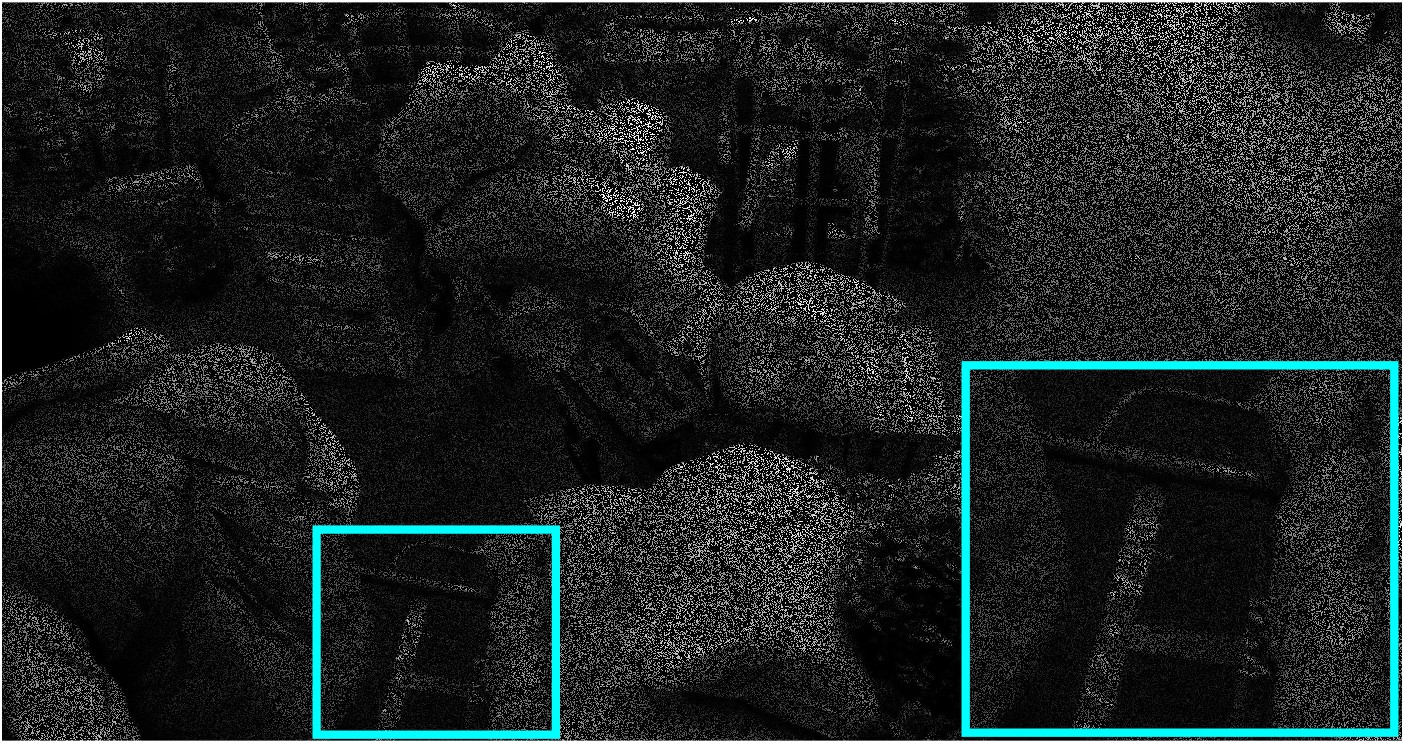} &
    \includegraphics[width=0.77in]{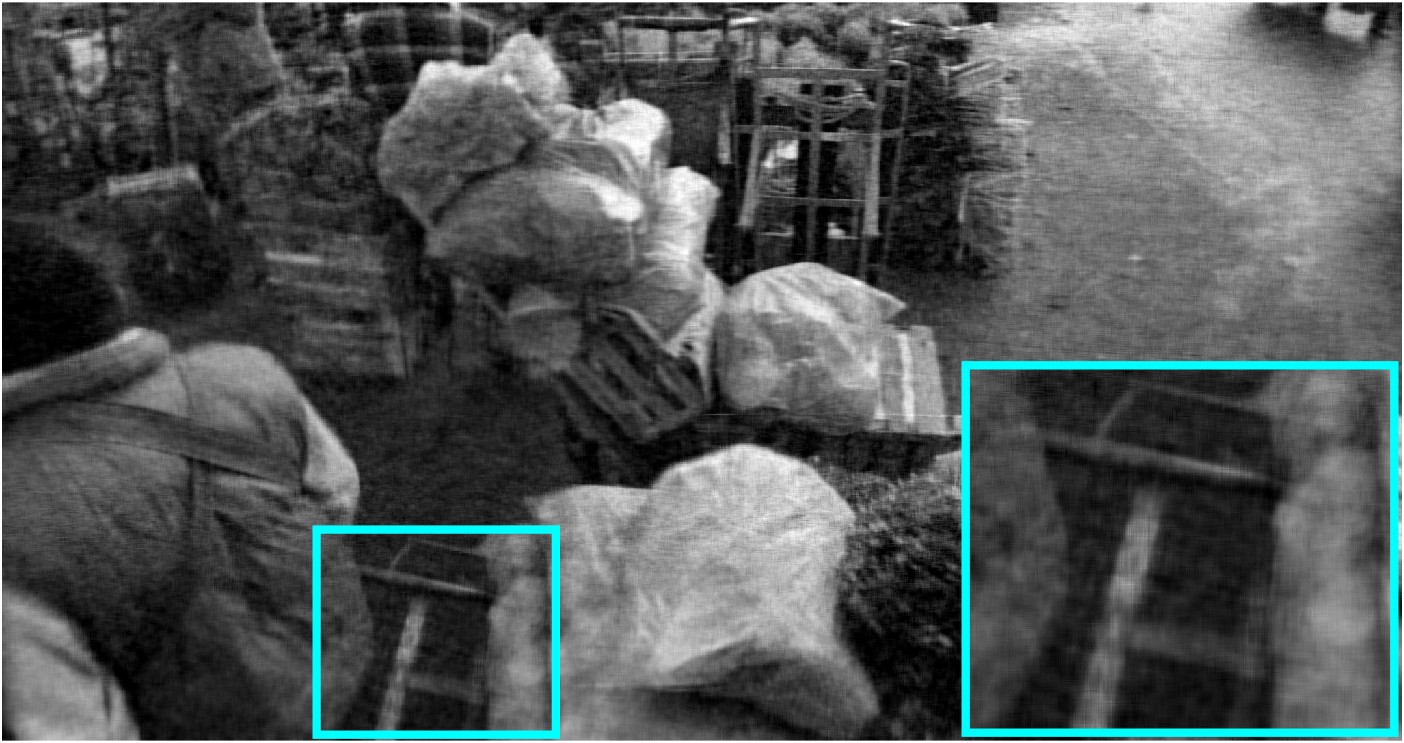} &
    \includegraphics[width=0.77in]{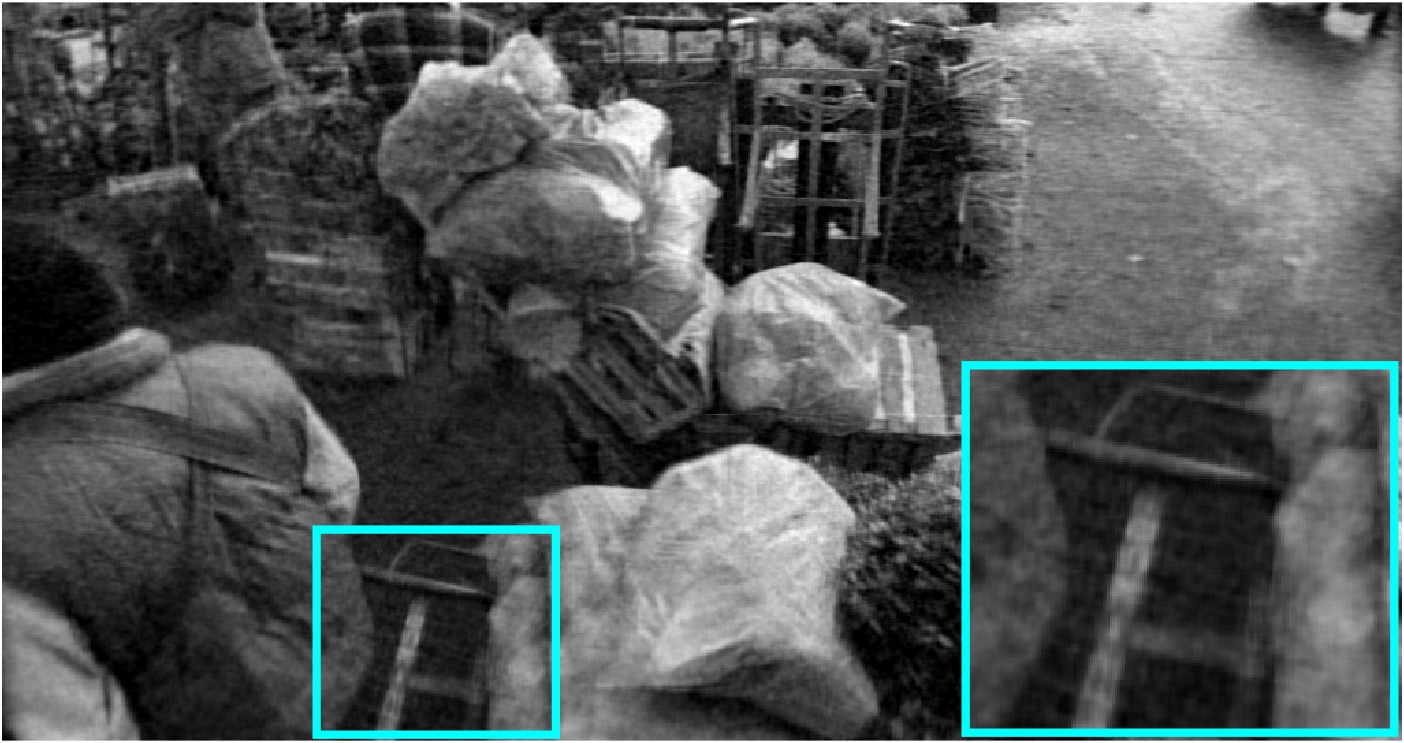}&
    \includegraphics[width=0.77in]{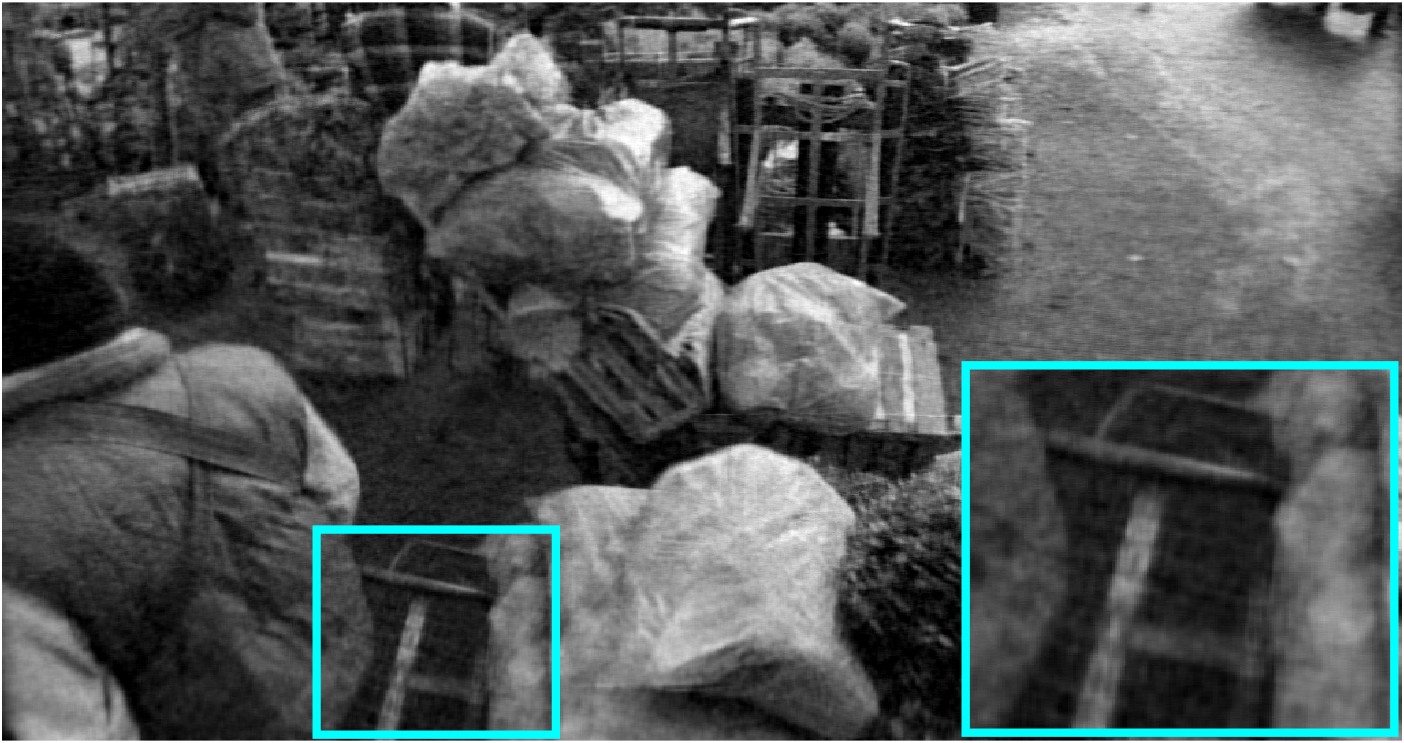} &
    \includegraphics[width=0.77in]{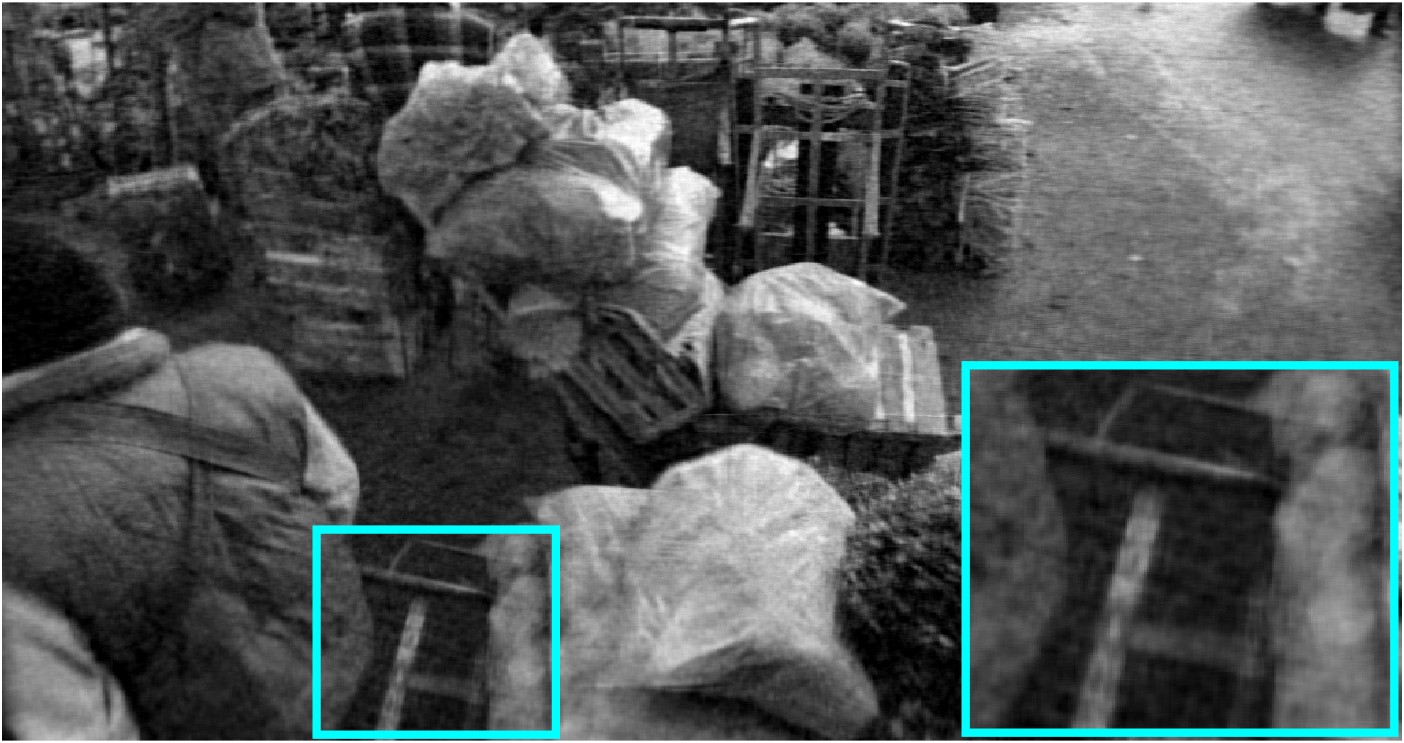} \\

    \includegraphics[width=0.77in]{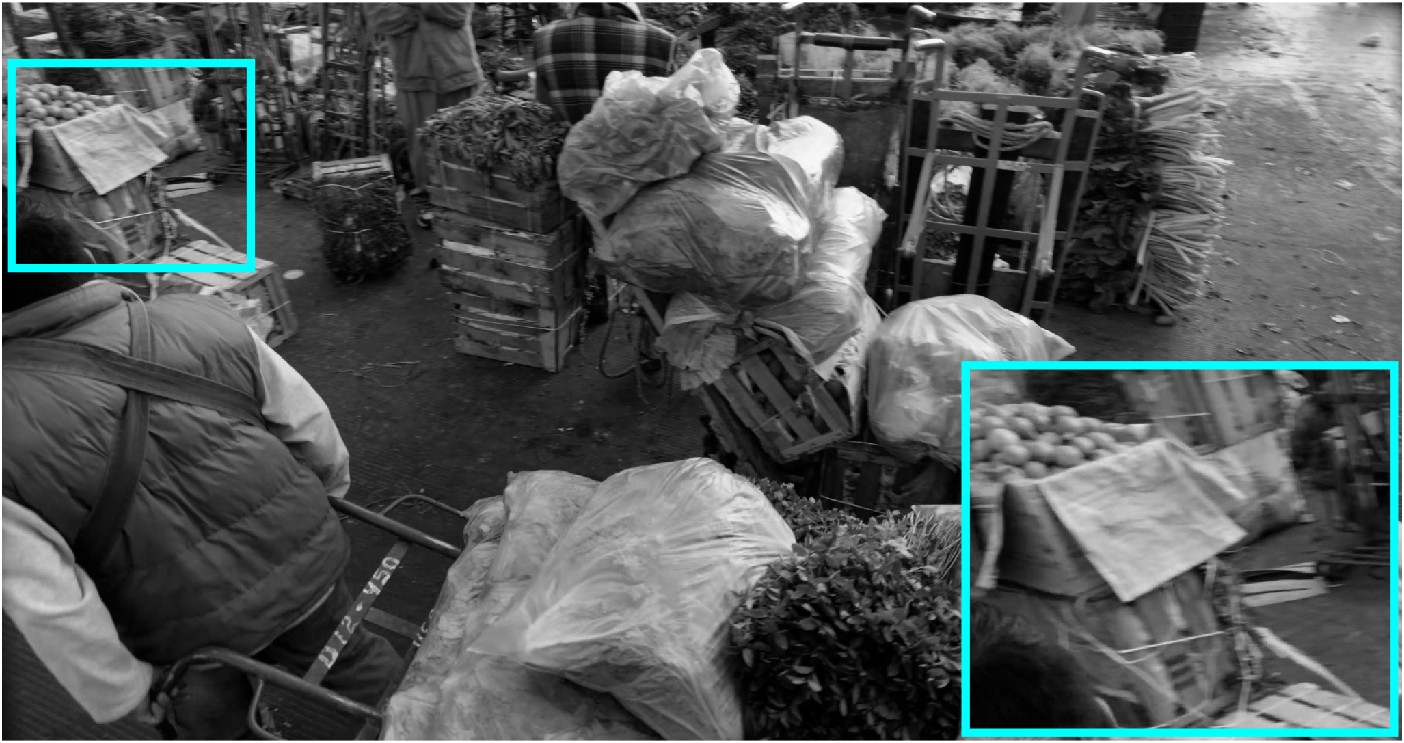} &
    \includegraphics[width=0.77in]{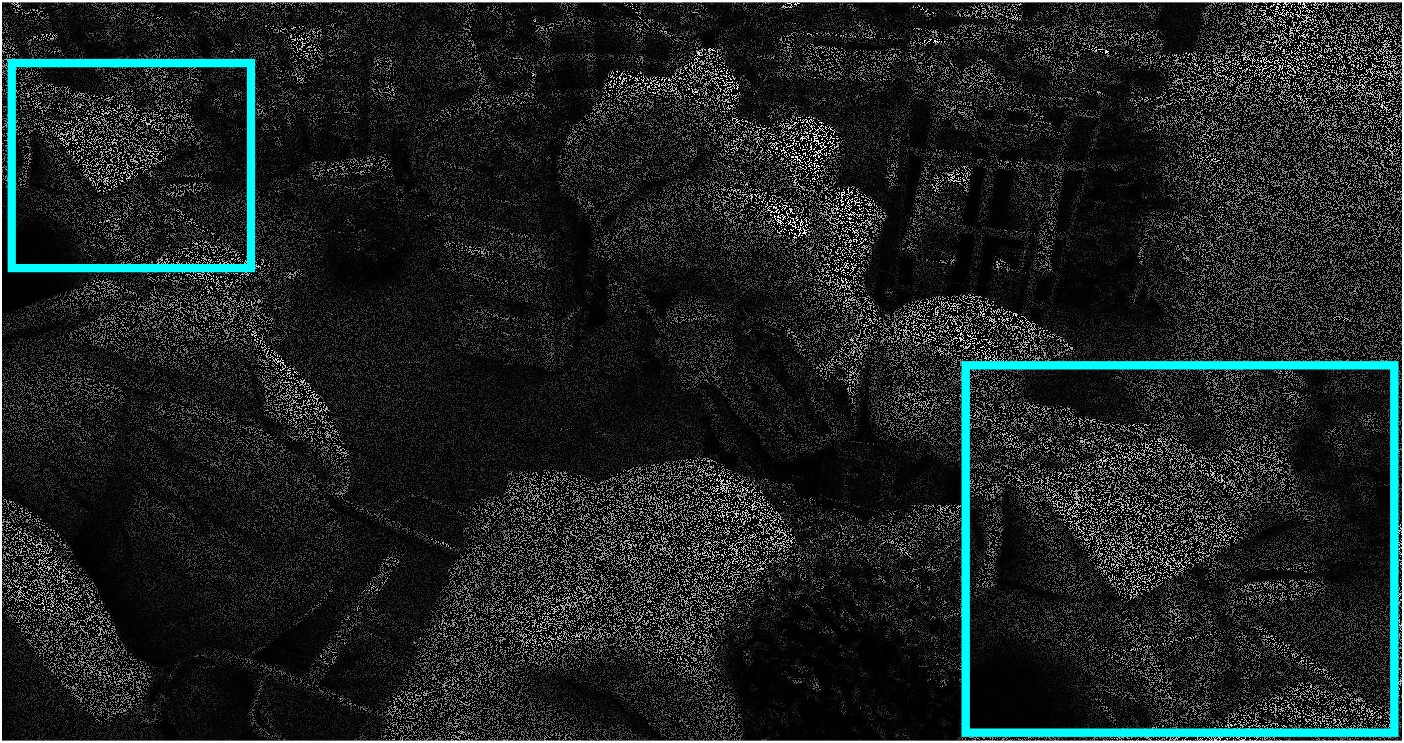} &
    \includegraphics[width=0.77in]{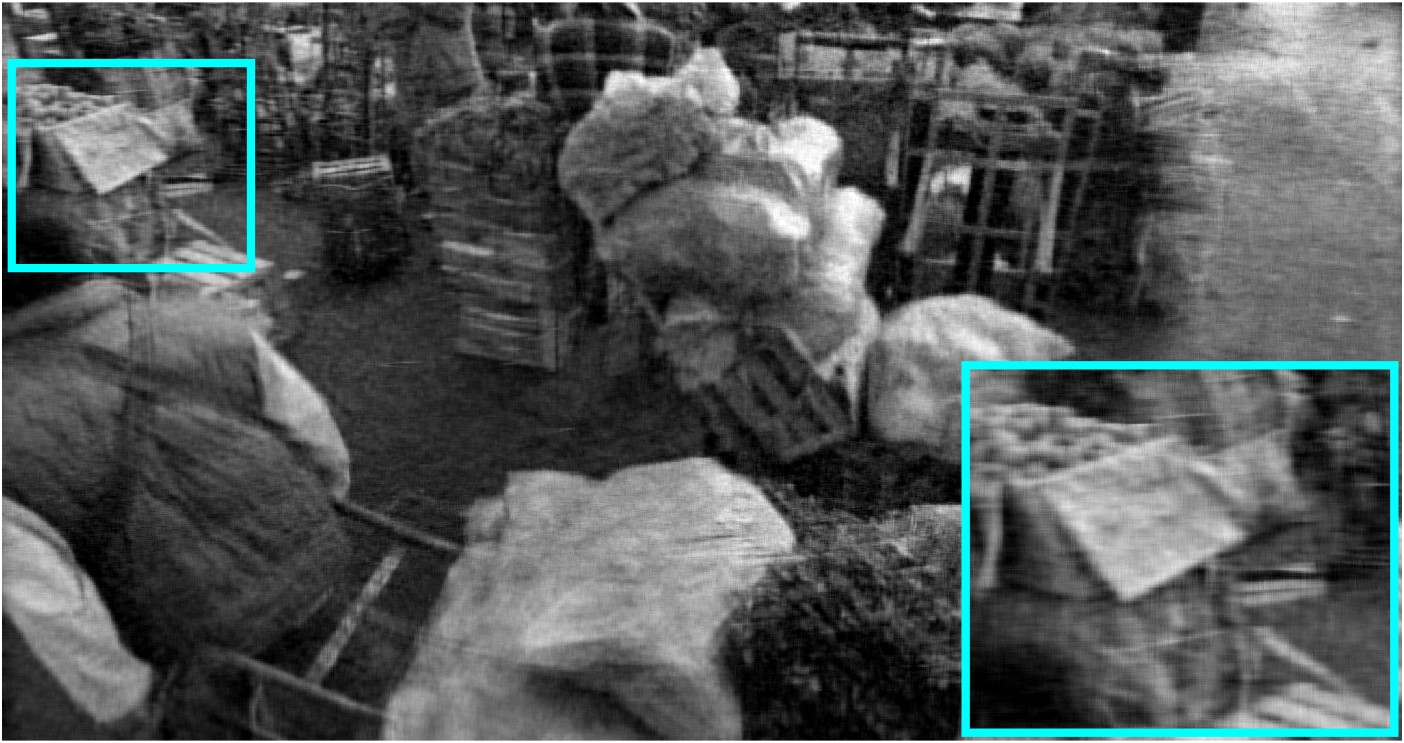} &
    \includegraphics[width=0.77in]{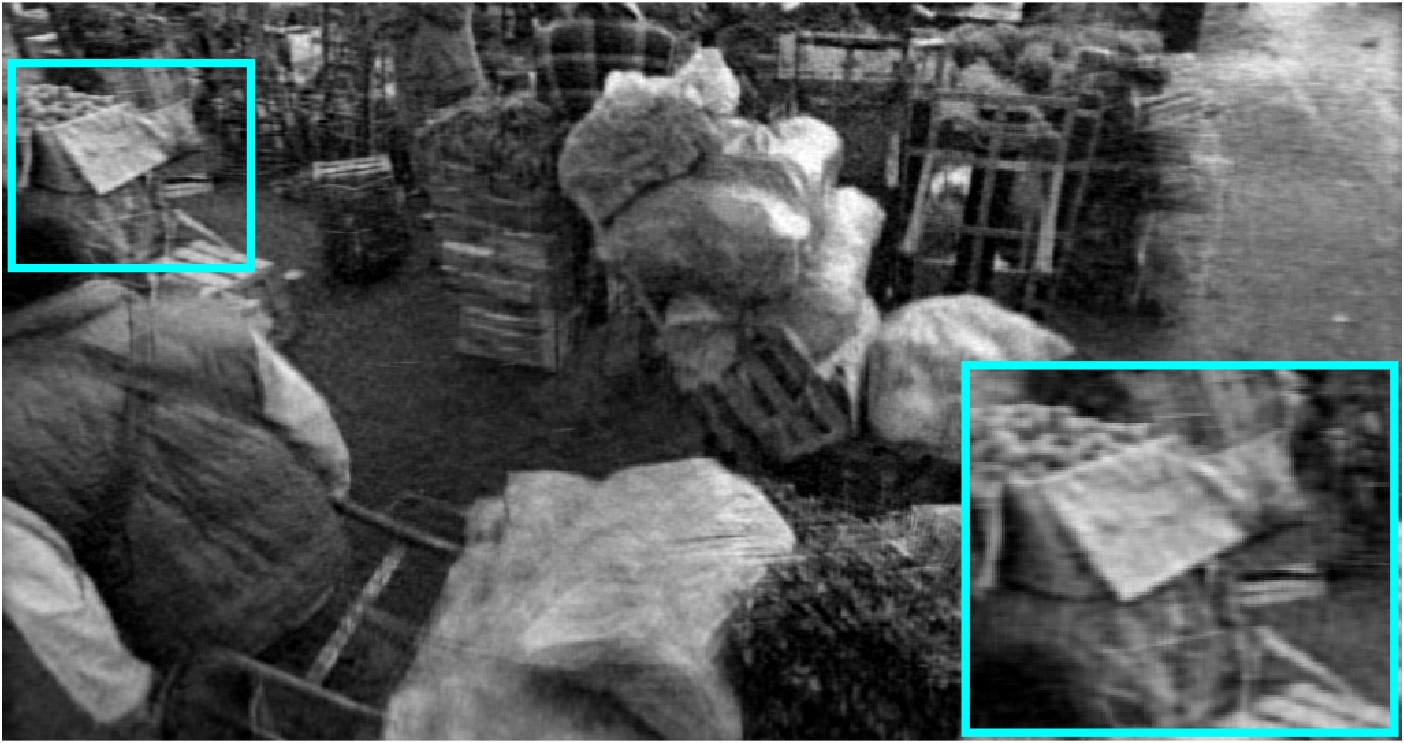}&
    \includegraphics[width=0.77in]{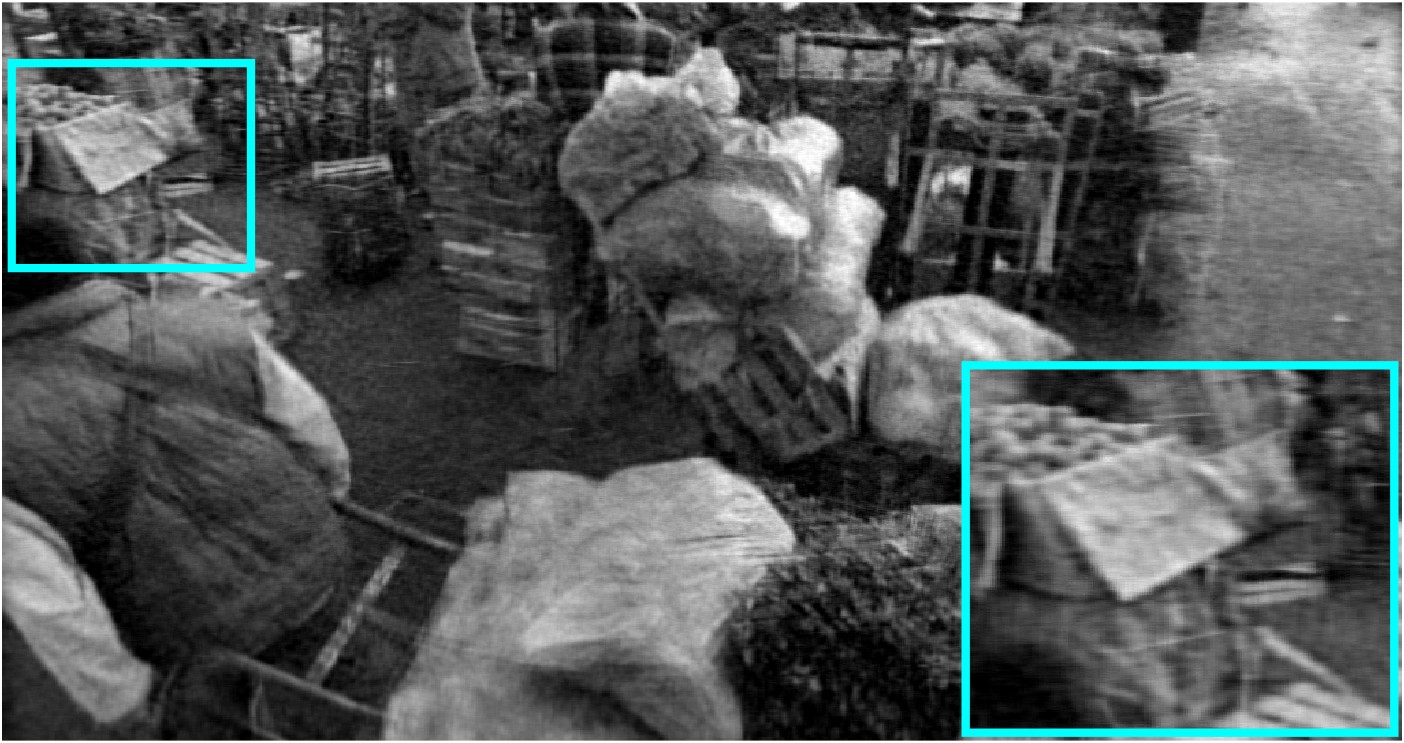} &
    \includegraphics[width=0.77in]{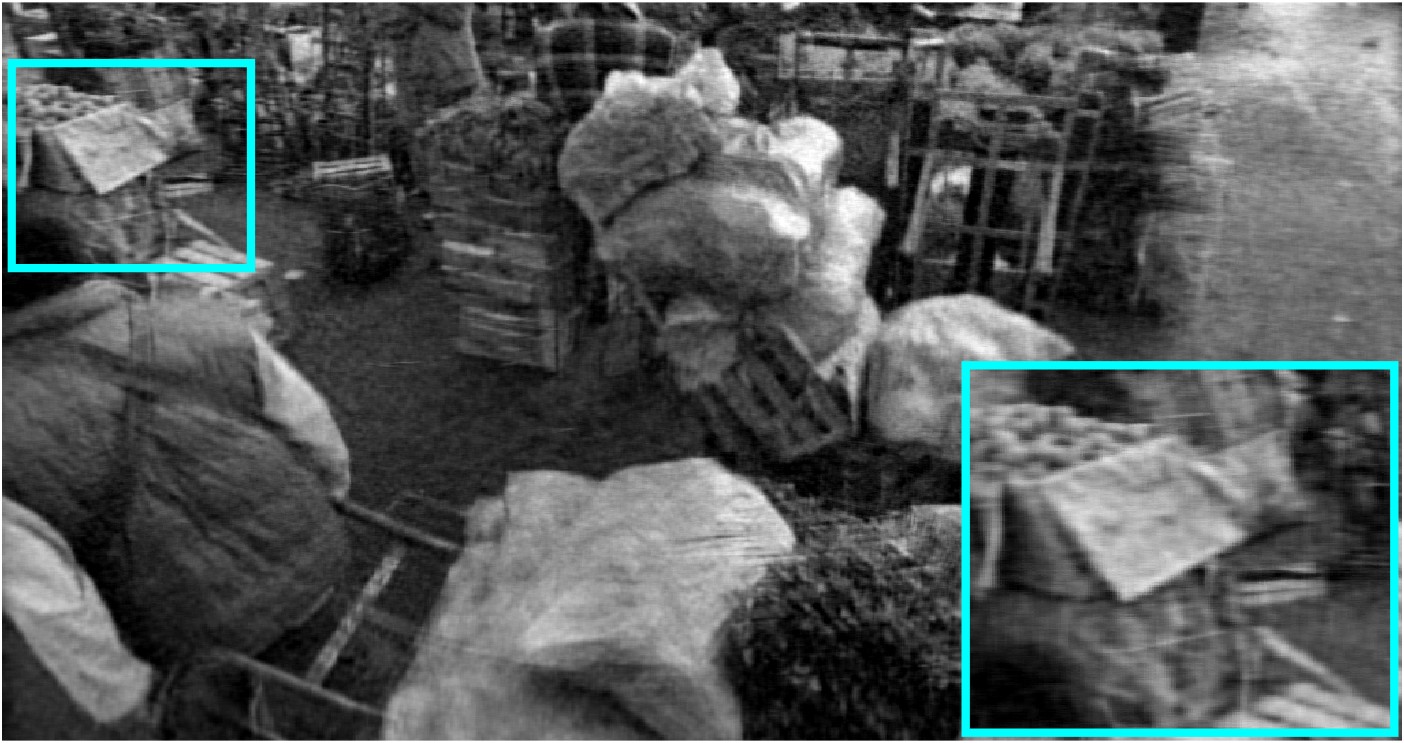} \\
    \scriptsize Original  &
    \scriptsize 70\% missing &
    \scriptsize  TT-SVD &
    \scriptsize  TSTP-SVD&
    \scriptsize  TMSTP-SVD &
    \scriptsize  TMRSTP-SVD
  \end{tabular}

\caption{Objective and visual comparisons of video recovery methods on two test sequences, including PSNR-SSIM curves, runtime, and reconstruction results under 70\% missing pixels.}
  \label{fig:video_complete}
\end{figure}
Video completion experiments are conducted on the four sequences from Subsection from Subsection \ref{video}, with 70\% of pixels randomly masked in each video tensor. All experimental parameters follow the settings in Subsection \ref{video}.
Fig.~\ref{fig:video_complete} shows frame-wise PSNR/SSIM curves, runtime statistics and visual reconstructions for Crosswalk and Market; results for Narrator and Aerial are provided in Supplementary Material. Our TMSTP-SVD and TMRSTP-SVD outperform baselines in reconstruction fidelity. The randomized TMRSTP-SVD achieves comparable PSNR and SSIM with negligible quality loss, and reduces total completion runtime by around 25\% relative to its deterministic counterpart. Visually, our methods recover finer textures in missing regions while baselines produce noticeable artifacts. Consistent performance gains can be observed for the two evaluated sequences, which are further validated by the supplementary results.
Overall, the video completion results corroborate the findings from image completion. The proposed randomized algorithm maintains reconstruction quality comparable to the deterministic version while achieving substantial computational savings, making it well suited for iterative completion on large-scale video data.
\section{Conclusion}
\label{sec:conclusion}
This paper proposes a novel semi-tensor product for third-order tensors under a generalized t-product framework with arbitrary invertible linear transforms. Unlike fixed-transform alternatives, our construction preserves the closed-form T-SVD structure while accommodating any unitary transform for enhanced flexibility. On this basis, we develop a multi-term decomposition model (MSTP-SVD) with multiple orthogonal components, which notably improves low-rank approximation accuracy over single-term schemes.
To address the computational bottleneck in large-scale applications, we further introduce a randomized variant (MRSTP-SVD) combining projection and power iteration, achieving a practical balance between reconstruction fidelity and efficiency. Theoretical error bounds are derived for both deterministic and randomized formulations to clarify the roles of key parameters. Image/video compression experiments validate the proposed method, and tensor completion experiments confirm its effectiveness as a low-rank prior, with the randomized variant delivering substantial acceleration at negligible accuracy cost. This work opens avenues for further algorithmic acceleration, adaptive parameter tuning, and extensions to higher-order tensors and broader algebraic structures.
\section*{CRediT authorship contribution statement}
\textbf{Xingchen Xiao:} Writing -- original draft, Visualization, Methodology, Software, Validation, Conceptualization.
\textbf{Feng Zhang}:  Writing -- review \& editing, Supervision, Resources, Funding acquisition, Conceptualization.
\textbf{Wenjin Qin}: Writing -- review \& editing, Formal analysis, Investigation, Data curation.
\textbf{Jianjun Wang}: Writing -- review \& editing, Supervision,  Project administration.
\section*{Declaration of competing interest}
The authors declare that they have no known competing financial interests or personal relationships that could have appeared to influence the work reported in this paper.
\section*{Acknowledgment}
This work was supported in part by the National Key Research and
Development Program of China under Grant 2023YFA1008502; in part by Fundamental Research Funds for the Central Universities under Grant SWU-KR25013; and in part by National Natural Science Foundation of China under Grant 12101512.
\section*{Data availability}
The datasets used in this study are publicly available from the sources cited in the numerical experiments section.


%
%


\FloatBarrier
\clearpage
\bibliographystyle{elsarticle-num}
\bibliography{refcopy} 

\clearpage
\addcontentsline{toc}{section}{Supplementary Material}

{\centering
\fontsize{14}{17}\selectfont
Supplementary Material of Semi-Tensor Product-Based
Multi-Term Randomized T-SVD and Its Visual Applications\par
}

\vspace{2em}
\appendix
This supplementary material accompanies the main paper by providing supporting mathematical preliminaries, full proofs of all theoretical results, complete algorithm pseudocodes, and additional experimental validations. All results presented here are included for completeness and do not affect the core contributions of the main work.
\section{Auxiliary definitions and properties of Kronecker product}
\label{app:aux}
\setcounter{equation}{0}
\setcounter{table}{0}
\setcounter{figure}{0}
\setcounter{definition}{0}
\setcounter{theorem}{0}
\setcounter{lemma}{0}

\renewcommand{\theequation}{S1.\arabic{equation}}
\renewcommand{\thetable}{S1.\arabic{table}}
\renewcommand{\thedefinition}{S1.\arabic{definition}}
\renewcommand{\thefigure}{S1.\arabic{figure}}
\renewcommand{\thetheorem}{S1.\arabic{theorem}}
\renewcommand{\thelemma}{S1.\arabic{lemma}}

We first review key definitions and properties of the Kronecker product for matrices and tensors. These are standard results from the literature, reproduced here to keep the main text concise and to support the theoretical developments in later sections.
\begin{definition}\cite{bellman1997introduction}
        If $\mathbf{A} \in \mathbb{R}^{m_1 \times m_2}$ and $\mathbf{B} \in \mathbb{R}^{n_1 \times n_2}$, the Kronecker product between them  is defined as 
        \[ \mathbf{A} \otimes \mathbf{B}=\begin{bmatrix} \mathbf{A}_{11}&\cdots&\mathbf{A}_{1 m_2}\\\mathbf{A}_{21}&\cdots&\mathbf{A}_{2 m_2}\\\vdots&\ddots&\vdots\\\mathbf{A}_{m_1 1}&\cdots&\mathbf{A}_{m_1 m_2} \end{bmatrix}\otimes\mathbf{B}
        =\begin{bmatrix} \mathbf{A}_{11}\mathbf{B}&\cdots&\mathbf{A}_{1 m_2}\mathbf{B}\\\mathbf{A}_{21}\mathbf{B}&\cdots&\mathbf{A}_{2 m_2}\mathbf{B}\\\vdots&\ddots&\vdots\\\mathbf{A}_{m_1 1}\mathbf{B}&\cdots&\mathbf{A}_{m_1 m_2}\mathbf{B} \end{bmatrix} \in \mathbb{R}^{m_1 n_1 \times m_2 n_2}.
        \]
\end{definition}

\begin{lemma}\label{lem:kp property}\cite{cheng2012introduction}
Let $\mathbf{A}, \mathbf{B}, \mathbf{C}, \mathbf{D}$ be matrices of compatible dimensions, and let $\alpha$ denote a scalar. The following properties of the Kronecker product hold:
    \\(i) $\mathbf{A}\mathbf{B} \otimes \mathbf{C}\mathbf{D} = (\mathbf{A} \otimes \mathbf{C})(\mathbf{B} \otimes \mathbf{D}).$
    \\(ii) $\mathbf{A} \otimes (\mathbf{B} \pm \mathbf{C}) = (\mathbf{A} \otimes \mathbf{B}) \pm (\mathbf{A} \otimes \mathbf{C})$ and $(\mathbf{B} \pm \mathbf{C}) \otimes \mathbf{A} = \mathbf{B} \otimes \mathbf{A} \pm \mathbf{C} \otimes \mathbf{A}.$
    \\(iii) $(\mathbf{A} \otimes \mathbf{B})^\top = \mathbf{A}^\top \otimes \mathbf{B}^\top.$
    \\(iv) $(\mathbf{A} \otimes \mathbf{B})^{-1} = \mathbf{A}^{-1} \otimes \mathbf{B}^{-1}$, provided $\mathbf{A}$ and $\mathbf{B}$ are invertible.
    \\(v) $(\mathbf{A} \otimes \mathbf{B}) \otimes \mathbf{C} = \mathbf{A} \otimes (\mathbf{B} \otimes \mathbf{C}).$
    \\(vi) $(\alpha \mathbf{A}) \otimes \mathbf{B} = \mathbf{A} \otimes (\alpha \mathbf{B}) = \alpha (\mathbf{A} \otimes \mathbf{B}).$
\end{lemma}
\begin{definition}[Kronecker product of tensors \cite{batselier2017constructive}]\label{de:tkp} 
Consider two $p$-order tensors  $\mathcal{A}\in\mathbb{R}^{n_1\times n_2\times \dots \times n_p}$ and 
$\mathcal{B}\in\mathbb{R}^{m_1\times m_2\times \dots \times m_p}$.
Their tensor Kronecker product $\mathcal{C}=\mathcal{A}\otimes\mathcal{B}\in\mathbb{R}^{n_1m_1\times \cdots \times n_pm_p}$ is defined entrywise by
\[
\mathcal{C}_{[i_1 j_1]\cdots [i_p j_p]} =
\mathcal{A}_{i_1\cdots i_p}\,\mathcal{B}_{j_1\cdots j_p}.
\]
\end{definition}
\par Several fundamental properties of the tensor Kronecker product are summarized below without elaborate derivations.
\begin{lemma}\cite{batselier2017constructive}
Let $\mathcal{A}$, $\mathcal{B}$, $\mathcal{C}$ be $p$-order tensors and let $\alpha$ is a scalar. The following identities hold:
 \begin{enumerate}[(i)]
 \item $\mathcal{A}\otimes (\mathcal{B}\pm\mathcal{C})=\mathcal{A}\otimes \mathcal{B}\pm\mathcal{A}\otimes\mathcal{C}$;
 \item $(\mathcal{B}\pm\mathcal{C})\otimes\mathcal{A}=\mathcal{B}\otimes\mathcal{A}\pm \mathcal{C}\otimes\mathcal{A}$;
 \item $(\mathcal{A}\otimes\mathcal{B})\otimes\mathcal{C}=\mathcal{A}\otimes(\mathcal{B}\otimes\mathcal{C})$;
 \item $(\alpha\mathcal{A})\otimes\mathcal{B}=\mathcal{A}\otimes(\alpha\mathcal{B})=\alpha(\mathcal{A}\otimes\mathcal{B})$.
 \end{enumerate}
\end{lemma}

\section{Proofs of theorems}

\setcounter{equation}{0}
\setcounter{table}{0}
\setcounter{figure}{0}
\setcounter{theorem}{0}
\setcounter{lemma}{0}

\renewcommand{\theequation}{S2.\arabic{equation}}
\renewcommand{\thetable}{S2.\arabic{table}}
\renewcommand{\thefigure}{S2.\arabic{figure}}
\renewcommand{\thetheorem}{S2.\arabic{theorem}}
\renewcommand{\thelemma}{S2.\arabic{lemma}}
\subsection{Proof of Theorem 4.2}
To keep the main text focused on algorithmic frameworks and experimental validation, we defer the detailed theoretical derivations to this section. Specifically, we first prove Theorem 4.2, which establishes the multi-term STP-SVD decomposition for matrices and characterizes its truncation error in terms of singular values. We then extend the result to the higher-order tensor setting and prove Theorem 4.3, which presents the MSTP-SVD decomposition with verified orthogonality of factor tensors and a closed-form error bound. Finally, we provide the full proof of Theorem 5.1, which derives the expected reconstruction error bound for the randomized MRSTP-SVD algorithm by decomposing the total error into a deterministic truncation component and a randomized approximation component.
\begin{proof}
        Based on Lemma 4.2, for any matrix $\mathbf{A} \in \mathbb{R}^{m_1 m_2 \times n_1 n_2}$, there exist matrices $\mathbf{B}_i \in \mathbb{R}^{m_1 \times n_1}$ and $\mathbf{C}_i \in \mathbb{R}^{m_2 \times n_2}$ that matrix $\mathbf{A}$ can be decomposed as the sum of $k$ Kronecker product terms plus an error term, i.e.,
        \begin{equation}\label{eq:proof1}
                \mathbf{A} = \sum_{i=1}^{k}\mathbf{B}_i\otimes \mathbf{C}_i + \mathbf{E}_k,
        \end{equation}
        where $\mathbf{E}_k \in \mathbb{R}^{m_1 m_2 \times n_1 n_2}$ denotes the approximation error matrix, whose squared Frobenius norm satisfies
\begin{equation}\label{eq:mstpsvd_matrix_error}
\|\mathbf{E}_k\|_F^2 = \sum_{i=k+1}^{v} \sigma_i^2,        
\end{equation}
where $\sigma_{k+1} \ge \sigma_{k+2} \ge \dots \ge \sigma_{v} \ge 0$ are the singular values of $\mathscr{R}(\mathbf{A}) \in \mathbb{R}^{m_1 n_1 \times m_2 n_2}$, 
and $v = \min\{m_1 n_1, m_2 n_2\}$. Next, computing the SVD of each $\mathbf{B}_i$ yields $\mathbf{B}_i = \mathbf{U}_i \mathbf{\Sigma}_{\mathbf{B}_i} \mathbf{V}_i^\top$, where $\mathbf{U}_i \in \mathbb{R}^{m_1 \times m_1}$, $\mathbf{\Sigma}_{\mathbf{B}_i} \in \mathbb{R}^{m_1 \times n_1}$, $\mathbf{V}_i \in \mathbb{R}^{n_1 \times n_1}$.
        Accordingly, (\ref*{eq:proof1}) can be rewritten as 
        \begin{equation}\label{eq:proof2}
                \mathbf{A} = \sum_{i=1}^{k}(\mathbf{U}_i \mathbf{\Sigma}_{\mathbf{B}_i} \mathbf{V}_i^\top) \otimes \mathbf{C}_i + \mathbf{E}_k,
        \end{equation}
where $\mathbf{\Sigma}_{\mathbf{B}_i}$ is a diagonal matrix whose diagonal entries are the singular values of $\mathbf{B}_i$. Let $\sigma_{i1}, \sigma_{i2}, \cdots, \sigma_{ip}$ with $p=\min\{m_1 , n_1\}$ denote the singular values of $\mathbf{B}_i$ sorted in decreasing order such that $\sigma_{i1} \ge \sigma_{i2} \ge \cdots \ge \sigma_{ip}$. Then $\mathbf{\Sigma}_{\mathbf{B}_i} = \mathrm{diag}(\sigma_{i1}, \sigma_{i2}, \cdots, \sigma_{ip})$.
With Lemma \ref*{lem:kp property} and Lemma 2.2, (\ref*{eq:proof2}) can be reformulated as 
\begin{equation*}
        \begin{split}
        \mathbf{A} 
        &= \sum_{i=1}^{k}(\mathbf{U}_i \mathbf{\Sigma}_{\mathbf{B}_i} \mathbf{V}_i^\top) \otimes (\mathbf{I}_{m_2} \mathbf{C}_i \mathbf{I}_{n_2}) +\mathbf{E}_k \\
        &= \sum_{i=1}^{k}(\mathbf{U}_i \otimes \mathbf{I}_{m_2})(\mathbf{\Sigma}_{\mathbf{B}_i} \otimes \mathbf{C}_i)(\mathbf{V}_i^\top \otimes \mathbf{I}_{n_2})+ \mathbf{E}_k\\
        &= \sum_{i=1}^{k}\mathbf{U}_i \ltimes \mathbf{\Sigma}_i \ltimes \mathbf{V}_i^\top + \mathbf{E}_k,
        \end{split}
\end{equation*}
where $\mathbf{\Sigma}_i = \mathbf{\Sigma}_{\mathbf{B}_i} \otimes \mathbf{C}_i \in \mathbb{R}^{m_1 m_2 \times n_1 n_2}$ is a block-diagonal matrix whose diagonal blocks are given by $\mathbf{S}_{i1} = \sigma_{i1}\mathbf{C}_i$, $\mathbf{S}_{i2} = \sigma_{i2}\mathbf{C}_i$, $\cdots$, $\mathbf{S}_{ip} = \sigma_{ip}\mathbf{C}_i$.
Since $\sigma_{ij} (j=1, 2, \cdots, p)$ are non-negative scalars, norm properties give $\|\sigma_{ij} \mathbf{C}_i\|_F = \sigma_{ij} \|\mathbf{C}_i\|_F$. Consequently, 
\[
\|\sigma_{i1} \mathbf{C}_i\|_F \ge \|\sigma_{i2} \mathbf{C}_i\|_F \ge \cdots \ge \|\sigma_{ip} \mathbf{C}_i\|_F,
\]
or equivalently 
\[
\|\mathbf{S}_{i1} \|_F \ge \|\mathbf{S}_{i2}\|_F \ge \cdots \ge \|\mathbf{S}_{ip}\|_F.
\]
\end{proof}
\subsection{Proof of Theorem 4.3}\label{app:proof_mstpsvd}
\begin{proof}
        We first assume that each frontal slice $\bar{\mathcal{A}}^{(j)}$ ($j=1,2,\dots,l$) admits the multi-term STP-SVD decomposition derived in Theorem 4.2, i.e.,
\[
\bar{\mathcal{A}}^{(j)} = \sum_{i=1}^{k}\bar{\mathcal{U}_i}^{(j)} \ltimes \bar{\mathcal{S}_i}^{(j)} \ltimes \bigl(\bar{\mathcal{V}_i}^{(j)}\bigr)^\top + \bar{\mathcal{E}_k}^{(j)}.
\]
        Based on the definition of $\mathrm{fold}$ and $\mathrm{bdiag}$ operators, we have 
\begin{equation*}
                \begin{split}
\mathcal{\bar{A}} 
&= \mathrm{fold}(\mathrm{bdiag}(\mathcal{\bar{A}})) \\
&=\mathrm{fold}\begin{bmatrix}
                \bar{\mathcal{A}}^{(1)} & & & \\
                & \bar{\mathcal{A}}^{(2)} & & \\
                & & \ddots & \\
                & & & \bar{\mathcal{A}}^{(l)}
                \end{bmatrix} \\
&=\mathrm{fold} \left[\sum_{i=1}^{k}\mathrm{bdiag}(\bar{\mathcal{U}_i}) \ltimes \mathrm{bdiag}(\bar{\mathcal{S}_i}) \ltimes  \mathrm{bdiag}(\bar{\mathcal{V}_i}^\top) + \mathrm{bdiag}(\bar{\mathcal{E}_k})\right] \\
&=\mathrm{fold} \left[ \sum_{i=1}^{k}  (\mathrm{bdiag}(\bar{\mathcal{U}_i})\otimes\mathbf{I}_{m_2} ) \times \mathrm{bdiag}(\bar{\mathcal{S}_i}) \times (\mathrm{bdiag}(\bar{\mathcal{V}_i}^\top) \otimes \mathbf{I}_{n_2})+ \mathrm{bdiag}(\bar{\mathcal{E}_k}) \right] \\
&=\mathrm{fold} \left[ \sum_{i=1}^{k} \mathrm{bdiag}(\bar{\mathcal{U}_i} \otimes\mathcal{I}_{m_2}) \times  \mathrm{bdiag}(\bar{\mathcal{S}_i}) \times \mathrm{bdiag}(\bar{\mathcal{V}_i}^\top \otimes \mathcal{I}_{n_2}) + \mathrm{bdiag}(\bar{\mathcal{E}_k}) \right] \\
&=\mathrm{fold} \left[ \sum_{i=1}^{k} \mathrm{bdiag}(\overline{\mathcal{U}_i \otimes\mathcal{I}_{m_2}}) \times \mathrm{bdiag}(\bar{\mathcal{S}_i}) \times \mathrm{bdiag}(\overline{\mathcal{V}_i^\top \otimes \mathcal{I}_{n_2}}) + \mathrm{bdiag}(\bar{\mathcal{E}_k})\right].                     
                \end{split}
\end{equation*}
Then,\begin{equation*}
        \begin{split}
\mathcal{A} 
        = L^{-1} (\bar{\mathcal{A}}) 
        = \sum_{i=1}^{k} \mathcal{U}_i \ltimes_L  \mathcal{S}_i \ltimes_L \mathcal{V}_i^\top + \mathcal{E}_k.
        \end{split}
\end{equation*}  
Since $(\bar{\mathcal{U}_i}^{(j)})^\top$ is orthogonal, we have,
\begin{equation*}
        \begin{split}
L(\mathcal{U}_i^\top *_L \mathcal{U}_i)
        &=L(\mathcal{U}_i^\top \ltimes_L \mathcal{U}_i) \\
        &= \mathrm{fold}\left[\mathrm{bdiag}(\bar{\mathcal{U}_i}^\top) \ltimes \mathrm{bdiag}(\bar{\mathcal{U}_i}) \right] \\
        &=\mathrm{fold} 
                \begin{bmatrix}
                (\bar{\mathcal{U}_i}^{(1)})^\top \ltimes \bar{\mathcal{U}_i}^{(1)}& & & \\
                & (\bar{\mathcal{U}_i}^{(2)})^\top \ltimes \bar{\mathcal{U}_i}^{(2)}& & \\
                & & \ddots & \\
                & & & (\bar{\mathcal{U}_i}^{(l)})^\top \ltimes \bar{\mathcal{U}_i}^{(l)}
                \end{bmatrix}\\ 
        &=\mathrm{fold} 
                \begin{bmatrix}
                (\bar{\mathcal{U}_i}^{(1)})^\top \times \bar{\mathcal{U}_i}^{(1)}& & & \\
                & (\bar{\mathcal{U}_i}^{(2)})^\top \times \bar{\mathcal{U}_i}^{(2)}& & \\
                & & \ddots & \\
                & & & (\bar{\mathcal{U}_i}^{(l)})^\top \times \bar{\mathcal{U}_i}^{(l)}
                \end{bmatrix}\\ 
        &=\mathrm{fold}\left[\mathrm{bdiag}(\bar{\mathcal{I}}_{m_1 m_1 l})\right] \\
        &=\bar{\mathcal{I}}_{m_1 m_1 l}.
        \end{split}
\end{equation*}
By analogous arguments, $L\bigl(\mathcal{U}_i *_L \mathcal{U}_i^\top\bigr) = \bar{\mathcal{I}}_{m_1 m_1 l}$, which verifies that $\mathcal{U}_i$ is an orthogonal tensor. Following identical reasoning, $\mathcal{V}_i$ is also orthogonal.

Suppose that the transform matrix $\mathbf{L}$ satisfies $\mathbf{L}^\mathrm{H}\mathbf{L} = \mathbf{L}\mathbf{L}^\mathrm{H} = \rho\mathbf{I}_{l}$ and $\mathbf{L}^{-1} =  \mathbf{L}^\mathrm{H}/\rho$ for some constant $\rho > 0$. For the approximation error tensor $\mathcal{E}_k$, its squared Frobenius norm satisfies
\begin{equation}
\|\mathcal{E}_k\|_F^2 = \left(\frac{1}{\sqrt{\rho}} \|\mathrm{bdiag}(\bar{\mathcal{E}_k})\|_F\right)^2 = \frac{1}{\rho} \sum_{j=1}^{l} \|\bar{\mathcal{E}_k}^{(j)}\|_F^2 = \frac{1}{\rho} \sum_{j=1}^{l} \sum_{i=k+1}^v \bigl(\hat{\sigma}_i^{(j)}\bigr)^2,  
\end{equation}
where the last equality is from (\ref*{eq:mstpsvd_matrix_error}), and $\hat{\sigma}_i^{(j)}$ is the $i$-th singular value of $\mathscr{R}(\bar{\mathcal{A}}^{(j)})$ and $v = \min\{m_1 n_1, m_2 n_2\}$.
\end{proof}

\subsection{Proof of Theorem 5.1}\label{app:proof_mrstpsvd}
\begin{proof}
Let $\mathcal{A}_{\text{MSTP}}$ denotes the exact  MSTP-SVD approximation of tensor \(\mathcal{A}\), which obeys the deterministic error bound from Theorem 4.2: 
         \begin{equation}\label{eq:mrstpsvd_error1}
         \left\|\mathcal{A}-\mathcal{A}_{\text{MSTP}}\right\|_{F}^{2} = \frac{1}{\rho}\sum_{j=1}^{l}\sum_{i=k+1}^{v} (\hat{\sigma}_i^{(j)})^2.       
         \end{equation}
The tensor \(\tilde{\mathcal{A}}\) produced by the MRSTP-SVD algorithm is built upon the exact deterministic MSTP-SVD approximation \(\mathcal{A}_{\text{MSTP}}\). To cut the heavy computational cost of full SVD on each rearranged matrix \(\mathscr{R}(\bar{\mathcal{A}}^{(j)})\), we substitute the exact SVD routine with randomized projection and power iteration. Given this two-stage construction flow, we can naturally decompose the overall reconstruction error into two additive residual parts as
\begin{equation}\label{eq:mrstpsvd_error2}
\mathcal{A}-\tilde{\mathcal{A}} = \left(\mathcal{A}-\mathcal{A}_{\text{MSTP}}\right) + \left(\mathcal{A}_{\text{MSTP}}-\tilde{\mathcal{A}}\right),        
\end{equation}
where the first term corresponds to fixed truncation residual of multi-term decomposition, and the second term represents the error induced by randomized subspace approximation. As given in Theorem 2 of \cite{qin2024nonconvex}, the error bound for the standalone randomized T-SVD approximation satisfies
\begin{equation}\label{eq:mrstpsvd_error3}
\mathbb{E}\left\|\mathcal{A}_{\text{MSTP}}-\tilde{\mathcal{A}}\right\|_{F}^{2} \leq \frac{1}{\rho} \sum_{j=1}^{l} \left(1+\frac{k}{p-1} (\tau_k^{(j)})^{4q}\right) \left(\sum_{i=k+1}^{v} (\hat{\sigma}_i^{(j)})^2 \right).
\end{equation}
Combining (\ref*{eq:mrstpsvd_error1}), (\ref*{eq:mrstpsvd_error2}) and (\ref*{eq:mrstpsvd_error3}), we apply the parallelogram identity for the Frobenius norm and linearity of expectation to expand the total expected error:
\begin{equation}
\begin{aligned}
\mathbb{E}\left\|\mathcal{A}-\tilde{\mathcal{A}}\right\|_{F}^{2}
&= \mathbb{E}\left\| \left(\mathcal{A}-\mathcal{A}_{\text{MSTP}}\right) + \left(\mathcal{A}_{\text{MSTP}} - \tilde{\mathcal{A}}\right) \right\|_{F}^{2} \\
&\leq 2\mathbb{E}\left\|\mathcal{A}-\mathcal{A}_{\text{MSTP}}\right\|_{F}^{2} + 2\mathbb{E}\left\|\mathcal{A}_{\text{MSTP}}-\tilde{\mathcal{A}}\right\|_{F}^{2} \\
&= \frac{2}{\rho} \sum_{j=1}^{l}\left[\left(2+\frac{k}{p-1}\left(\tau_{k}^{(j)}\right)^{4q}\right)\left(\sum_{i=k+1}^{v}\left(\hat{\sigma}_{i}^{(j)}\right)^{2}\right)\right].
\end{aligned}
\end{equation}
\end{proof}

\section{Pseudocodes for all truncated variants}\label{app:truncated_algorithm}

\setcounter{equation}{0}
\setcounter{table}{0}
\setcounter{figure}{0}
\setcounter{theorem}{0}
\setcounter{lemma}{0}

\renewcommand{\theequation}{S3.\arabic{equation}}
\renewcommand{\thetable}{S3.\arabic{table}}
\renewcommand{\thefigure}{S3.\arabic{figure}}
\renewcommand{\thetheorem}{S3.\arabic{theorem}}
\renewcommand{\thelemma}{S3.\arabic{lemma}}

For completeness and ease of reproducibility, we collect the full pseudocodes for all truncated variants of the STP-SVD framework. Specifically, we detail the truncated STP-SVD for matrices, the truncated multi-term STP-SVD for matrices, the truncated MSTP-SVD for tensors, and the truncated randomized MRSTP-SVD for tensors.
\begin{algorithm}[!ht]
\caption{Truncated STP-SVD of matrices \cite{chen2023tensor}}
\label{alg:truncated stpsvd}
\KwIn{$\mathbf{A} \in \mathbb{R}^{m_1 m_2 \times n_1 n_2}$, truncated parameter $r$.}
\KwOut{$\mathbf{U}, \mathbf{\Sigma}, \mathbf{V}$.}
Calculate matrices $\mathbf{B} \in \mathbb{R}^{m_1 \times n_1}$ and $\mathbf{C} \in \mathbb{R}^{m_2 \times n_2}$ via Lemma 4.1, such that $\mathbf{A} \approx \mathbf{B} \otimes \mathbf{C}$\;
Perform truncated SVD on $\mathbf{B}$: $[\mathbf{U}_\mathbf{B}, \mathbf{\Sigma}_{\mathbf{B}}, \mathbf{V}_\mathbf{B}] = \mathrm{svds}(\mathbf{B}, r)$\;
\Return{$\mathbf{U} = \mathbf{U}_\mathbf{B}, \mathbf{\Sigma} = \mathbf{\Sigma_B} \otimes \mathbf{C}, \mathbf{V}=\mathbf{V}_\mathbf{B}$.}
\end{algorithm}
\begin{algorithm}[!ht]
\caption{Truncated multi-term STP-SVD of matrices}
\label{alg: truncated mstpsvd of matrices}
\KwIn{$\mathbf{A} \in \mathbb{R}^{m_1 m_2 \times n_1 n_2}$, the number of terms $k$, truncated parameter $r$.}
\KwOut{$\mathbf{U}_i$, $\mathbf{\Sigma}_i$, $\mathbf{V}_i$.}
Calculate matrices $\mathbf{B}_i \in \mathbb{R}^{m_1 \times n_1}$ and $\mathbf{C}_i \in \mathbb{R}^{m_2 \times n_2}$ ($i=1,\ldots,k$) via Lemma 4.2, such that
$\mathbf{A} \approx \sum_{i=1}^{k} \mathbf{B}_i \otimes \mathbf{C}_i$\;
\For{$i = 1$ \KwTo $k$}{
    Perform truncated SVD on $\mathbf{B}_i$: $[\mathbf{U}_i, \mathbf{\Sigma}_{\mathbf{B}_i}, \mathbf{V}_i] = \operatorname{svds}(\mathbf{B}_i,r)$\;
    $\mathbf{\Sigma}_i = \mathbf{\Sigma}_{\mathbf{B}_i} \otimes \mathbf{C}_i$\;
}
\Return{$\mathbf{U}_i, \mathbf{\Sigma}_i, \mathbf{V}_i$.}
\end{algorithm}
\begin{algorithm}[!ht]
\caption{Truncated MSTP-SVD method of tensors (TMSTP-SVD)}
\label{alg:truncated mstpsvd}
\KwIn{$\mathcal{A} \in \mathbb{R}^{m_1 m_2 \times n_1 n_2 \times l}$, the number of terms $k$, truncated rank matrix $\mathbf{R}$.}
\KwOut{$\mathcal{U}_i$, $\mathcal{S}_i$, $\mathcal{V}_i$.}
Obtain $\bar{\mathcal{A}}$ by applying an invertible linear transform $L$ on $\mathcal{A}$\;
\For{$j = 1$ \KwTo $l$}{
    Approximate the $j$-th frontal slice by Lemma 4.2: $\bar{\mathcal{A}}^{(j)} \approx \sum_{i=1}^{k} \mathbf{B}_i^{(j)} \otimes \mathbf{C}_i^{(j)}$\;
    \For{$i = 1$ \KwTo $k$}{
        Compute the truncated SVD of $\mathbf{B}_i^{(j)}$: $[\mathbf{U}_i^{(j)}, \mathbf{\Sigma}_{\mathbf{B}_i}^{(j)}, \mathbf{V}_i^{(j)}] = \mathrm{svds}\left(\mathbf{B}_i^{(j)},R_{ij}\right)$\;
        $\mathbf{\Sigma}_i^{(j)} = \mathbf{\Sigma}_{\mathbf{B}_i}^{(j)} \otimes \mathbf{C}_i^{(j)}$\;
        Store $\mathbf{U}_i^{(j)}$, $\mathbf{\Sigma}_i^{(j)}$, $\mathbf{V}_i^{(j)}$ into $\mathcal{U}_i$, $\mathcal{S}_i$, $\mathcal{V}_i$, respectively\;
    }
}
\Return{$\mathcal{U}_i = L^{-1}(\mathcal{U}_i)$, $\mathcal{S}_i = L^{-1}(\mathcal{S}_i)$, $\mathcal{V}_i = L^{-1}(\mathcal{V}_i)$.}
\end{algorithm}

\begin{algorithm}[!ht]
\caption{Truncated MRSTP-SVD method of tensors (TMRSTP-SVD)}
\label{alg:truncated mrstpsvd}
\KwIn{$\mathcal{A} \in \mathbb{R}^{m_1 m_2 \times n_1 n_2 \times l}$, the number of terms $k$, oversampling parameter $s \ge 0$, iteration parameter $q \ge 0$,  truncated rank matrix $\mathbf{R}$.}
\KwOut{$\mathcal{U}_i$, $\mathcal{S}_i$, $\mathcal{V}_i$.}
Generate a Gaussian random tensor $\mathcal{G} \in \mathbb{R}^{n_1 n_2 \times (k+s) \times l}$\;
Compute $\bar{\mathcal{A}} = L(\mathcal{A})$ and $\bar{\mathcal{G}} = L(\mathcal{G})$\;
\For{$j = 1$ \KwTo $l$}{
    Obtain $\mathscr{R}(\bar{\mathcal{A}}^{(j)})$ by reorganizing the blocks of $\bar{\mathcal{A}}^{(j)}$\;
    Compute $\mathbf{Y} = \bigl(\mathscr{R}(\bar{\mathcal{A}}^{(j)}) \mathscr{R}(\bar{\mathcal{A}}^{(j)})^\top\bigr)^q \mathscr{R}(\bar{\mathcal{A}}^{(j)}) \bar{\mathcal{G}}^{(j)}$\;
    Compute thin-QR factorization $\mathbf{Y} = \mathbf{Q}_j \mathbf{R}$\;
    Compute $\mathbf{B} = \mathbf{Q}_j^\top \mathscr{R}(\bar{\mathcal{A}}^{(j)})$\;
    Compute the SVD of $\mathbf{B}$: $\mathbf{B} = \mathbf{U} \mathbf{S} \mathbf{V}^\top$\;
    Form $\mathbf{U}_k$, $\mathbf{V}_k$, $\mathbf{S}_k$ by truncating $\mathbf{Q}_j\mathbf{U}$, $\mathbf{V}$, $\mathbf{S}$ with $k$\;
    \For{$i = 1$ \KwTo $k$}{
        $\operatorname{vec}(\mathbf{B}_i^{(j)}) = \sqrt{\mathbf{S}_k(i,i)} \, \mathbf{U}_k(:, i)$, \quad$\operatorname{vec}(\mathbf{C}_i^{(j)}) = \sqrt{\mathbf{S}_k(i,i)} \, \mathbf{V}_k(:, i)$
        such that $\bar{\mathcal{A}}^{(j)} \approx \sum_{i=1}^{k} \mathbf{B}_i^{(j)} \otimes \mathbf{C}_i^{(j)}$\;
        Compute the truncated SVD of $\mathbf{B}_i^{(j)}$: $[\mathbf{U}_i^{(j)}, \mathbf{\Sigma}_{\mathbf{B}_i}^{(j)}, \mathbf{V}_i^{(j)}] = \operatorname{svds}\left(\mathbf{B}_i^{(j)},R_{ij}\right)$\;
        $\mathbf{\Sigma}_i^{(j)} = \mathbf{\Sigma}_{\mathbf{B}_i}^{(j)} \otimes \mathbf{C}_i^{(j)}$\;
        Store $\mathbf{U}_i^{(j)}$, $\mathbf{\Sigma}_i^{(j)}$, $\mathbf{V}_i^{(j)}$ into $\mathcal{U}_i$, $\mathcal{S}_i$, $\mathcal{V}_i$, respectively\;
    }
}
\Return{$\mathcal{U}_i = L^{-1}(\mathcal{U}_i)$, $\mathcal{S}_i = L^{-1}(\mathcal{S}_i)$, $\mathcal{V}_i = L^{-1}(\mathcal{V}_i)$.}
\end{algorithm}
\clearpage
\section{Supplementary experimental results}

\setcounter{equation}{0}
\setcounter{table}{0}
\setcounter{figure}{0}
\setcounter{theorem}{0}
\setcounter{lemma}{0}

\renewcommand{\theequation}{S4.\arabic{equation}}
\renewcommand{\thetable}{S4.\arabic{table}}
\renewcommand{\thefigure}{S4.\arabic{figure}}
\renewcommand{\thetheorem}{S4.\arabic{theorem}}
\renewcommand{\thelemma}{S4.\arabic{lemma}}
This section presents supplementary experimental results to further validate the effectiveness and efficiency of the proposed method for compression and completion tasks. We extend the empirical evaluation in the main text to additional benchmark datasets and provide extra qualitative visual comparisons. All results presented here complement the conclusions of the main manuscript and deliver a more comprehensive empirical verification of our approach.

Fig.~\ref*{fig:conv_three_subplot_sup}presents the impact of three invertible linear transforms (DFT, DCT, and ROT) on the compression performance of both deterministic MSTP-SVD and randomized MRSTP-SVD on the Night and Fruit test images. Consistent with the observations in the main text, all three transforms yield comparable reconstruction quality in terms of PSNR and SSIM, 
while DFT consistently achieves the lowest computational overhead across both test images. These supplementary results further validate the rationality of selecting DFT as the default transform in all subsequent experiments.
\begin{figure}[!htbp]
  \centering
    \includegraphics[width=\textwidth]{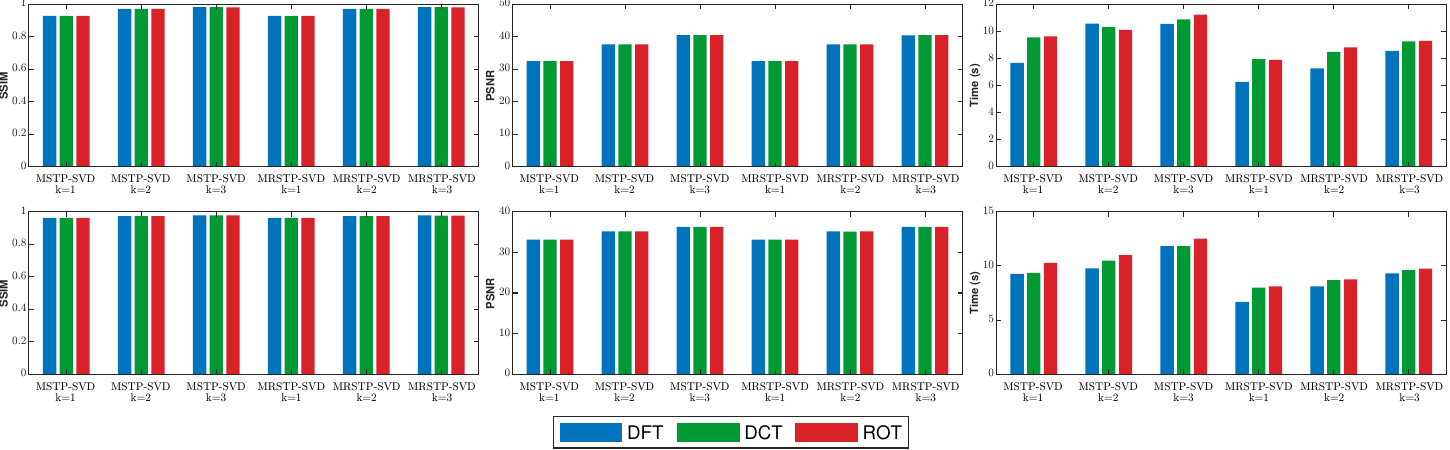}
    \caption{Quantitative metrics (PSNR, SSIM, runtime) of MSTP-SVD and randomized MRSTP-SVD for image compression under invertible transforms (DFT, DCT, ROT). Top: Night; Bottom: Fruit.}
    \label{fig:conv_three_subplot_sup}
  \end{figure}

Fig.~\ref*{fig:image_compression_sup} provides supplementary visual reconstruction comparisons and corresponding quantitative PSNR and runtime measurements on two additional representative test images, covering all competing baselines and the proposed multi-term, truncated, and randomized variants. The block partition sizes and truncation rank settings are identical to those in the main text to ensure fair comparison. Quantitatively, MSTP-SVD (k=3) achieves $7.9$ dB and $3.1$ dB PSNR improvements over the single-term STP-SVD on the two images, respectively, and outperforms the TT-SVD baseline by $5.8-7.3$ dB. The randomized MRSTP-SVD yields comparable reconstruction fidelity to its deterministic counterpart, with PSNR loss less than $0.08$ dB, while reducing the runtime by approximately 22\%. It can be observed that the multi-term schemes consistently outperform single-term baselines and preserve finer local structural details, whereas baseline methods tend to produce over-smoothed outputs. These results corroborate the effectiveness and generalization capability of the proposed methods across diverse image content.
\begin{figure}[!ht]
\centering
\renewcommand{\arraystretch}{0.3}
\setlength\tabcolsep{0.1pt}
\begin{tabular}{@{}ccccccc@{}}
\tiny Original& \tiny TT-SVD & \tiny STP-SVD & \tiny TSTP-SVD &\tiny\makecell[c]{MSTP-SVD\\[-4pt](k=2)} & \tiny\makecell[c]{MSTP-SVD\\[-4pt](k=3)} &\tiny\makecell[c]{TMSTP-SVD\\[-4pt](k=2)} \\
\includegraphics[width=0.672in]{image/night/fig_output/night.jpg} &
\includegraphics[width=0.672in]{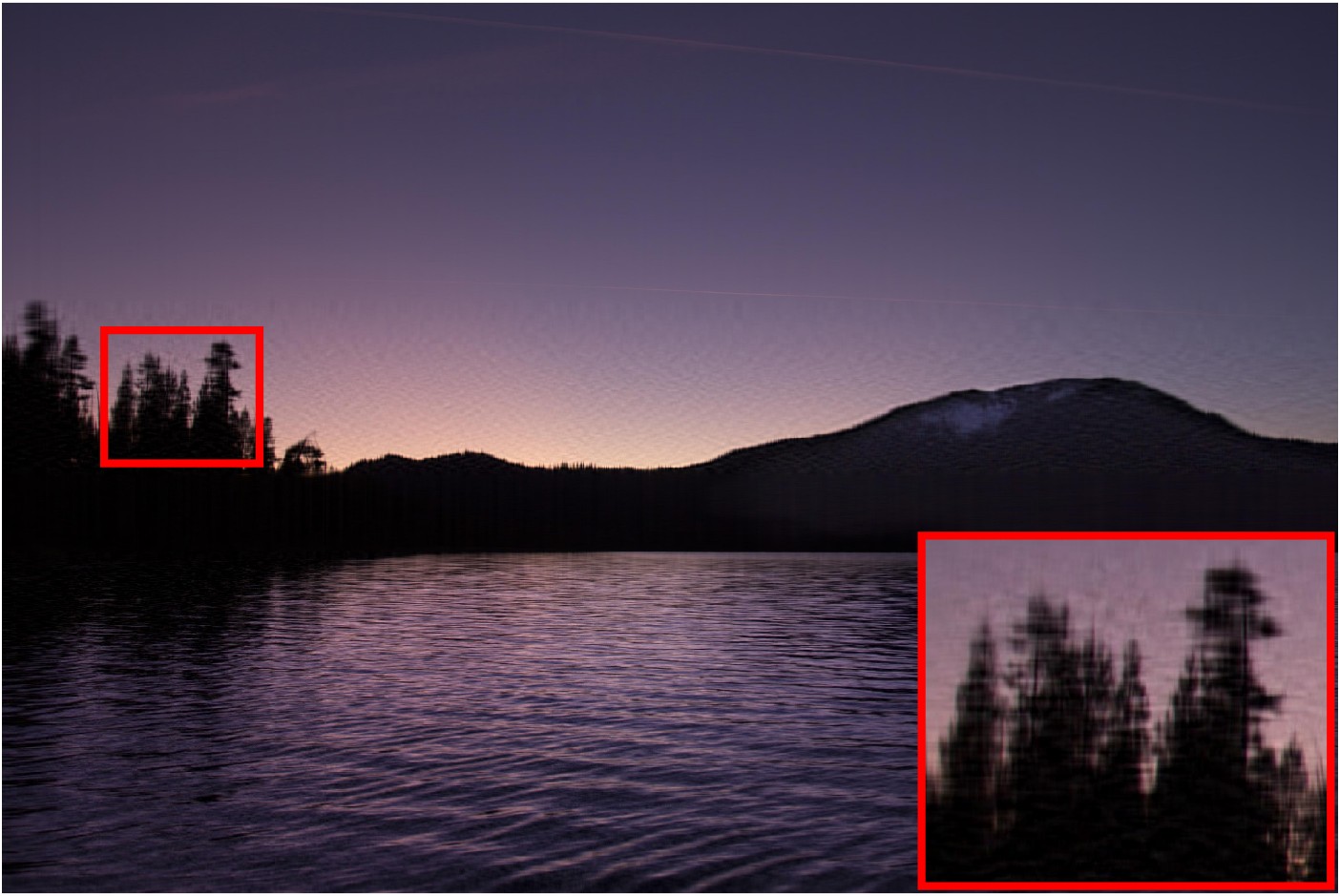} &
\includegraphics[width=0.672in]{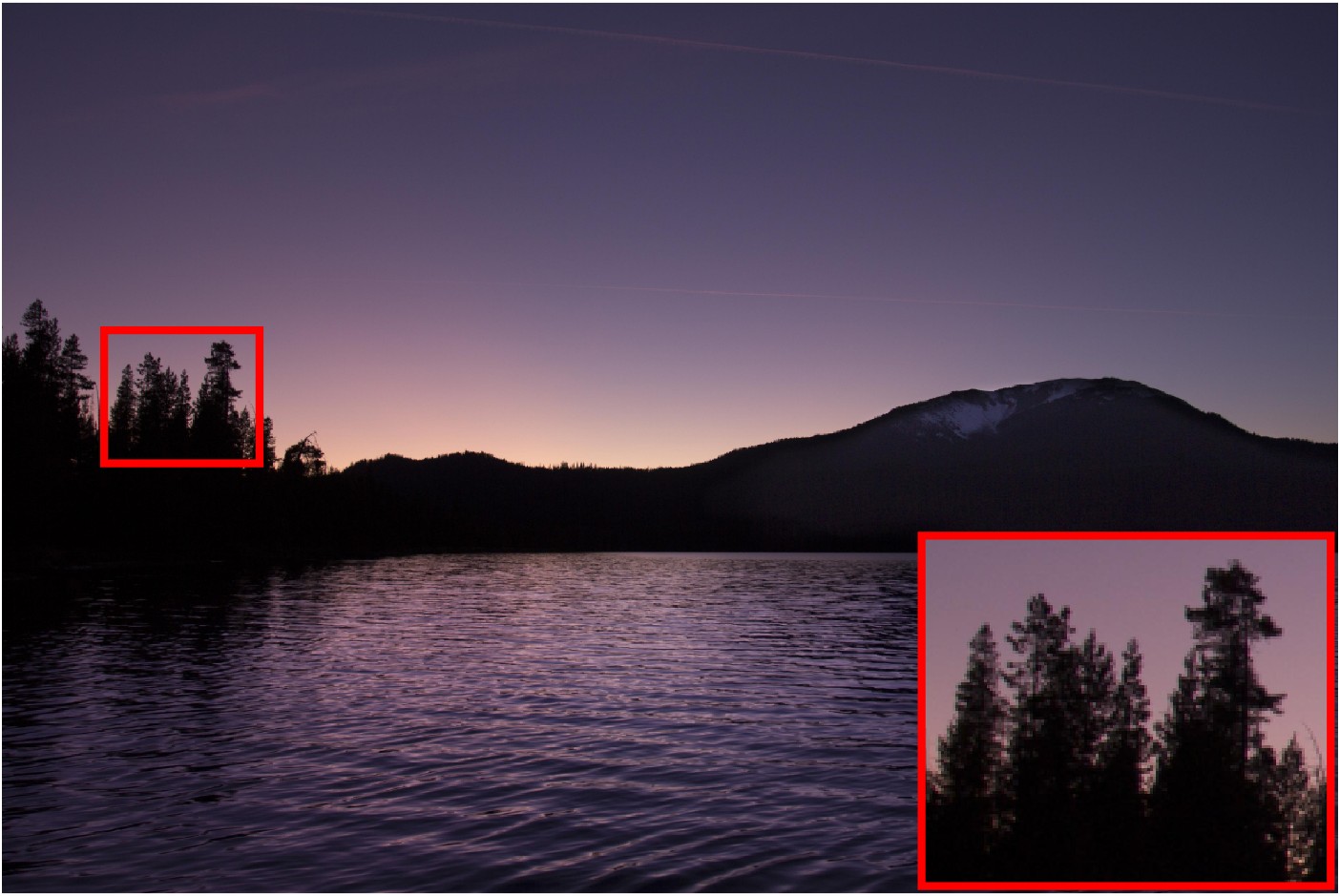} &
\includegraphics[width=0.672in]{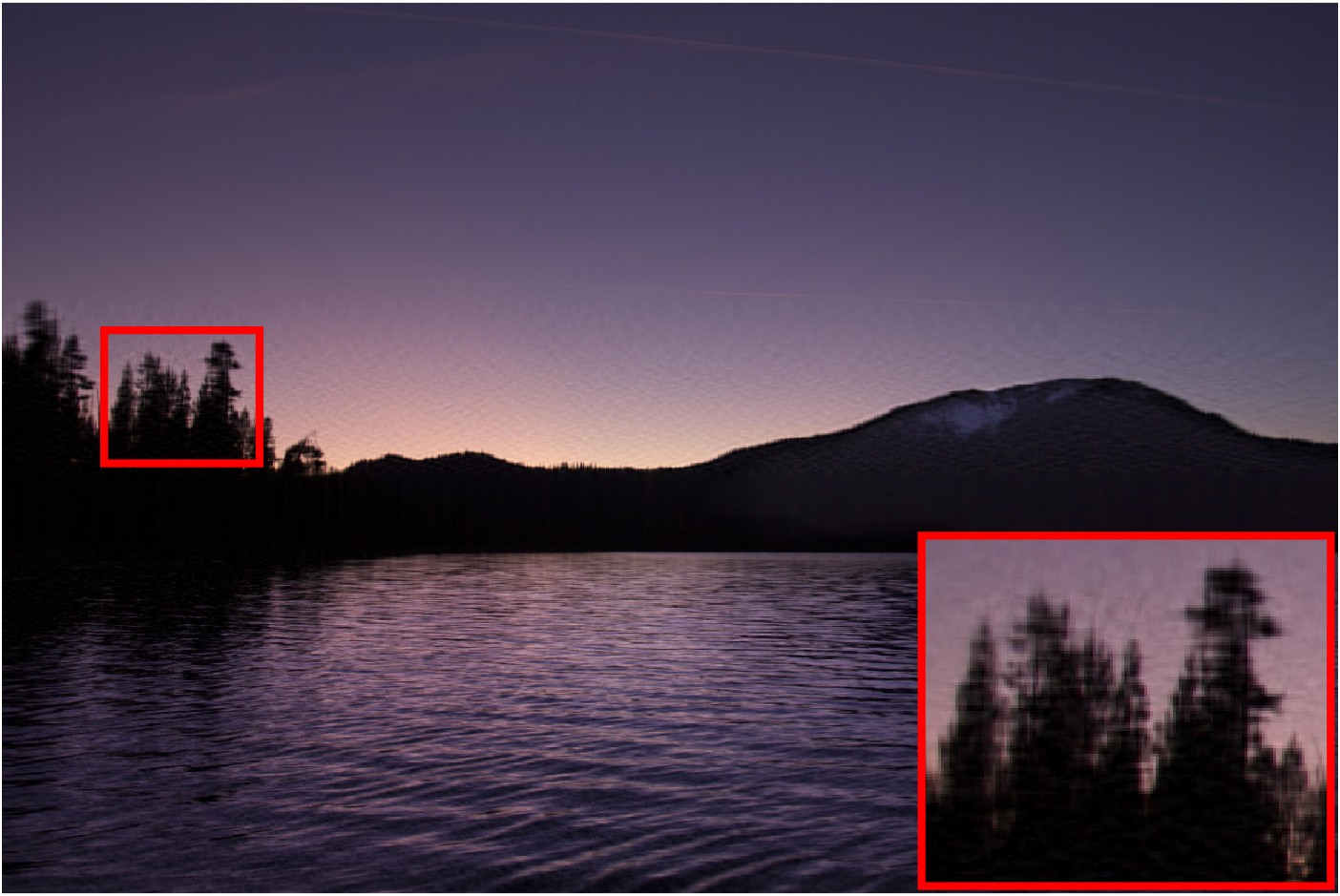} &
\includegraphics[width=0.672in]{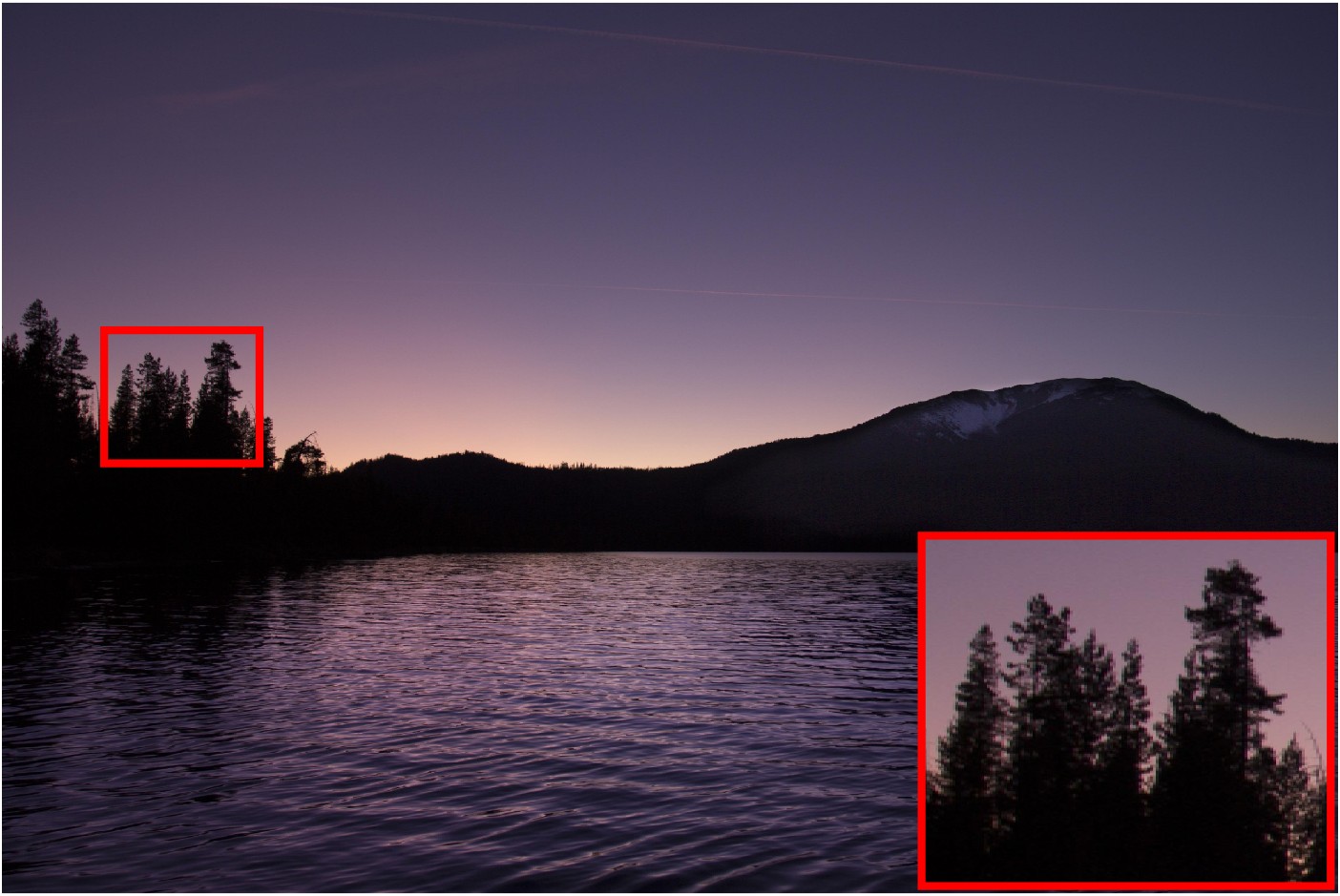} &
\includegraphics[width=0.672in]{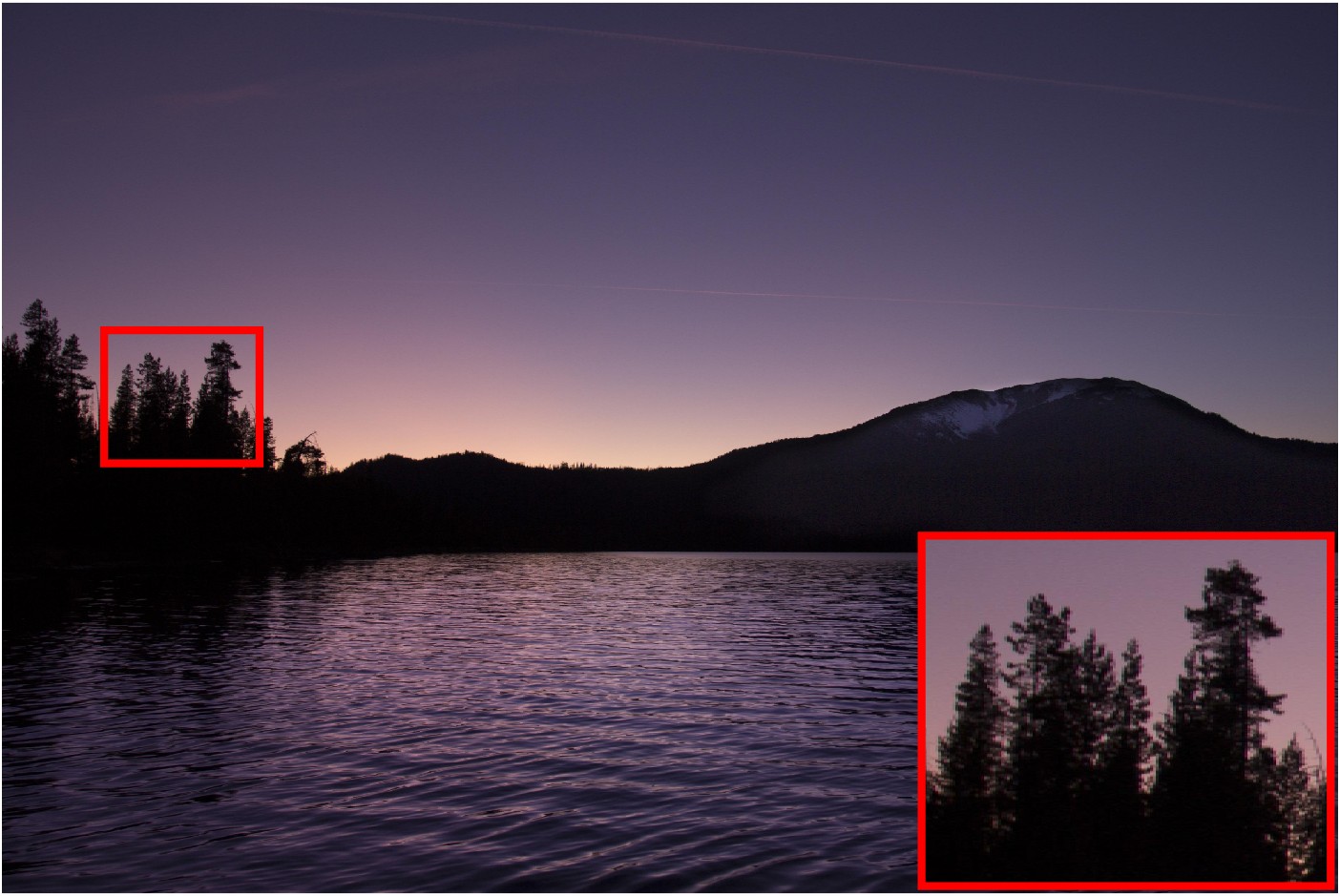} &
\includegraphics[width=0.672in]{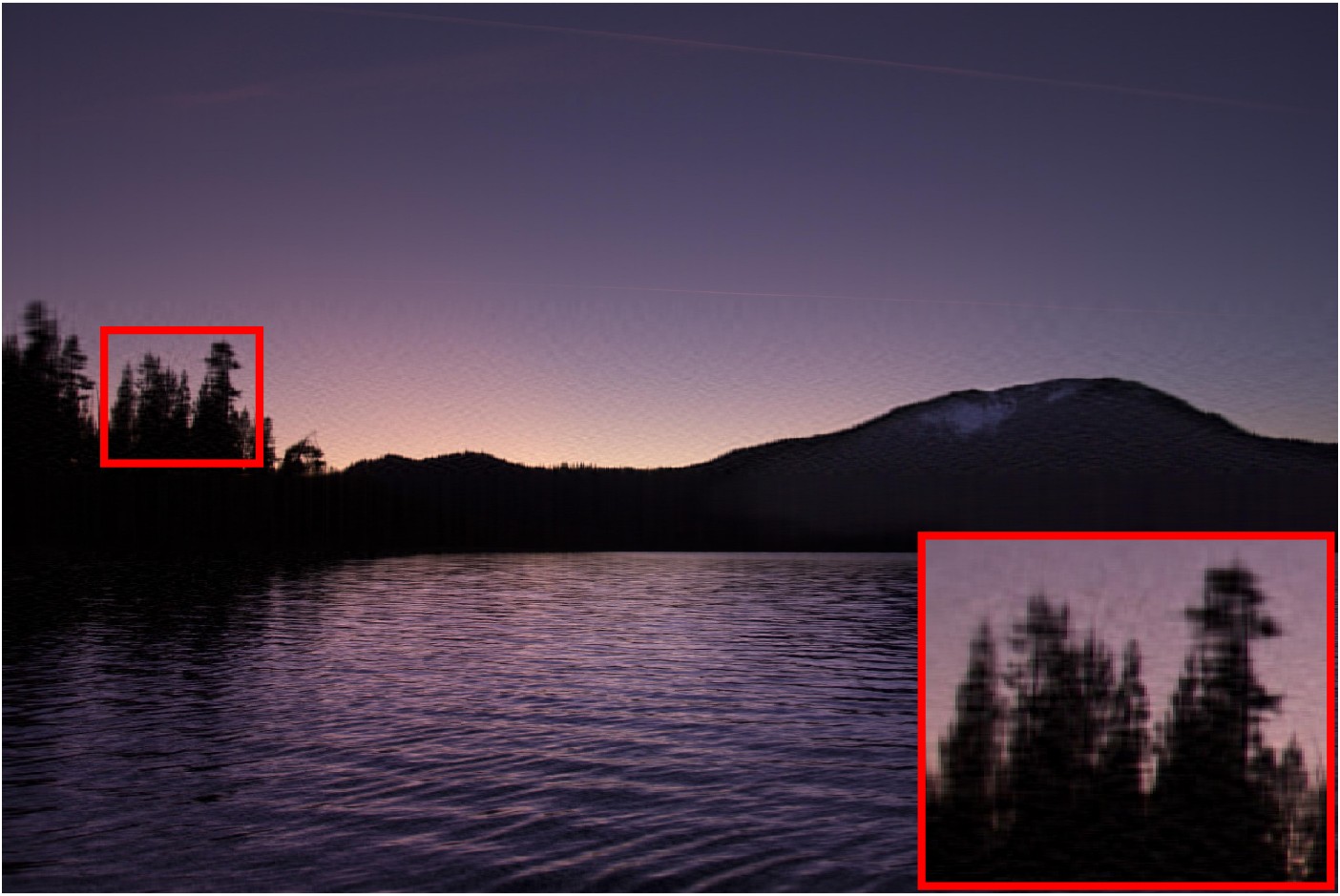} \\
&\tiny PSNR:33.22 & \tiny PSNR:32.53 & \tiny PSNR:30.91& \tiny PSNR:37.67 & \tiny PSNR:40.52 & \tiny PSNR:33.15\\ 
&\tiny Time:19.02s & \tiny Time:9.67s & \tiny Time:9.51s & \tiny Time:11.01s & \tiny Time:11.99s & \tiny Time:10.77s\\ 
\tiny\makecell[c]{TMSTP-SVD\\[-4pt](k=3)} &\tiny\makecell[c]{MRSTP-SVD\\[-4pt](k=1)} & \tiny\makecell[c]{MRSTP-SVD\\[-4pt](k=2)} & \tiny\makecell[c]{MRSTP-SVD\\[-4pt](k=3)}& \tiny\makecell[c]{TMRSTP-SVD\\[-4pt](k=1)} & \tiny\makecell[c]{TMRSTP-SVD\\[-4pt](k=2)} & \tiny\makecell[c]{TMRSTP-SVD\\[-4pt](k=3)}\\
\includegraphics[width=0.672in]{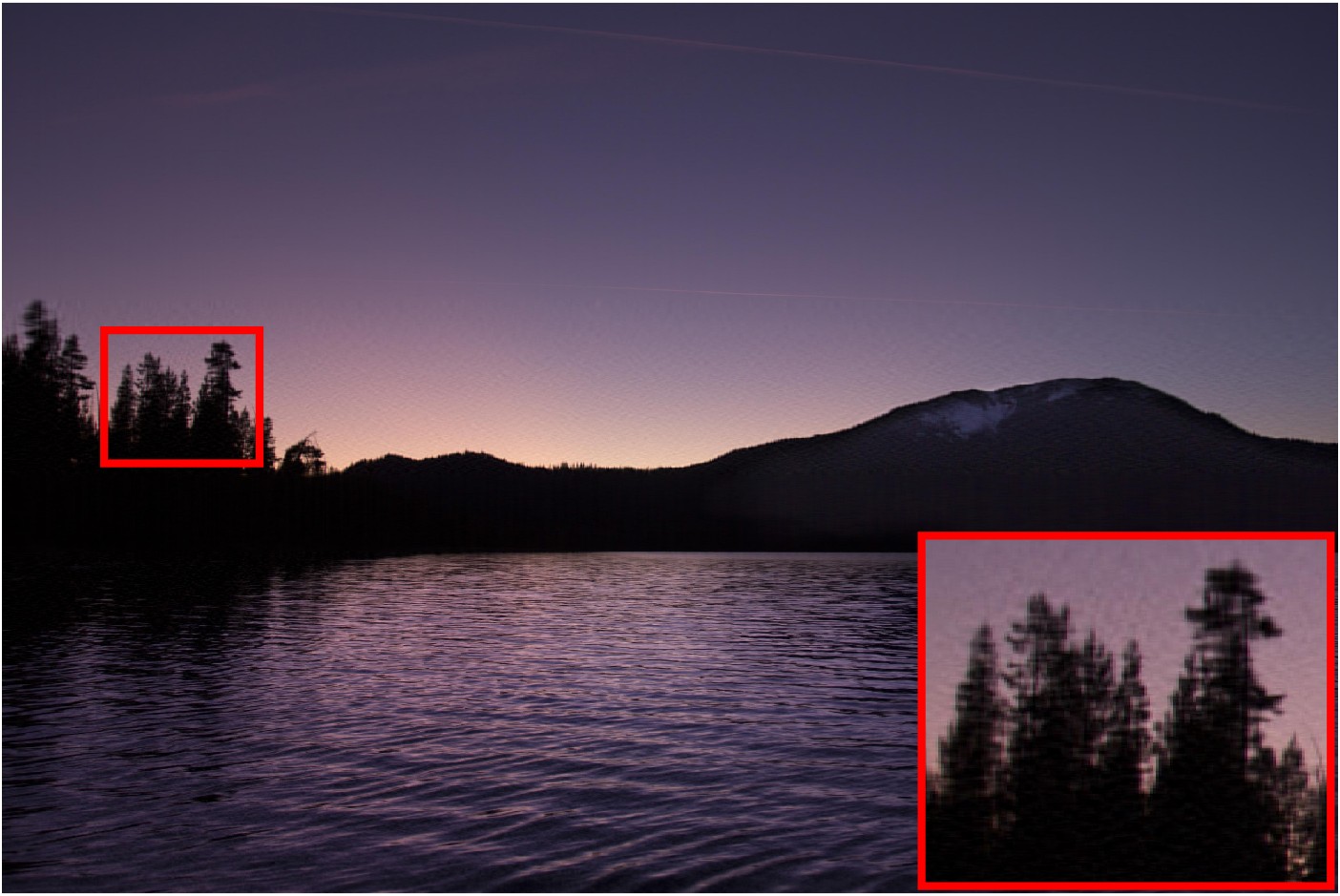} &
\includegraphics[width=0.672in]{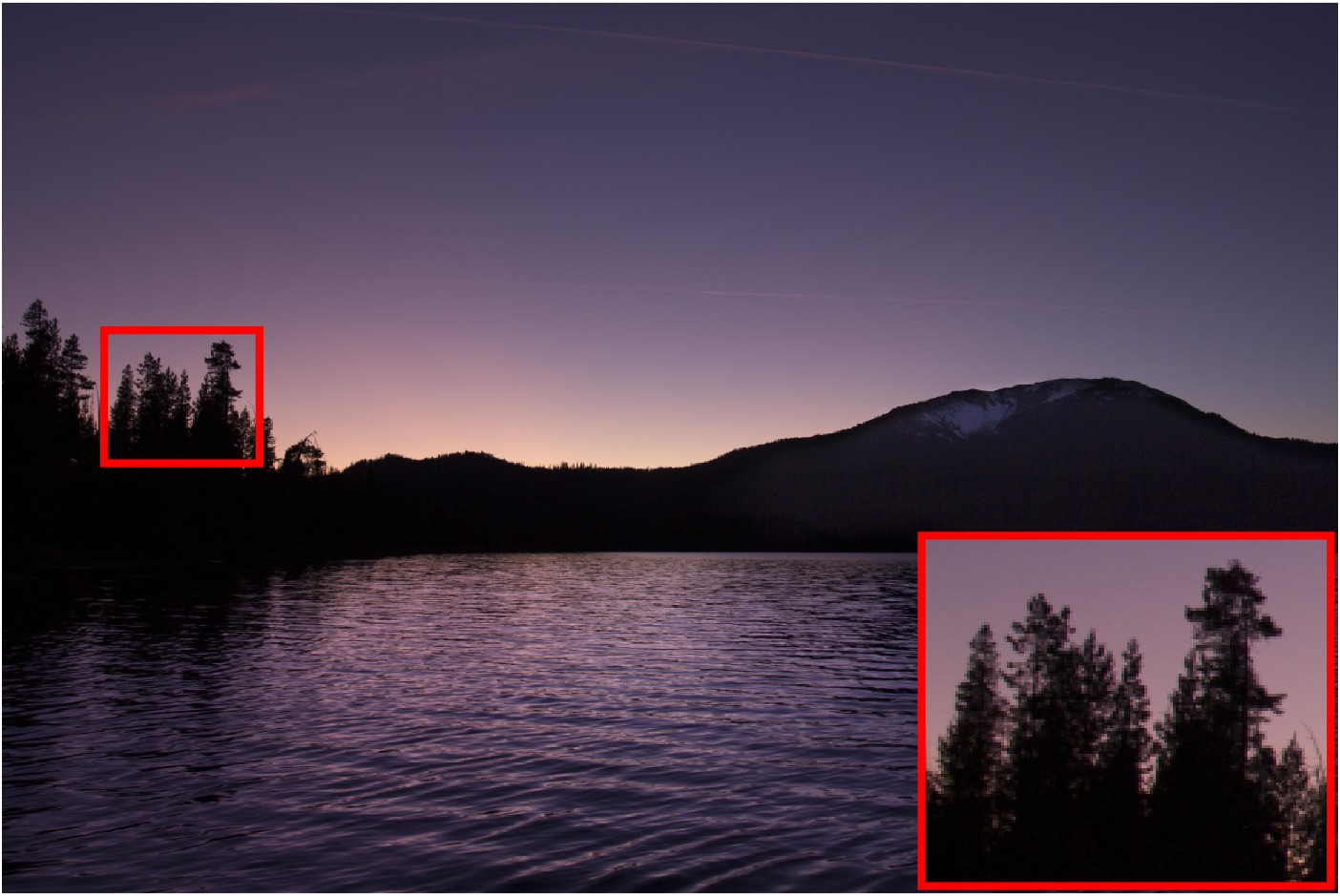} &
\includegraphics[width=0.672in]{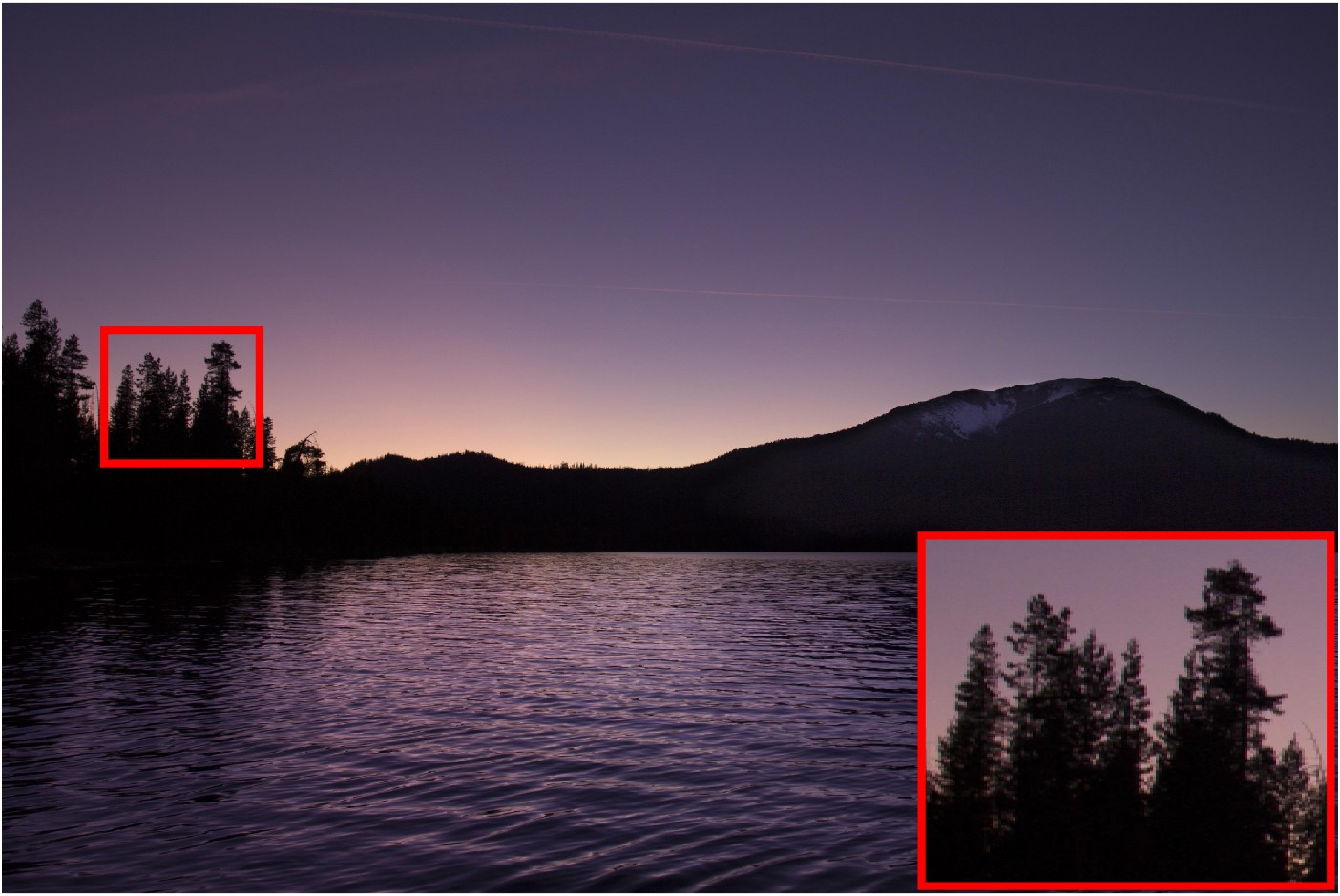} &
\includegraphics[width=0.672in]{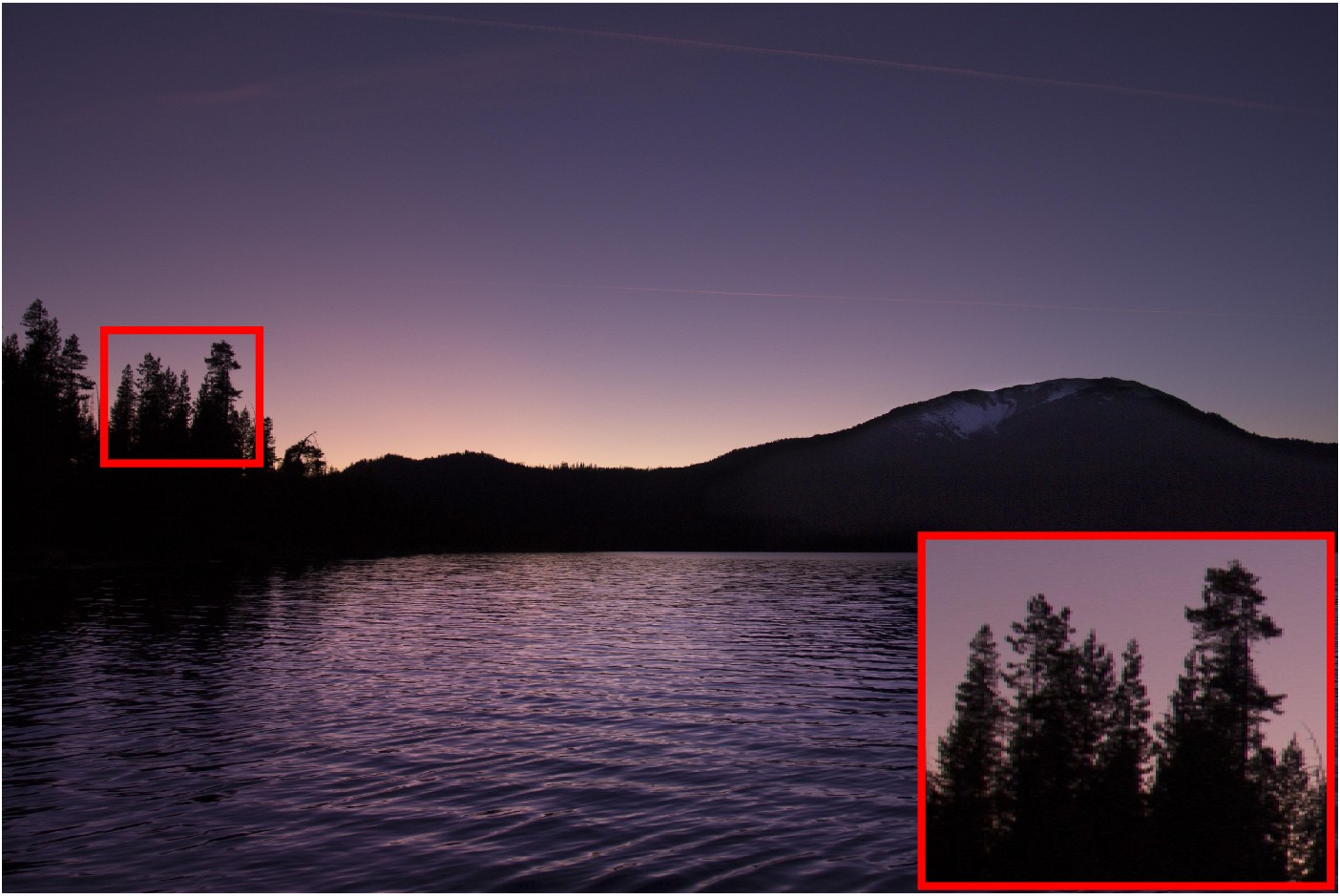}&
\includegraphics[width=0.672in]{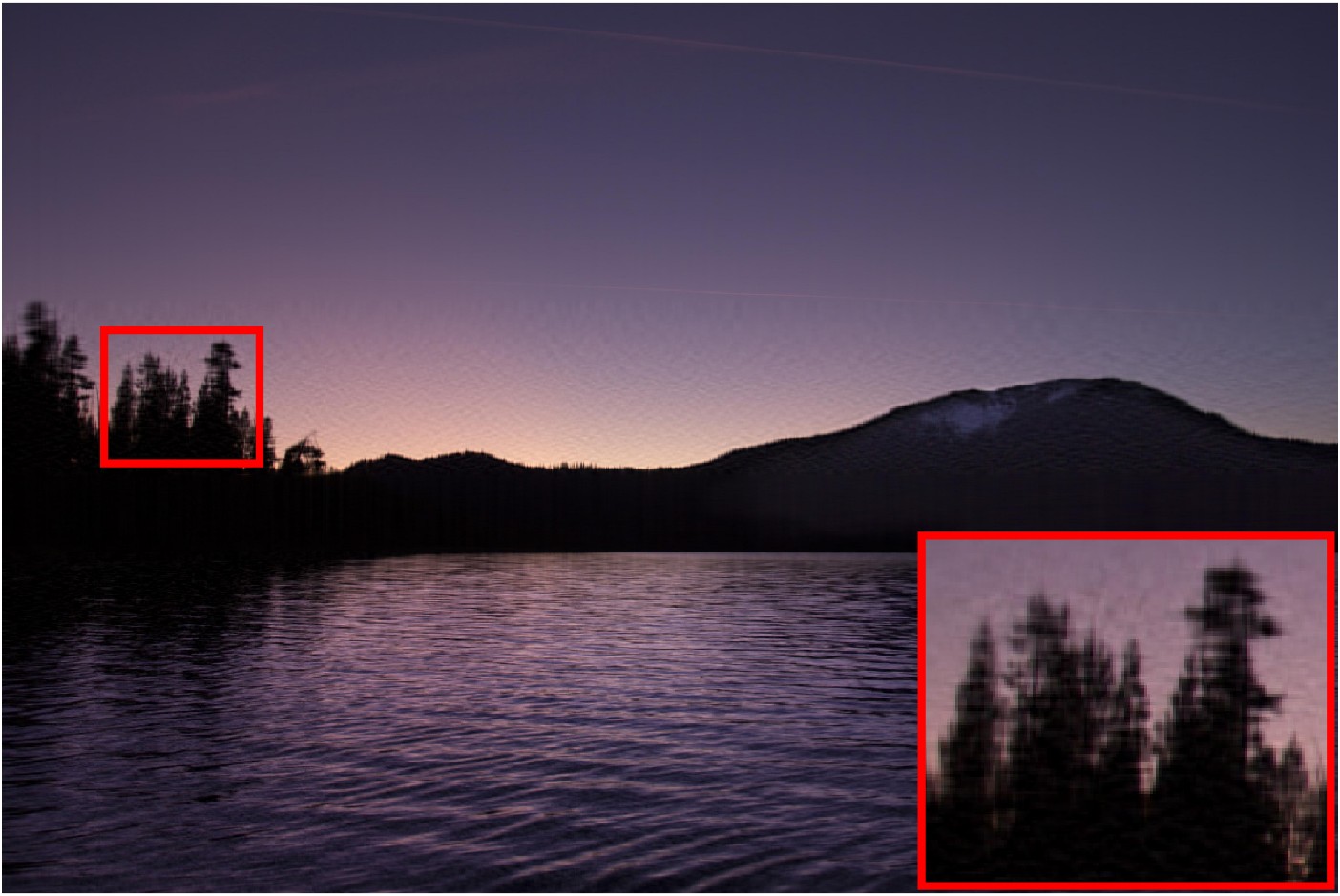} &
\includegraphics[width=0.672in]{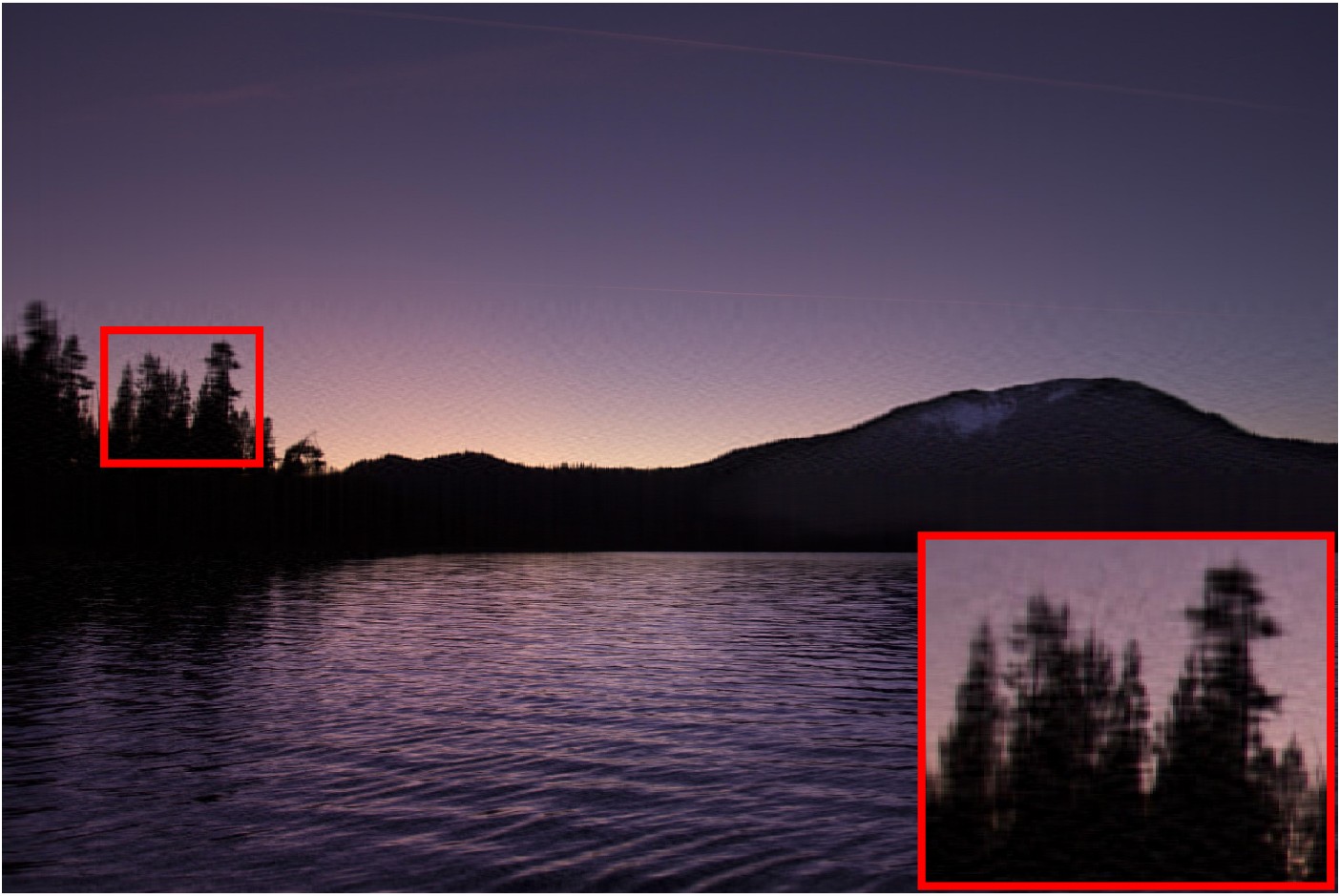} &
\includegraphics[width=0.672in]{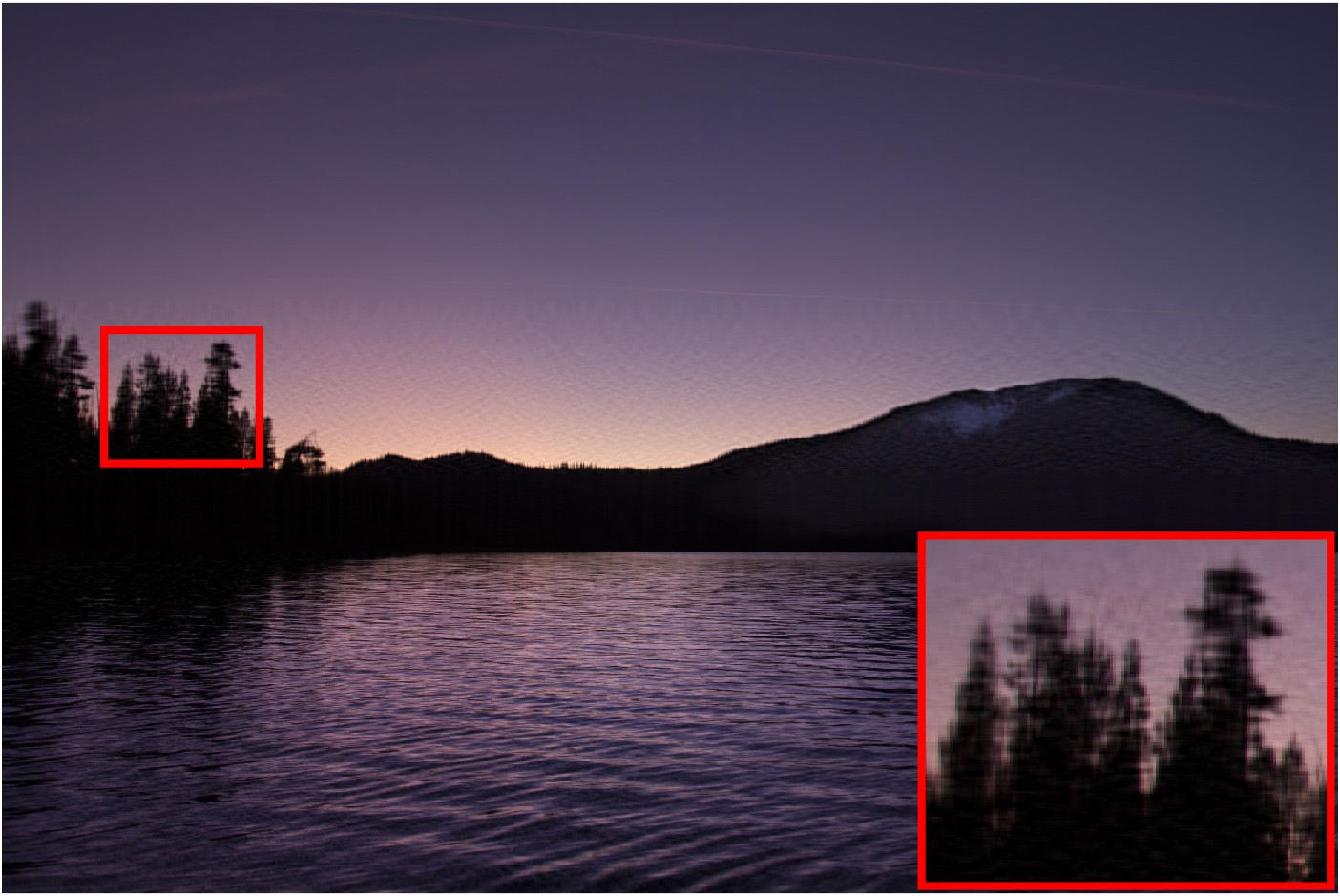} \\
\tiny PSNR:33.79& \tiny PSNR:32.53 & \tiny PSNR:37.60 & \tiny PSNR:40.52 & \tiny PSNR:30.91 & \tiny PSNR:33.13 & \tiny PSNR:33.79\\
\tiny Time: 11.79s & \tiny Time:8.39s & \tiny Time:8.67s & \tiny Time:9.20s & \tiny Time:7.37s & \tiny Time:8.28s & \tiny Time:9.01s \\

\tiny Original &\tiny TT-SVD & \tiny STP-SVD & \tiny TSTP-SVD &\tiny\makecell[c]{MSTP-SVD\\[-4pt](k=2)} & \tiny\makecell[c]{MSTP-SVD\\[-4pt](k=3)} &\tiny\makecell[c]{TMSTP-SVD\\[-4pt](k=2)} \\
\includegraphics[width=0.672in]{image/strawberry/fig_output/strawberry.jpg} &
\includegraphics[width=0.672in]{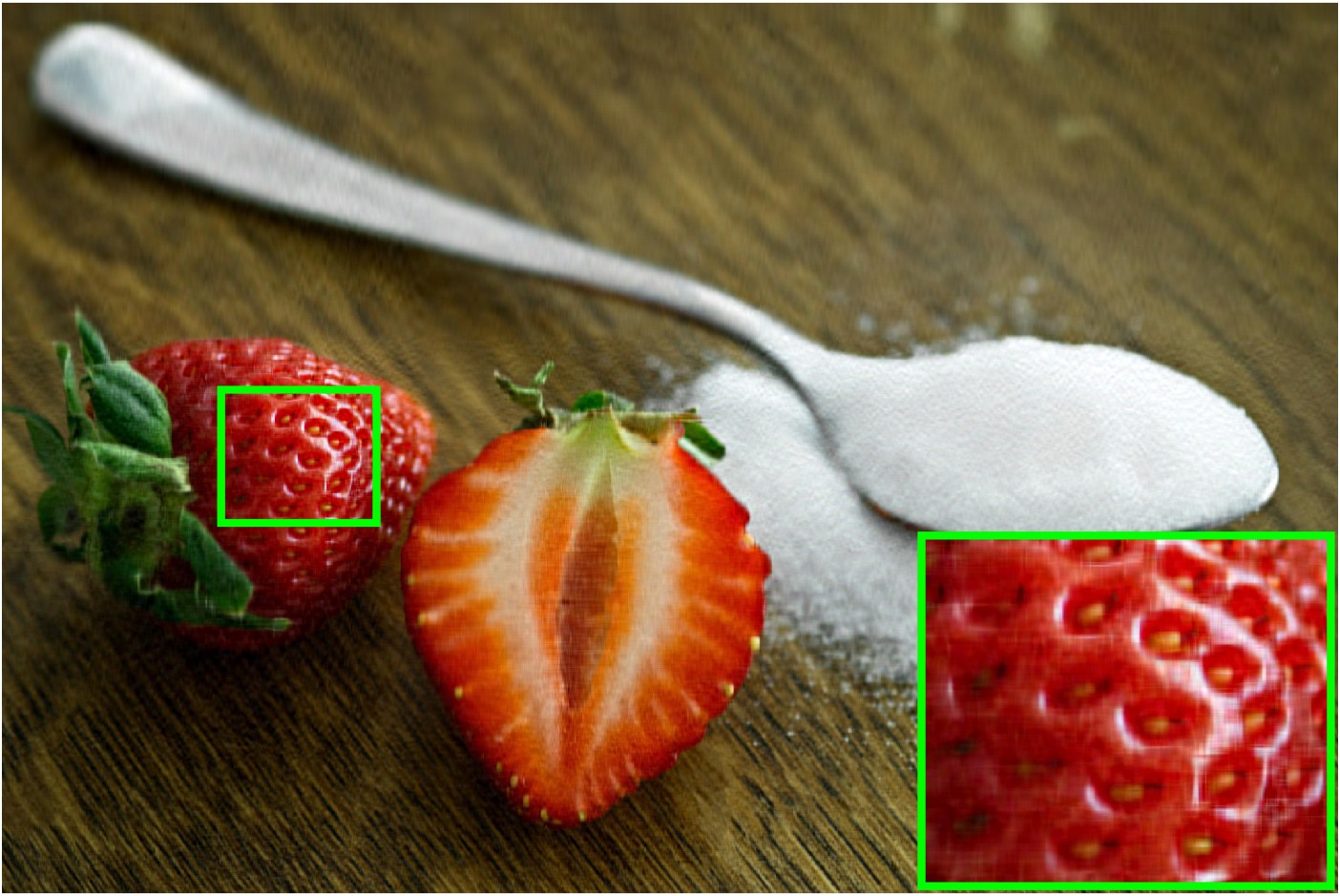} &
\includegraphics[width=0.672in]{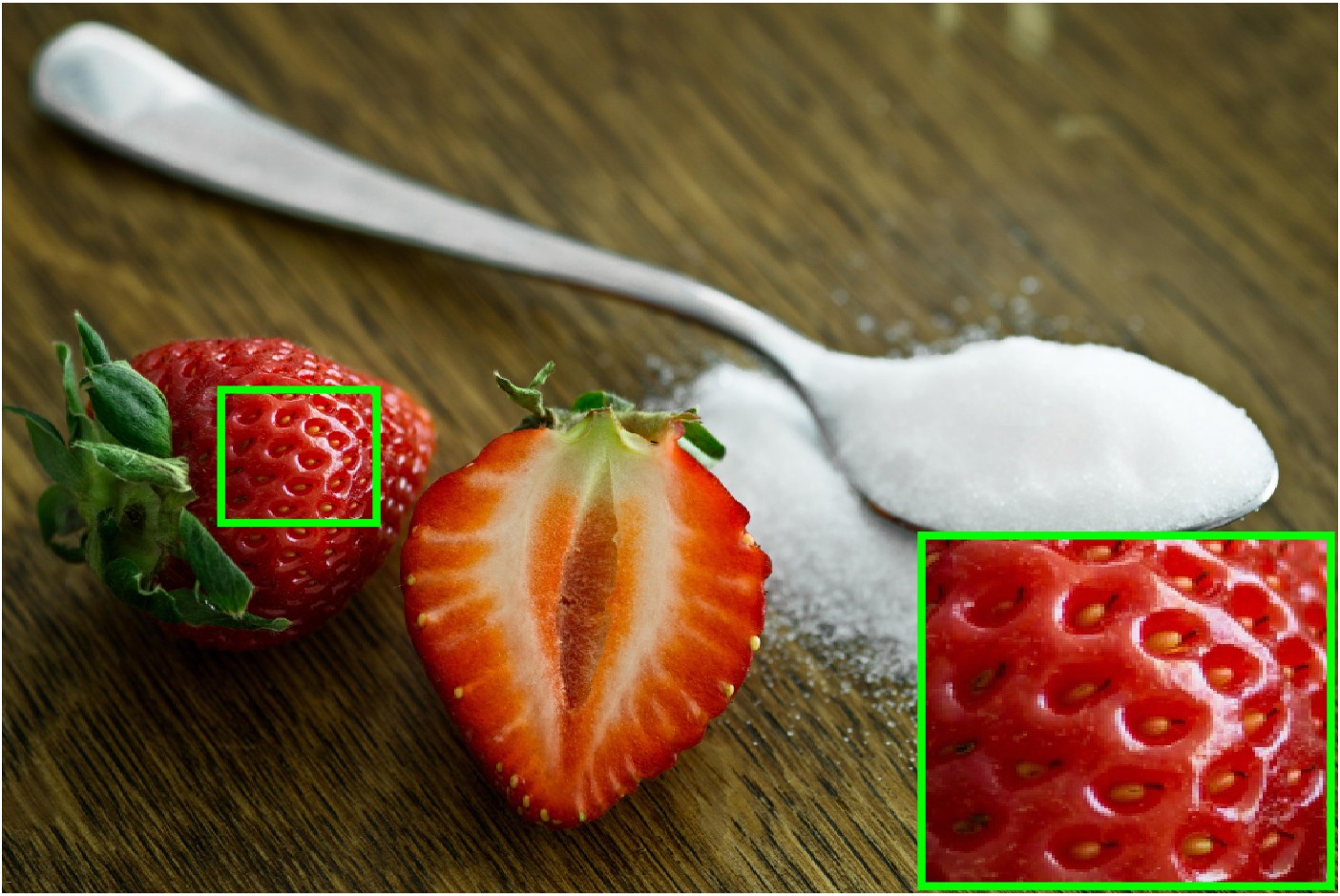} &
\includegraphics[width=0.672in]{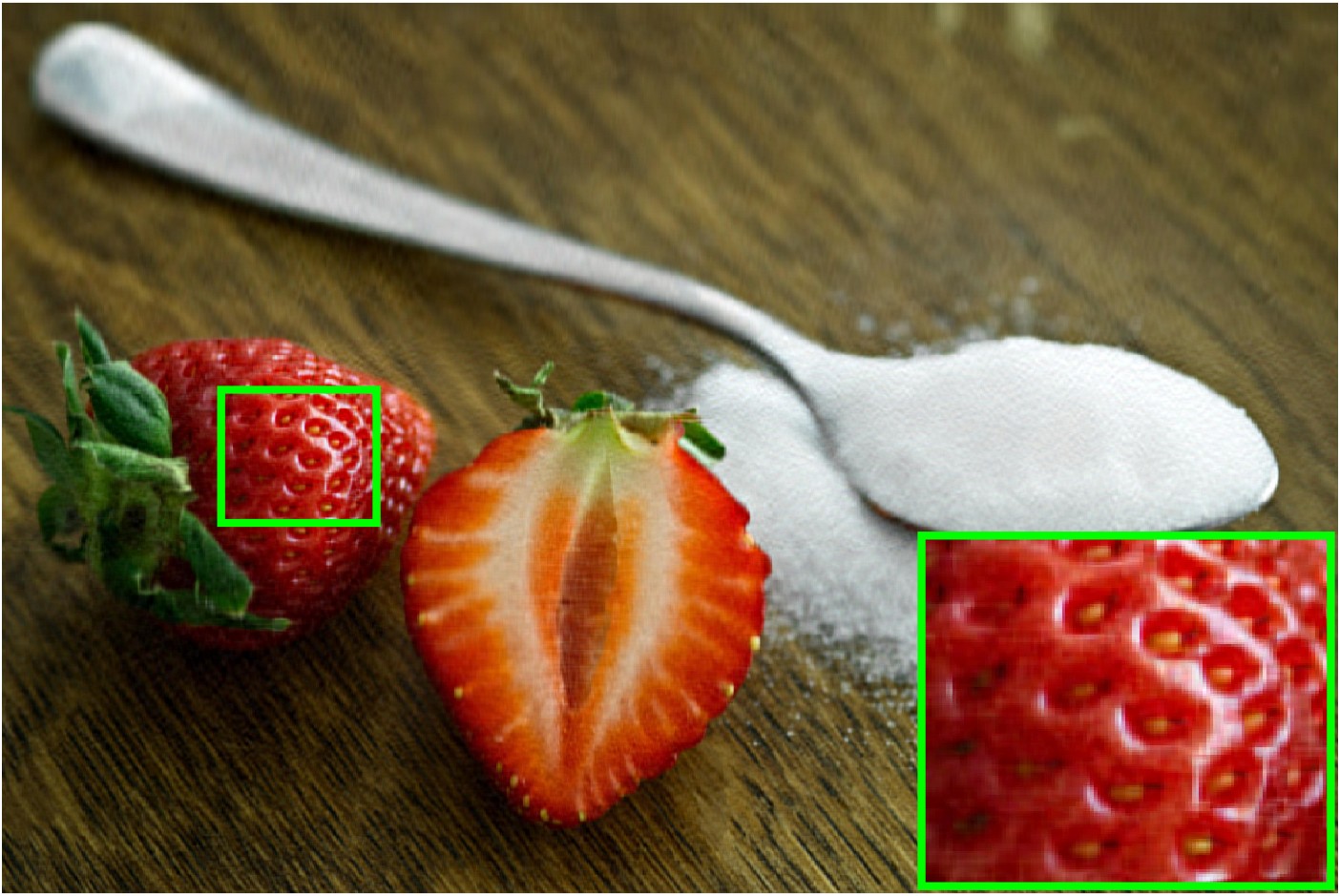} &
\includegraphics[width=0.672in]{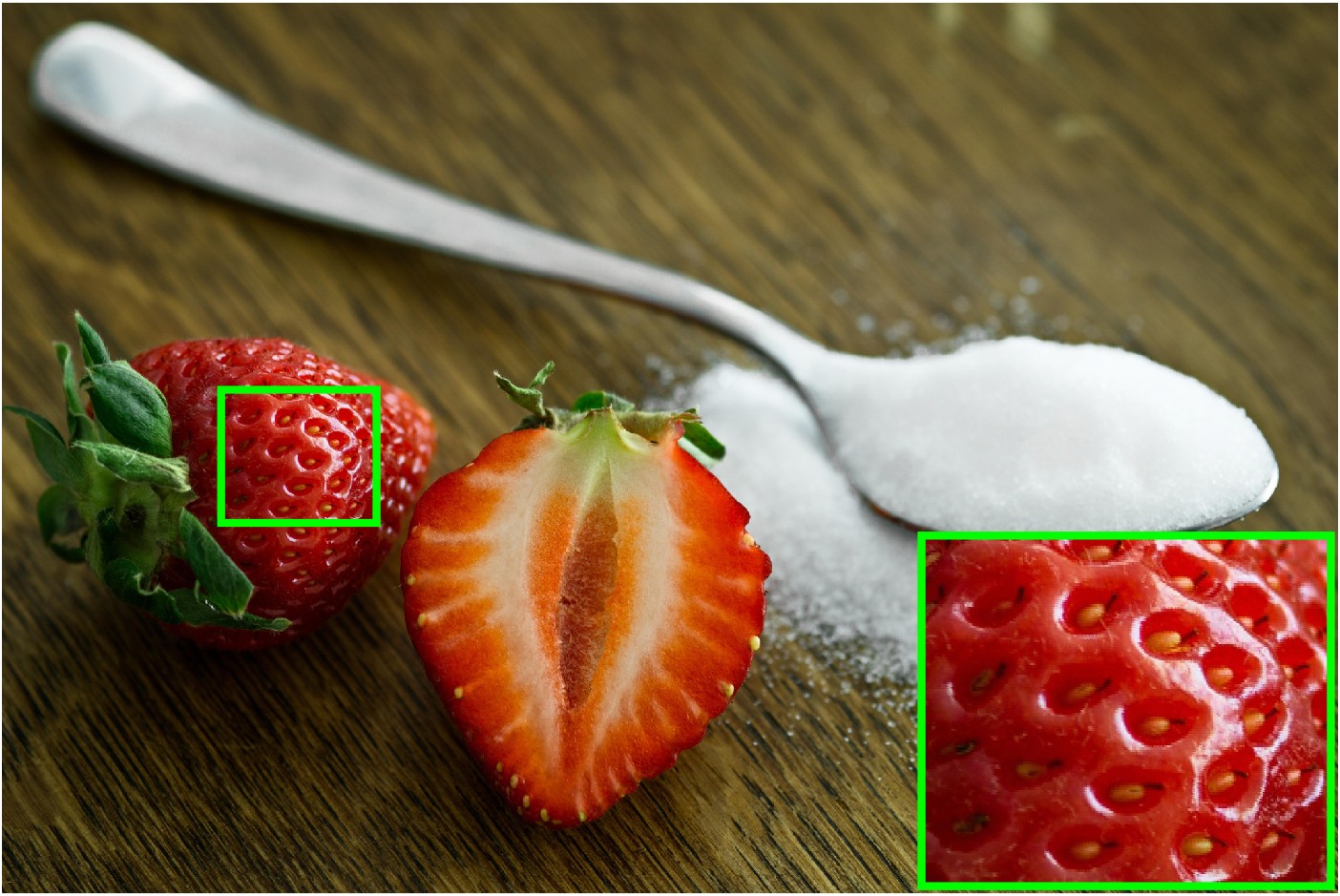} &
\includegraphics[width=0.672in]{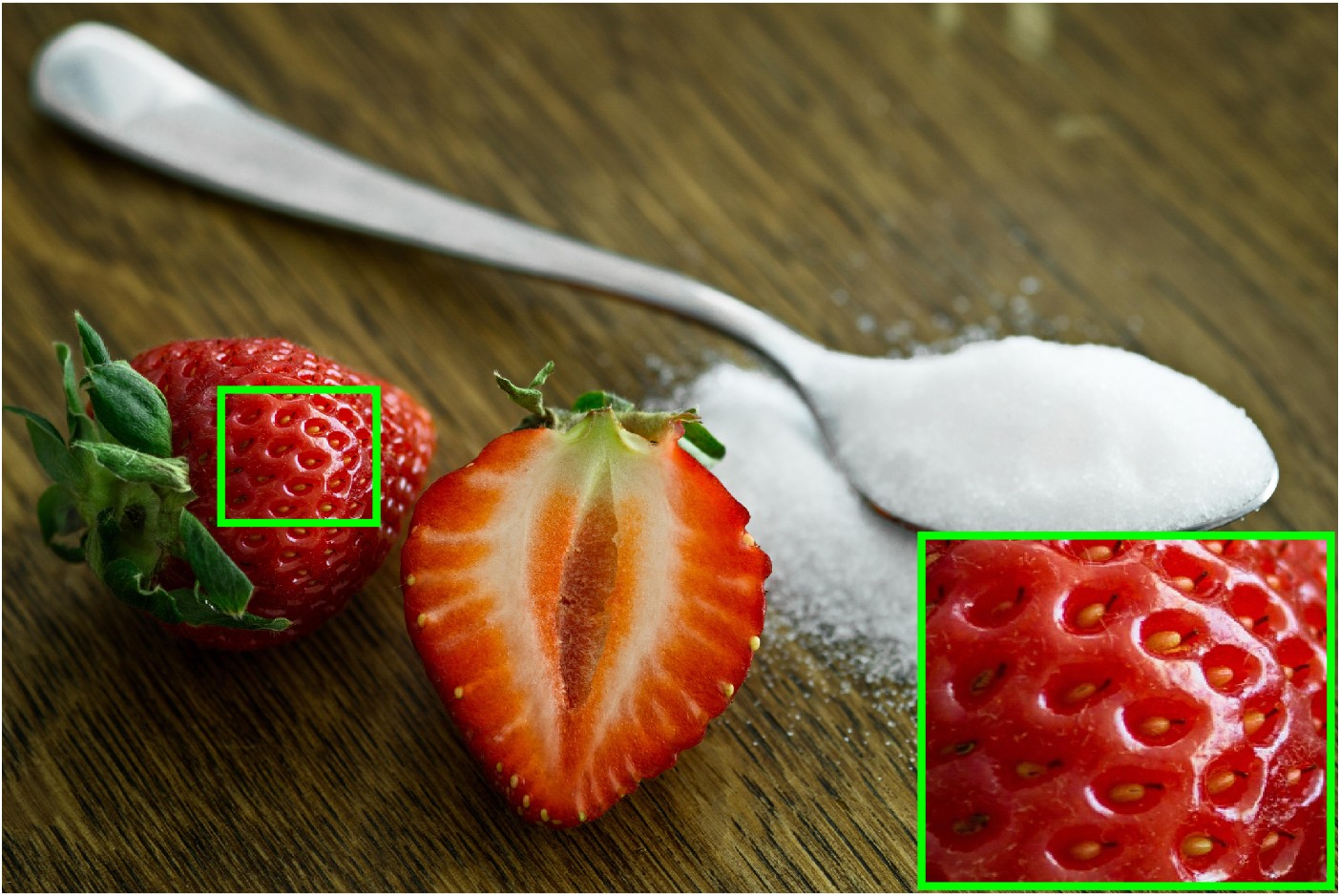} &
\includegraphics[width=0.672in]{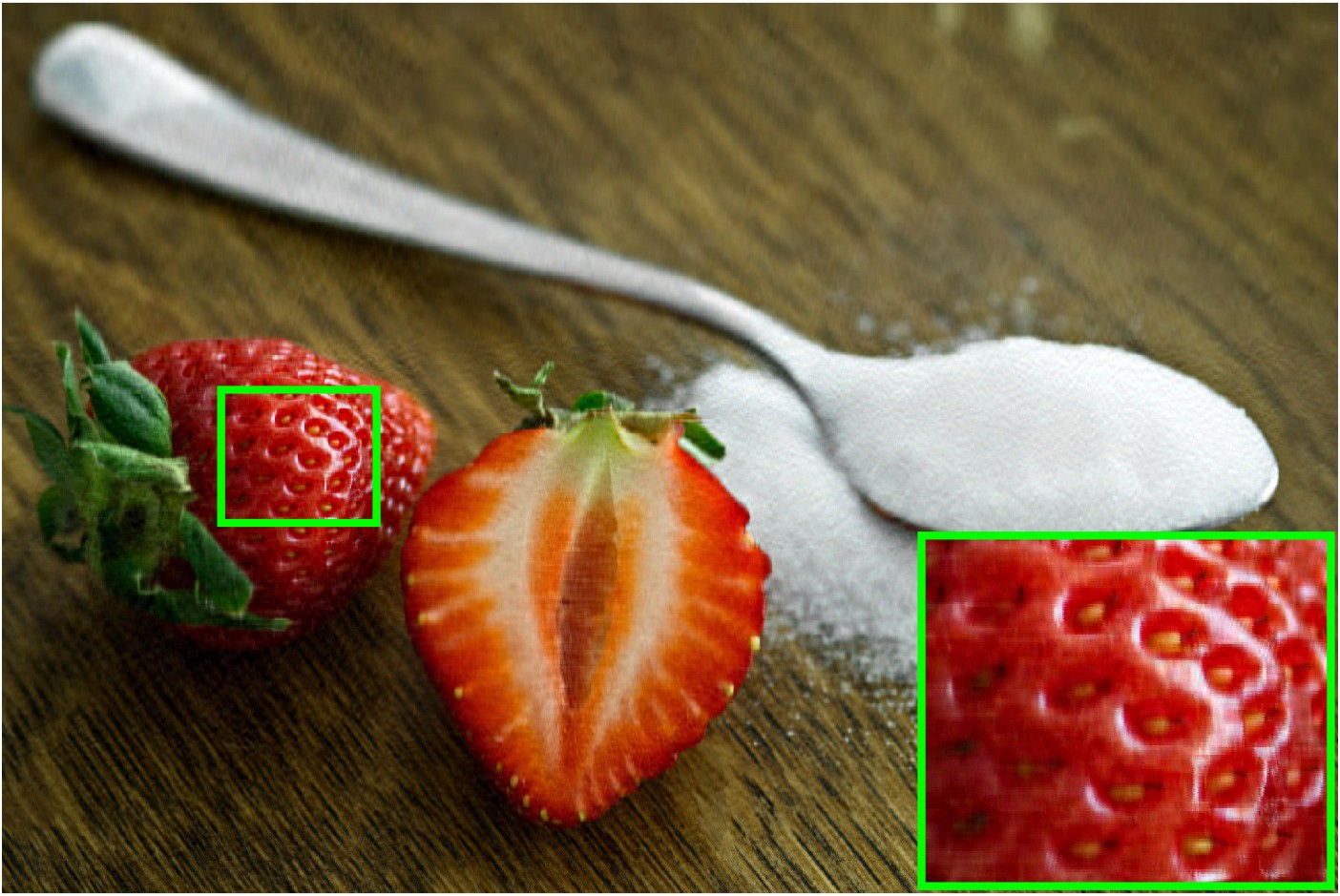} \\
&\tiny PSNR:30.46& \tiny PSNR:33.12 & \tiny PSNR:29.95 & \tiny PSNR:35.17 & \tiny PSNR:36.30 & \tiny PSNR:30.46\\
&\tiny Time:26.84s & \tiny Time:10.45s & \tiny Time:10.30s & \tiny Time:12.23s & \tiny Time:14.14s & \tiny Time:11.35s\\
\tiny\makecell[c]{TMSTP-SVD\\[-4pt](k=3)} &\tiny\makecell[c]{MRSTP-SVD\\[-4pt](k=1)} & \tiny\makecell[c]{MRSTP-SVD\\[-4pt](k=2)} & \tiny\makecell[c]{MRSTP-SVD\\[-4pt](k=3)}& \tiny\makecell[c]{TMRSTP-SVD\\[-4pt](k=1)} & \tiny\makecell[c]{TMRSTP-SVD\\[-4pt](k=2)} & \tiny\makecell[c]{TMRSTP-SVD\\[-4pt](k=3)}\\
\includegraphics[width=0.672in]{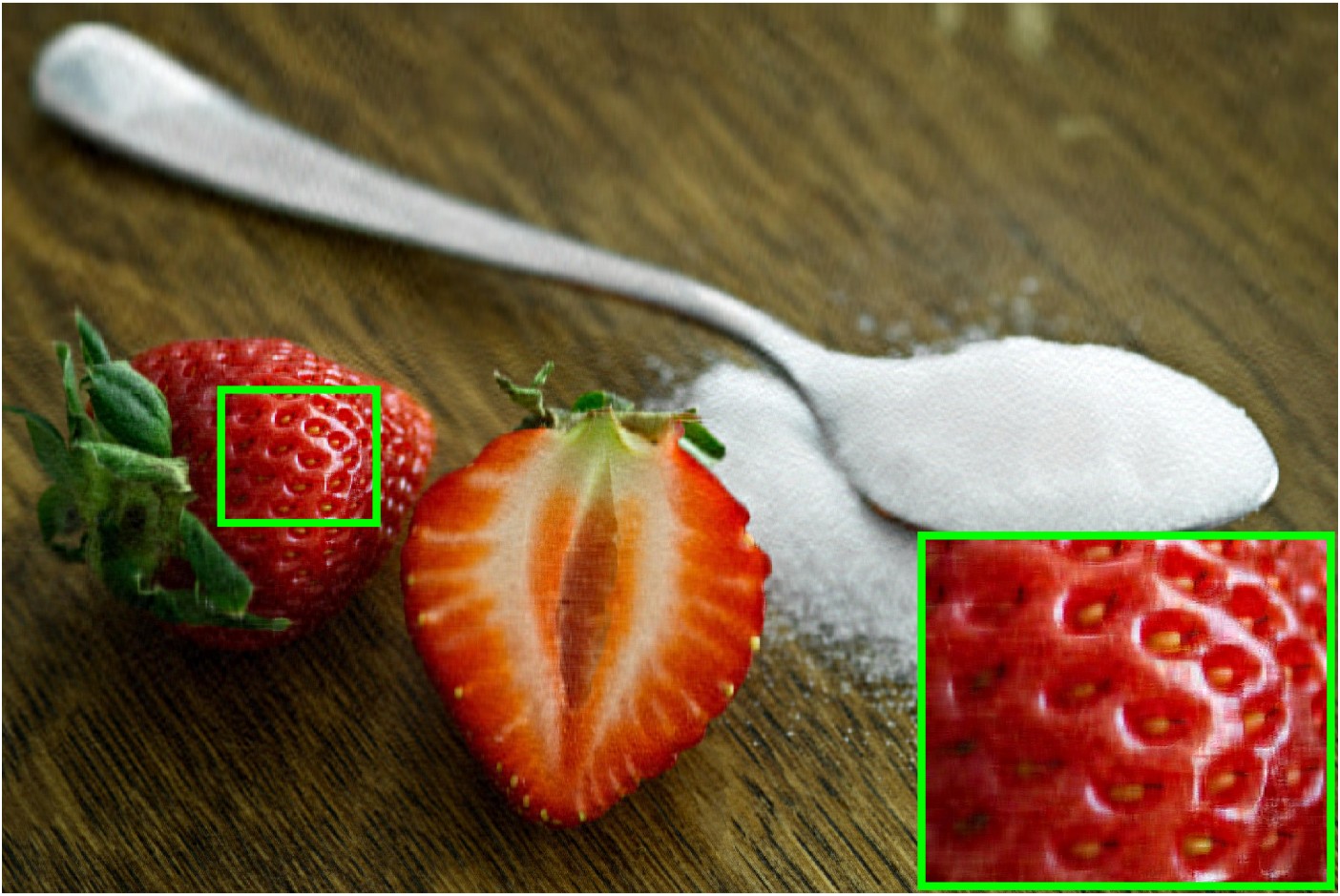} &
\includegraphics[width=0.672in]{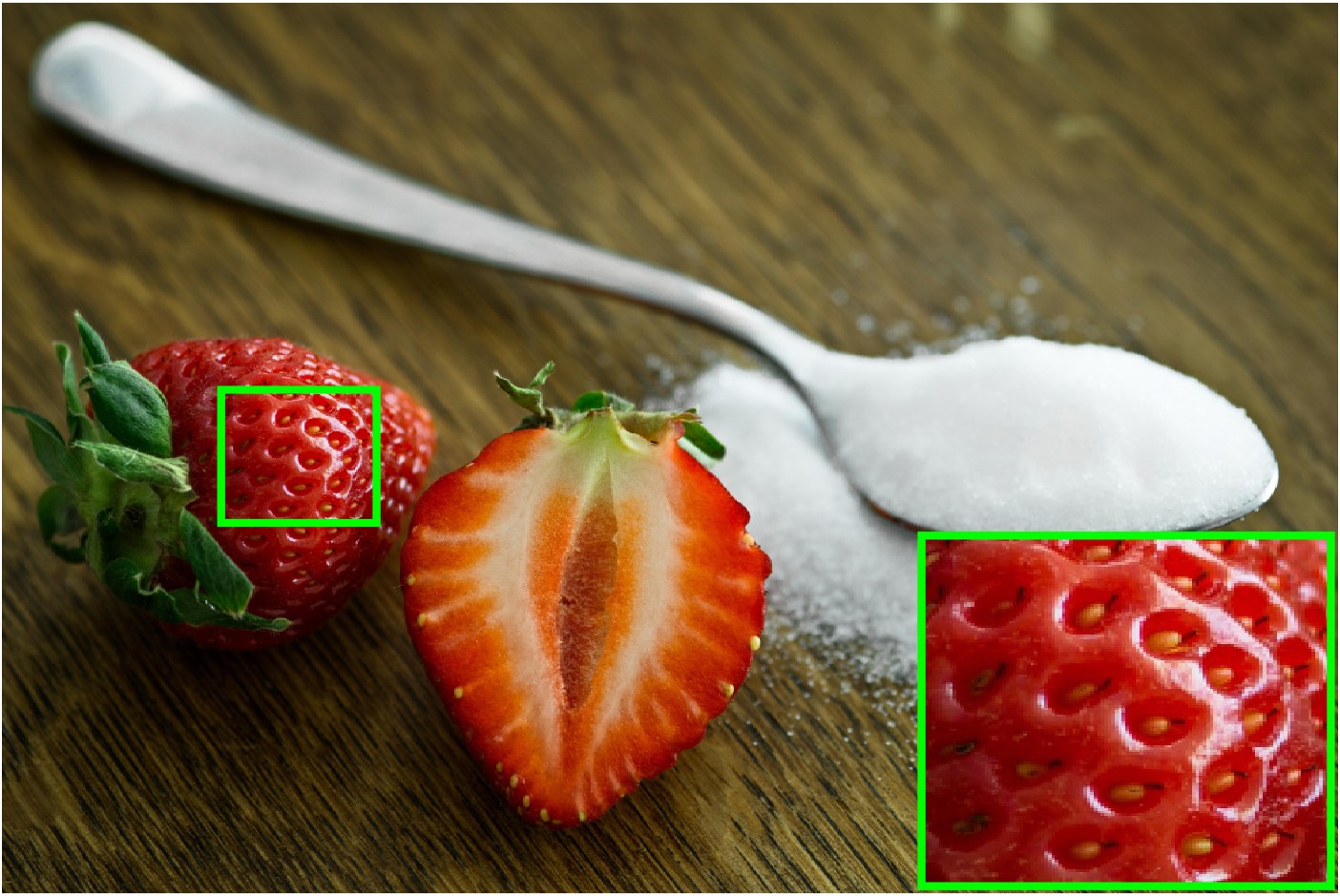} &
\includegraphics[width=0.672in]{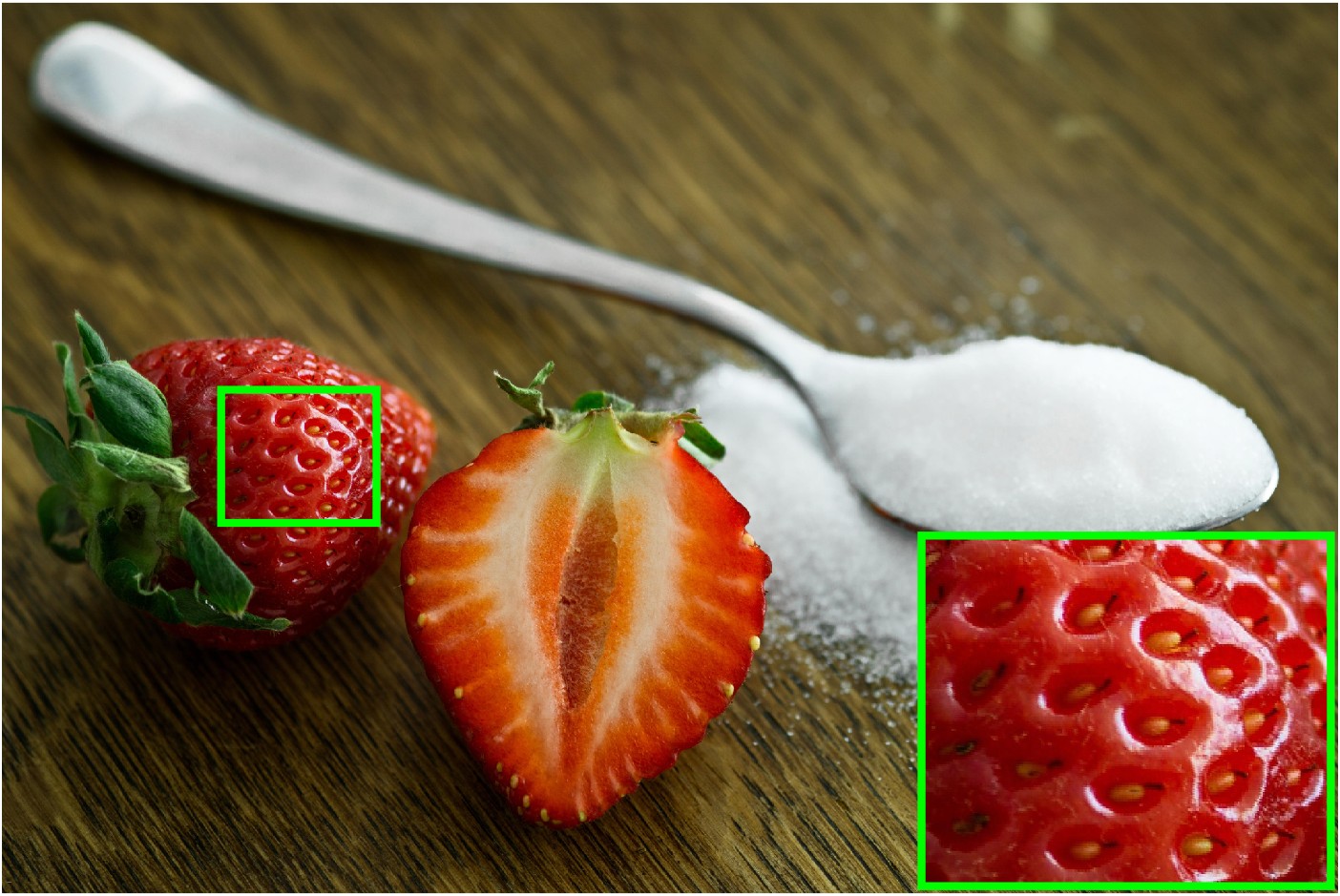} &
\includegraphics[width=0.672in]{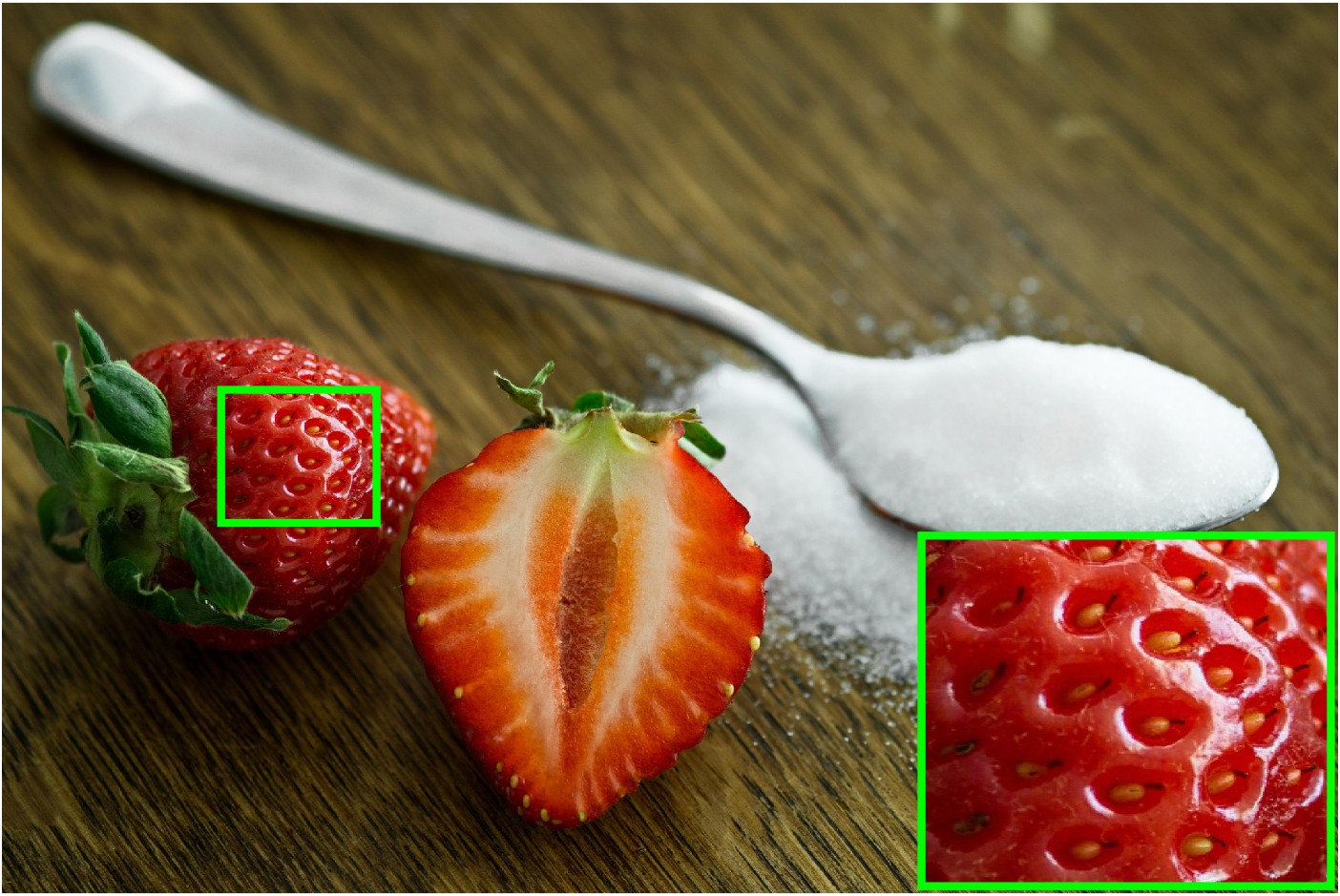}&
\includegraphics[width=0.672in]{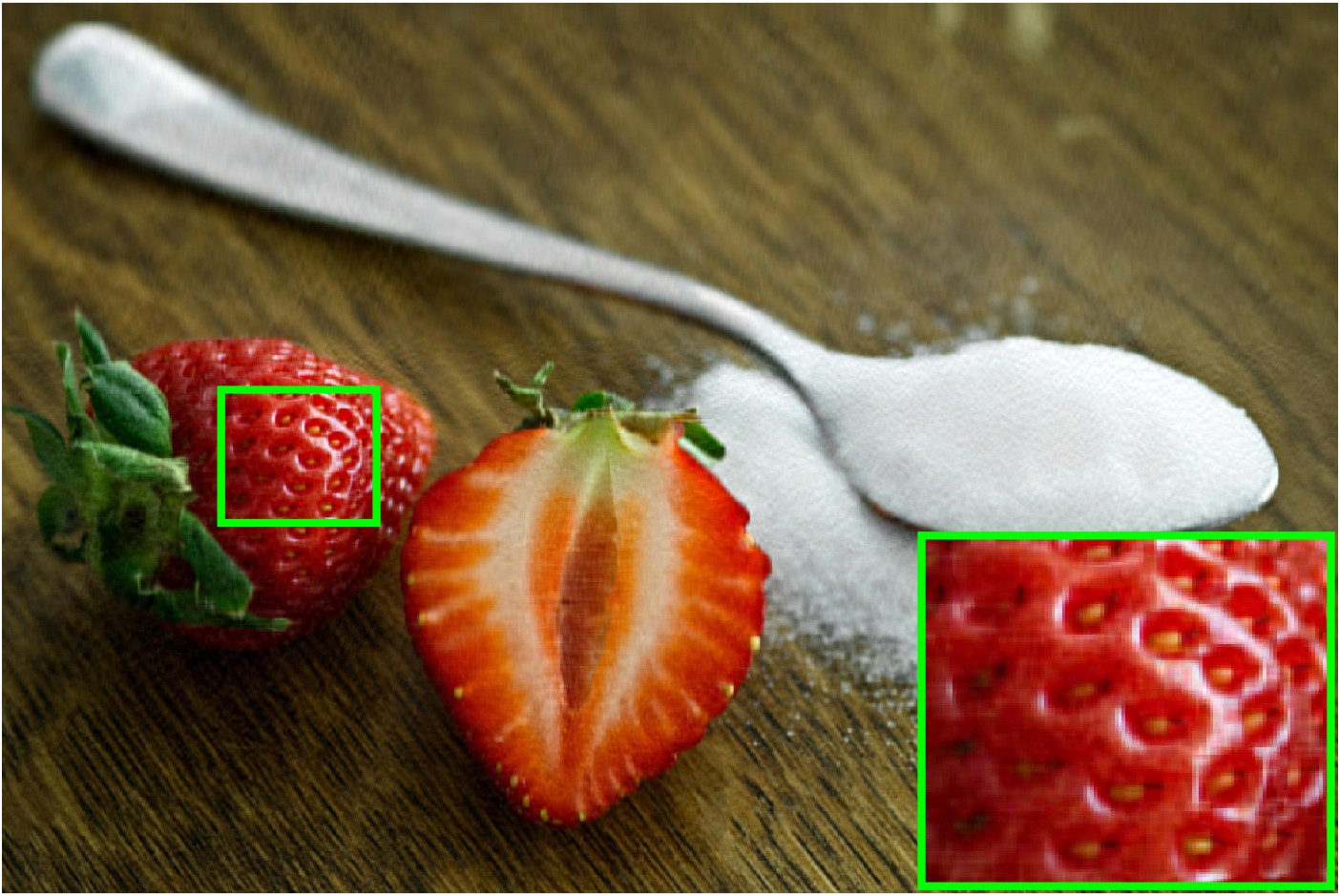} &
\includegraphics[width=0.672in]{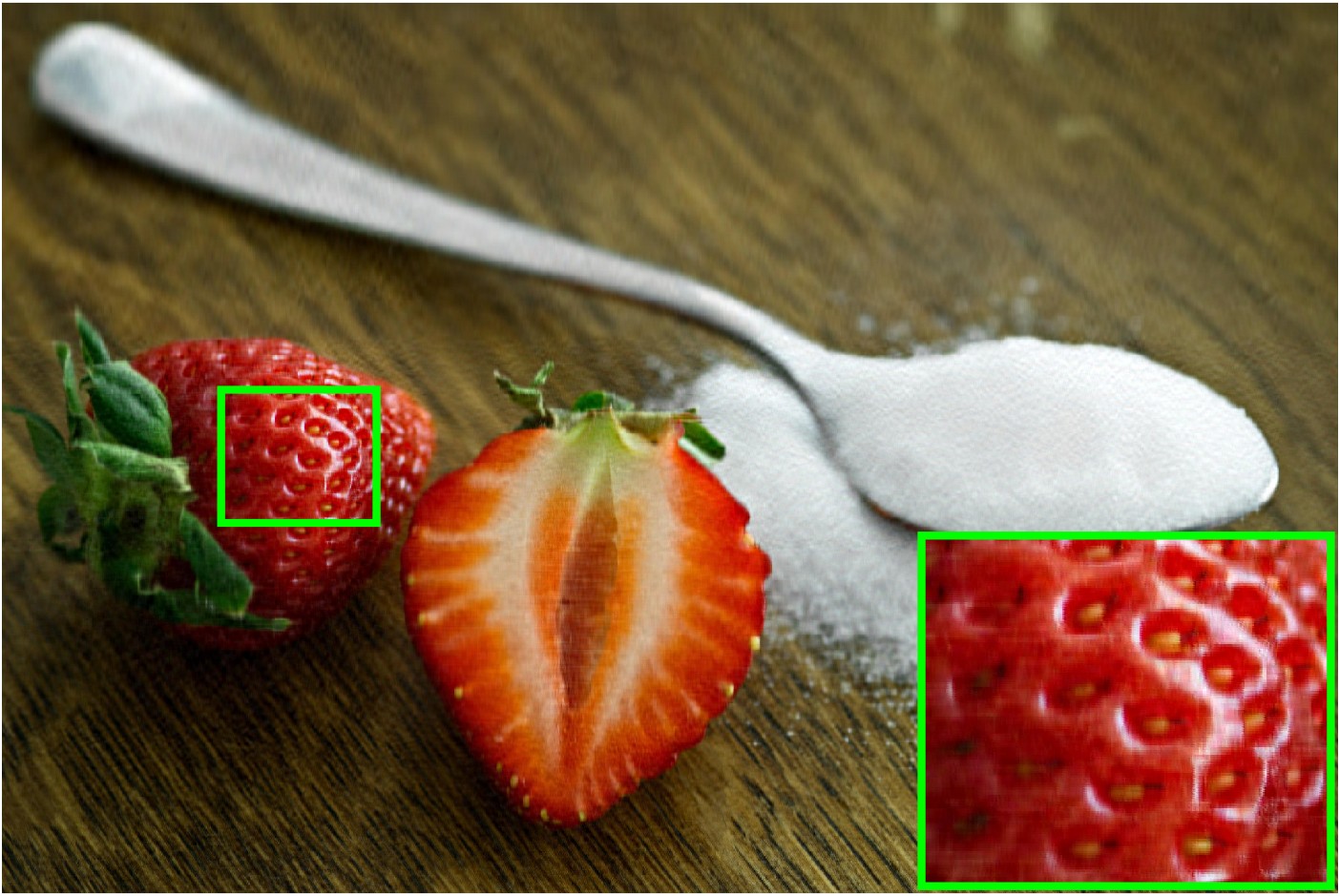} &
\includegraphics[width=0.672in]{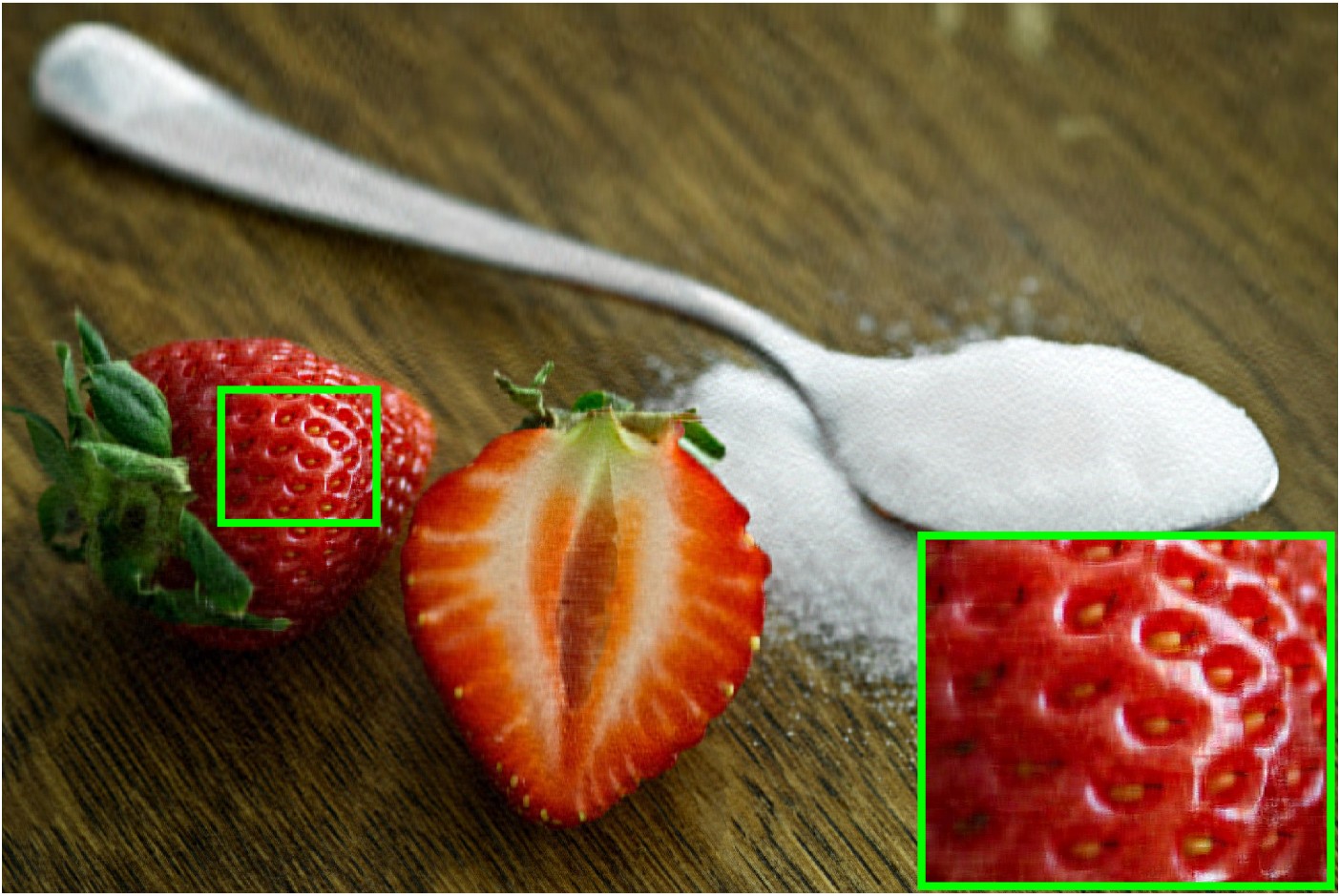}  \\
\tiny PSNR:30.63& \tiny PSNR:33.12 & \tiny PSNR:35.16 & \tiny PSNR:36.28 & \tiny PSNR:29.94 & \tiny PSNR:30.46 & \tiny PSNR:30.64\\
\tiny Time:13.31s& \tiny Time:7.79s & \tiny Time:9.87s & \tiny Time:10.84s & \tiny Time:7.13s & \tiny Time:8.54s & \tiny Time:9.43s\\   

\end{tabular}
\caption{Supplementary comparisons of reconstruction performance, PSNR, and runtime for competing methods and the proposed approach on two additional representative test images.}
\label{fig:image_compression_sup}
\end{figure}  

Fig.~\ref*{fig:video compression_sup} shows supplementary visual reconstruction examples and quantitative PSNR-SSIM comparisons on randomly sampled frames from two additional test video sequences, including all competing baselines as well as the original and truncated variants of the proposed methods. All experimental configurations, including block partition sizes, truncation ranks, and randomized hyperparameters, follow the settings used in the main text. Quantitative results demonstrate that MSTP-SVD (k=3) improves PSNR by 2.9 dB and 5.2 dB over single-term STP-SVD on the two sequences, respectively, and achieves around 8 dB of PSNR gain over the TT-SVD baseline for both sequences. The randomized MRSTP-SVD reduces average  runtime by 15\%-25\% with negligible quality degradation. Visually, the proposed methods recover richer textures and finer details than baseline approaches. These supplementary video samples further verify the robustness of the proposed framework across different video content.
\begin{figure}[!ht]
\centering
\renewcommand{\arraystretch}{0.3}
\setlength\tabcolsep{0.1pt}
\begin{tabular}{@{}ccccccc@{}}

\tiny Original &\tiny TT-SVD & \tiny STP-SVD & \tiny TSTP-SVD &\tiny\makecell[c]{MSTP-SVD\\[-4pt](k=2)} & \tiny\makecell[c]{MSTP-SVD\\[-4pt](k=3)} &\tiny\makecell[c]{TMSTP-SVD\\[-4pt](k=2)} \\
\includegraphics[width=0.672in]{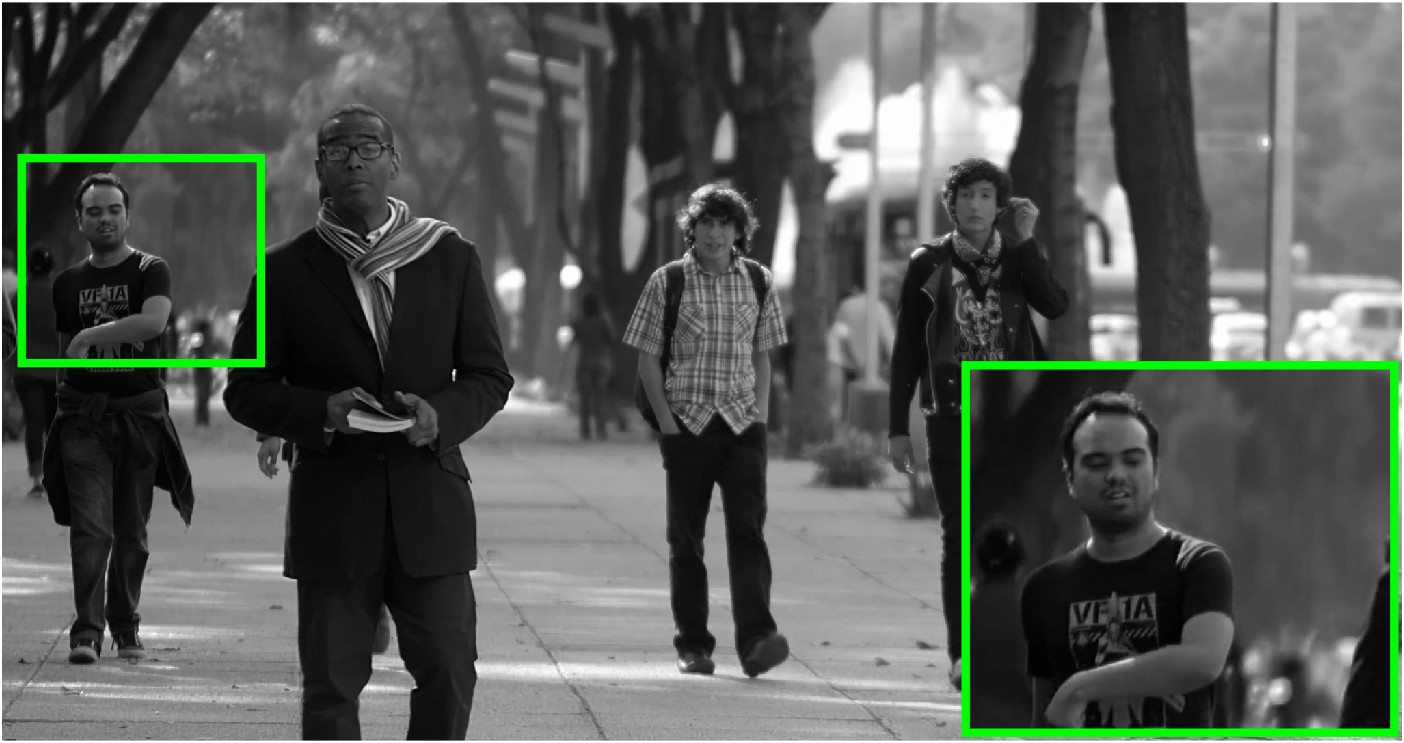} &
\includegraphics[width=0.672in]{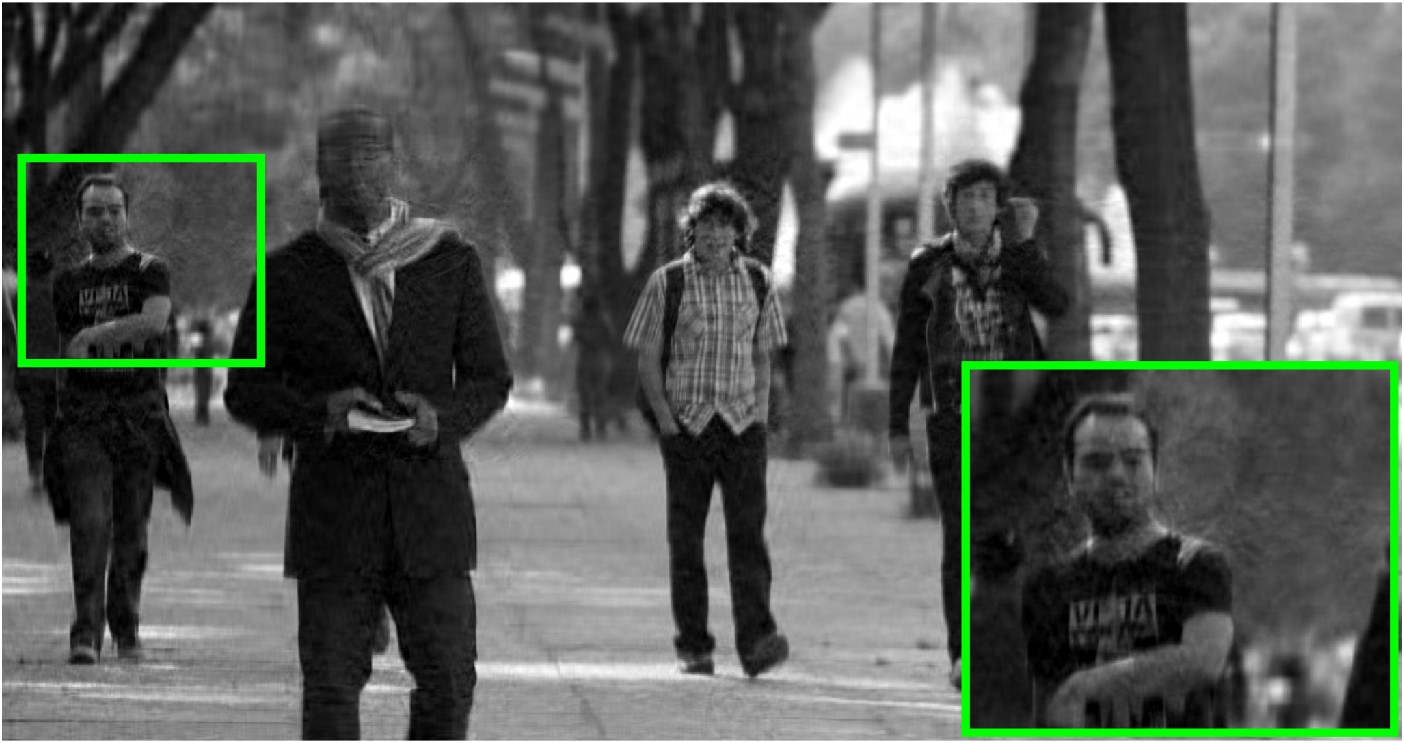} &
\includegraphics[width=0.672in]{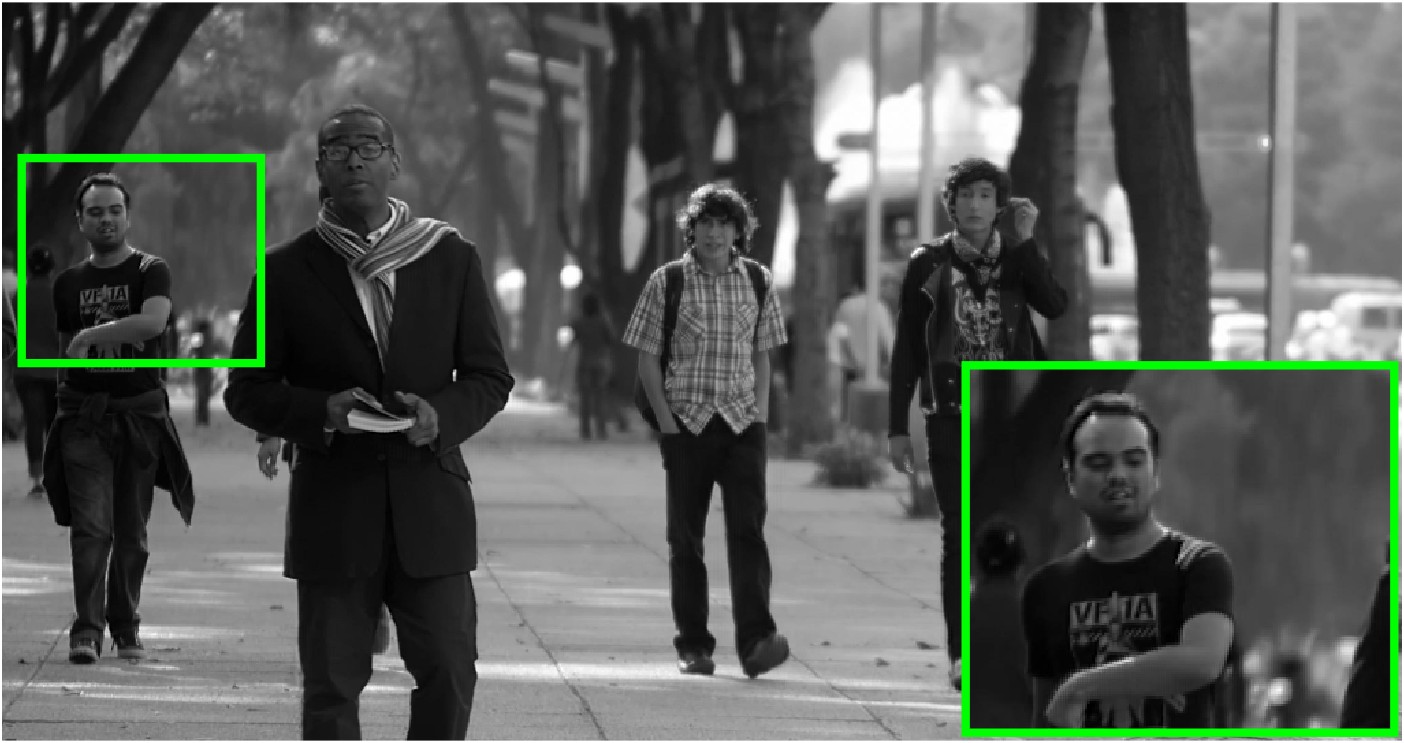} &
\includegraphics[width=0.672in]{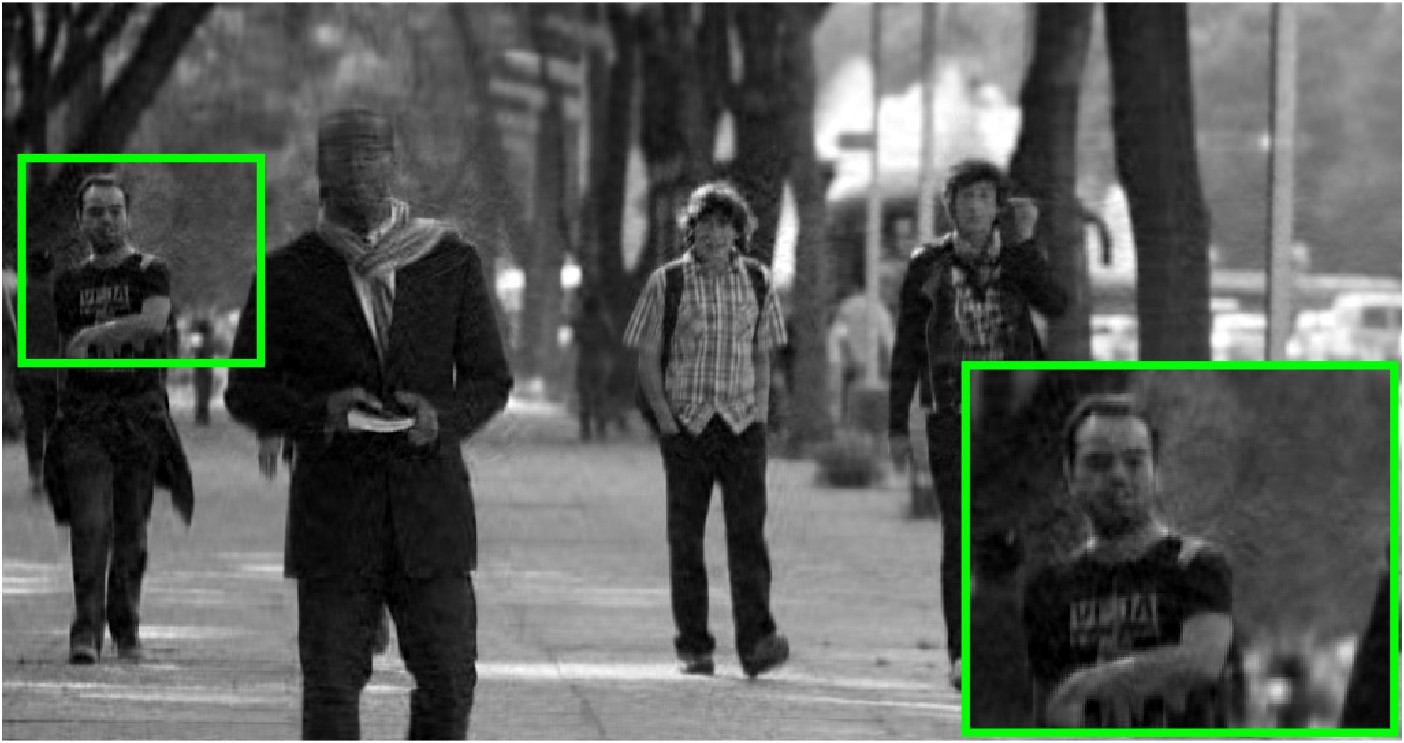} &
\includegraphics[width=0.672in]{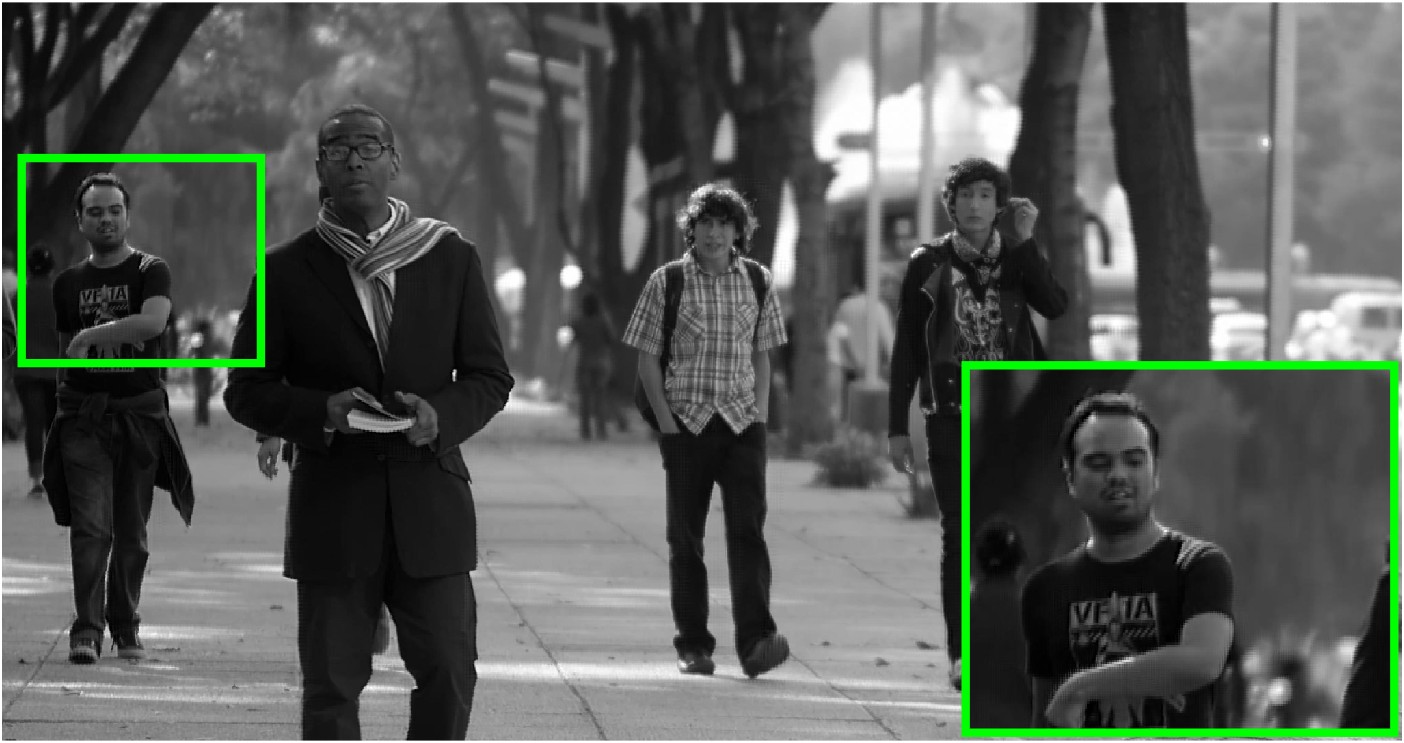} &
\includegraphics[width=0.672in]{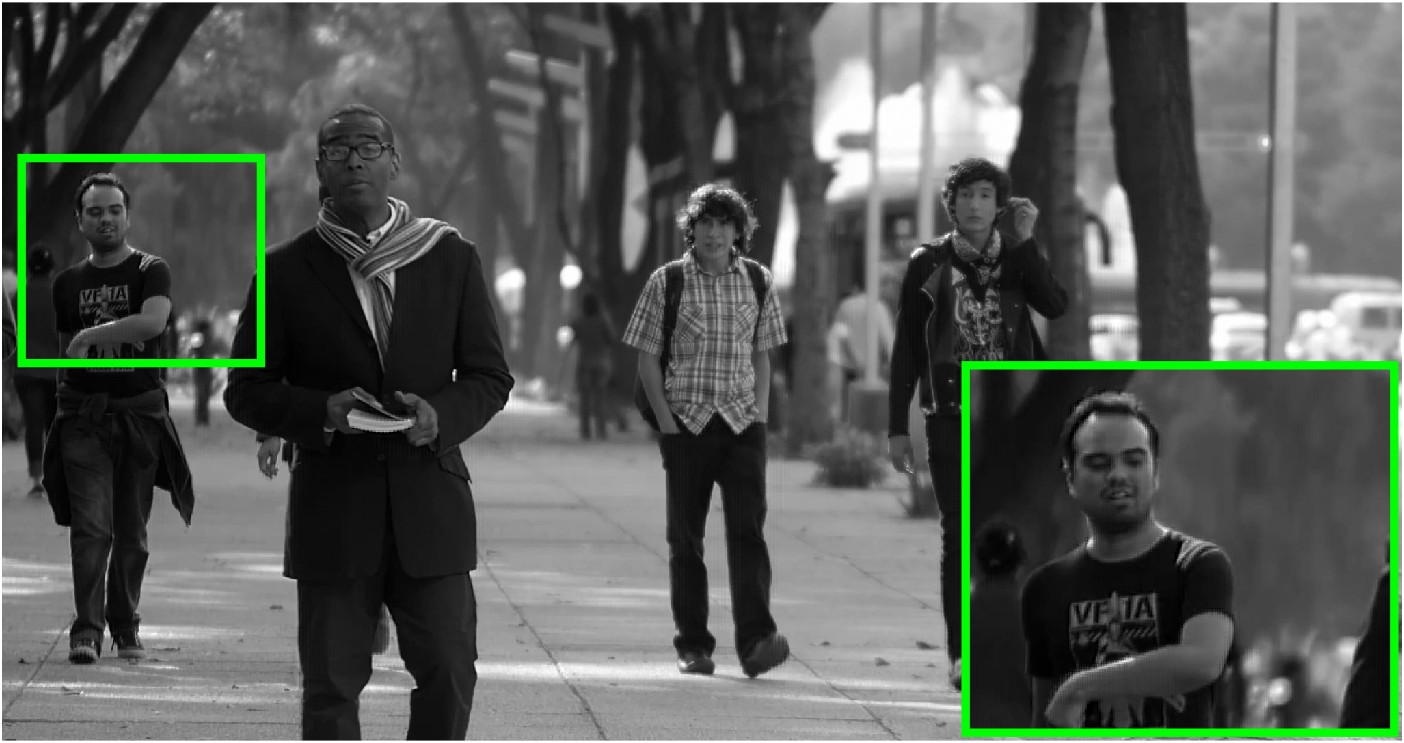} &
\includegraphics[width=0.672in]{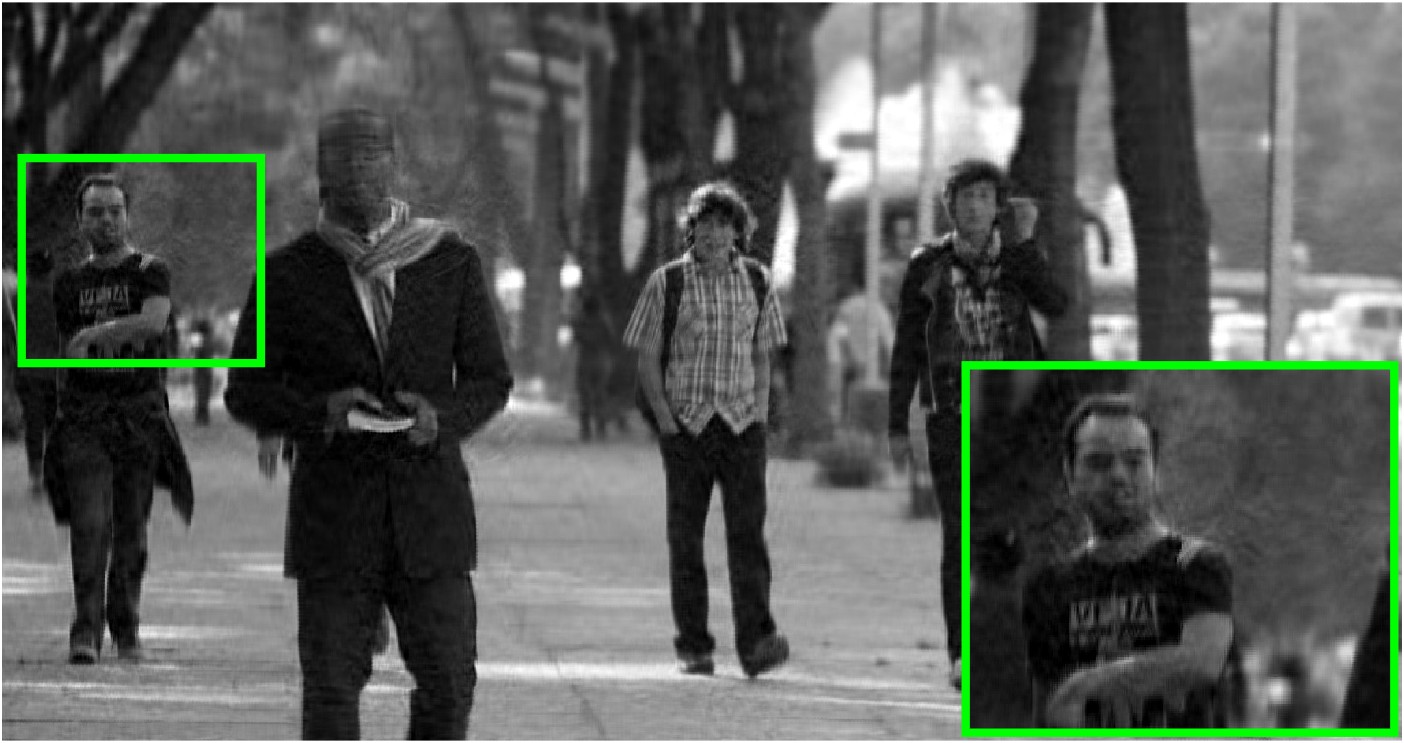} \\
&\tiny PSNR:31.68 & \tiny PSNR:34.7 & \tiny PSNR:30.33 &\tiny PSNR:36.31& \tiny PSNR:37.69 & \tiny PSNR:30.69\\
&\tiny SSIM:0.932& \tiny SSIM:0.942& \tiny SSIM:0.895& \tiny SSIM:0.962 & \tiny SSIM:0.979 & \tiny SSIM:0.904\\
\tiny\makecell[c]{TMSTP-SVD\\[-4pt](k=3)} &\tiny\makecell[c]{MRSTP-SVD\\[-4pt](k=1)} & \tiny\makecell[c]{MRSTP-SVD\\[-4pt](k=2)} & \tiny\makecell[c]{MRSTP-SVD\\[-4pt](k=3)}& \tiny\makecell[c]{TMRSTP-SVD\\[-4pt](k=1)} & \tiny\makecell[c]{TMRSTP-SVD\\[-4pt](k=2)} & \tiny\makecell[c]{TMRSTP-SVD\\[-4pt](k=3)}\\
\includegraphics[width=0.672in]{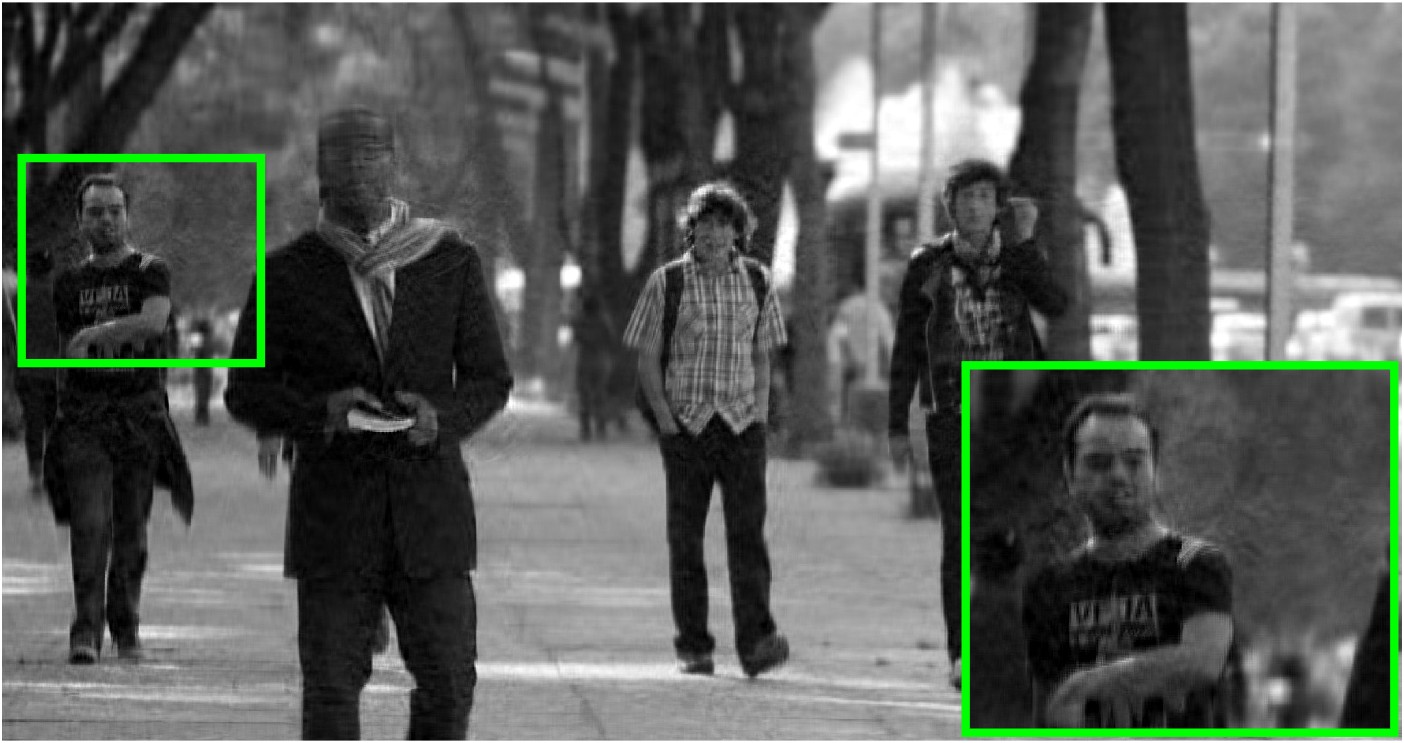} &
\includegraphics[width=0.672in]{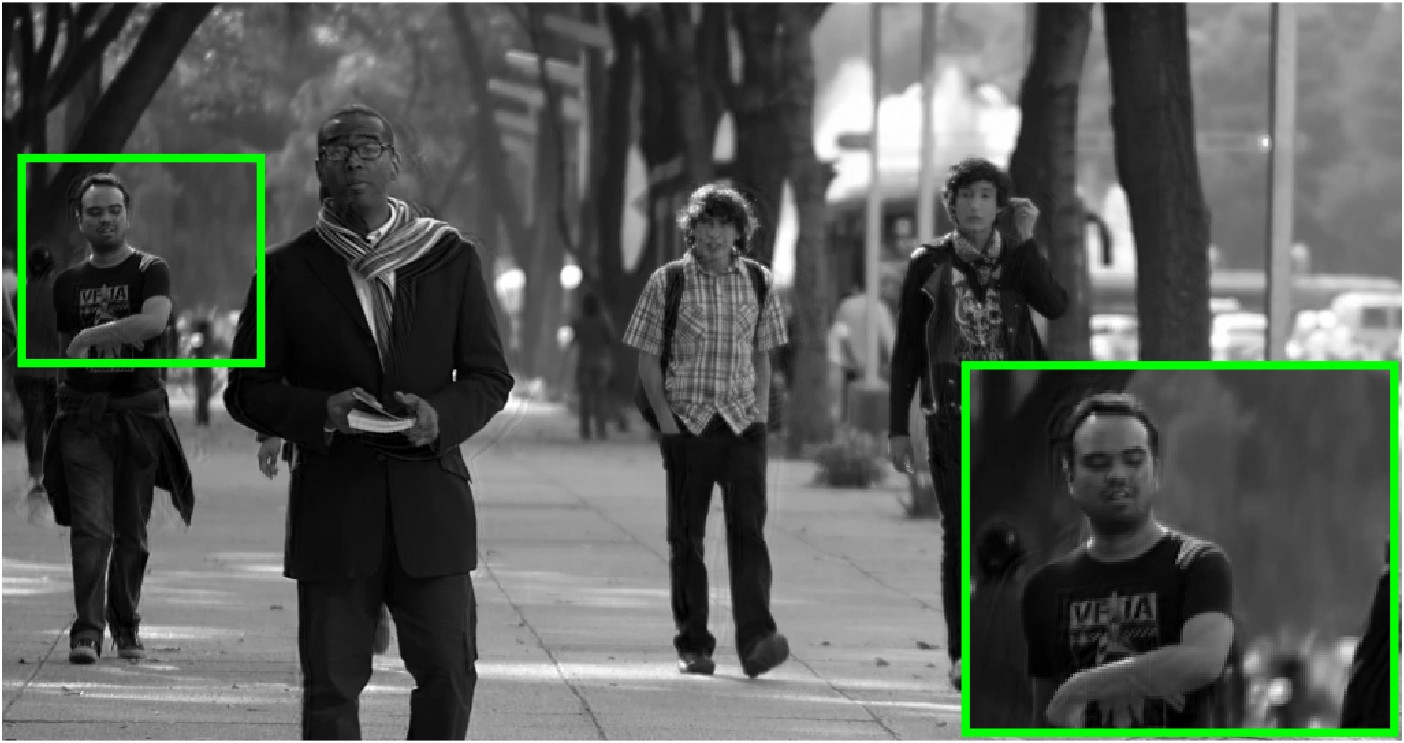} &
\includegraphics[width=0.672in]{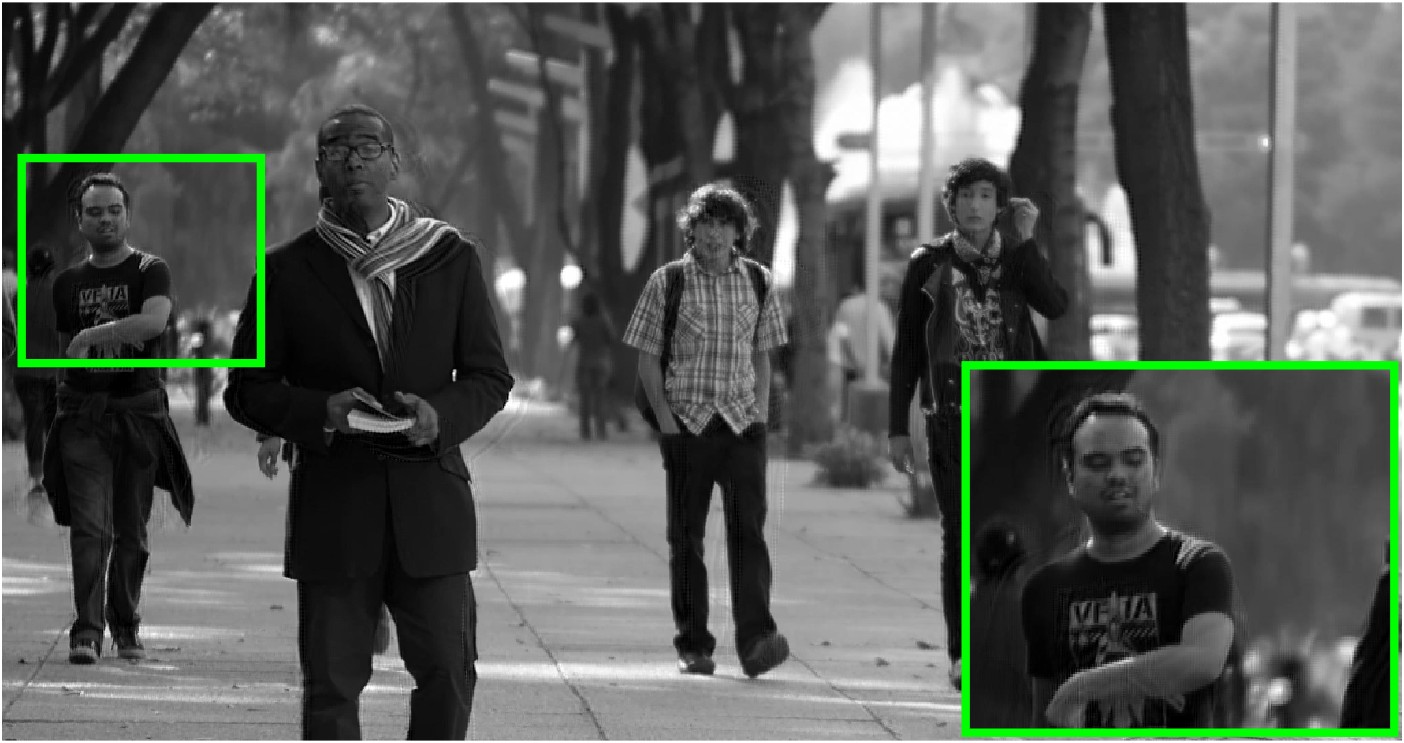} &
\includegraphics[width=0.672in]{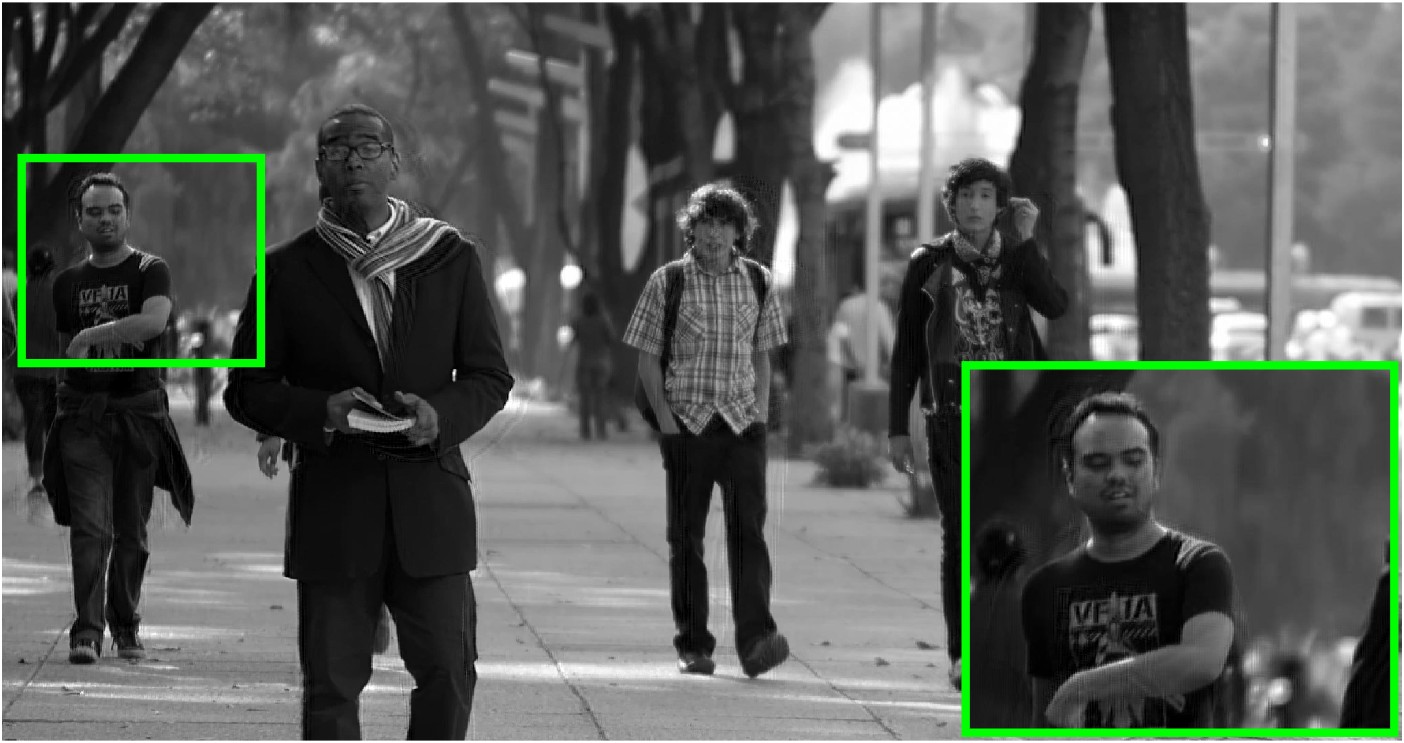}&
\includegraphics[width=0.672in]{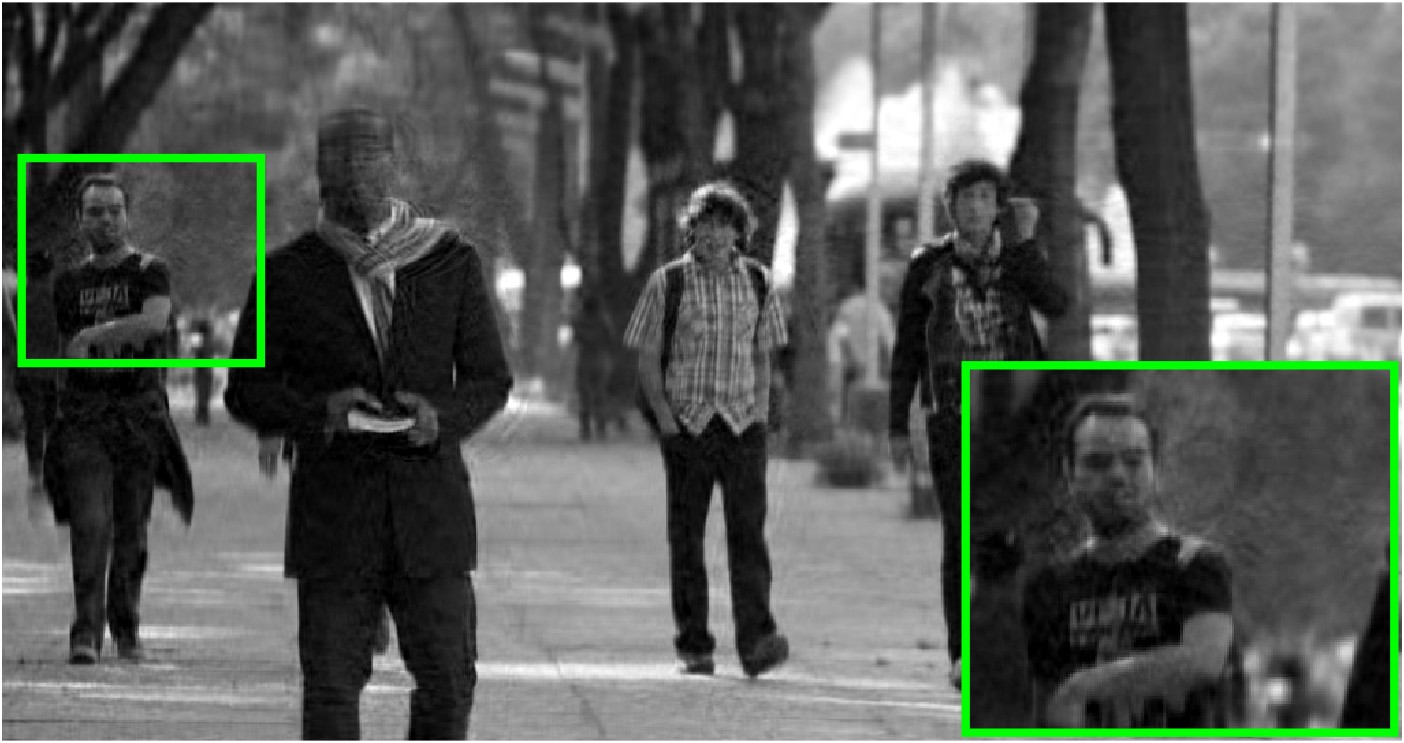} &
\includegraphics[width=0.672in]{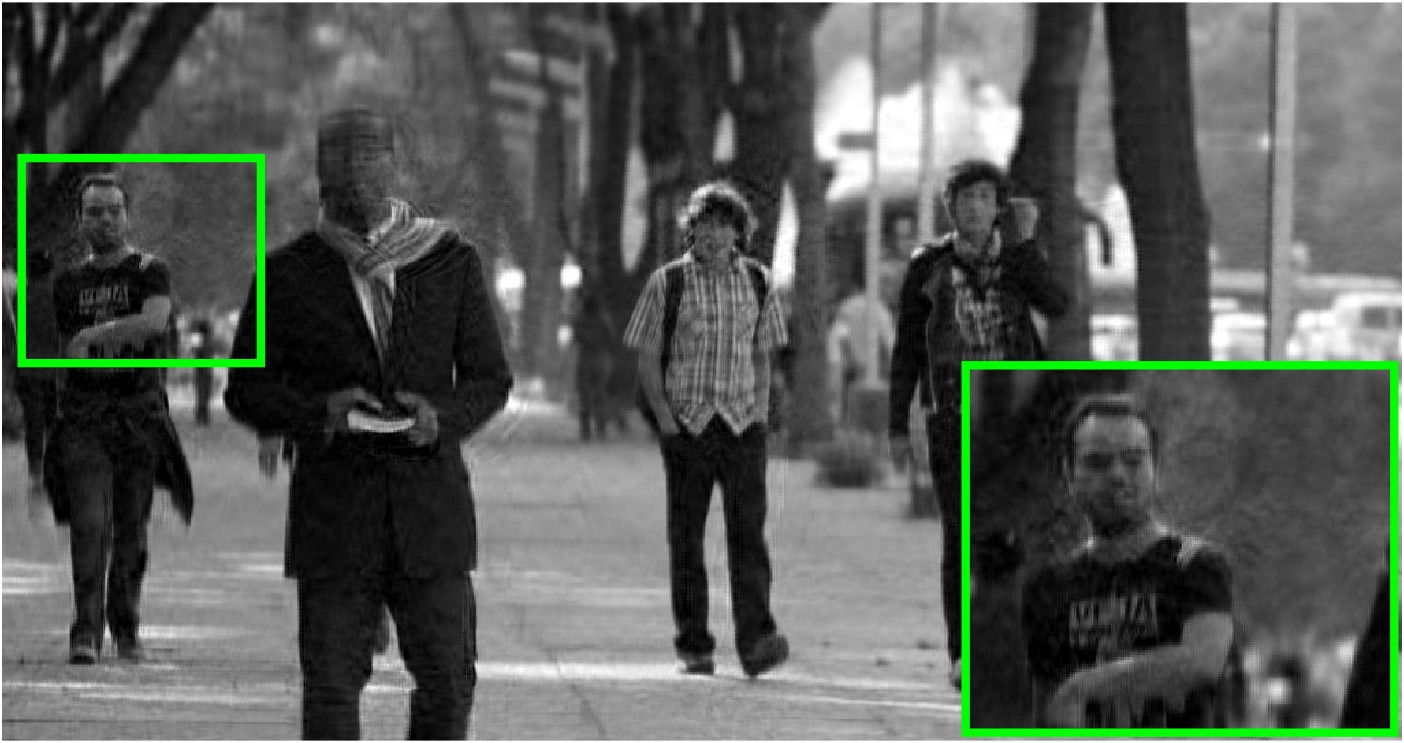} &
\includegraphics[width=0.672in]{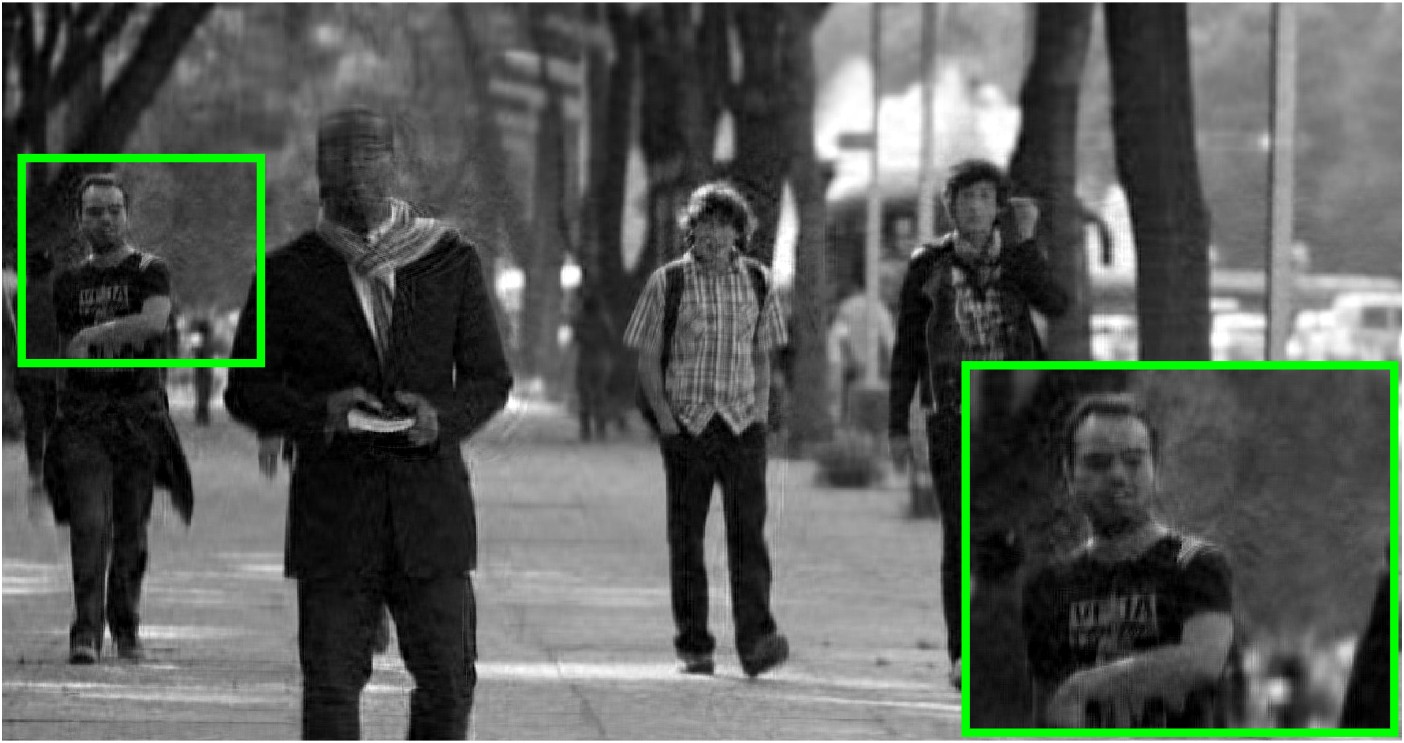}  \\
\tiny PSNR:30.92& \tiny PSNR:35.13 & \tiny PSNR:35.6 & \tiny PSNR:36.74 & \tiny PSNR:30.57 & \tiny PSNR:30.67 & \tiny PSNR:30.89\\
\tiny SSIM:0.912 & \tiny SSIM:0.935 & \tiny SSIM:0.932 & \tiny SSIM:0.947 & \tiny SSIM:0.889& \tiny SSIM:0.885 & \tiny SSIM:0.891\\

\tiny Original &\tiny TT-SVD & \tiny STP-SVD & \tiny TSTP-SVD &\tiny\makecell[c]{MSTP-SVD\\[-4pt](k=2)} & \tiny\makecell[c]{MSTP-SVD\\[-4pt](k=3)} &\tiny\makecell[c]{TMSTP-SVD\\[-4pt](k=2)} \\
\includegraphics[width=0.672in]{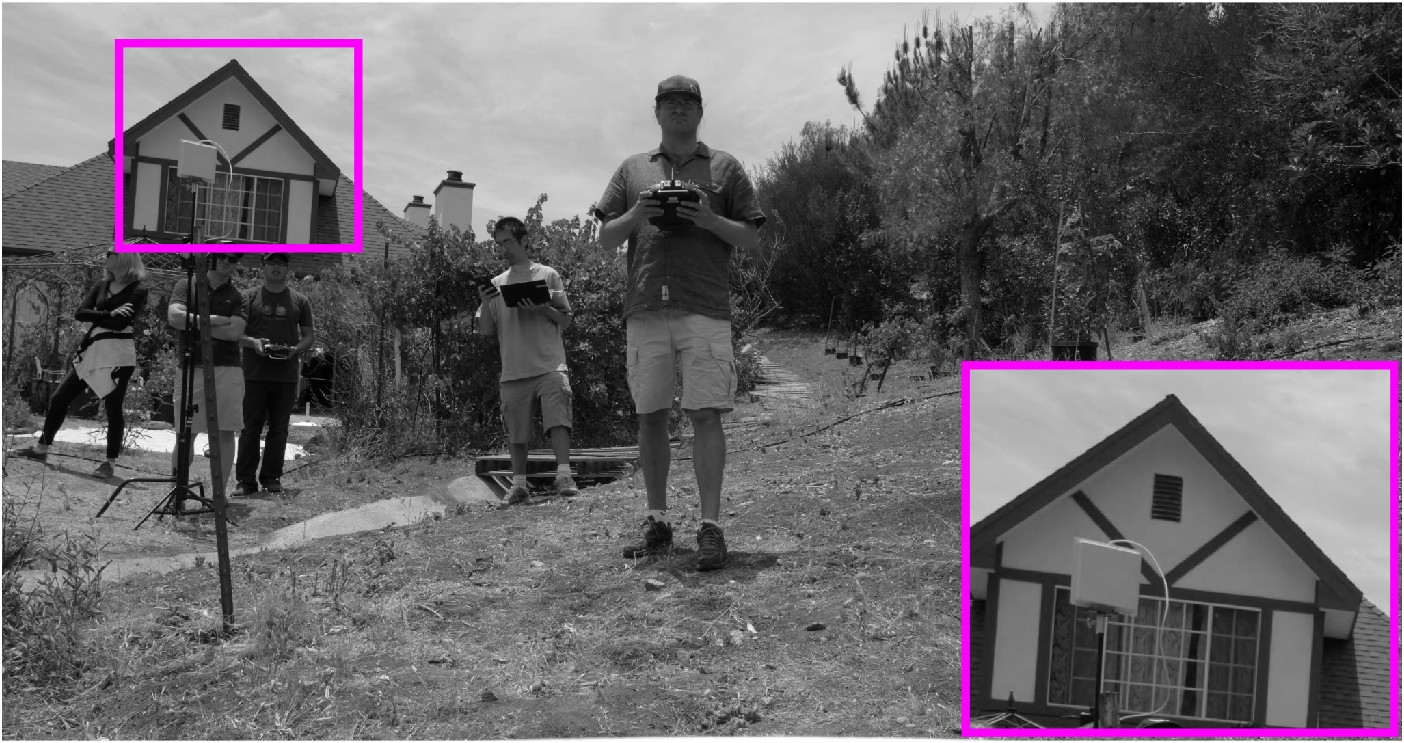} &
\includegraphics[width=0.672in]{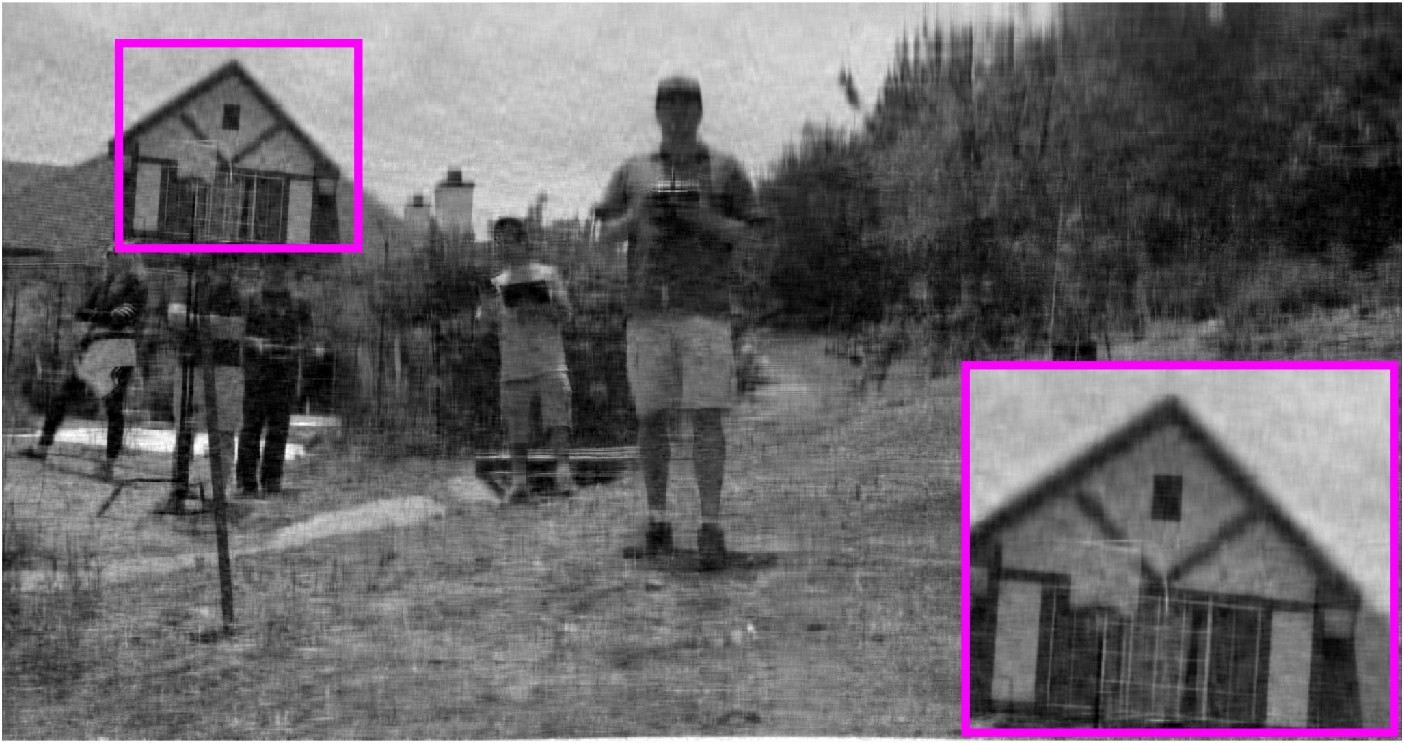} &
\includegraphics[width=0.672in]{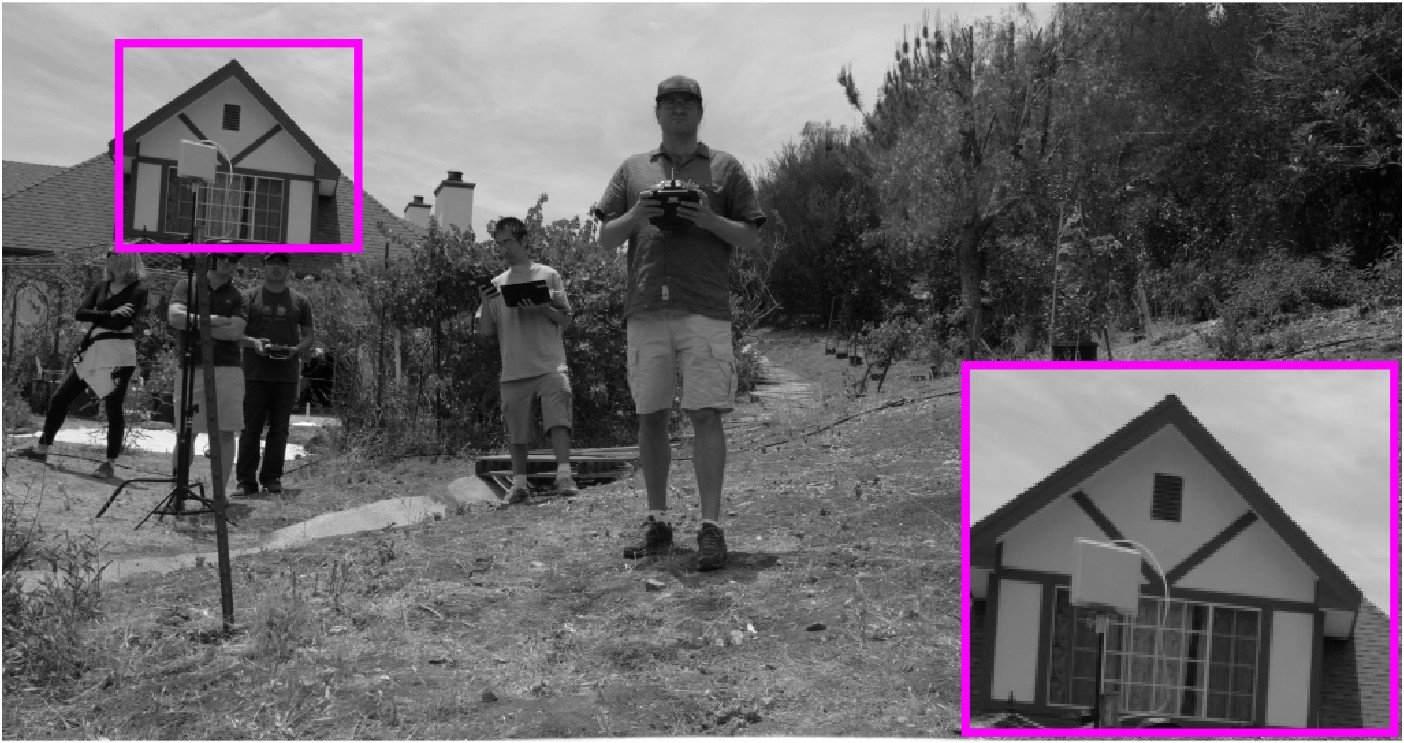} &
\includegraphics[width=0.672in]{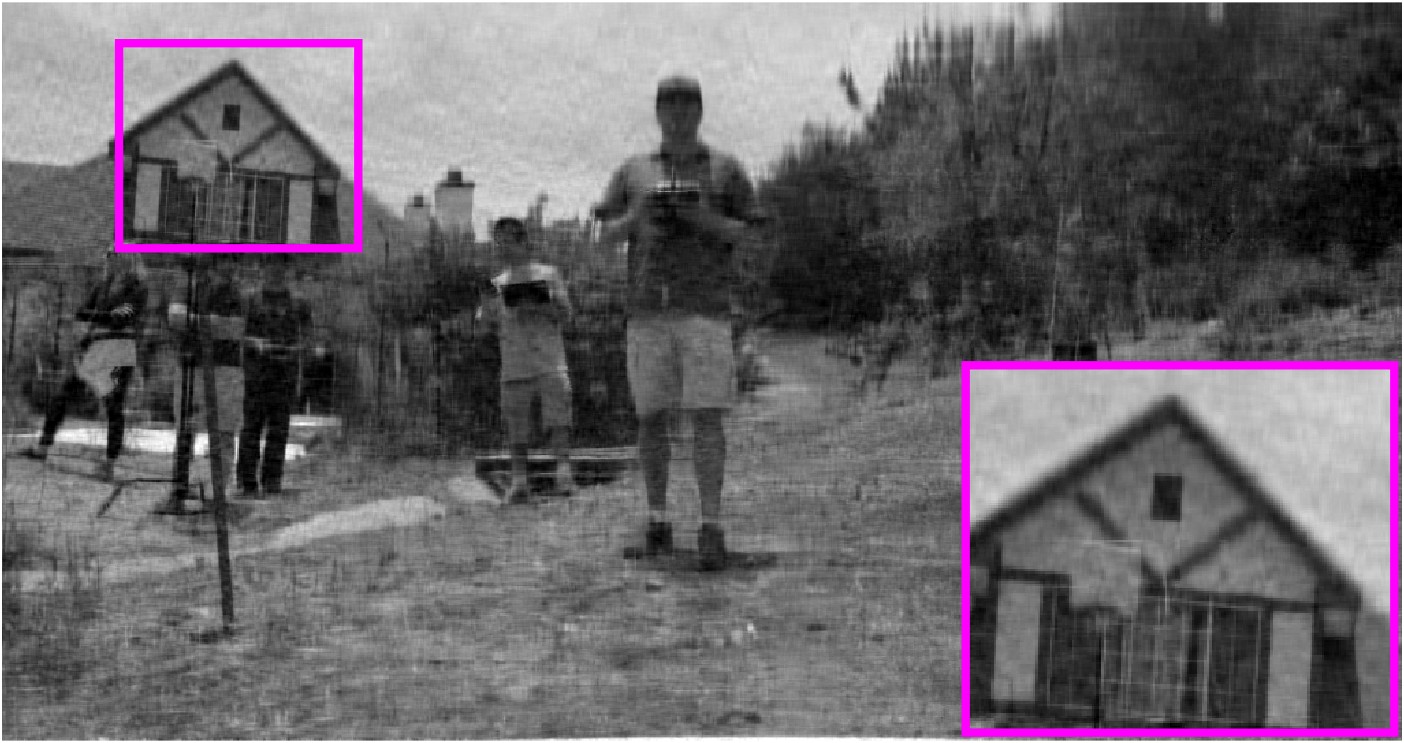} &
\includegraphics[width=0.672in]{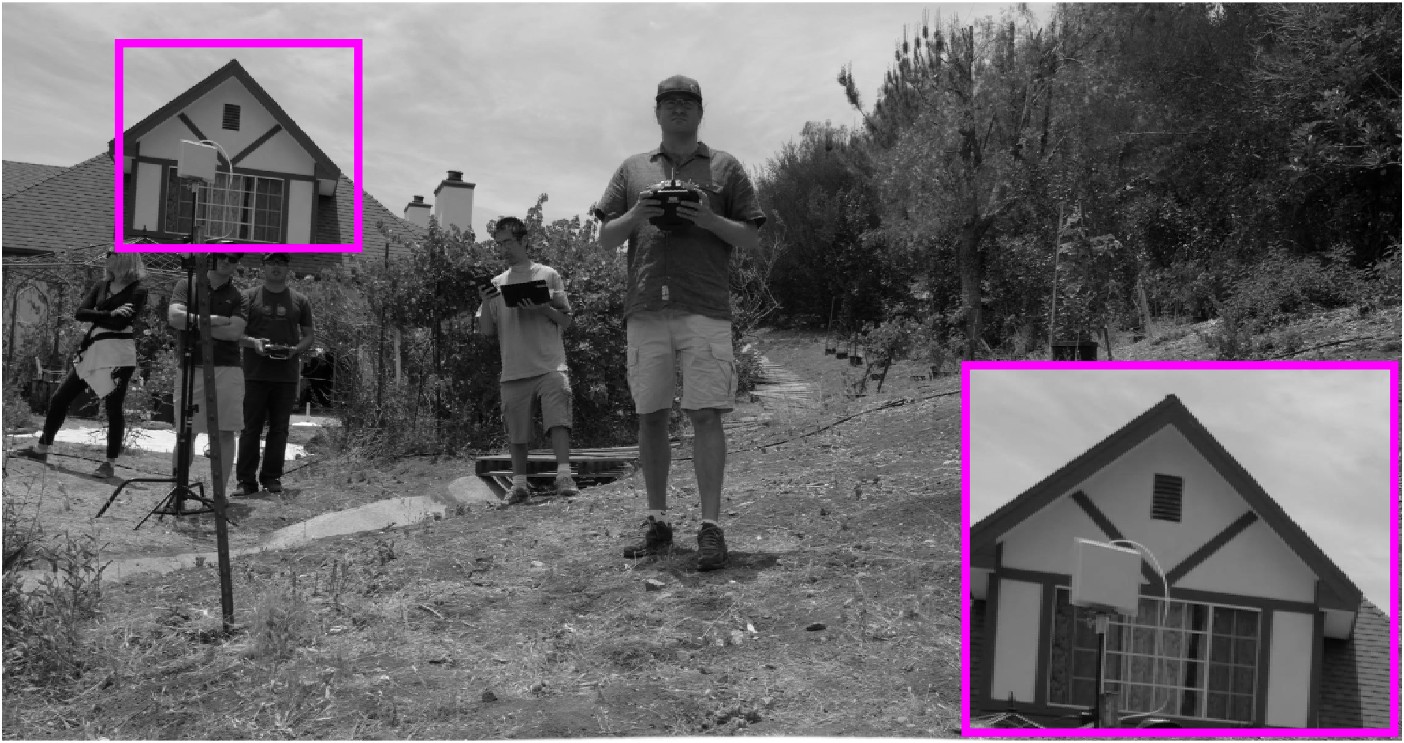} &
\includegraphics[width=0.672in]{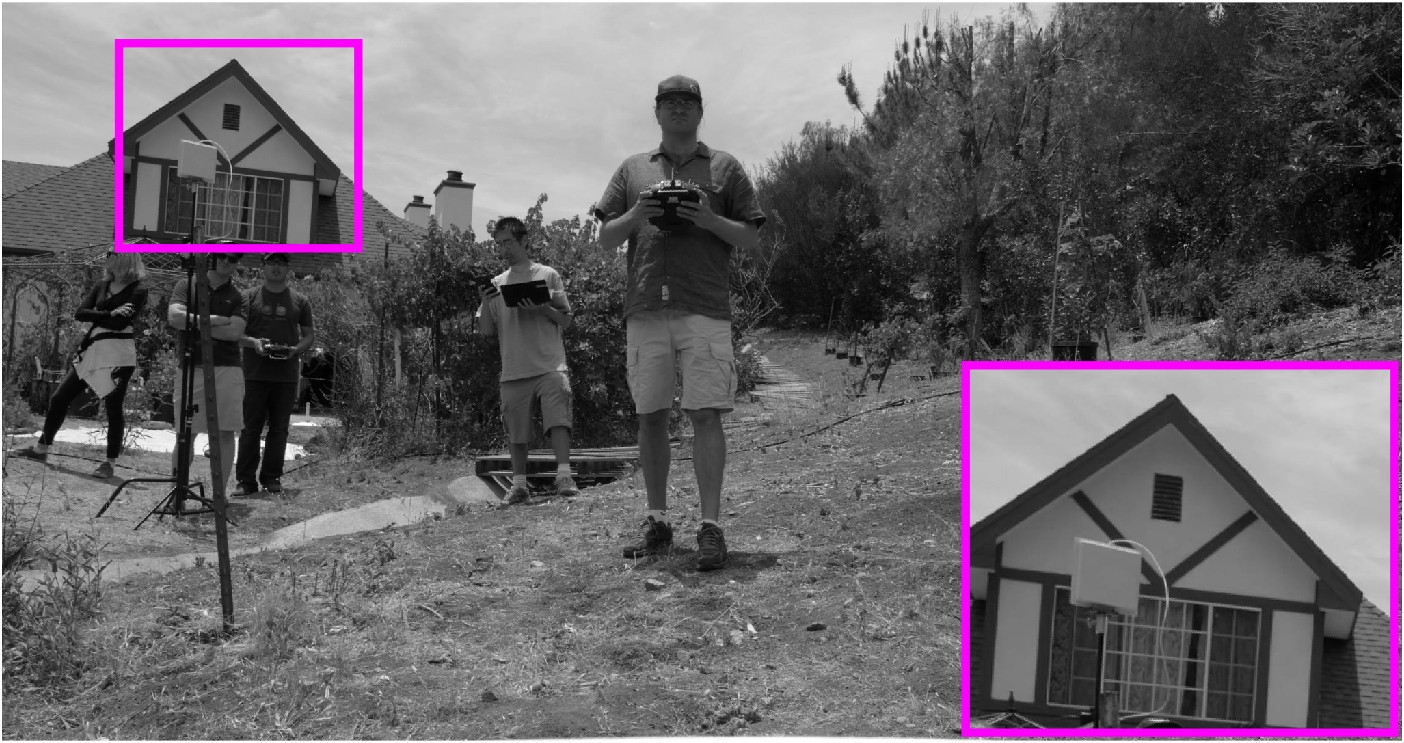} &
\includegraphics[width=0.672in]{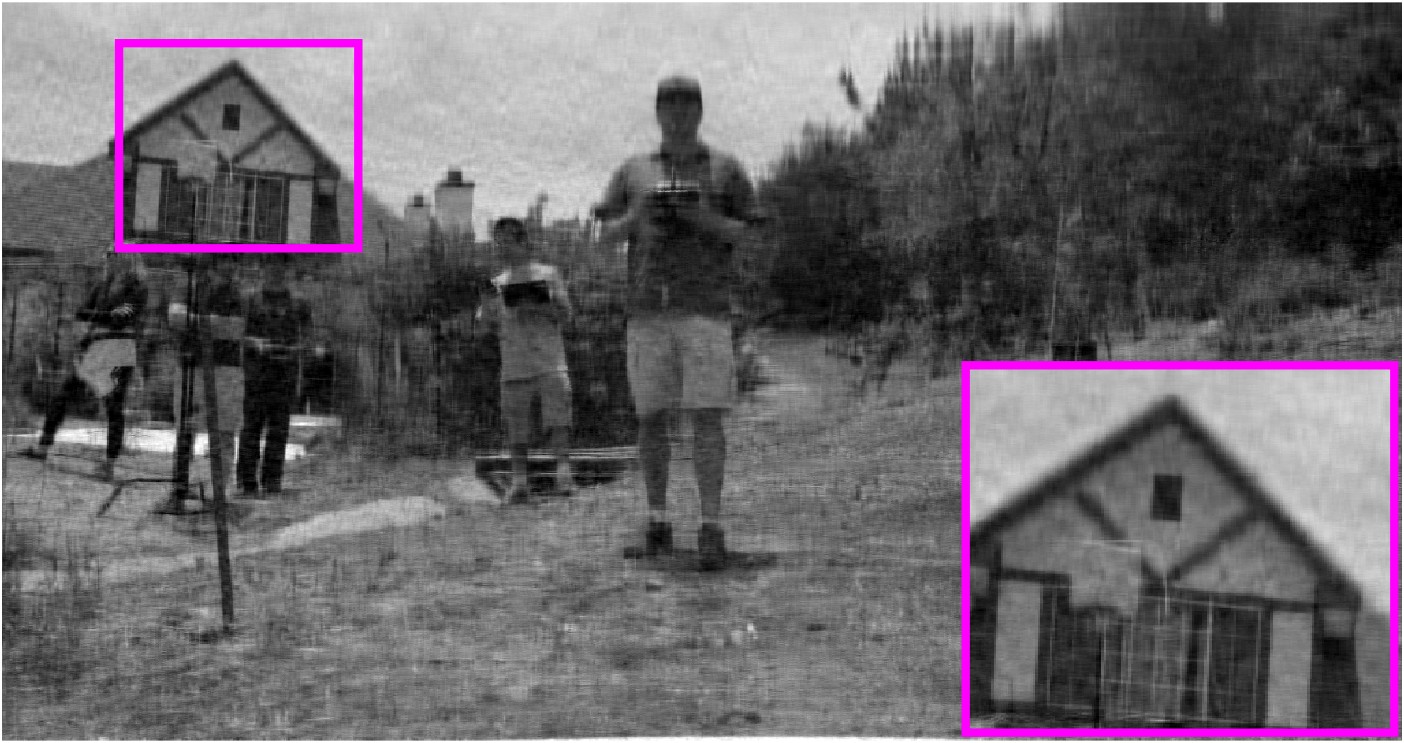} \\
&\tiny PSNR:24.45& \tiny PSNR:28.23 & \tiny PSNR:24.15 & \tiny PSNR:30.82 & \tiny PSNR:33.47 & \tiny PSNR:24.42\\
&\tiny SSIM:0.601 & \tiny SSIM:0.745 & \tiny SSIM:0.557 & \tiny SSIM:0.857 & \tiny SSIM:0.916 & \tiny SSIM:0.587\\
\tiny\makecell[c]{TMSTP-SVD\\[-4pt](k=3)} &\tiny\makecell[c]{MRSTP-SVD\\[-4pt](k=1)} & \tiny\makecell[c]{MRSTP-SVD\\[-4pt](k=2)} & \tiny\makecell[c]{MRSTP-SVD\\[-4pt](k=3)}& \tiny\makecell[c]{TMRSTP-SVD\\[-4pt](k=1)} & \tiny\makecell[c]{TMRSTP-SVD\\[-4pt](k=2)} & \tiny\makecell[c]{TMRSTP-SVD\\[-4pt](k=3)}\\
\includegraphics[width=0.672in]{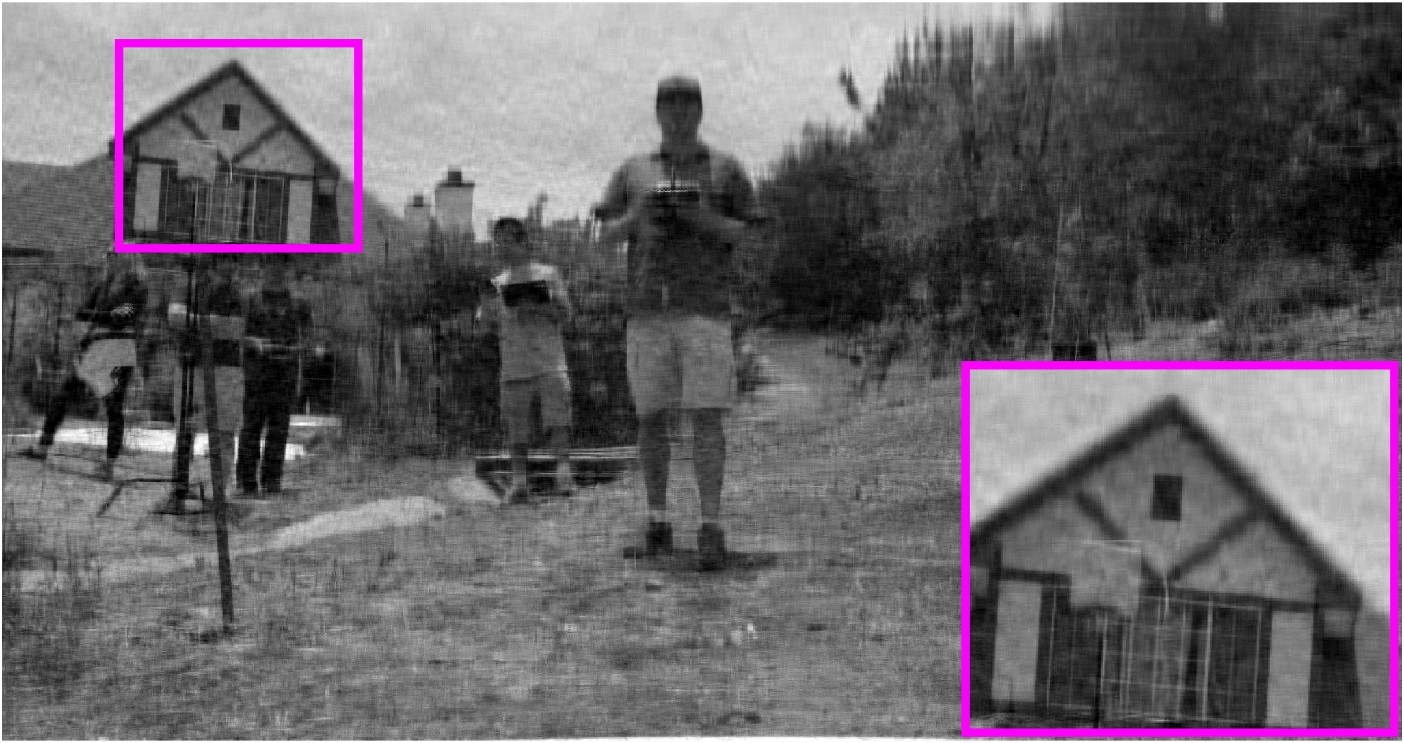} &
\includegraphics[width=0.672in]{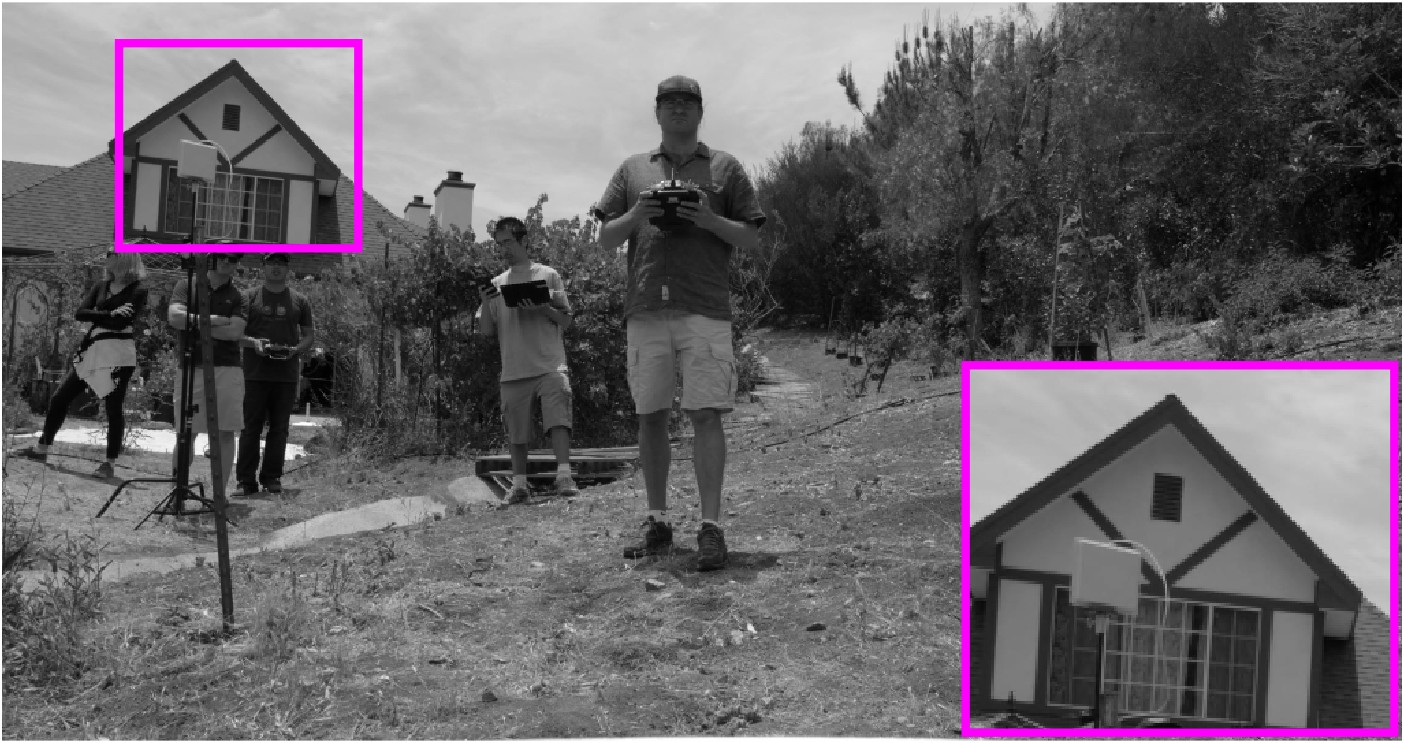} &
\includegraphics[width=0.672in]{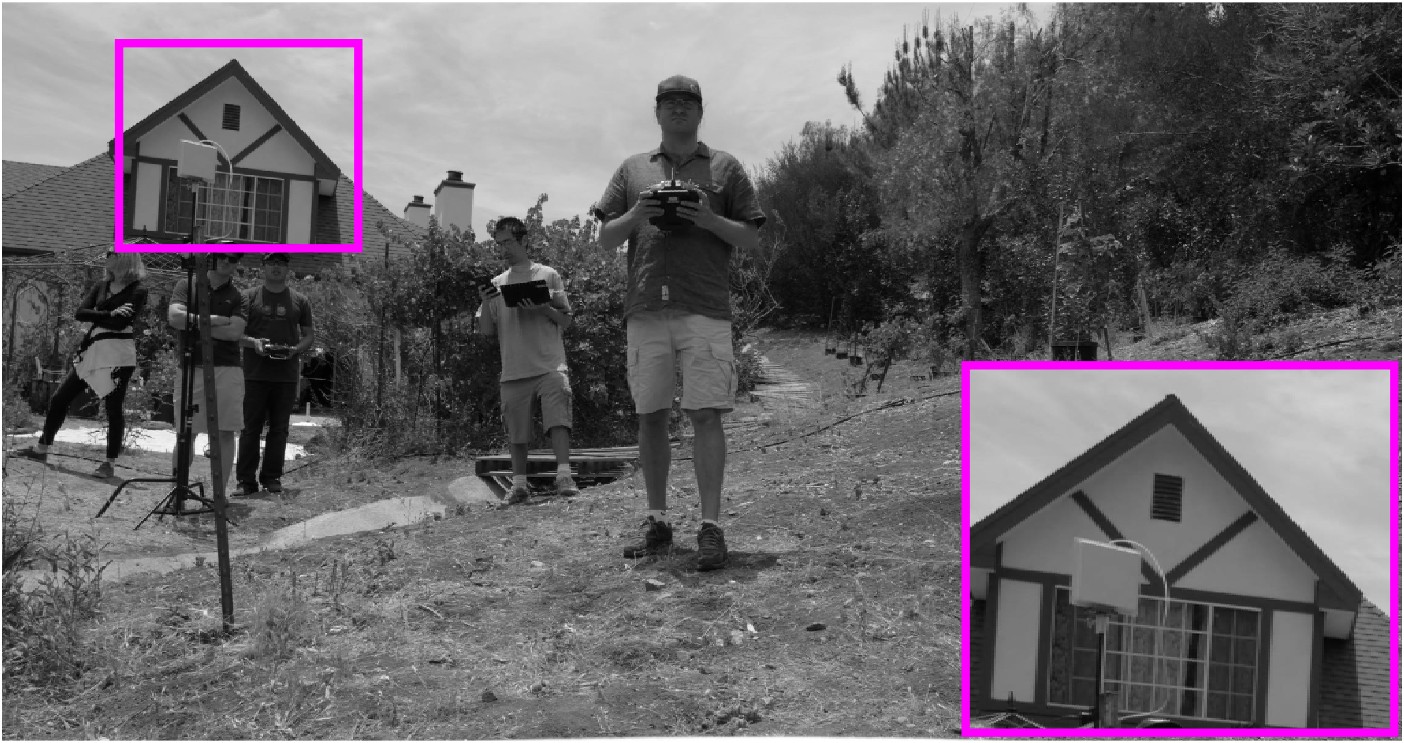} &
\includegraphics[width=0.672in]{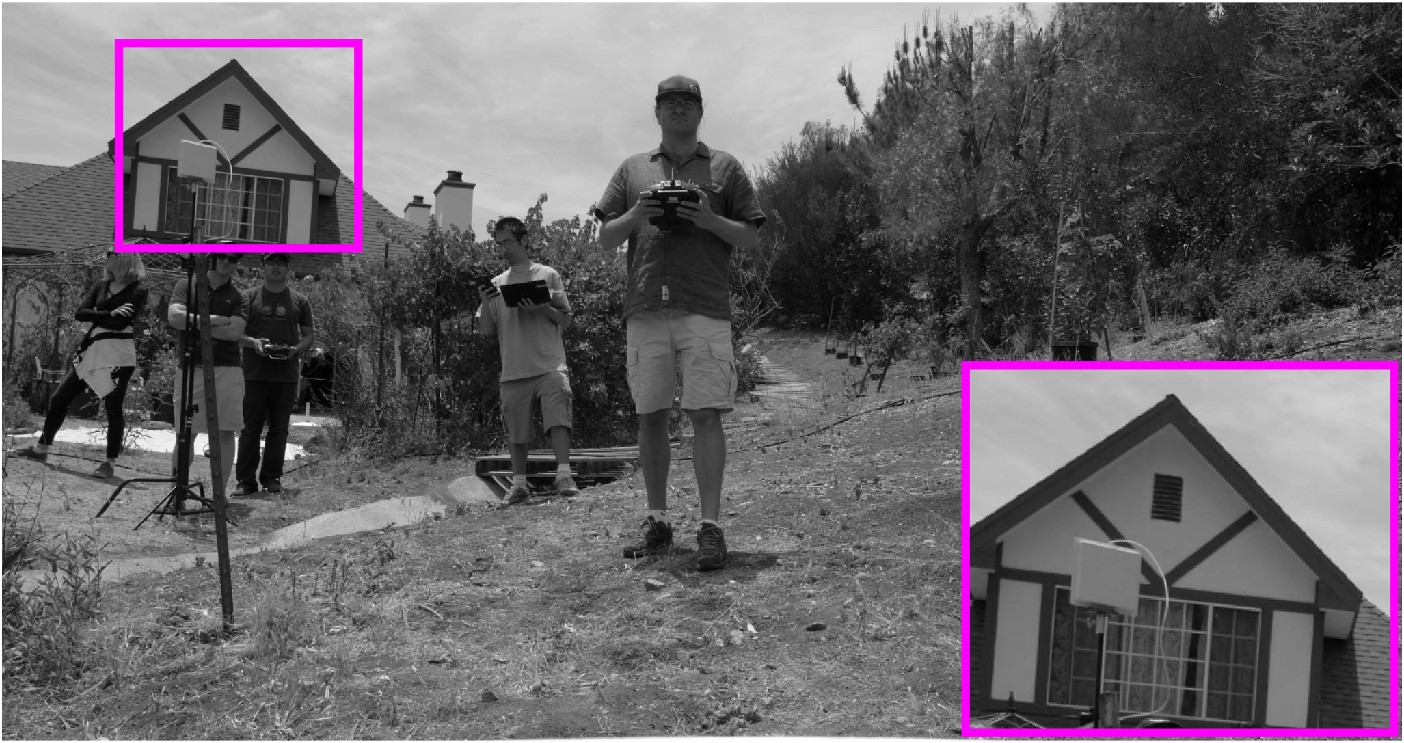}&
\includegraphics[width=0.672in]{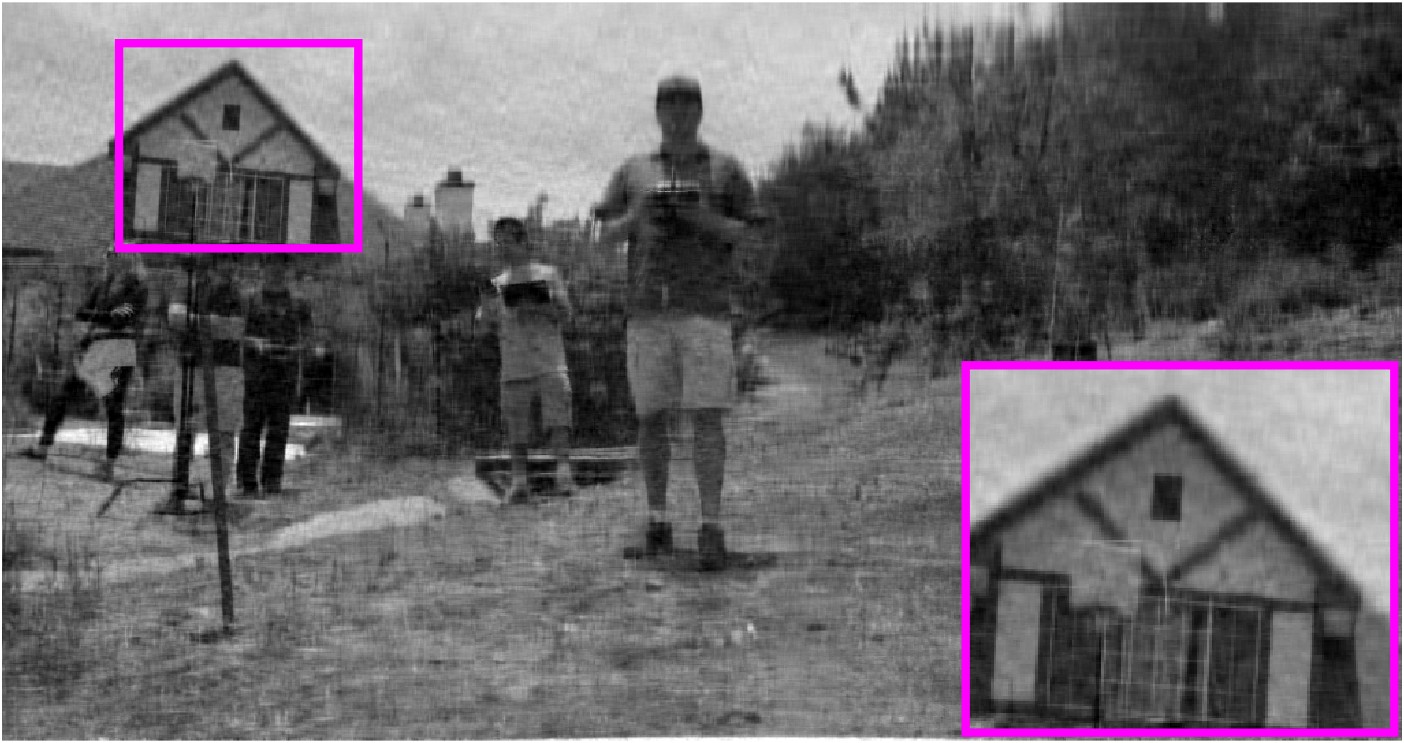} &
\includegraphics[width=0.672in]{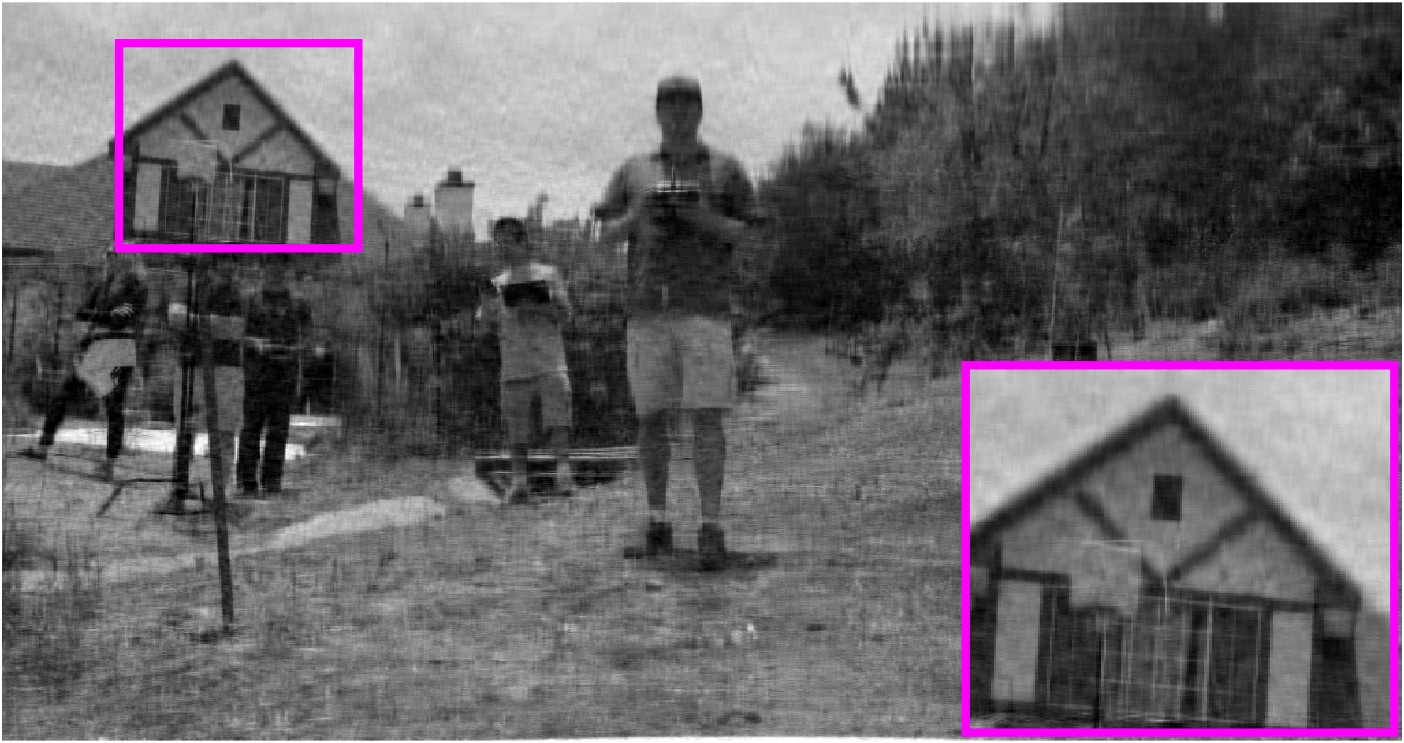} &
\includegraphics[width=0.672in]{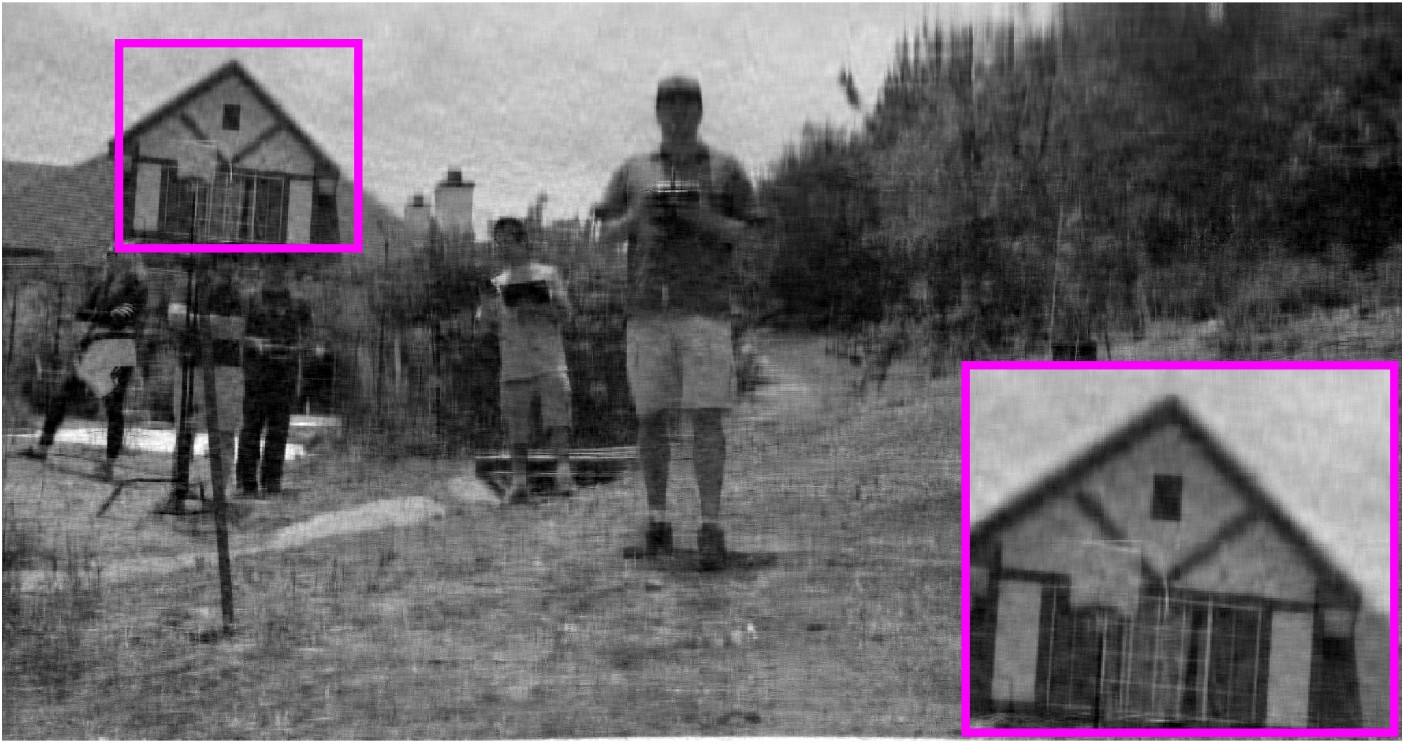}  \\
\tiny PSNR:24.58& \tiny PSNR:28.38 & \tiny PSNR:30.94 & \tiny PSNR:33.88 & \tiny PSNR:24.16 & \tiny PSNR:24.43 & \tiny PSNR:24.60\\
\tiny SSIM:0.600& \tiny SSIM:0.751 & \tiny SSIM:0.860 & \tiny SSIM:0.922 & \tiny SSIM:0.558 & \tiny SSIM:0.588 & \tiny SSIM:0.602\\      
\end{tabular}

\caption{Supplementary visual reconstruction examples and quantitative PSNR-SSIM comparisons between competing baselines and our approaches (original and truncated variants) on randomly sampled frames from two additional test video sequences.}
\label{fig:video compression_sup}
\end{figure}

Fig.~\ref*{fig:video_complete_sup} reports supplementary video completion results on two additional test sequences under 70\% randomly missing pixels, including frame-wise PSNR-SSIM curves, total runtime statistics, and visual reconstruction examples on sampled frames. The experiments adopt the same iterative completion framework as the main text, with DFT as the default transform and 70\% pixel missing ratio. Quantitative results show that the proposed methods outperform baseline competitors in reconstruction accuracy. Compared with the deterministic TMSTP-SVD, the randomized TMRSTP-SVD substantially cuts down the overall completion runtime, while TMRSTP-SVD incurs nearly negligible reconstruction quality loss. Visually, the proposed methods recover finer textures and introduce fewer visible artifacts in missing regions. These supplementary results further confirm the generalization ability of the proposed methods for video completion tasks.
\begin{figure}[!ht]  
  \centering
  \includegraphics[width=0.49\linewidth]{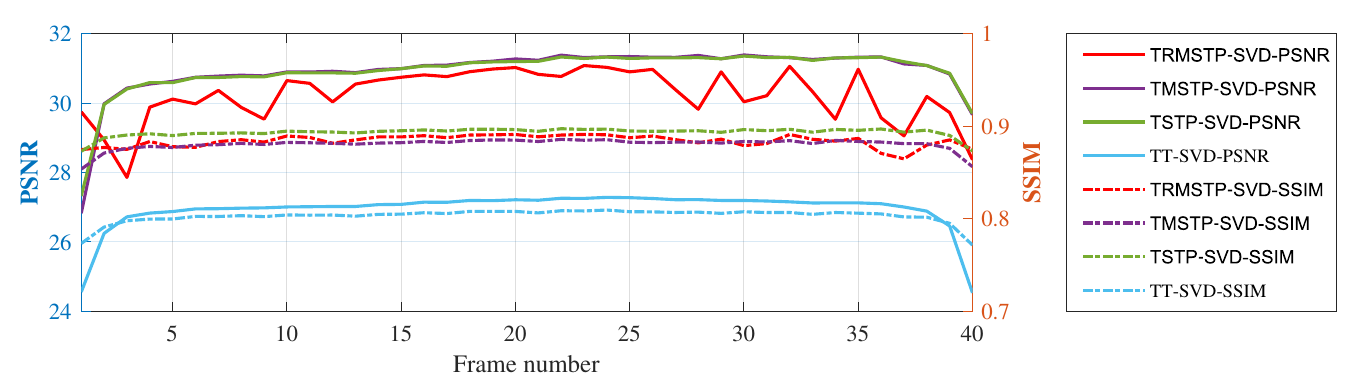}
\hfill
  \includegraphics[width=0.49\linewidth]{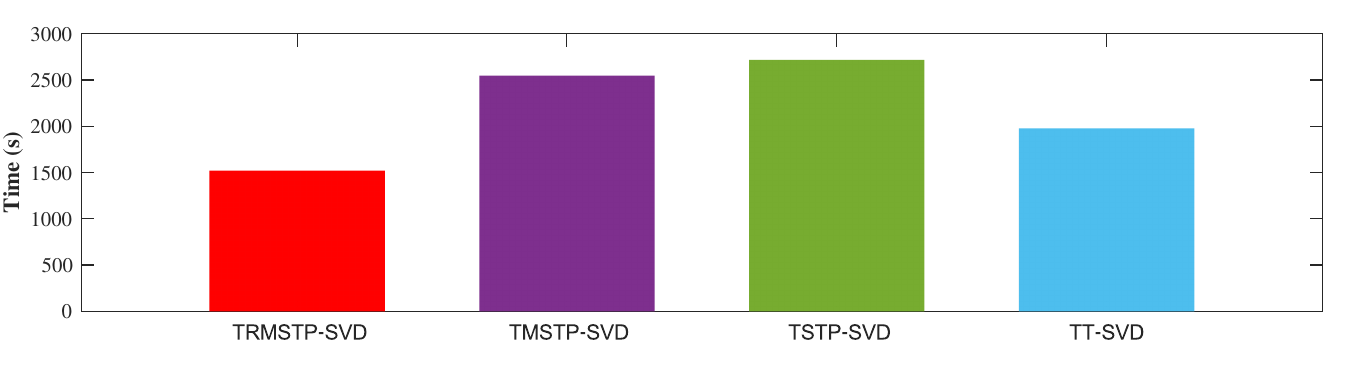}
\renewcommand{\arraystretch}{0.5} 
  \setlength{\tabcolsep}{0.3pt}       
  \begin{tabular}{@{}cccccc@{}}
    \includegraphics[width=0.77in]{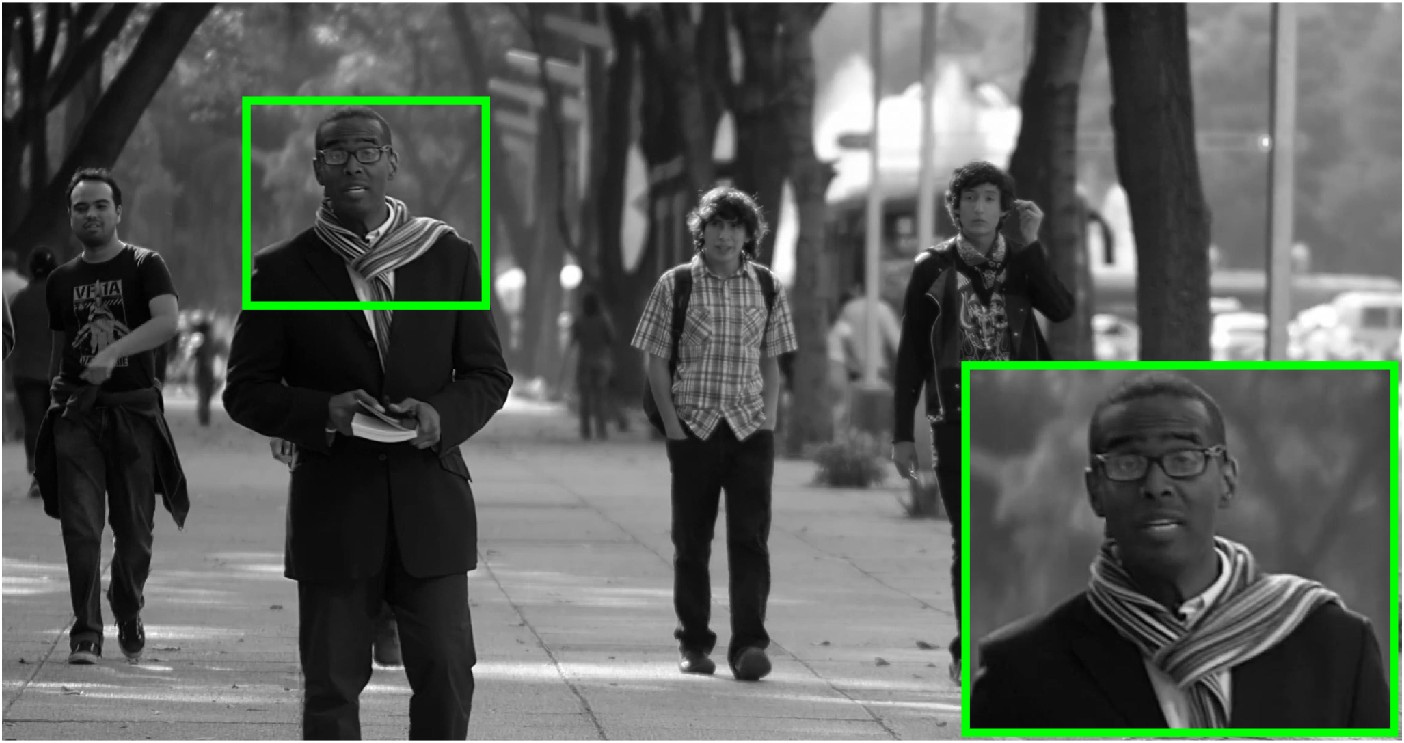} &
    \includegraphics[width=0.77in]{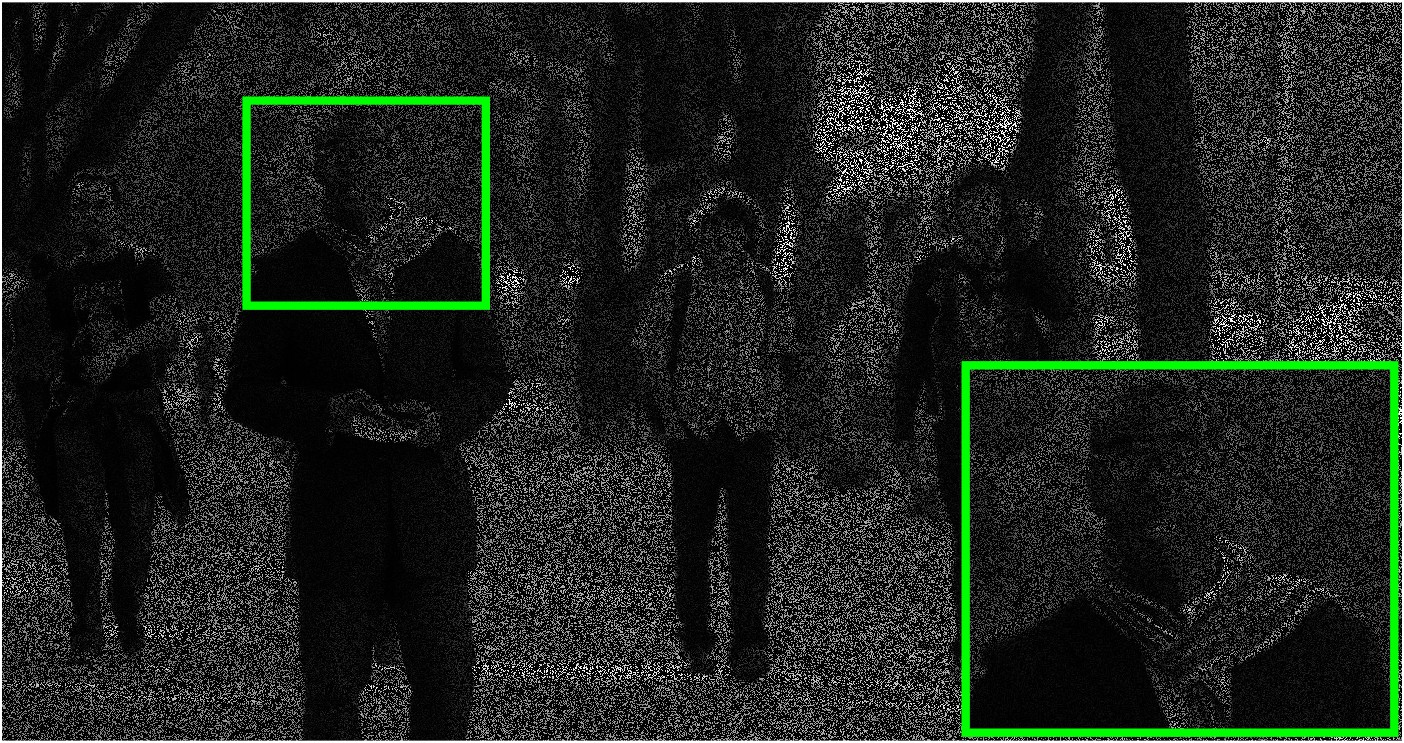} &
    \includegraphics[width=0.77in]{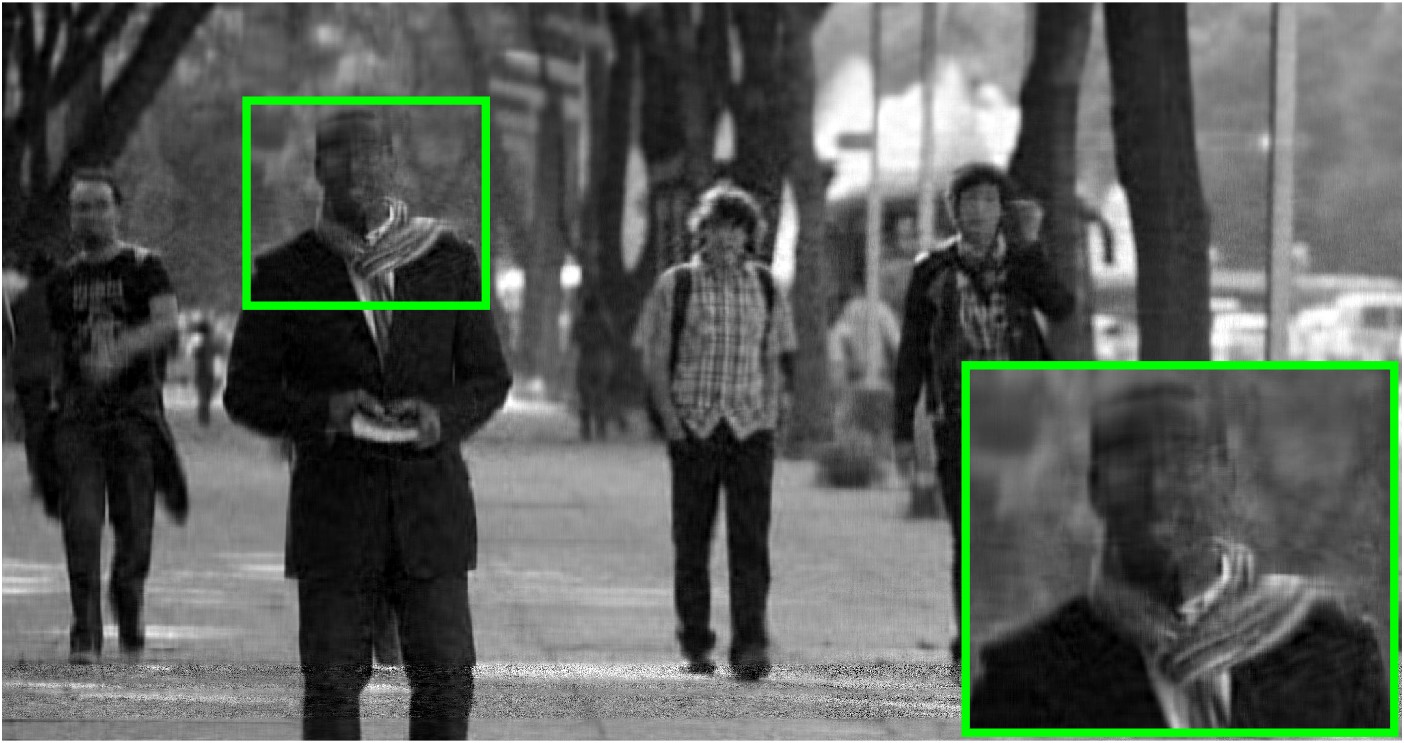} &
    \includegraphics[width=0.77in]{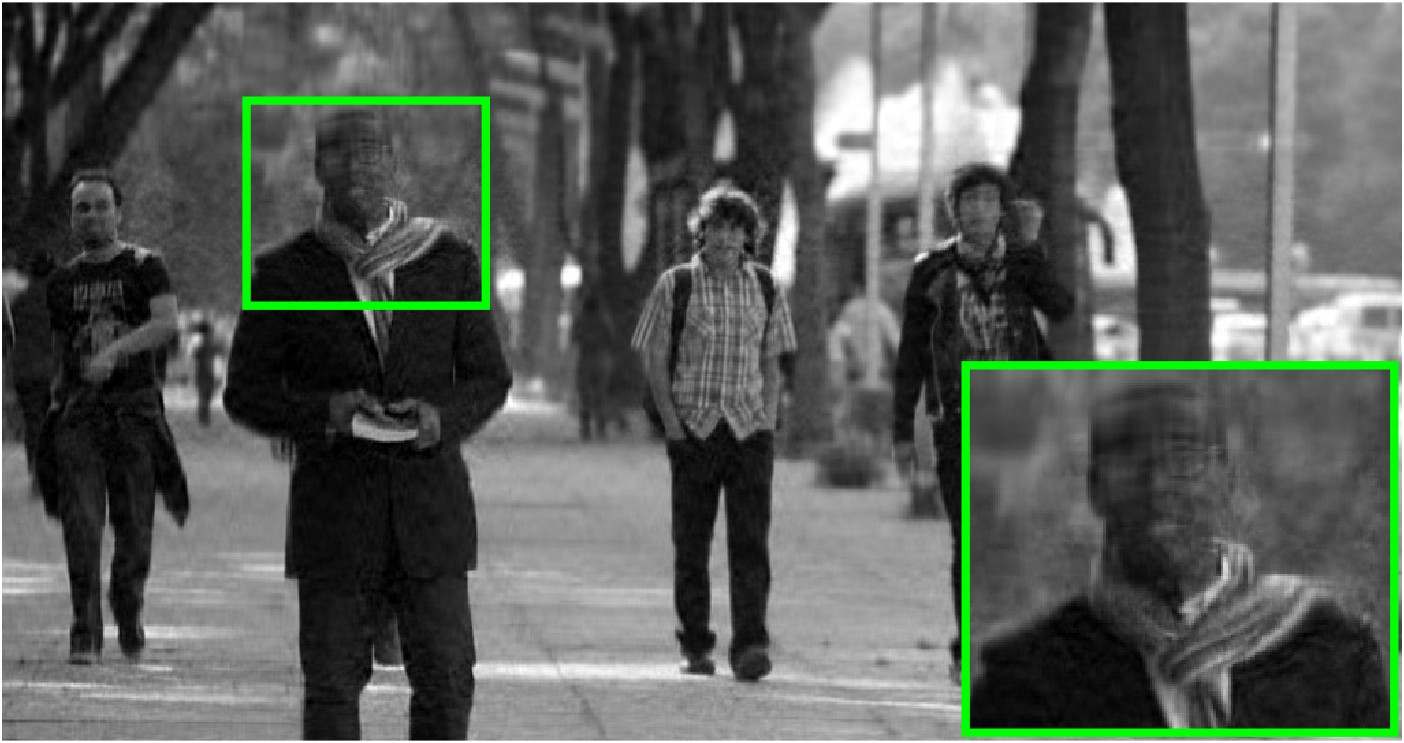}&
    \includegraphics[width=0.77in]{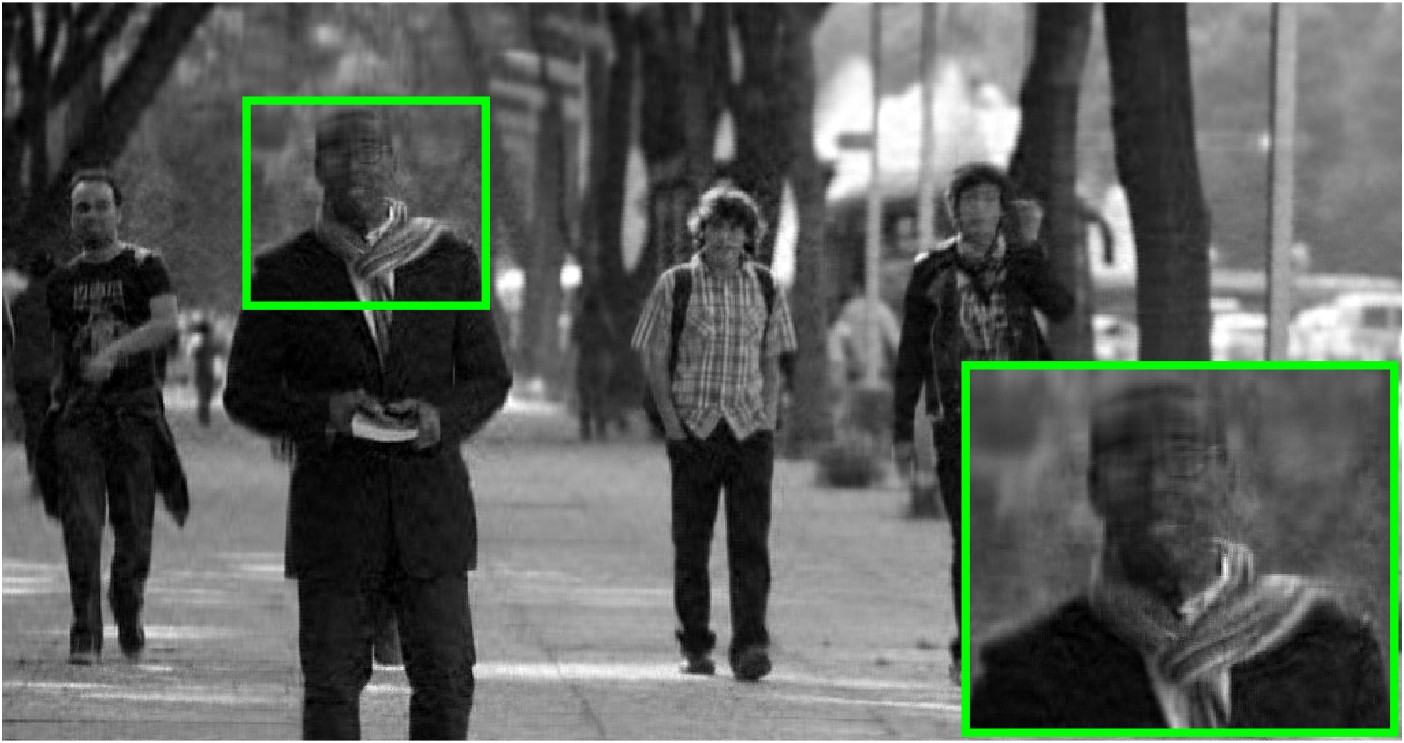} &
    \includegraphics[width=0.77in]{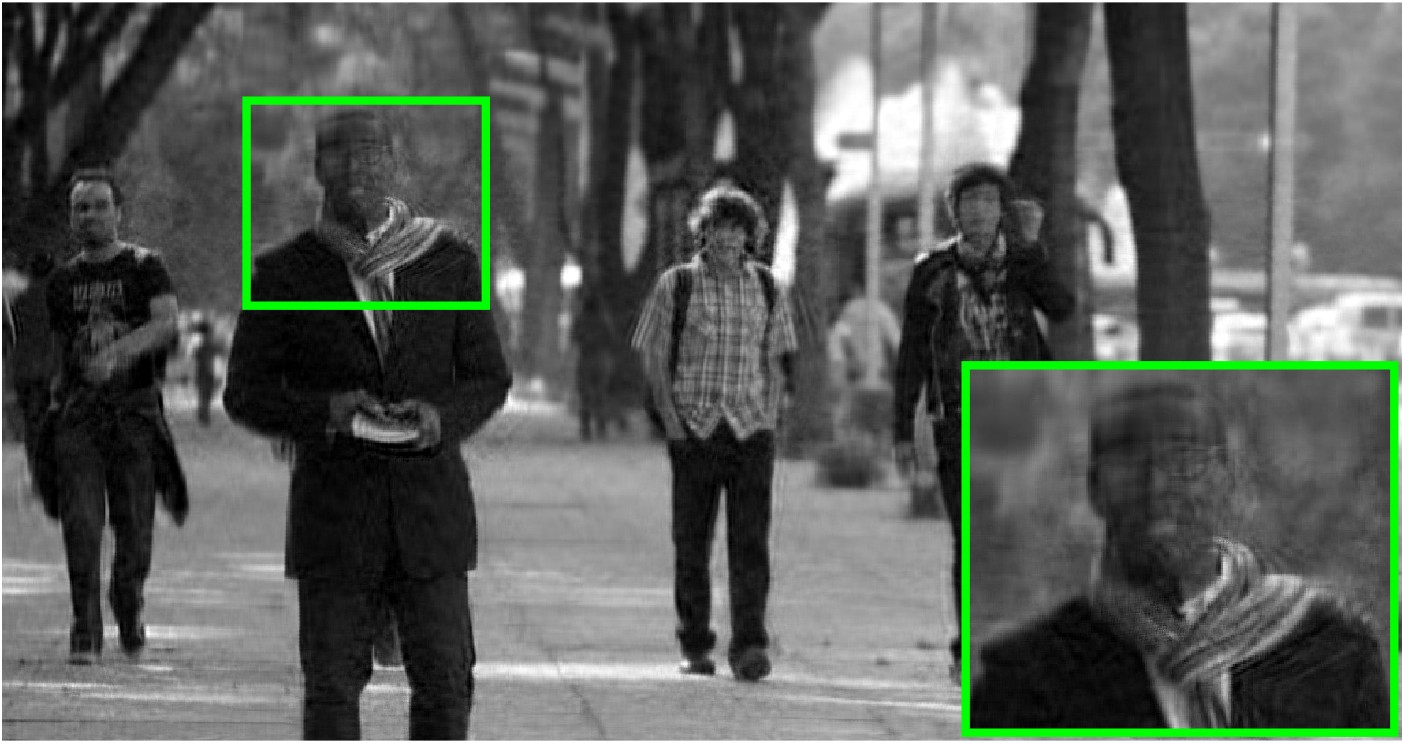} \\
    \includegraphics[width=0.77in]{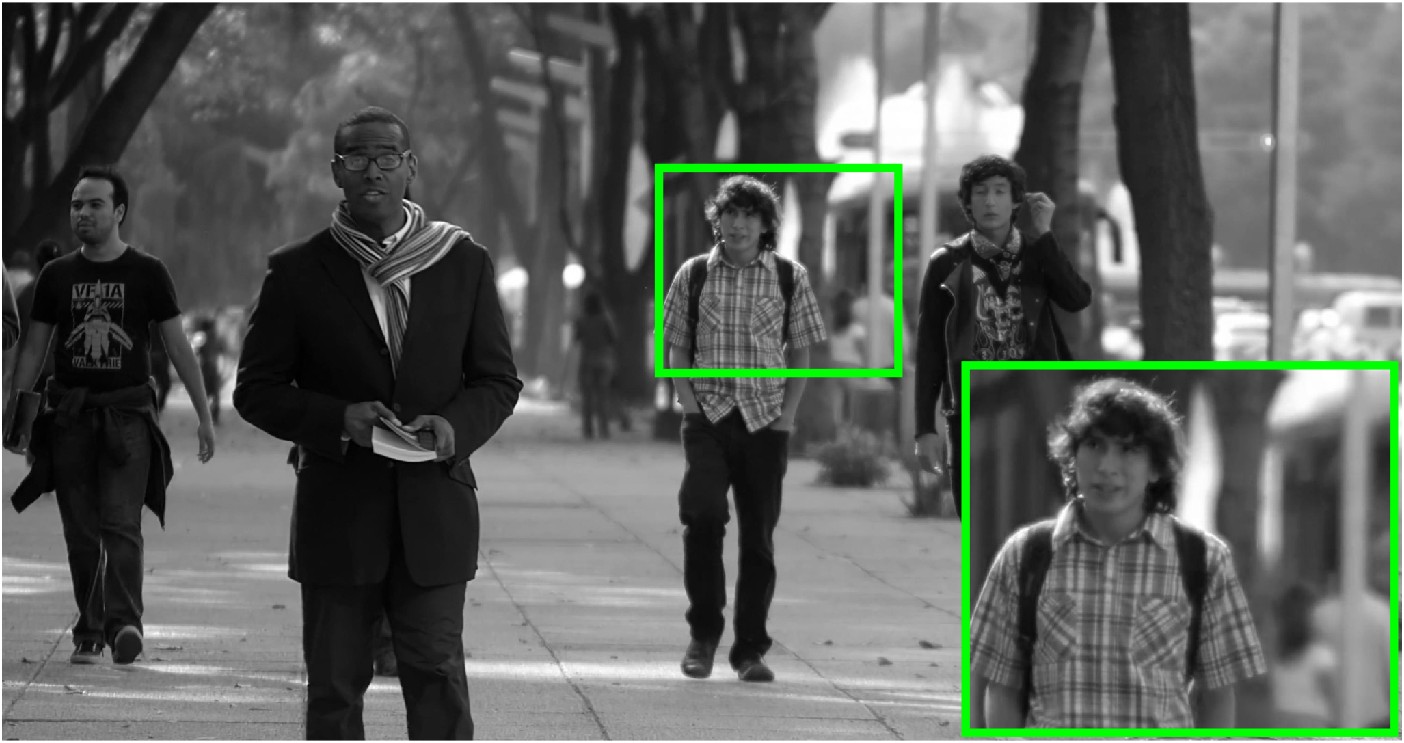} &
    \includegraphics[width=0.77in]{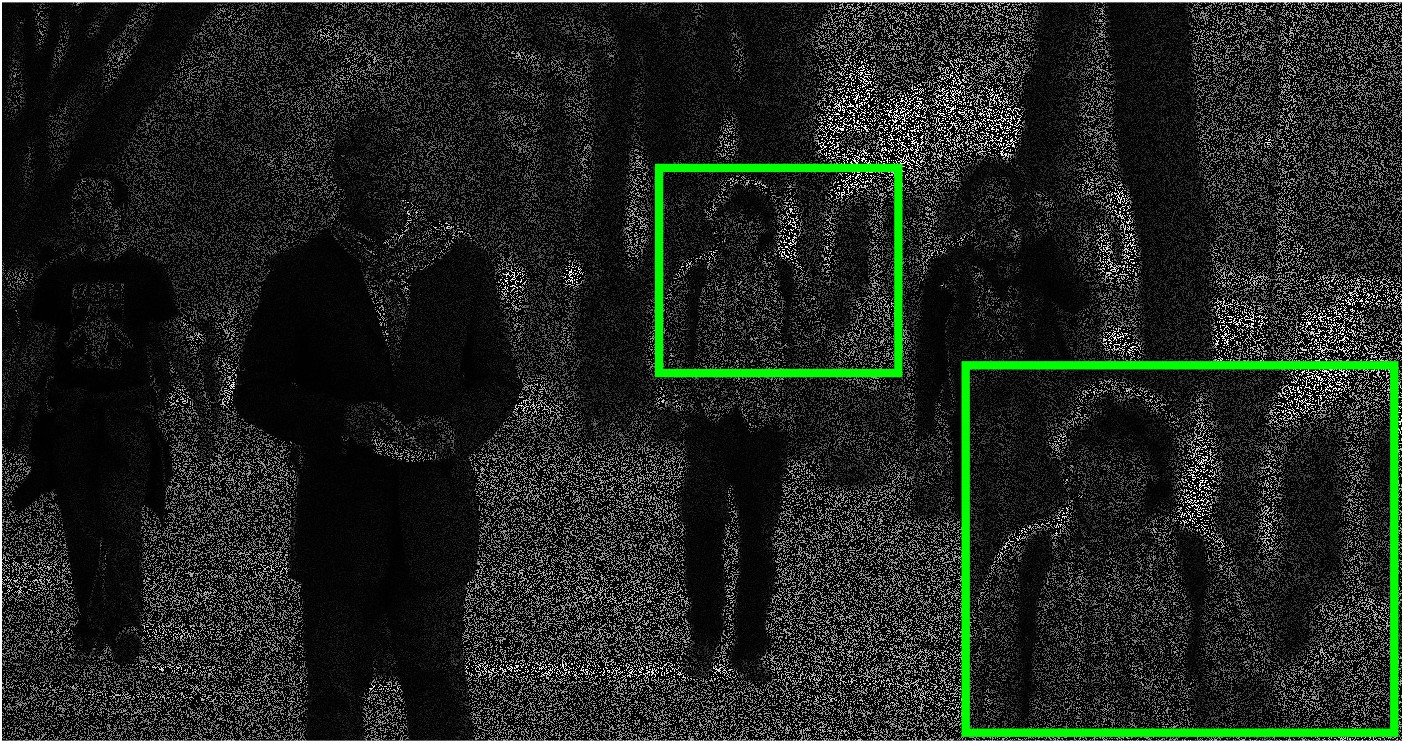} &
    \includegraphics[width=0.77in]{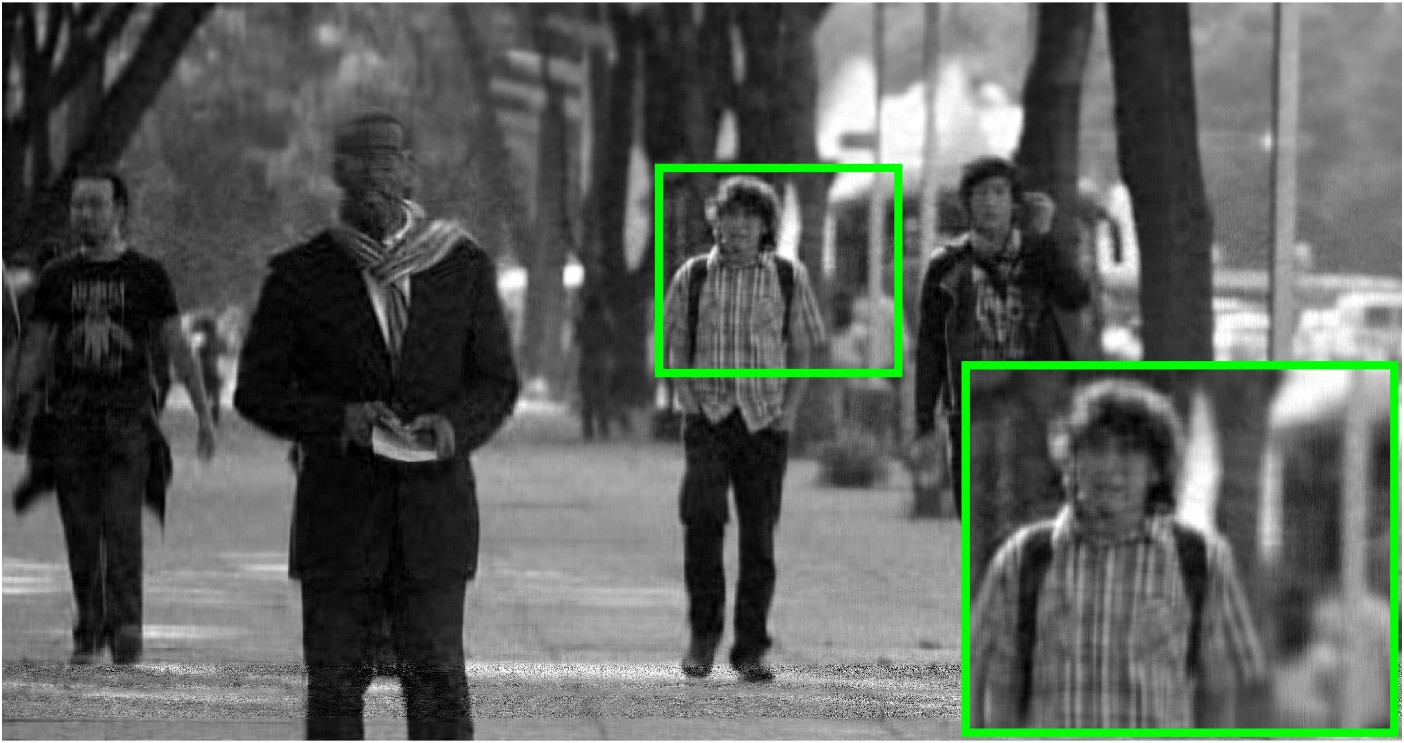} &
    \includegraphics[width=0.77in]{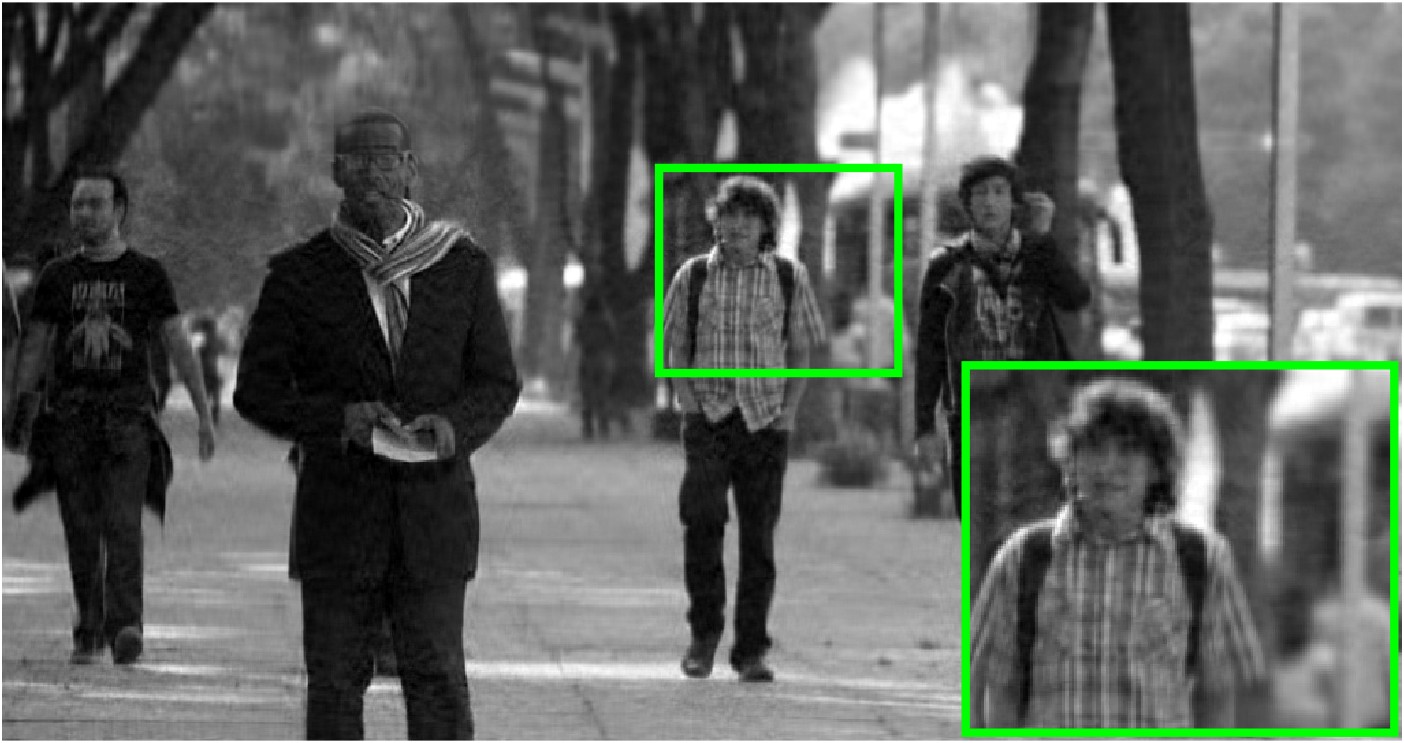}&
    \includegraphics[width=0.77in]{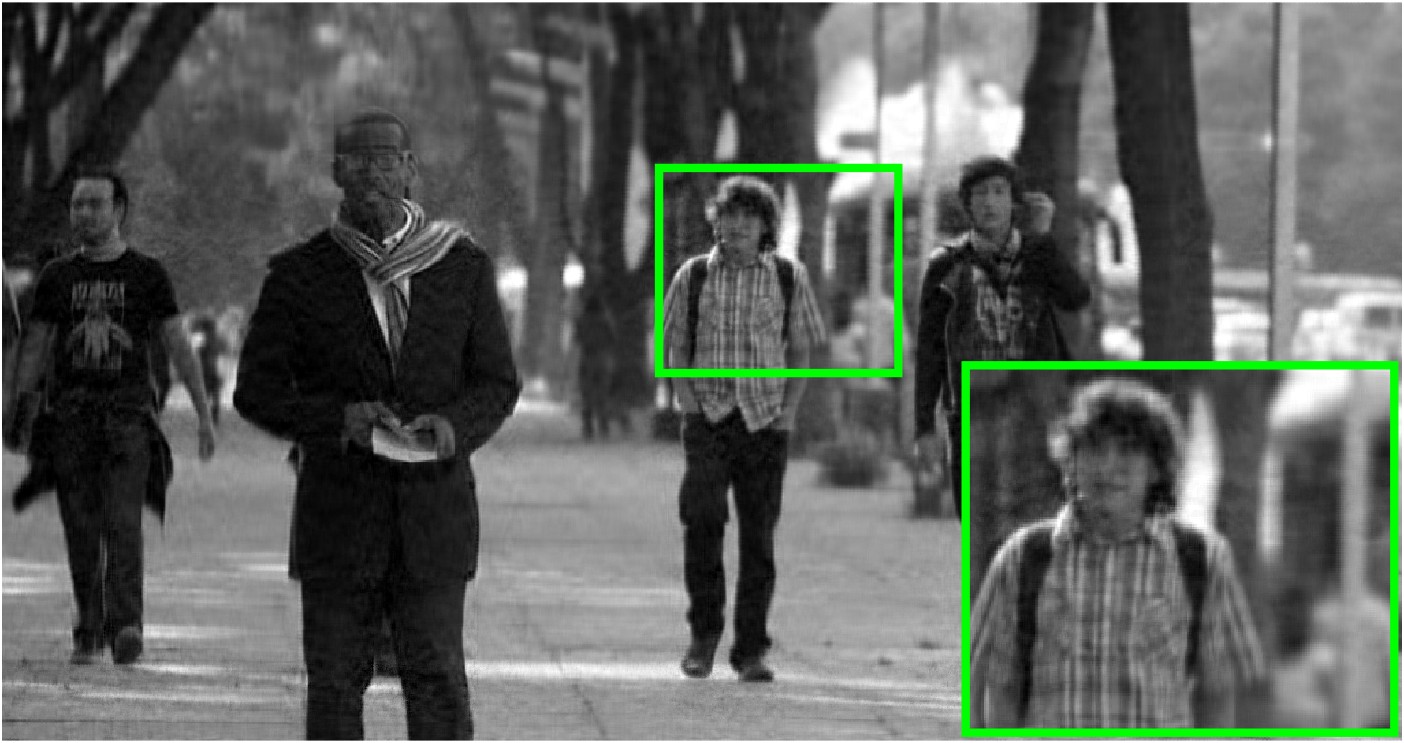} &
    \includegraphics[width=0.77in]{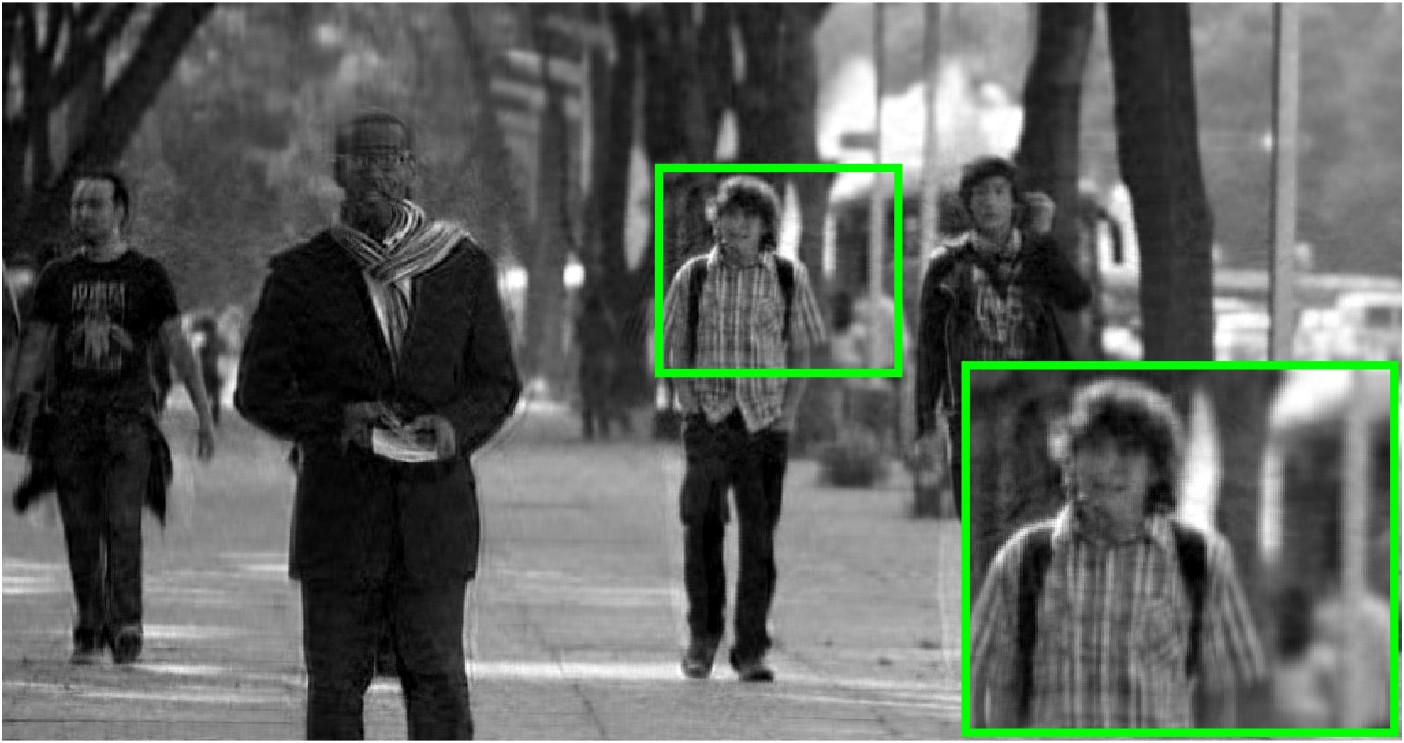} \\
    \scriptsize Original  &
    \scriptsize 70\% missing &
    \scriptsize  TT-SVD &
    \scriptsize  TSTP-SVD&
    \scriptsize  TMSTP-SVD &
    \scriptsize  TMRSTP-SVD
  \end{tabular}
\includegraphics[width=0.49\linewidth]{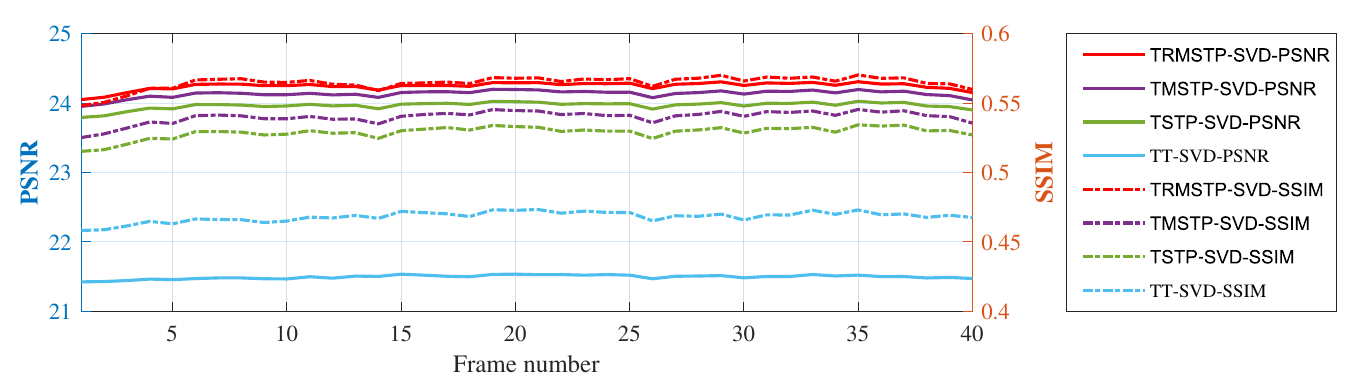}
\hfill
\includegraphics[width=0.49\linewidth]{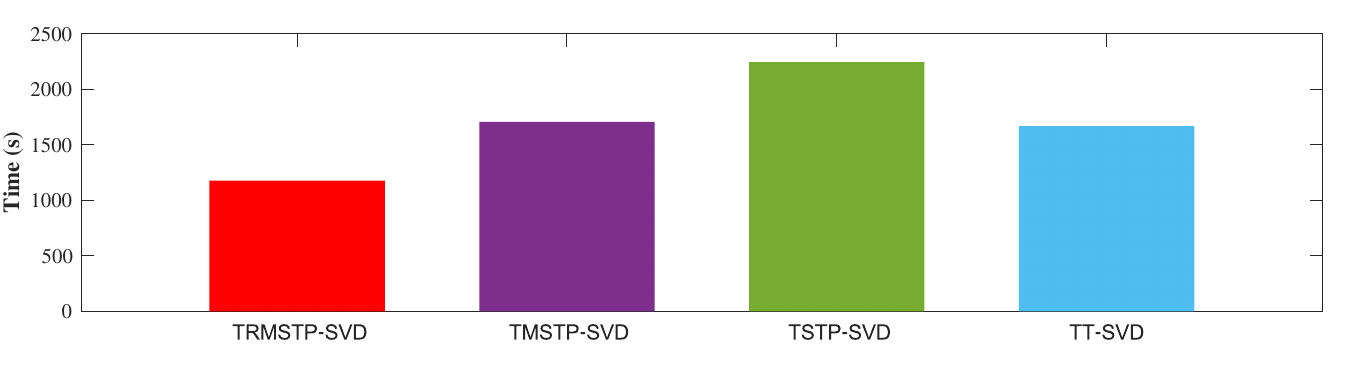}
\renewcommand{\arraystretch}{0.5} 
  \setlength{\tabcolsep}{0.3pt}       
  \begin{tabular}{@{}cccccc@{}}
    \includegraphics[width=0.77in]{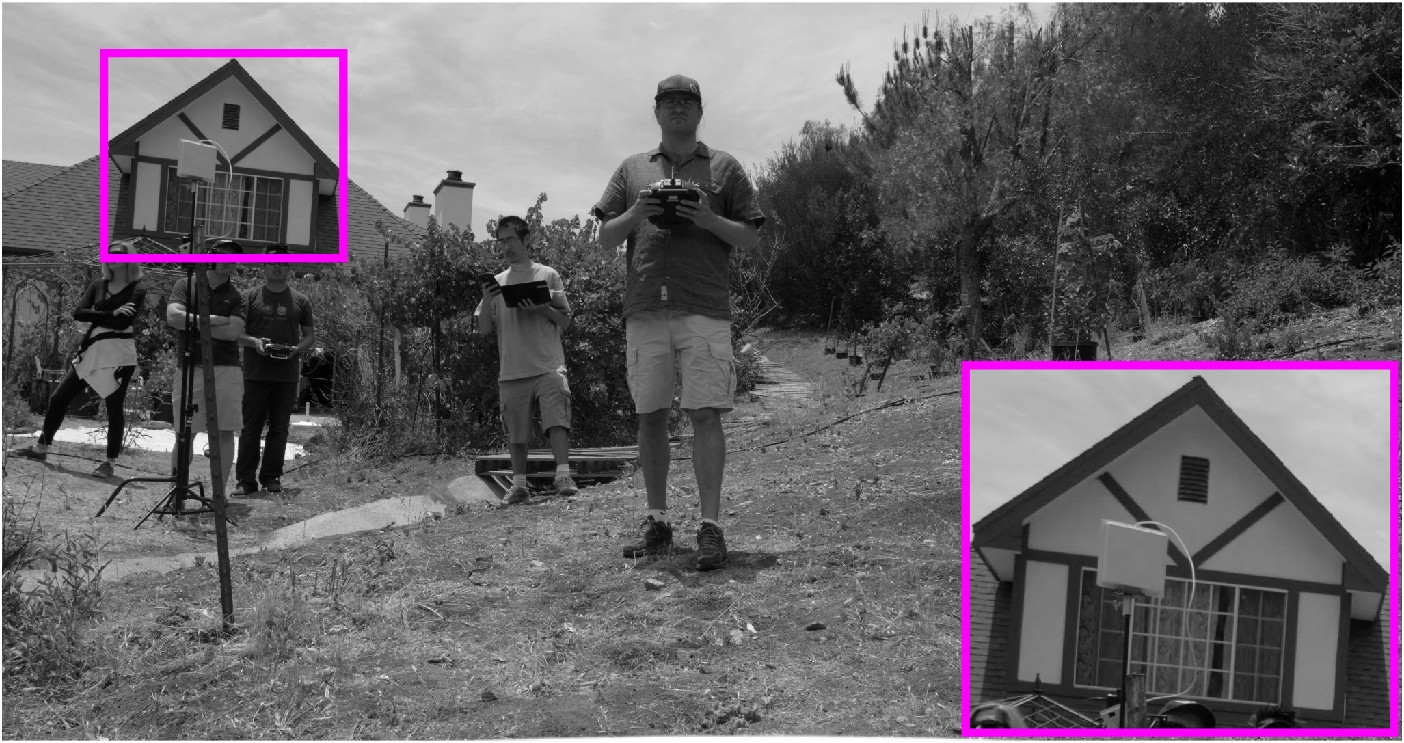} &
    \includegraphics[width=0.77in]{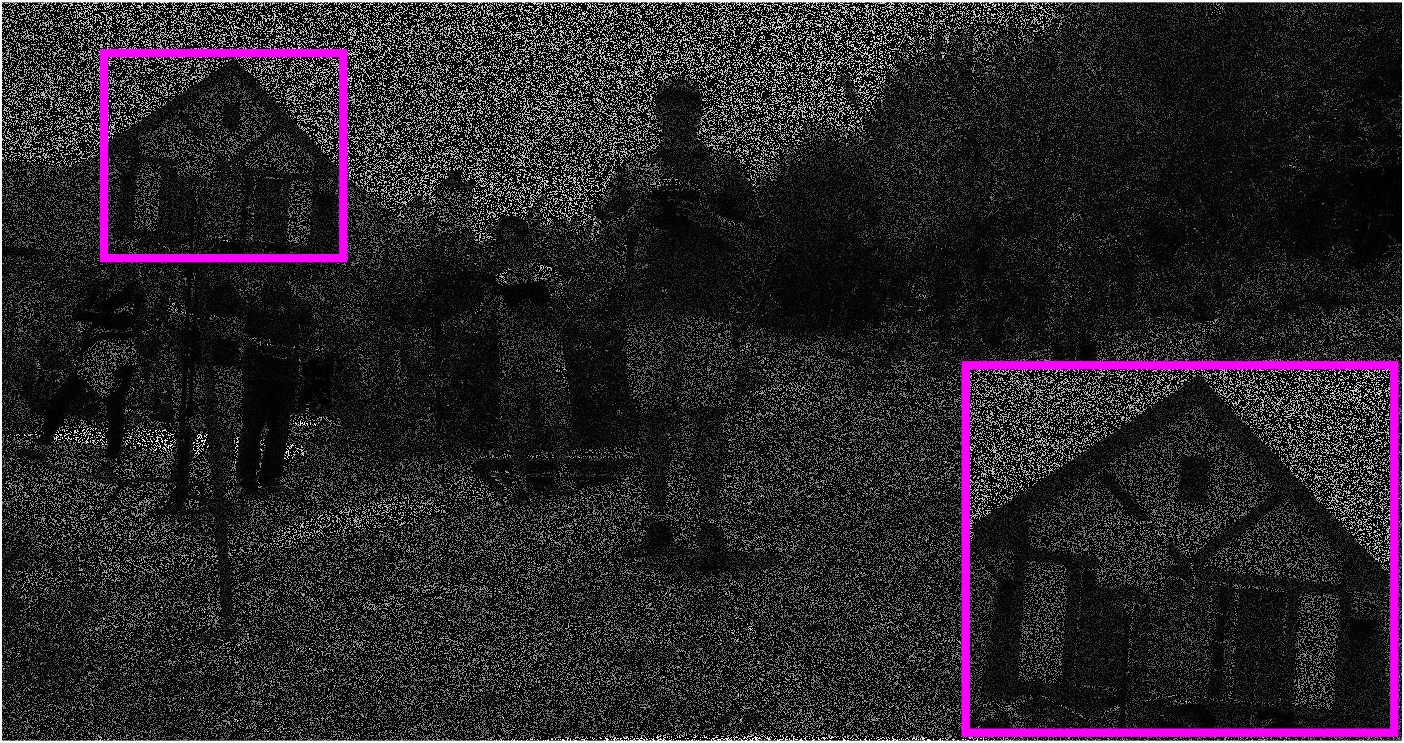} &
    \includegraphics[width=0.77in]{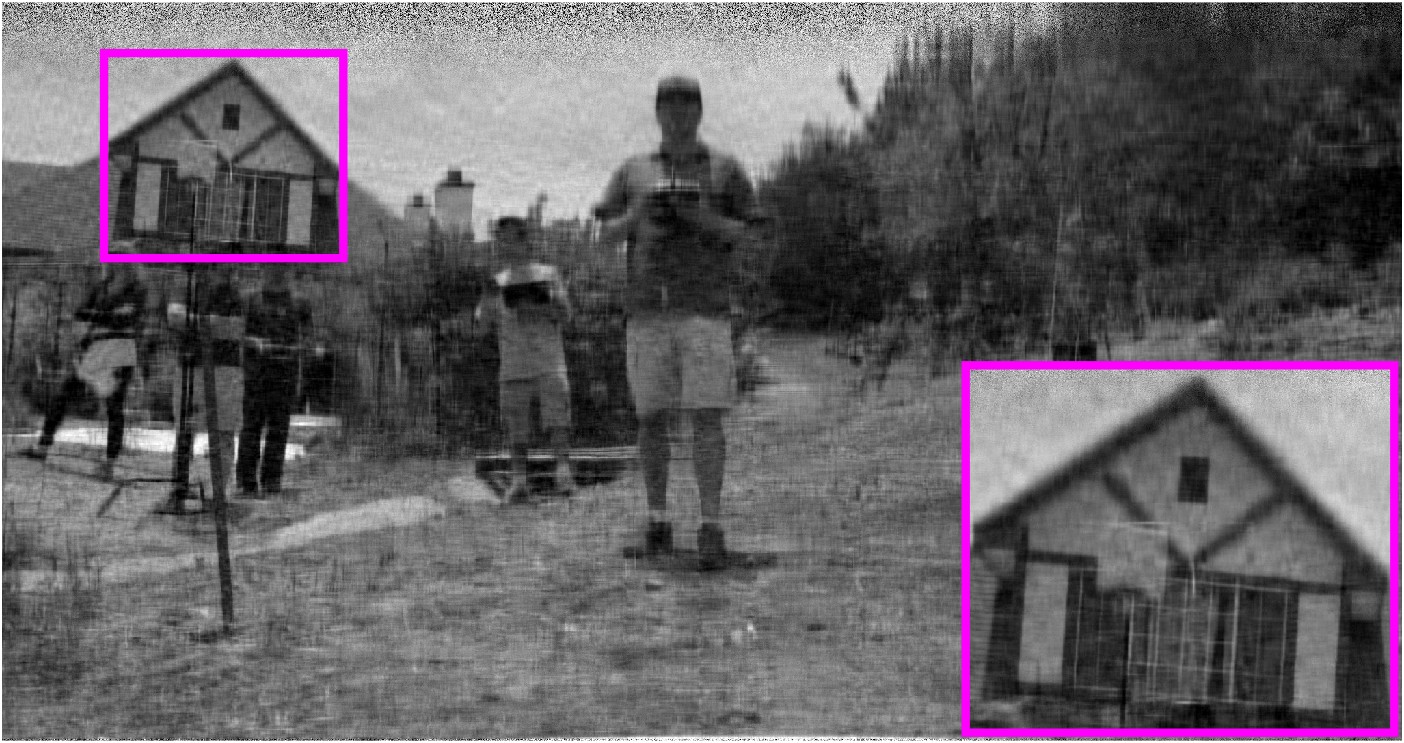} &
    \includegraphics[width=0.77in]{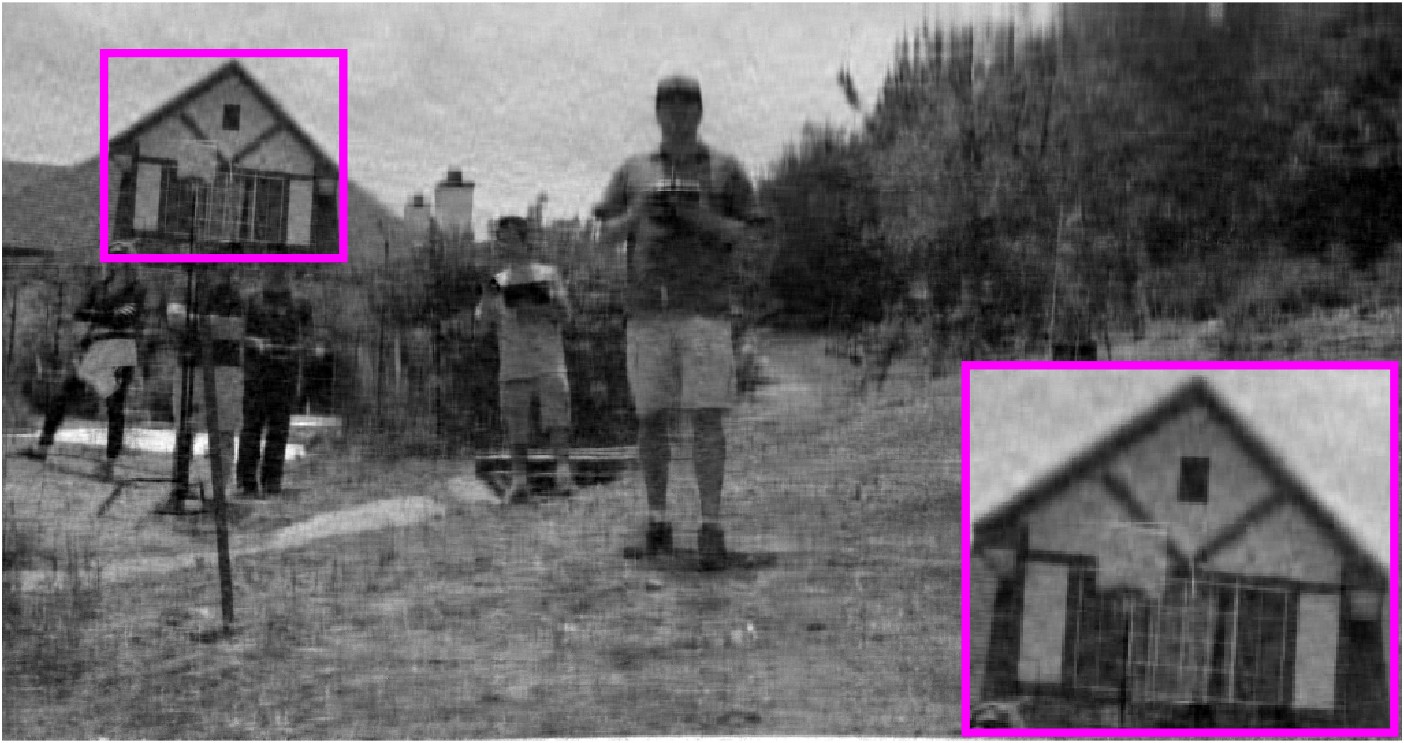}&
    \includegraphics[width=0.77in]{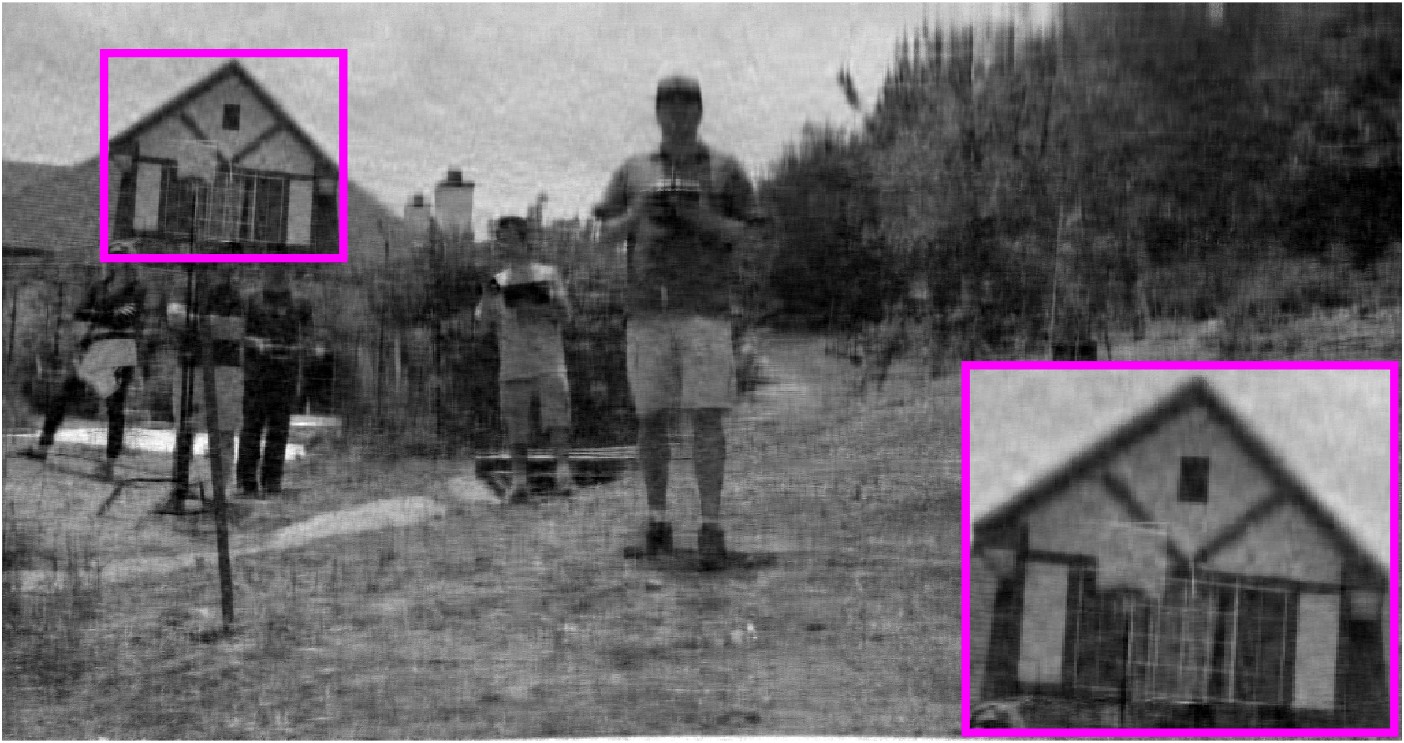} &
    \includegraphics[width=0.77in]{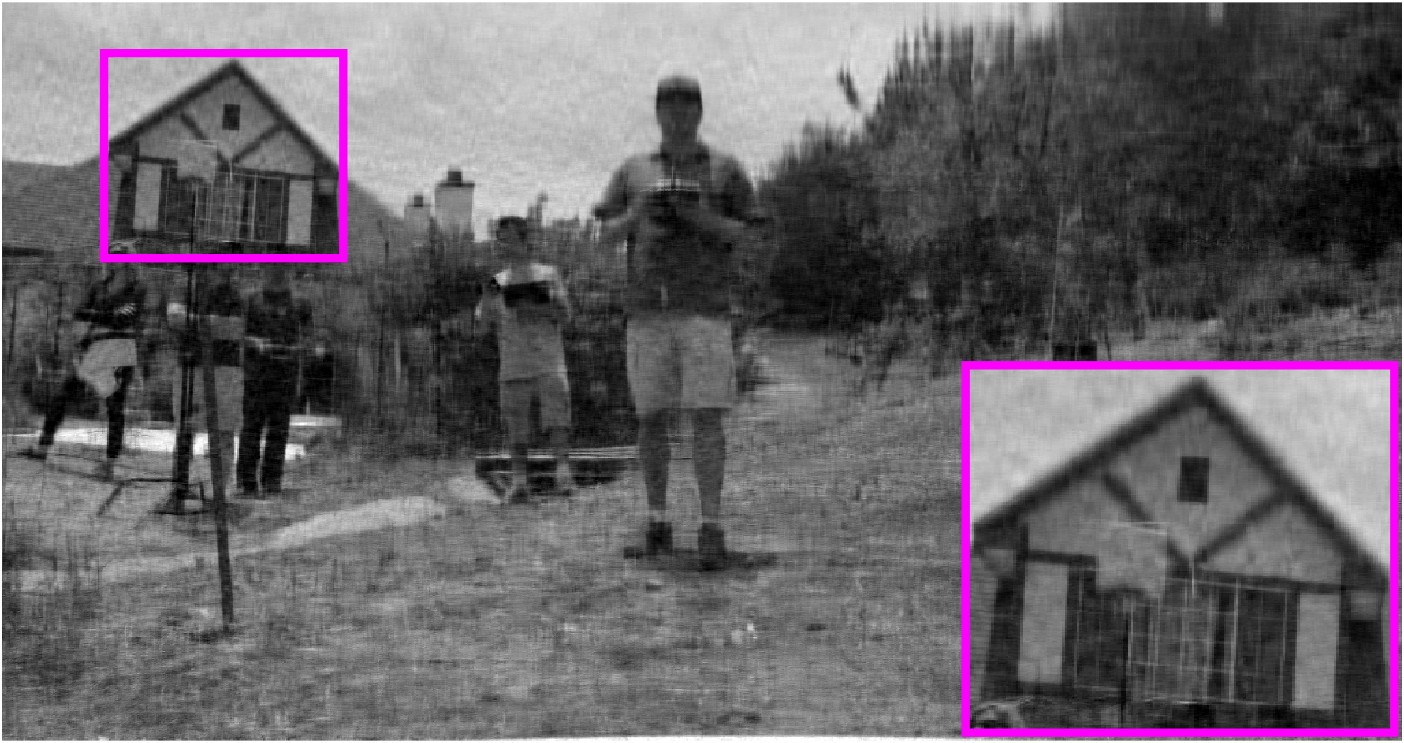} \\
    \includegraphics[width=0.77in]{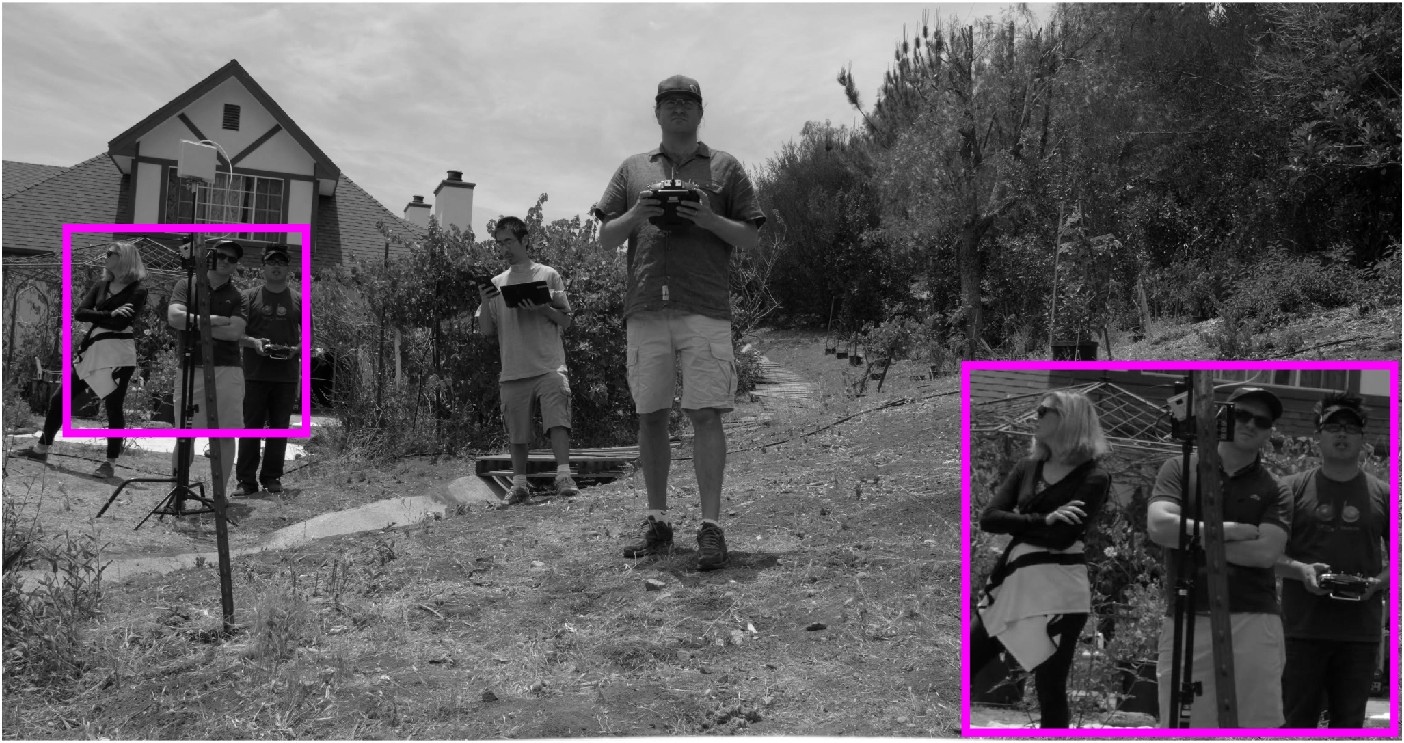} &
    \includegraphics[width=0.77in]{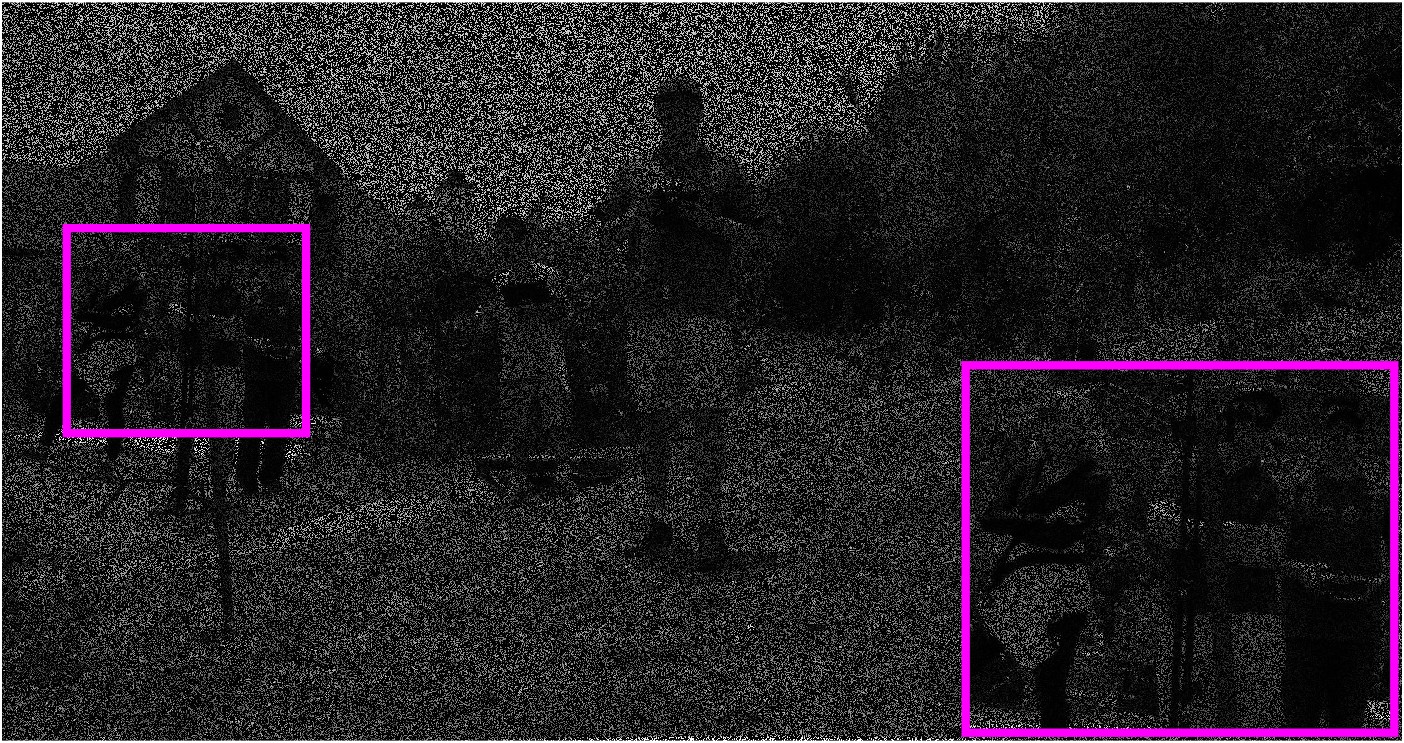} &
    \includegraphics[width=0.77in]{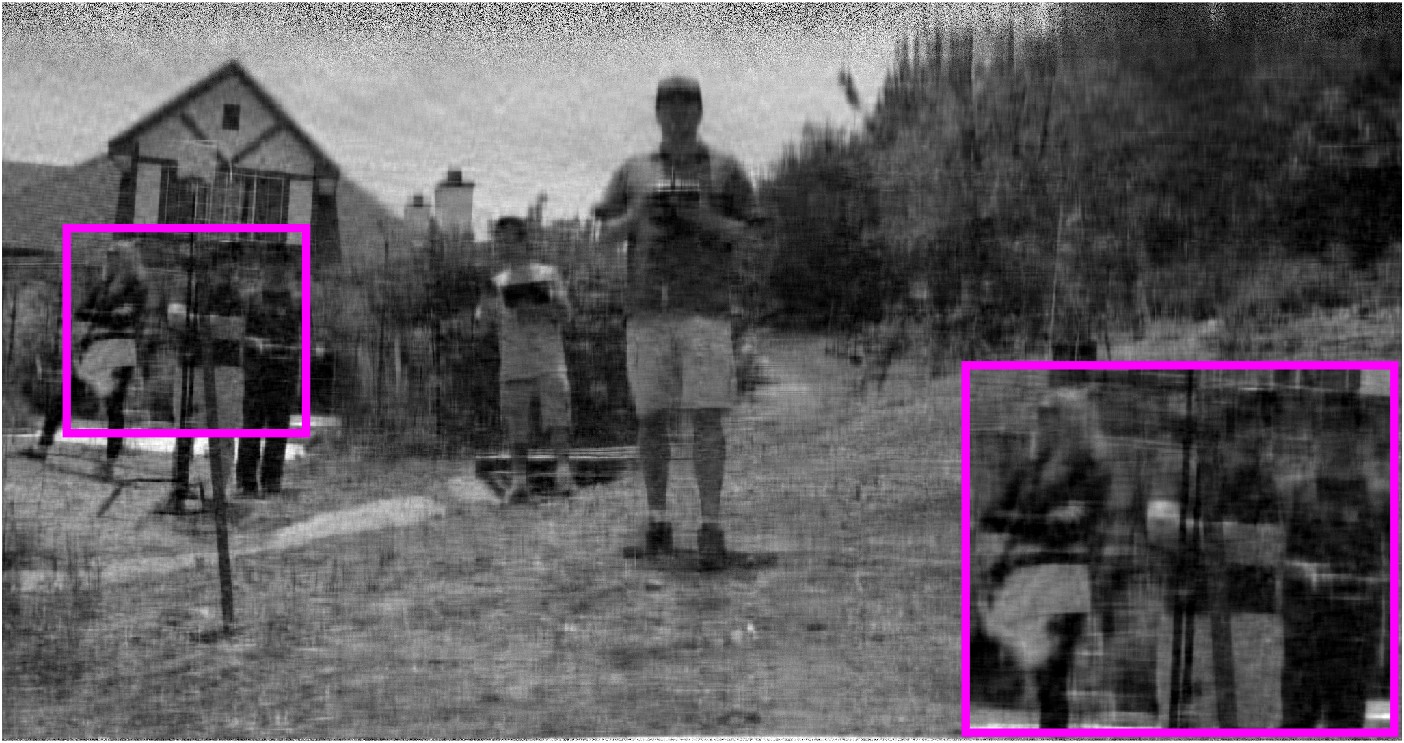} &
    \includegraphics[width=0.77in]{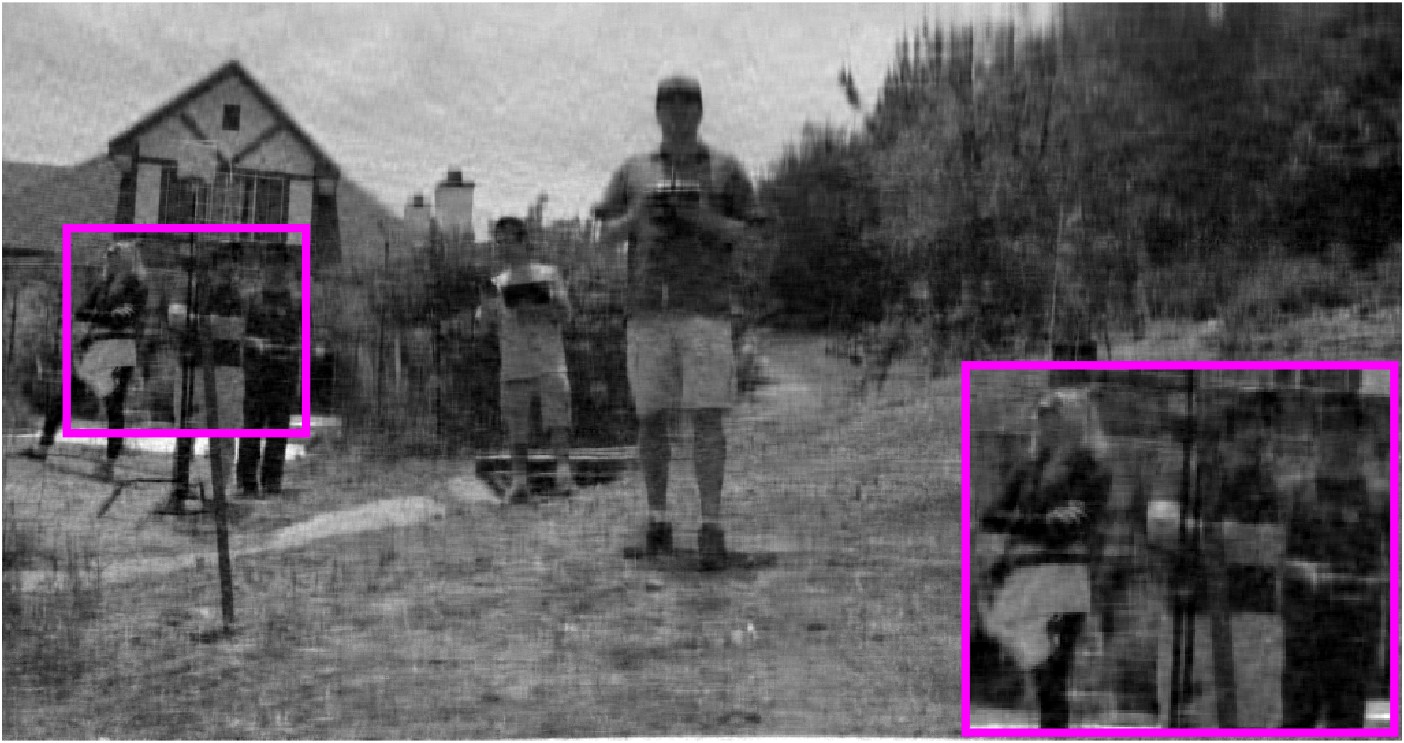}&
    \includegraphics[width=0.77in]{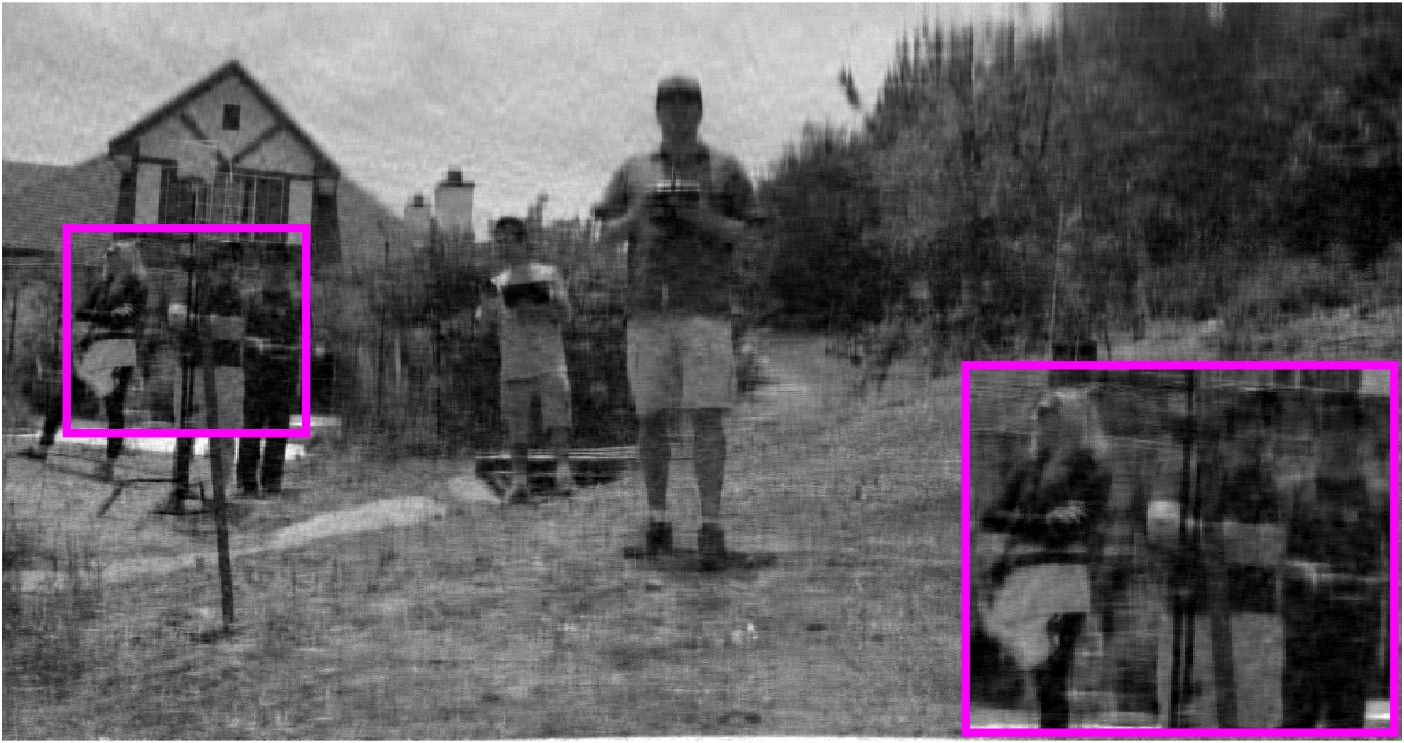} &
    \includegraphics[width=0.77in]{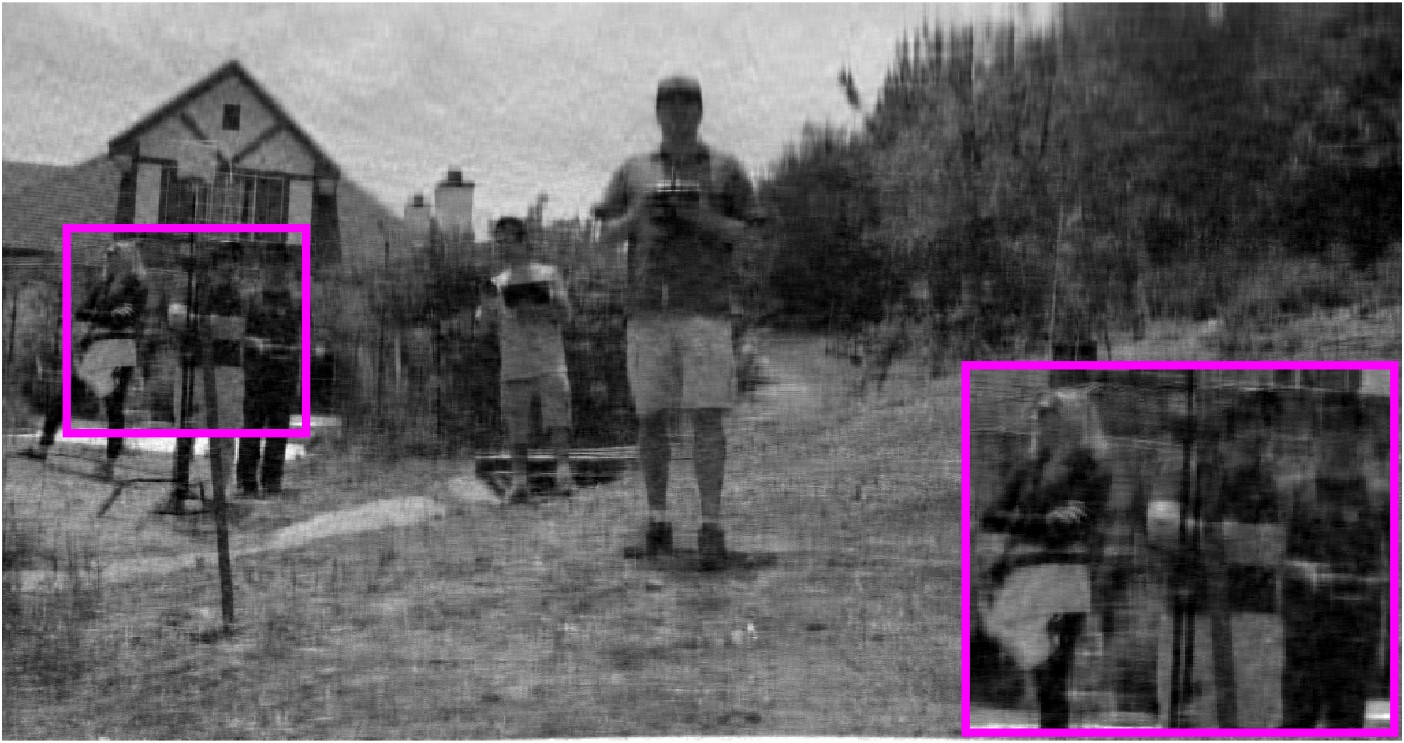} \\
    \scriptsize Original  &
    \scriptsize 70\% missing &
    \scriptsize  TT-SVD &
    \scriptsize  TSTP-SVD&
    \scriptsize  TMSTP-SVD &
    \scriptsize  TMRSTP-SVD
  \end{tabular}        
 \caption{Supplementary objective and visual comparisons of video recovery methods on two additional test sequences, including PSNR-SSIM curves, runtime measurements, and reconstruction results under 70\% missing pixels.}
  \label{fig:video_complete_sup}
\end{figure}

\end{document}